\documentclass[10pt, journal, final, letterpaper, twocolumn]{IEEEtran}

\usepackage{microtype}
\usepackage{graphicx}
\usepackage{booktabs} 
\usepackage{hyperref}
\usepackage{url}
\usepackage{amsmath}
\usepackage{amssymb}
\usepackage{mathtools}
\usepackage{amsthm}
\usepackage{enumitem}
\usepackage[numbers]{natbib}

\usepackage{amsmath,amsfonts,bm}

\def\eqref#1{equation~\ref{#1}}

\def\1{\bm{1}}

\DeclareMathAlphabet{\mathsfit}{\encodingdefault}{\sfdefault}{m}{sl}
\SetMathAlphabet{\mathsfit}{bold}{\encodingdefault}{\sfdefault}{bx}{n}

\usepackage{algorithm}
\usepackage{algorithmic}
\usepackage{bbm}
\usepackage{tabularx}
\usepackage{pifont}
\usepackage{multirow}
\usepackage{wrapfig}
\usepackage{makecell}
\usepackage{tikz}
\usepackage{listings}
\usepackage{courier}
\usepackage{titletoc}

\newcommand{\rot}[1]{\rotatebox{90}{#1}}

\newcommand{\tile}[2]{%
	\begin{tikzpicture}[baseline=(img.south)]
		\node[inner sep=0pt] (img) {\includegraphics[width=0.15\textwidth]{#1}};
		\node[anchor=north west, font=\scriptsize, fill=white, fill opacity=.85,
		text opacity=1, inner sep=1pt, xshift=1pt, yshift=-1pt]
		at (img.north west) {#2};
	\end{tikzpicture}%
}

\newtheorem{theorem}{Theorem}[section]

\newtheorem{corollary}[theorem]{Corollary}
\theoremstyle{definition}
\newtheorem{definition}[theorem]{Definition}

\theoremstyle{remark}
\newtheorem{remark}[theorem]{Remark}

\usepackage[capitalize,noabbrev]{cleveref}

\begin{document}

	\title{Weierstrass Positional Encoding for Vision Transformers}
	
	\author{Zhihang~Xin,
		Rui~Wang$^{*}$,
		Xitong~Hu,
		and~Xiaojun~Wu$^{*}$%
		\thanks{Zhihang Xin and Xitong Hu are with the School of Mathematics and Data Science, Jiangnan University, Wuxi 214122, China, e-mail: \{1131230221; 1131230110\}@stu.jiangnan.edu.cn.}%
		\thanks{Rui Wang and Xiaojun Wu are with the School of Artificial Intelligence and Computer Science, Jiangnan University, Wuxi 214122, China, e-mail: \{cs\_wr; wu\_xiaojun\}@jiangnan.edu.cn.}%
		\thanks{$^{*}$Corresponding authors. Rui Wang is the first corresponding author.}%
		\thanks{All theoretical results are established under explicit assumptions, with complete proofs provided in the Appendix. The code will be released upon acceptance.}%
	}
	
	\maketitle
	
	\begin{abstract}
		Vision Transformers (ViTs) have demonstrated remarkable success in computer vision tasks.  However, their reliance on learnable one-dimensional positional encoding  disrupts the inherent two-dimensional spatial structure of images due to patch flattening. Existing positional encoding approaches lack geometric constraints and fail to preserve a monotonic correspondence between Euclidean spatial distances and sequential index distances, thereby limiting the model's capacity to leverage spatial proximity priors effectively. Recognizing that periodicity is particularly beneficial for positional encoding, we propose Weierstrass elliptic Positional Encoding (WePE), a mathematically principled approach that encodes two-dimensional coordinates in the complex domain. This method maps the normalized two-dimensional patch coordinates onto the complex plane and constructs a compact four-dimensional positional feature based on the Weierstrass elliptic function $\wp(z)$ and its derivative. 
		The doubly periodic property of $\wp(z)$ enables a principled encoding of 2D positional information, while their intrinsic lattice structure aligns naturally with the geometric regularities of patch grids in images. Their nonlinear geometric characteristics enable faithful modeling of spatial distance relationships,  while the associated algebraic addition formula allows relative positional information between arbitrary patch pairs to be derived directly from their absolute encodings. WePE is a plug-and-play, resolution-agnostic positional module that integrates seamlessly with existing ViTs. Extensive experiments demonstrate that WePE delivers consistent performance gains in most scenarios, while its implementation with precomputed lookup tables ensures that these improvements incur no noticeable computational or memory overhead. In addition, several analyses and ablation studies bring further confirmation to the effectiveness of our method. 
	\end{abstract}
	
	
	\section{INTRODUCTION}
	\label{sec:introduction}
	
	\IEEEPARstart{V}{ision} Transformers (ViTs)~\citep{dosovitskiy2020image} 
	have recently emerged as a powerful representation learning architecture in computer vision,  challenging the long-standing dominance of Convolutional Neural Networks (CNNs)~\citep{lecun2002gradient}. 
	By partitioning an image into a sequence of patches, ViTs leverage the self-attention mechanism to model global dependencies~\citep{vaswani2017attention}, a stark contrast to the localized inductive biases inherent in CNNs~\citep{zeiler2014visualizing}. While this design enables greater flexibility in capturing long-range interactions, it also introduces a critical limitation: ViTs lack an intrinsic understanding of spatial geometry. As a result, their performance heavily relies on positional encodings (PE)~\citep{shaw2018self,parmar2018image} to provide the necessary spatial information.
	
	The standard formulation of ViTs adopts simple, learnable 1D positional embeddings~\citep{dosovitskiy2020image}. Beyond learnable absolute encodings, researchers have proposed and adopted a spectrum of positional encodings, including sinusoidal (trigonometric) schemes~\citep{vaswani2017attention}, Fourier Position Embedding (FoPE)~\citep{hua2024fourier}, Rotary Position Embedding (RoPE)~\citep{su2024roformer}, Lie-group–based rotational encodings (LieRE)~\citep{ostmeier2025liere}, the RoPE-Mixed variant specialized for 2D vision~\citep{heo2024rotary}, and others. However, most positional encoding schemes for ViTs entail a structural limitation: the 2D patch grid is serialized by flattening into a 1D token sequence to conform to the sequence-based formulation of Transformer self-attention~\citep{vaswani2017attention}, thereby disrupting the image’s intrinsic spatial geometry. More importantly, such encodings operate essentially as a lookup table without geometric constraints. As a consequence, no monotonic correspondence is ensured between the true Euclidean distance of image patches and their relative positions in the embedding space~\citep{wu2021rethinking}. 
	This deficiency severely limits the model's capacity to exploit spatial proximity priors, which acts as a cornerstone for effective visual understanding~\citep{cordonnier2019relationship}.
	
	To overcome these  limitations, we introduce the Weierstrass elliptic Positional Encoding (WePE), a mathematically principled framework rooted in the theory of complex analysis. Instead of flattening the patch grid, we map the 2D coordinates of each patch directly onto the complex plane, thereby preserving their geometric integrity. We then utilize the Weierstrass elliptic function~\citep{weierstrass1854theorie}, $\wp(z)$, a doubly periodic meromorphic function (As given in Definition~\ref{def:weierstrass}), to construct a continuous and structured spatial representation. It establishes a profound connection between a complex torus (defined by a lattice in the complex plane) and an algebraic elliptic curve through a specific differential equation. Formally, $\wp(z)$ possesses several  mathematical properties, such as being a doubly periodic meromorphic function (as detailed in~\Cref{def:elliptic_function}), satisfying a specific differential Equation (\Cref{thm:weierstrass_ode}), and adhering to a distinctive addition formula (\Cref{thm:addition_formula}). 
	
	From a three-dimensional perspective, The Weierstrass $\wp$-function resembles an array of identical volcano-like structures arranged regularly across the plane, where the sharp peak of each “volcano” corresponds to a second-order pole of the function. These peaks rise steeply to infinity at the lattice points, while the surface between them exhibits smooth, doubly periodic undulations. Each “volcano” is embedded within the parallelogram cell spanned by its two fundamental half-periods $\omega_{1}$ and $\omega_{3}$, which together form the period lattice of the Weierstrass elliptic function.
	This choice is not arbitrary: Our discussion in ~\Cref{sec:periodicity_benefits} concludes that periodicity is beneficial, and even optimal under the standard criterion of translation-equivariant positive definite attention kernels, for positional encodings. This insight leads us to select a function with a doubly periodic function to adapt to two-dimensional images. The intrinsic lattice structure of $\wp(z)$ naturally aligns with the geometric regularities of image patch grids, while its continuity ensures resolution invariance, a pivotal advantage for fine-tuning across resolutions. Moreover,  we demonstrate that WePE possesses a provable distance-decay property, and that its algebraic addition formula allows the direct derivation of relative positional information between arbitrary patches.
	
	This paper presents the design, theoretical grounding, and empirical validation of WePE. By replacing heuristic positional encoding with a mathematically  principled formulation, our method endows ViTs with a robust geometric inductive bias. The main contributions of this paper are summarized as follows:
	\begin{enumerate}
		\item \textbf{Geometrically principled positional encoding:} We propose WePE, a mathematically grounded framework that maps 2D image coordinates to the complex plane via the Weierstrass elliptic function. This design preserves the intrinsic spatial structure of images, inherently aligns with translational regularities, and provides a continuous, resolution-invariant positional representation.
		
		\item \textbf{WePE with several key properties:} Mathematical analysis shows that WePE has some key properties, such as relative position modeling via the elliptic function’s addition formula, the inherent distance-decay property, and the periodicity advantages that, under certain conditions, may even be optimal. Additionally, the continuous function evaluation ensures resolution-invariant fine-tuning, while industrial-level acceleration schemes simplify implementation.
		
		\item \textbf{Empirical validations across multiple datasets:} 
		We conduct experiments under the scenarios of pre-training and fine-tuning. The results demonstrate the empirical advantages of WePE, showing consistent improvements over existing models where positional encoding is the only variable. Besides, several ablations also corroborate the significance of WePE's intrinsic properties.

	\end{enumerate}

	\begin{figure*}[htbp] 
		\centering
		\includegraphics[width=\textwidth]{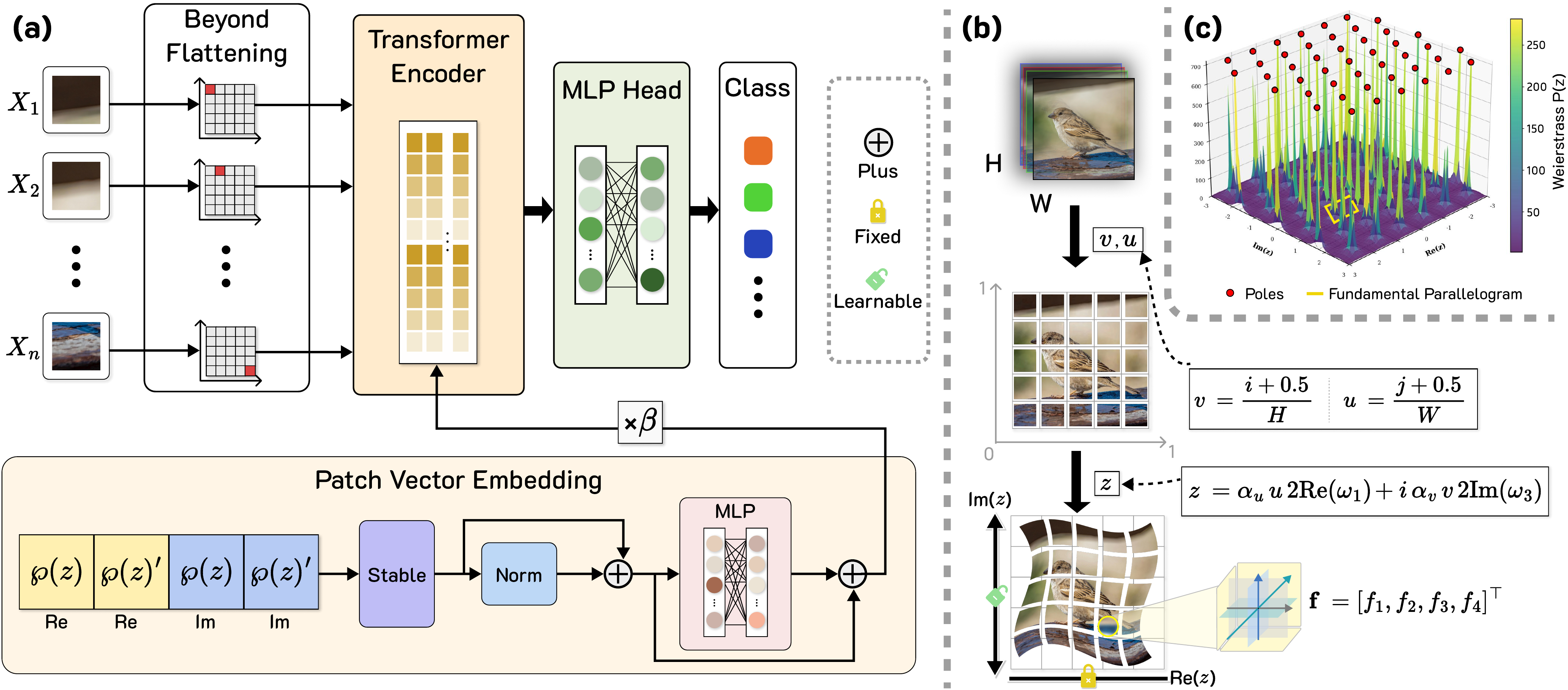} 
		\caption{Overview of how WePE encodes 2D spatial information. (a) Four-dimensional WePE features are mapped to patch embeddings for Transformer encodings. (b) Image patches are normalized and mapped onto the complex plane to construct WePE coordinates. (c) A three-dimensional visualization of the Weierstrass elliptic function, illustrating its doubly periodic structure and pole distribution across the complex plane.}
		\label{fig:Algorithm summary} 
	\end{figure*}

	\section{Weierstrass Elliptic Positional Encoding}
	\label{headings}

	\subsection{Foundational Framework of WePE}
	\label{subsec:wepe_framework}

	\paragraph{Coordinate System and Complex Plane Mapping}
	
	The input images of ViTs~\citep{dosovitskiy2020image} often exhibit varying resolutions, resulting in different numbers of image patches after partitioning. Consequently, the direct use of absolute indices is problematic, as their valid range changes with the total number of patches. In contrast, relative positional information, when represented within a normalized coordinate system, remains consistent and resolution-independent.
	
	Given an input image, we conventionally partition it into a grid of $H \times W$ patches. For patch coordinates $(i,j)$ of the input image, where $i \in \{0,1,\ldots,H-1\}, j \in \{0,1,\ldots,W-1\}$, we first normalize them to the $[0,1]$ interval:
	\begin{equation}
		u = \frac{j + 0.5}{W}, \qquad v = \frac{i + 0.5}{H} \label{eq:uv_norm}.
	\end{equation}
	
	The effective ranges for $u$ and $v$ are respectively 
	$\left(\frac{0.5}{W}, \frac{W-0.5}{W}\right)$ 
	and 
	$\left(\frac{0.5}{H}, \frac{H-0.5}{H}\right)$, 
	which are proper subsets of $[0, 1]$. This is advantageous in certain situations. Taking $z=0$ as an example, it is actually a pole in the mapping, and we want to avoid having the center of any patch located at a pole.

	Subsequently, the normalized coordinates can be mapped onto the complex plane via the following formula:
	\begin{equation}
		z = \alpha_u \cdot u \cdot 2\text{Re}(\omega_1) + i \cdot \alpha_v \cdot v \cdot 2\text{Im}(\omega_3),
		\label{eq:complex_mapping}
	\end{equation}
	where $\alpha_u, \alpha_v$ are adjustable scaling factors, and $\omega_1, \omega_3$ represent the real half-period and imaginary half-period of the elliptic function, respectively. This mapping serves as a pivotal step, as it embeds the rich geometric and analytic properties of the Weierstrass elliptic function into the spatial representation of each image patch. Intuitively, the elliptic function acts like "weaving a fishing net" across the image, thereby naturally coupling positional information along both spatial dimensions. 
	
	\paragraph{\boldmath Feature Extraction from the $\wp(z)$ and its Derivative}
	
	For the mapped complex coordinate $z$, to fully utilize the information in complex numbers, we extract the real and imaginary of $\wp(z)$ and its derivative, given below:
	\begin{equation}
		\begin{aligned}
			f_1 &= \operatorname{Re}(\wp(z)), & f_2 &= \operatorname{Im}(\wp(z)), \\
			f_3 &= \operatorname{Re}(\wp'(z)), & f_4 &= \operatorname{Im}(\wp'(z)).
		\end{aligned}
	\end{equation}
	
	This yields a four-dimensional positional feature, denoted by 
	$\mathbf{f} = [f_1, f_2, f_3, f_4]^{\top}$.

	In certain regions, the Weierstrass elliptic function may attain extremely large magnitudes, especially near its poles (also illustrated in~\Cref{fig:Algorithm summary}(c)). Such unbounded behavior poses a risk for training stability, often manifesting as gradient explosion and computational failure~\citep{lozier2003nist}. As a countermeasure, we instead propose two empirically validated and robust solutions (\Cref{subsec:wefpe-pretrain} and \Cref{subsec:wefpe-finetune} ), which are respectively more suitable for the scenarios of pre-training and fine-tuning.

	\subsection{WePE Implementation for From-Scratch Pre-training}
	\label{subsec:wefpe-pretrain}
	
	\paragraph{Numerical Computation via Direct Lattice Summation}
	
	When computing the series expansion of the Weierstrass elliptic function, we truncate it (Definition~\ref{def:weierstrass}) to a finite sum over indices $|m| \leq M, |n| \leq N$ (excluding the origin), where each lattice point is given by  $\omega_{mn} = 2m\omega_1 + 2n\omega_3$. At this point, the truncated approximation of $\wp(z)$ is expressed as:
	\begin{equation}
		\wp(z) \approx \frac{1}{z^2} + \sum_{\substack{|m| \leq M, |n| \leq N, \\ (m,n) \neq (0,0)}} \left( \frac{1}{(z-\omega_{mn})^2} - \frac{1}{\omega_{mn}^2} \right).
		\label{eq:truncated_series}
	\end{equation}
	For $w_{m,n}$ with a large modulus ($|w_{m,n}| \gg |z|$), the asymptotic contribution behaves as:
	\begin{equation} \label{eq:asymptotic_behavior}
		\begin{split}
			|T_{m,n}(z)| &= \left| \frac{1}{(z - w_{m,n})^2} - \frac{1}{w_{m,n}^2} \right| \\
			&\approx \frac{2|z|}{|w_{m,n}|^3} + O\left(\frac{|z|^2}{|w_{m,n}|^4}\right).
		\end{split}
	\end{equation}
	This indicates  that the contribution decays as a power law of $|w_{m,n}|^{-3}$. 
	
	To improve convergence, we sort the lattice points by their modulus, such that 
	$|\pi(w_1)| \le |\pi(w_2)| \le \cdots$. 
	Here, $\pi$ denotes a permutation that reorders the lattice points according to the magnitude of their modulus. 
	The reordered partial sum is then expressed as 
	$
	S_K(z) = \sum_{k=1}^{K} T_{\pi(k)}(z),
	$
	where $K$ is the truncation index. 
	This ordering ensures that terms corresponding to lattice points with smaller modulus are accumulated first, 
	significantly accelerating the convergence of the truncated series compared with the conventional lexicographic scheme. The truncation error can be bounded as:
	\begin{equation} \label{eq:truncation_error}
		|S_\infty(z) - S_K(z)| \le \sum_{k=K+1}^{\infty} |T_{\pi(k)}(z)| \le C \sum_{k=K+1}^{\infty} \frac{1}{|\pi(w_k)|^3}.
	\end{equation}
	For a 2D lattice, the tail sum can be approximated by an integral, $\sum_{|w|>R} |w|^{-3} \sim \int_{R}^{\infty} r^{-2} dr = 1/R$, which ensures that the error decays as $O(1/R_K)$, where $R_K = |\pi(w_K)|$. Compared with a conventional lexicographic ordering, this modulus-based summation significantly improves convergence, reducing the truncation error from $O(\log K / \sqrt{K})$ to $O(1/\sqrt{K})$. 
	
	Finally, the upper bound of the truncation error can be estimated as:
	\begin{equation} \label{eq:trunc_bound}
		E_{\text{trunc}}
		= \left| \sum_{|w|>R_{\max}} T_w(z) \right|
		\le \sum_{|w|>R_{\max}} \frac{2|z|}{|w|^3}
		\le \frac{2C_\Lambda\,|z|}{R_{\max}},
	\end{equation}
	where $C_\Lambda > 0$ is a lattice-dependent constant. This follows from comparing the
	tail sum with a planar integral: since the number of lattice points in an annulus
	$[r, r+\mathrm{d}r]$ is approximately $2\pi r\,\mathrm{d}r / a^2$, where
	$a = 2\omega_1$ is the fundamental period length, one obtains
	\begin{equation}
		\sum_{|w|>R} |w|^{-3}
		\;\lesssim\;
		\int_R^\infty \frac{2\pi r}{a^2}\cdot r^{-3}\,\mathrm{d}r
		= \frac{2\pi}{a^2 R},
	\end{equation}
	so that $C_\Lambda \lesssim 2\pi / a^2$. Taking $M_{\text{search}} = N_{\text{search}} = 12$
	as an example, we obtain $R_{\max} \approx 30$. In the lemniscatic setting used in our
	experiments, $a = 2\omega_1 \approx 5.244$, giving $C_\Lambda \lesssim 2\pi/a^2 \approx 0.229$.
	Substituting into Equation~\ref{eq:trunc_bound} yields
	\begin{equation}
		E_{\text{trunc}}
		\;\lesssim\;
		\frac{2 \times 0.229}{30}\,|z|
		\;\approx\;
		1.5\times 10^{-2}\,|z|.
	\end{equation}
	
	Our discussion in~\Cref{sec:wef_pe_lut} presents a solution that balances computational precision and implementation efficiency through the pre-computation of a high-resolution WePE look-up table. Once training is complete, the optimal parameters learned by the proposed WePE module are fixed. We then generate a high-resolution look-up table by calculating a four-dimensional positional feature vector for each point on a fine, fixed-size grid. During the inference stage, the normalized coordinates of the patches for any given input image are used as query points. During inference, the GPU utilizes hardware-accelerated bilinear interpolation to efficiently approximate and retrieve the corresponding positional encodings from the pre-computed table. This method transforms the computational burden from complex, real-time calculations into a one-time pre-computation and a rapid, memory-based retrieval operation, thus bringing the time complexity of online inference down to a level comparable to that of simple grid-based encoding schemes. We have shown in~\Cref{sec:wef_pe_lut} that the error introduced by this scheme is negligible, as it is far below the inherent stochastic error sources during deep neural network training and inference.
	
	\paragraph{Numerical Stability and Convergence Acceleration}
	
	To mitigate potential numerical explosion, we apply an adaptive tanh-based compression of each feature component in $\boldsymbol{\mathrm{f}}$, \textit{i.e.},$\tilde{f}_i = \tanh(\alpha_{\text{scale}} \cdot f_i)$ for $i=1, \dots, 4$, where the scaling parameter is parameterized as $\alpha_{\text{scale}} = \text{softplus}(\alpha_{\text{raw}})$ for ensuring positivity. For the input $z$ near a pole (\textit{i.e.}, $|z| < 15\epsilon$, where $\epsilon$ is a small threshold), the function value is clipped to a large constant $C_{\text{large}}$; otherwise, $\wp(z)$ is computed using  Equation ~\ref{eq:truncated_series}. The accumulated round-off error for a sum over $K$ terms is bounded by $E_{\text{round}} \le K \cdot \epsilon_{\text{mach}} \cdot \max_k |T_k|$, where  $\epsilon_{\text{mach}} \approx 2.22 \times 10^{-16}$ is the machine precision. By further clipping individual terms such that $\max_k |T_k| \le M_{\text{clip}}$, the total round-off error is controlled at the order of $10^{-10}$. These measures collectively ensure numerical stability, while preserving both model accuracy and the fidelity of the encoded features~\citep{higham2002accuracy}.
	
	\paragraph{Network Architecture Integration}
	
	We use $\tilde{\mathbf{f}}_{ij}$ to denote the complete four-dimensional stabilized feature vector of the $(i,j)$-th image patch, which is composed of the components
	$\tilde{f}_{1}, \tilde{f}_{2}, \tilde{f}_{3}, \tilde{f}_{4}$. These 4D WePE features are then projected to the model dimension $d$ via a linear layer, yielding patch encodings $\mathbf{PE}_{ij} = \text{LayerNorm}(\mathbf{W}_{\text{proj}} \tilde{\mathbf{f}}_{ij} + \mathbf{b}_{\text{proj}})$
	\citep{ba2016layer}, where $\mathbf{W}_{\text{proj}} \in \mathbb{R}^{d \times 4}$ and
	$\mathbf{b}_{\text{proj}} \in \mathbb{R}^d$.
	For enhanced representational capacity, this linear layer can be substituted with a Multi-Layer Perceptron (MLP)~\citep{rumelhart1986learning}. The classification token is endowed with a separate learnable encoding $\mathbf{PE}_{\text{cls}} \in \mathbb{R}^{1 \times d}$~\citep{devlin2019bert}. Finally, these position encodings are added to their corresponding patch and token encodings to form the model's input sequence:
	\begin{equation}
		\mathbf{X}_{\text{input}} = \big[ \mathbf{x}_{\text{cls}} + \mathbf{PE}_{\text{cls}}, \, \{ \mathbf{x}_{i} + \mathbf{PE}_{i} \}_{i=0}^{HW} \big]^{\top}
		\label{eq:final_input_v1}
	\end{equation}

	\paragraph{Positional Encodings with Learnable Parameters}
	
	To enable adaptive spatial scaling, the imaginary half-period $\omega_3$ is parameterized as a learnable quantity. Since the imaginary part $\omega_3^{\prime}$ must be positive, it cannot be directly optimized using gradient-based methods. Instead, we introduce an unconstrained trainable parameter $\alpha_{\text{learn}}$ and define $\omega_3 = i \cdot \omega_3^{\prime}$, where $\omega_3^{\prime}$ is derived via the softplus function to ensure positivity: $\omega_3^{\prime} = \text{softplus}(\alpha_{\text{learn}}) = \log(1 + \exp(\alpha_{\text{learn}}))$. The other lattice parameter, $\omega_1$, is kept as a fixed constant to prevent potential overfitting. This configuration allows the lattice basis, which is initialized to be orthogonal and form a square, to adaptively deform into a rectangle during training, thereby learning the optimal aspect ratio for the given data.
	
	Recognizing that semantic and positional cues are not always of equal importance, we introduce a learnable global scaling factor, $\beta_{\text{pos}}$, to balance their contributions. The final positional encoding is then defined as $\mathbf{PE}_{\text{final}} = \beta_{\text{pos}} \cdot \mathbf{PE}_{\text{WePE}}$, where $\mathbf{PE}_{\text{WePE}}$ is the encoding generated by our method.
	
	\begin{figure*}[t]
		\centering
		\includegraphics[width=\linewidth]{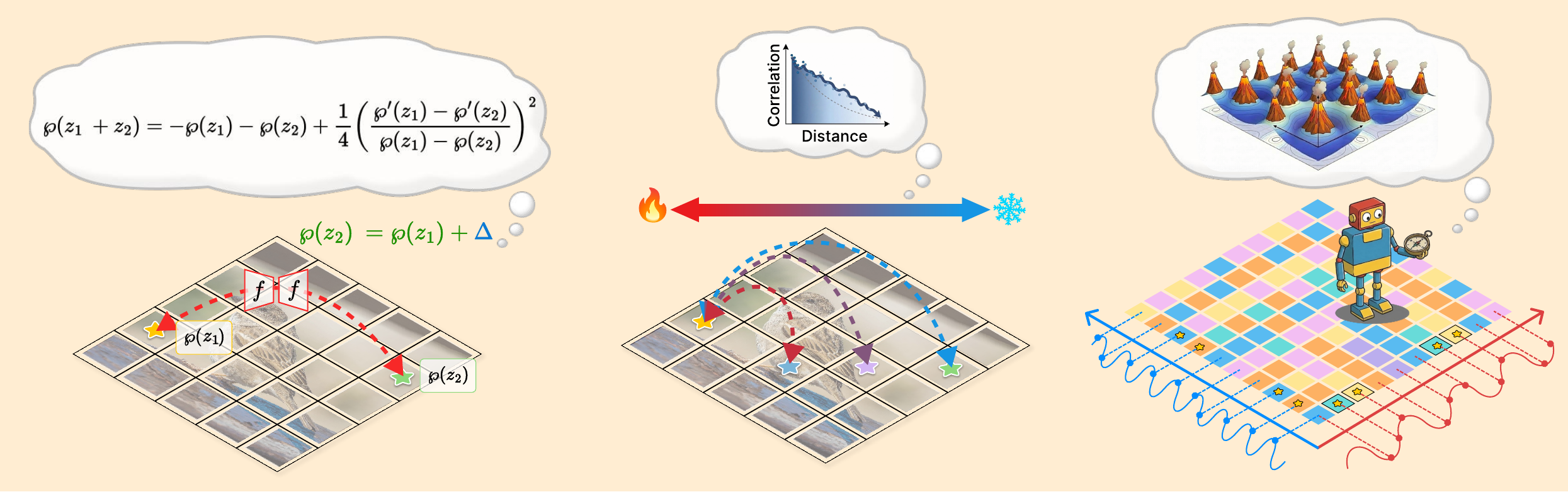}
		\caption{
			Diagram of key features of WePE.
			Left: WePE's addition formula enables explicit modeling of relative displacements on the patch lattice;
			Middle: The induced distance–aware correlation shows a decay with spatial separation;
			Right: The Weierstrass elliptic function forms a two-dimensional periodic positional structure on the complex plane.
		}
		\label{fig:wepe_properties}
	\end{figure*}
	
	\subsection{WePE Adaptation for Fine-tuning}
	\label{subsec:wefpe-finetune}
	
	Fine-tuning a pre-trained ViT typically involves far fewer optimization 
	steps than scratch training, starting from a well-initialized checkpoint. 
	In this regime, two practical difficulties arise when directly applying the 
	lattice-sum formulation of $\wp(z)$: the series evaluation introduces 
	non-trivial computational overhead at each gradient step, and the 
	unbounded behavior near lattice points can destabilize a training 
	trajectory that is already close to a good solution. Rather than 
	discarding the geometric structure established during pre-training, 
	we construct a stabilized surrogate field that retains the 
	inductive bias of the original encoding, a strong radial response 
	near the origin together with structured directional periodicity, while 
	ensuring numerical robustness throughout optimization.

	Let $(u, v) \in [0,1]^2$ denote the normalized patch coordinates. We embed them into the complex plane via
	\begin{equation}
		z = z_{\mathrm{re}} + i\, z_{\mathrm{im}},
	\end{equation}
	where
	\begin{equation}
		z_{\mathrm{re}} = \omega_1 u + \varepsilon_u \sin(2\pi u),
		\qquad
		z_{\mathrm{im}} = \omega_3' v + \varepsilon_v \cos(2\pi v).
		\label{eq:ft_coord}
	\end{equation}
	Here $\omega_1$ is the fixed real half-period inherited from pre-training, 
	$\omega_3' > 0$ is a learnable vertical scaling parameter constrained to 
	be positive via a shifted softplus transformation, and $\varepsilon_u, 
	\varepsilon_v$ are small fixed phase perturbations that improve 
	optimization robustness without introducing additional learned degrees 
	of freedom.

	Let $r(z) = \sqrt{(Re z)^2 + (Im z)^2}$ and $\theta(z) = 
	\mathrm{atan2}(Im z,\, Re z)$, and define normalized coordinates
	\begin{equation}
		u' = \frac{Re z}{\omega_1}, \qquad v' = \frac{Im z}{\omega_3'}.
	\end{equation}
	To avoid instability at the origin we adopt the clamped radius 
	$r_{\mathrm{safe}}(z) = \max\{r(z),\, \epsilon\}$ for a small constant 
	$\epsilon > 0$. A learnable shape parameter $\beta > 0$ governs the 
	principal radial term
	\begin{equation}
		M(z) = \frac{1}{r_{\mathrm{safe}}(z)^2 + \beta},
		\label{eq:M}
	\end{equation}
	which approximates the dominant $1/|z|^2$ singularity of $\wp(z)$ 
	while remaining finite everywhere. To encode directional periodicity 
	we augment $M(z)$ with a Fourier-like correction
	\begin{equation}
		C(z) = \sum_{k=1}^{K} a_k
		\Bigl[
		\cos(k\pi u')\,e^{-k\pi|v'|}
		+ \sin(k\pi v')\,e^{-k\pi|u'|}
		\Bigr],
		\label{eq:C}
	\end{equation}
	where $\{a_k\}$ are frequency-dependent amplitudes. In practice, 
	retaining only the first few low-frequency terms ($K = 3$) is 
	sufficient. The surrogate complex field is then
	\begin{equation}
		\widetilde{\wp}_{\mathrm{ft}}(z)
		=
		\bigl(M(z)\cos\theta(z) + C(z)\bigr)
		+ i\bigl(M(z)\sin\theta(z) + \eta\, C(z)\bigr),
		\label{eq:surr}
	\end{equation}
	where $\eta \in \mathbb{R}$ is a fixed mixing coefficient. This 
	construction yields a non-degenerate complex-valued field whose 
	real and imaginary channels are coupled through the polar angle, 
	while $\eta \neq 1$ breaks the symmetry between the two channels 
	in a controlled manner.

	To enrich the positional representation and maintain the same 
	four-channel interface as the pre-training formulation, we introduce 
	an auxiliary branch that mimics the role of $\wp'(z)$. This branch 
	is not the exact complex derivative of Equation ~\ref{eq:surr}; it is 
	designed to capture local spatial variation of the surrogate field 
	while remaining stable under gradient flow. Concretely, we define 
	a radial derivative proxy
	\begin{equation}
		M'(z) = -\frac{2}{r_{\mathrm{safe}}(z)^3 + \beta},
		\label{eq:Mprime}
	\end{equation}
	and a complementary correction term
	\begin{equation}
		C'(z) = \sum_{k=1}^{K'} b_k\, k
		\Bigl[
		-\sin(k\pi u')\,e^{-k\pi|v'|}
		+ \cos(k\pi v')\,e^{-k\pi|u'|}
		\Bigr],
		\label{eq:Cprime}
	\end{equation}
	where $b_k$ are fixed or learnable amplitudes. The auxiliary branch is
	\begin{equation}
		\widetilde{\wp}_{\mathrm{ft}}'(z)
		=
		\bigl(M'(z)\cos\theta(z) + C'(z)\bigr)
		+ i\bigl(M'(z)\sin\theta(z) + \eta'\,C'(z)\bigr),
		\label{eq:surr_deriv}
	\end{equation}
	with a fixed coefficient $\eta' \in \mathbb{R}$.

	The four positional features are assembled from the real and imaginary 
	parts of Equation ~\ref{eq:surr} and Equation ~\ref{eq:surr_deriv}, compressed 
	by a scaled $\tanh$ nonlinearity with a learnable or fixed factor 
	$\alpha > 0$:
	\begin{align}
		f_1 &= \tanh\!\bigl(\alpha\,Re(\widetilde{\wp}_{\mathrm{ft}}(z))\bigr), &
		f_2 &= \tanh\!\bigl(\alpha\,Im(\widetilde{\wp}_{\mathrm{ft}}(z))\bigr), 
		\label{eq:f12}\\
		f_3 &= \tanh\!\bigl(\alpha\,Re(\widetilde{\wp}_{\mathrm{ft}}'(z))\bigr), &
		f_4 &= \tanh\!\bigl(\alpha\,Im(\widetilde{\wp}_{\mathrm{ft}}'(z))\bigr).
		\label{eq:f34}
	\end{align}
	Collecting these gives $\mathbf{f}_{\mathrm{ft}}(z) = [f_1, f_2, f_3, 
	f_4]^\top \in \mathbb{R}^4$, which is projected to the model dimension 
	through a lightweight multilayer projection head:
	\begin{equation}
		\mathrm{PE}_{\mathrm{ft}}(z) = \mathrm{Proj}\!\bigl(\mathbf{f}_{\mathrm{ft}}(z)\bigr) + \mathbf{b}_{\mathrm{off}},
		\label{eq:PE_ft}
	\end{equation}
	where $\mathbf{b}_{\mathrm{off}}$ is a learnable offset vector.

	A key consideration in fine-tuning is that the pre-trained absolute 
	position embeddings $\mathbf{E}_{\mathrm{learned}} \in 
	\mathbb{R}^{(N+1)\times d}$ already encode task-relevant spatial 
	patterns acquired from large-scale pre-training. Discarding them 
	entirely would forfeit this prior. We therefore interpolate between 
	the pre-trained embeddings and the WePE encoding via a learnable 
	scalar gate $\lambda \in [0,1]$:
	\begin{equation}
		\mathbf{E}_{\mathrm{hybrid}}^{\mathrm{patch}}
		= \lambda\,\mathbf{E}_{\mathrm{WePE}}
		+ (1-\lambda)\,\mathbf{E}_{\mathrm{learned}}^{\mathrm{patch}},
		\label{eq:hybrid}
	\end{equation}
	where $\mathbf{E}_{\mathrm{WePE}} \in \mathbb{R}^{N \times d}$ is 
	obtained from Equation ~\ref{eq:PE_ft} for all $N$ patch positions. The 
	gate is parameterized as $\lambda = \sigma(\lambda_{\mathrm{raw}})$, 
	where $\sigma$ denotes the sigmoid function and $\lambda_{\mathrm{raw}}$ 
	is an unconstrained scalar optimized jointly with the rest of the model. 
	The class token retains its original pre-trained embedding throughout 
	fine-tuning. The complete position embedding matrix is then
	\begin{equation}
		\mathbf{E}_{\mathrm{final}}
		= \begin{bmatrix}
			\mathbf{E}_{\mathrm{cls}} \\[4pt]
			\mathbf{E}_{\mathrm{hybrid}}^{\mathrm{patch}}
		\end{bmatrix} \in \mathbb{R}^{(N+1)\times d}.
		\label{eq:E_final}
	\end{equation}
	This formulation allows the model to continuously interpolate between 
	the empirically learned spatial representation and the mathematically 
	principled WePE structure, with $\lambda$ adapting automatically to 
	the demands of each downstream task.

	The fine-tuning construction described above should be understood as a 
	geometry-aware surrogate rather than a faithful evaluation of the 
	Weierstrass elliptic function. Its purpose is to preserve the two 
	structural properties most essential for transfer, a strong radial 
	response concentrated near the origin, and directional periodic 
	modulation consistent with the patch lattice, while sidestepping 
	the numerical sensitivity that makes direct elliptic-function 
	evaluation impractical within a gradient-based optimization loop. 
	Importantly, the surrogate retains the same four-channel complex-valued 
	interface as the pre-training formulation, ensuring seamless 
	compatibility with the downstream transformer. A detailed derivation 
	of the surrogate from the Fourier expansion of $\wp(z)$ is provided 
	in Appendix~\ref{sec:derivation_rationale}.

	\subsection{Theoretical Explanation}
	\begin{figure}[t]
		\centering
		\includegraphics[width=1.0\linewidth]{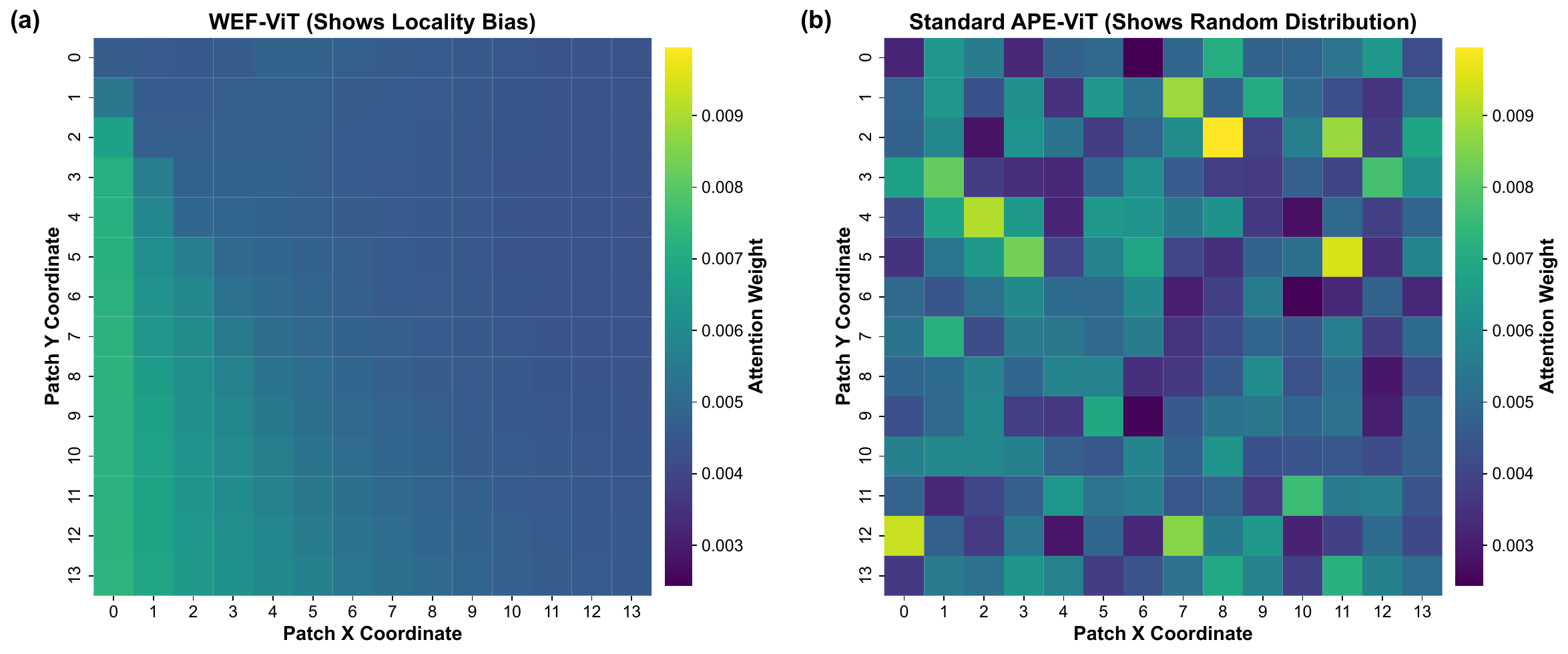}
		\caption{Comparison of geometric inductive bias between WePE-ViT and standard APE-ViT. (a) WePE-ViT exhibits structured locality-aware attention patterns with smooth, isotropic decay from the query patch. (b) Standard APE-ViT demonstrates unstructured attention distribution that lacks spatial coherence.}
		\label{fig:inductive_bias_comparison}
	\end{figure}
	\paragraph{Mapping Normalized Coordinates to the Complex Plane via Isomorphism} Mapping~\ref{eq:complex_mapping} can therefore be written as
	$
	T:\mathbb{R}^2\to\mathbb{C}, T(u,v)=c_1 u + i\,c_2 v, c_1c_2\neq 0.
	$
	To verify linearity, let $\mathbf{x}_1=(u_1,v_1)$ and $\mathbf{x}_2=(u_2,v_2)$ in $\mathbb{R}^2$ and $a,b\in\mathbb{R}$. Then
	$
	T(a\mathbf{x}_1+b\mathbf{x}_2)=T(au_1+bu_2,\,av_1+bv_2)=a\,T(\mathbf{x}_1)+b\,T(\mathbf{x}_2),
	$
	so $T$ is linear. To prove injectivity, observe that
	$
	T(u,v)=0 \iff c_1 u + i\,c_2 v = 0.
	$
	Since $c_1,c_2,u,v\in\mathbb{R}$ and $c_1\neq 0$, $c_2\neq 0$, we must have $u=v=0$, hence $\ker(T)=\{(0,0)\}$ and $T$ is injective. Because
	$
	\dim_{\mathbb{R}}\mathbb{R}^2=\dim_{\mathbb{R}}\mathbb{C}=2,
	$
	an injective linear map between equal-dimensional finite-dimensional real spaces is automatically surjective (rank–nullity). Therefore $T$ is bijective and thus an $\mathbb{R}$-linear isomorphism $\mathbb{R}^2 \cong \mathbb{C}$. Consequently, the spatial relationships among image patches are preserved when embedding normalized coordinates into the complex plane.

	\paragraph{Relative Position Modeling Based on Addition Formula}
	A distinctive property of Weierstrass elliptic function is their addition formula. Let the absolute positions of two image patches be $z_i$ and $z_j$ in the mapped complex plane $\mathbb{C}$. Their relative position can be represented as $\sigma_z = z_j - z_i$. By applying the addition formula (see Theorem~\ref{thm:addition_formula}), $\wp(z_j) = \wp(z_i + \sigma_z)$ can be expressed entirely in terms of $\wp(z_i)$, $\wp(\sigma_z)$, and their derivatives. This algebraic relationship enables direct derivation of relative positional information between any two points from their absolute encodings, without requiring additional relative position encoding modules. Within self-attention mechanisms, the interaction between query vector $Q_{z_i}$ and key vector $K_{z_j}$ can naturally utilize the $\wp(z_j - z_i)$ term from the addition formula. This design enables attention weights to inherently capture spatial relationships between positions, thereby enhancing the model's understanding of geometric structures. Compared to methods that require explicit computation of relative for all position pairs~\citep{shaw2018self}, WePE achieves continuous, precise, and computationally efficient relative position representations. 
	
	\paragraph{Local attenuation of WePE}
	
	In self-attention, interaction strength is determined by the similarity
	between token representations. If the interaction strength after adding
	positional encodings does not exhibit a locality-sensitive trend with
	spatial separation, the model cannot effectively distinguish nearby
	relationships from long-range ones. We formalize this through the
	following local attenuation result: for any two patch positions $z$ and
	$z+\delta$ drawn from a compact domain $D$ within a single fundamental
	parallelogram of the WePE lattice, and satisfying
	$\mathrm{dist}(D,\Lambda)\geq\rho>0$, the normalized similarity of their
	positional encodings satisfies
	\begin{equation*}
		0 \;\le\; 1 - s(z,\delta) \;\le\; C_D|\delta|^2,
	\end{equation*}
	where $s(z,\delta)=\langle q(z),\,q(z+\delta)\rangle$ denotes the inner
	product of the normalized encodings and $C_D>0$ is a domain-dependent
	constant. In particular, the normalized similarity attains its maximum
	at zero displacement and departs from this maximum under small spatial
	perturbations with a quadratic local bound. This is a local attenuation property: within the finite patch
	domain, patches that are close in image space have highly similar
	positional encodings, and this similarity attenuates locally as the
	displacement moves away from zero. This property arises from the
	combination of the linear embedding of normalized patch coordinates into
	the complex plane and the analytic structure of the Weierstrass elliptic
	function $\wp(z)$ away from its poles. Specifically, the final encodings
	are linear projections of 4D feature vectors constructed from $\wp(z)$
	and its derivative (see \Cref{subsec:wepe_framework}). We emphasize that
	this is a finite-domain statement and does not assert global
	asymptotic decay, the periodicity instead induces a translation-consistent tiling of the locality prior across the image,
	ensuring that the same structured proximity bias is reproduced
	consistently in every period cell. This equips the model with an
	explicit spatial proximity prior, which benefits a wide range of vision
	tasks~\citep{vaswani2017attention}. A detailed proof is provided in
	Appendix~\ref{sec:wef_decay_proof}.
	
	\paragraph{Resolution-Invariant Positional Encoding through Continuous Function Evaluation}
	Fine-tuning ViTs typically increases input resolution to capture fine details \citep{dosovitskiy2020image,steiner2021train}. Discrete learnable encodings are tied to a fixed grid and do not transfer; bilinear or bicubic interpolation attenuates high frequencies, introducing boundary aliasing and distorting long-range geometry \citep{keys2003cubic,touvron2021training}.As shown in Equation~\ref{eq:complex_mapping}, WePE effectively overcomes these limitations through its formulation as a continuous meromorphic function evaluated at arbitrary complex coordinates rather than a discrete lookup operation. The scaling factors $\alpha_u$ and $\alpha_v$ control the effective spatial frequency of the elliptic function across horizontal and vertical dimensions respectively, enabling the encoding to maintain optimal spatial discrimination at the increased resolution while preserving the fundamental periodic structure that encodes translational regularities in visual data. The imaginary half-period parameter adapts during fine-tuning to accommodate changes in the aspect ratio between patches and the overall spatial density of the representation, ensuring that the doubly periodic lattice structure characterized by the fundamental parallelogram maintains geometric consistency across resolutions. The lattice summation in Equation ~\ref{eq:truncated_series} remains numerically stable across different resolutions since the truncation parameters $M$ and $N$ are determined by the desired numerical precision rather than the specific image dimensions, ensuring consistent computational accuracy regardless of the resolution scaling factor. 
	This continuous formulation enables the generation of positional encodings at any spatial resolution without resorting to interpolation of pre-computed values, thereby preserving the mathematical precision and geometric fidelity of the spatial representation.

	\section{Experiments}
	
	To evaluate the effectiveness of WePE, we conduct experiments under both pre-training and fine-tuning settings. We further perform ablation studies to assess the contribution of each module component, and provide empirical analyses to reveal the underlying mechanism and theoretical rationale of WePE. More supplementary experiments, including robustness and invariance analysis (Appx.~\ref{sec:occlusion-robustness-analysis}, ~\ref{sec:geometric-invariance-analysis}), geometric properties and mechanism verification (Appx.~\ref{sec:relative-position-awareness}, \ref{sec:wepe-geometric-bias}, \ref{sec:supplement_long_term}, \ref{sec:attention-behavior-analysis}), comparison of different 2D encoding schemes (Appx.~\ref{sec:comparison-2d-pe-schemes}), adaptability and sensitivity analysis (Appx.~\ref{sec:extreme-aspect-ratio}, \ref{sec:fourier-params-sensitivity}), object detection and image segmentation (Appx.~\ref{app:coco_attention_seg}), and shortcomings and failure analysis (Appx.~\ref{sec:limitations-future-work}), are provided in Appendix~\ref{sec:additional_experiments}.

	\subsection{Understanding  WePE}

	\paragraph{WePE exhibits better geometric inductive bias} To investigate the inductive bias introduced by our proposed WePE, we first visualize self-attention maps of a randomly-initialized ViTs~\citep{dosovitskiy2020image}, without any training. Specifically, we focus on the attention distribution originating from a central query patch to all other patches within the sequence. As depicted in~\Cref{fig:inductive_bias_comparison}, the WePE-equipped model exhibits a highly structured and localized attention pattern: attention weights are concentrated on the query patch itself and decay smoothly and isotropically with increasing spatial distance. In contrast, the baseline using standard learnable Absolute Positional Encoding (APE)~\citep{dosovitskiy2020image} shows a largely uniform and unstructured attention distribution, where attention weights appear randomly scattered. This demonstrates that WePE inherently provides a strong spatial locality prior, predisposing the model to focus on local interactions even before learning, whereas standard encodings lack such an inductive bias. These results collectively demonstrate that WePE injects a robust and accurate 2D geometric inductive bias into the ViTs~\citep{dosovitskiy2020image} from initialization. This structural prior is absent in standard models, which must learn spatial relationships from data alone.

	\paragraph{Global Semantic Attention in ViTs~\citep{dosovitskiy2020image}} We trained two ViT-Tiny models under identical settings. We then visualized the complete information flow from input to output on unseen high-resolution images using the Attention Rollout method. The results, presented in ~\Cref{fig:semantic_attention}, consistently demonstrate a significant qualitative difference in the learned attention patterns. For instance, when presented with an image of a cat, the WePE-ViTs model's attention forms a coherent and complete silhouette that accurately envelops the entire animal. In stark contrast, the APE-ViTs attention is fragmented, focusing disproportionately on high-contrast edges where the subject meets the background, rather than the semantic object itself. WePE learns to associate features within a global spatial context, resulting in attention maps that align closely with the primary semantic content. The baseline APE model~\citep{dosovitskiy2020image}, lacking this structural prior, appears to overfit to low-level, local cues, leading to a fragmented attention mechanism that often fails to represent the complete semantic entity within the image. From these visualizations, we conclude that the geometric inductive bias inherent in our WePE enables the model to develop a more holistic and structurally-aware understanding of visual scenes.

	\begin{figure}[t]
		\centering
		\includegraphics[width=0.8\linewidth]{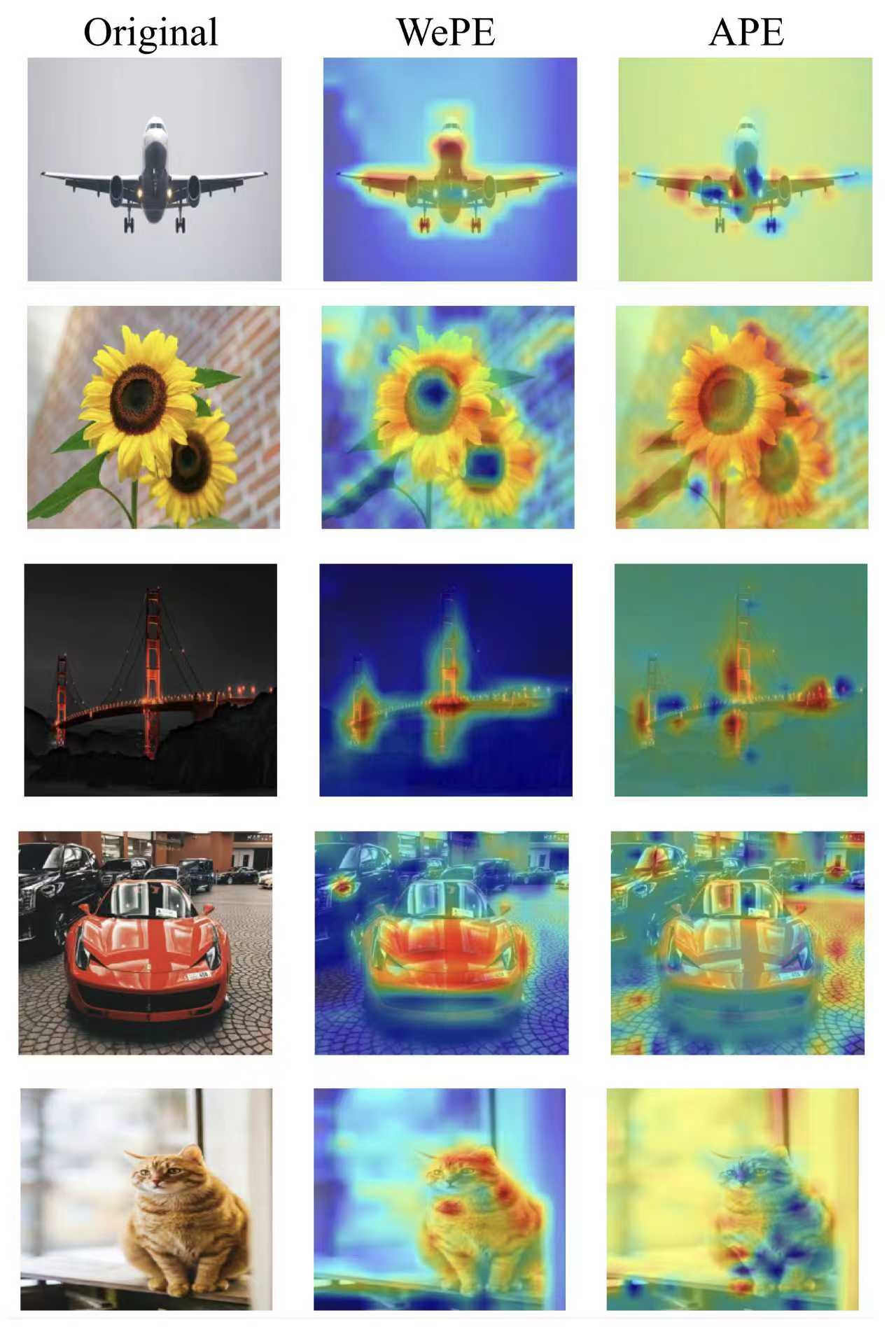}
		\caption{Attention rollout visualization comparing semantic focus patterns between WePE-ViTs and APE-ViTs models trained on CIFAR-100.}
		\label{fig:semantic_attention}
	\end{figure}

	\paragraph{Long-term Attenuation of Positional Encoding} \label{sec:long_term_attenuation}
	We verify the distance–decay property of WePE on a $14\times14$ patch grid (from $224\times224$ images). For all $\binom{196}{2}=19{,}110$ pairs we compute the normalized Euclidean distance $d_{\mathrm{rel}}\!\in\![0,100]$ and the cosine similarity $S$ between encodings $(\mathbf{p}_i,\mathbf{p}_j)$, rescale $S$ by min–max for visualization~\citep{jiawei2006data}, and aggregate results into 80 distance bins, taking the bin midpoint as the representative distance and the mean similarity as the representative score (see Appendix~\ref{sec:supplement_long_term}). The curve (see \Cref{fig:app-distance-decay-quad}) shows a pronounced negative correlation with distance ($\rho=-0.966$), evidencing long-range attenuation. 
	
	\begin{figure}[t]
		\centering
		\includegraphics[width=1.0\linewidth]{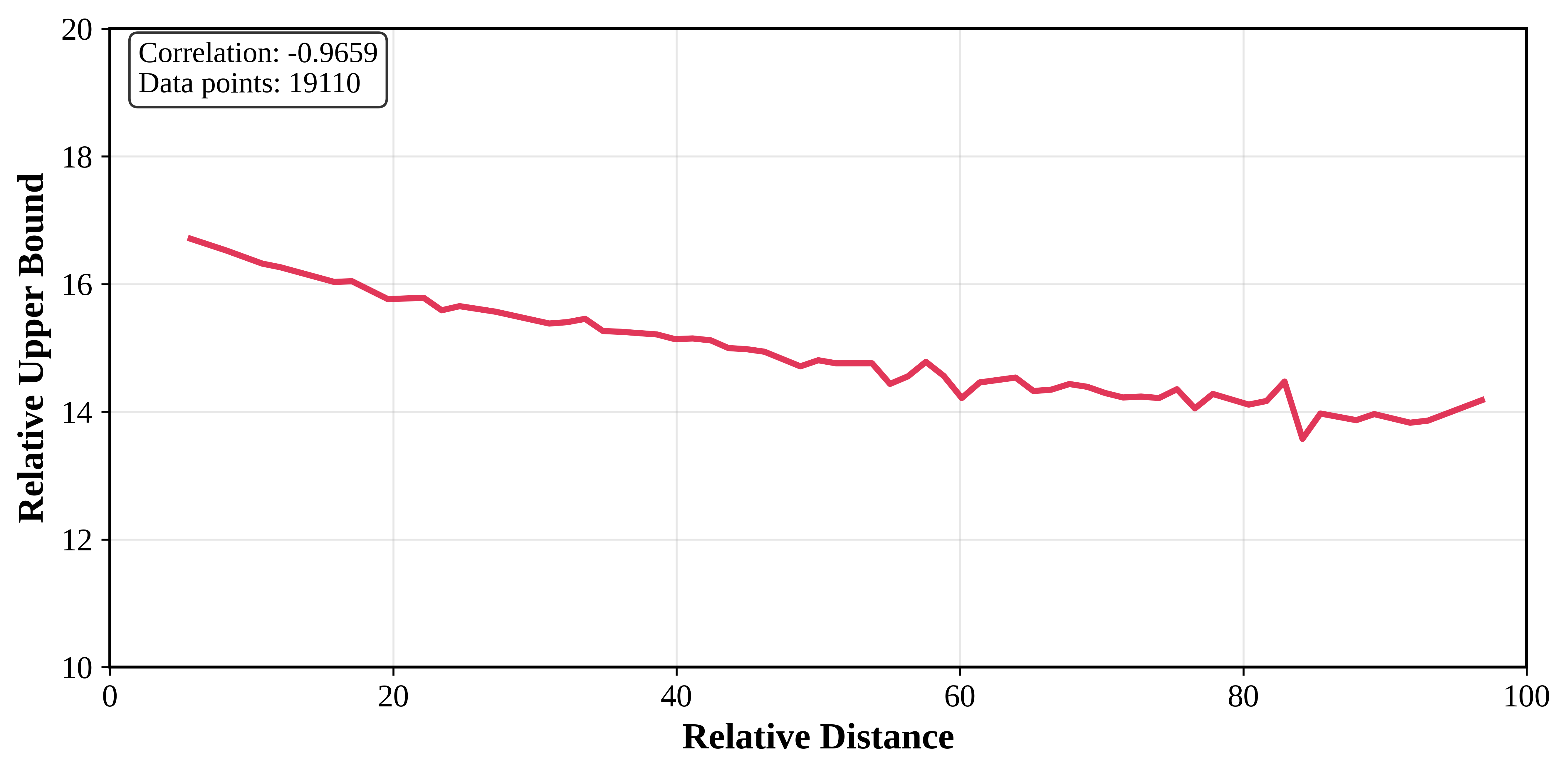}
		
		\caption{Empirical validation of the distance--decay property of WePE: x-axis denotes normalized patch spatial separation, y-axis denotes min--max scaled cosine similarity of positional encodings.}
		\label{fig:distance_decay_validation}
	\end{figure}

	In practical applications of ViTs~\citep{dosovitskiy2020image}, the self-attention mechanism depends not only on pure positional encodings, but more critically, on the fused representation of patch content features and their positional encodings.
	Similarly, to simulate a content-agnostic scenario and isolate the effect of the positional signal, we sample random content features $\mathbf{f}_{ij}\!\sim\!\mathcal{N}(\mathbf{0},\mathbf{I})$ in $\mathbb{R}^{192}$ and fuse them with WePE via $\mathbf{h}_{ij}=\mathbf{f}_{ij}+\mathbf{p}_{ij}$, repeating the analysis yields the same distance–decay trend (see~\Cref{fig:distance_decay_validation}).
	
	\subsection{Computational Efficiency}
	
	\begin{figure}[t]
		\centering
		\includegraphics[width=\columnwidth]{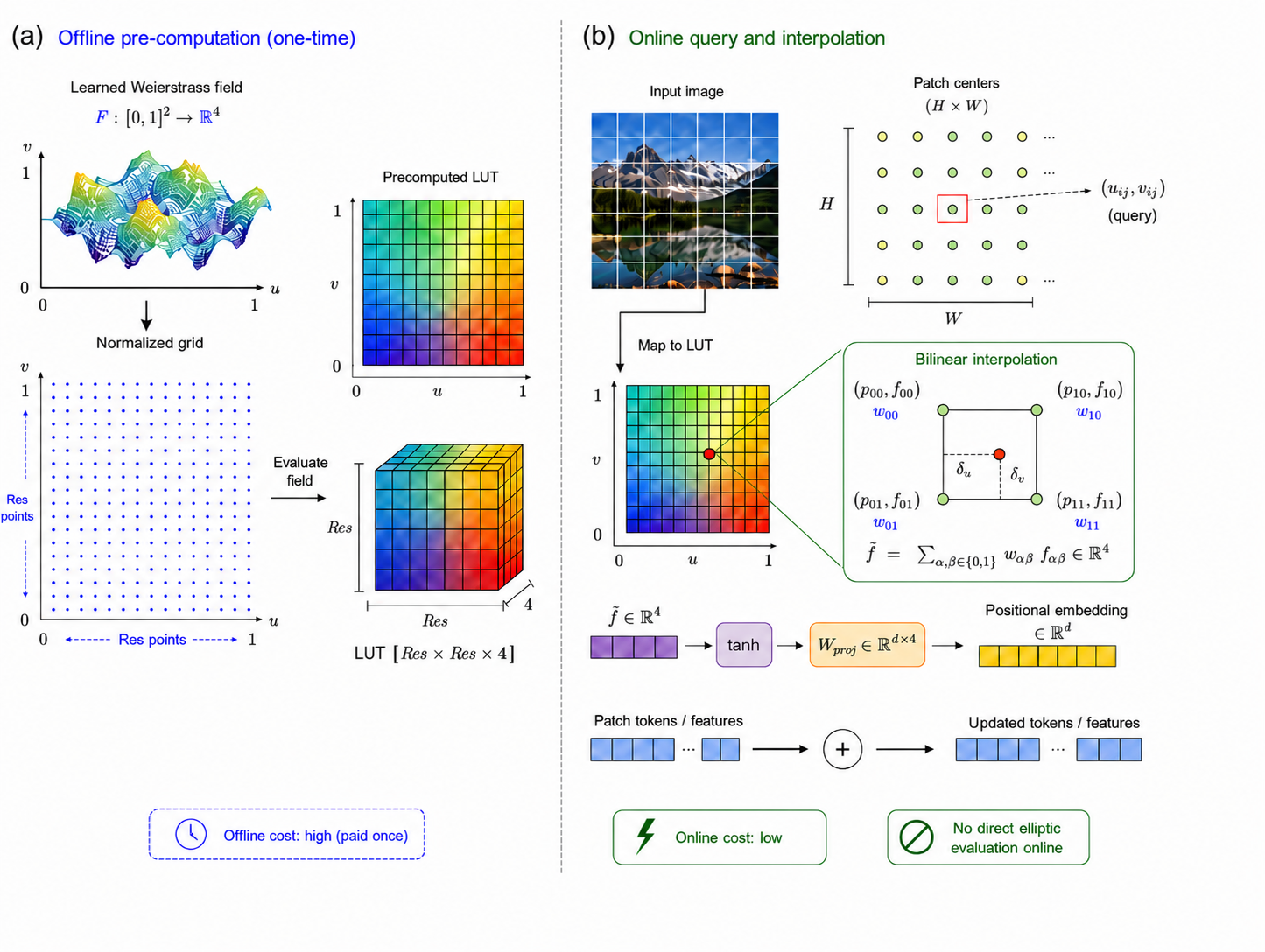}
		\caption{Illustration of the precomputed lookup-table implementation of WePE.
			The learned Weierstrass positional field is precomputed offline as a dense LUT,
			while online positional features are obtained by querying patch centers and
			bilinear interpolation, avoiding direct elliptic-function evaluation.}
		\label{fig:wepe_lut}
	\end{figure}

	Although WePE is defined through the Weierstrass elliptic function, its online implementation does not require repeated elliptic-function evaluation. After training, the learned elliptic-function parameters are fixed, and the corresponding four-channel positional field is precomputed on a dense regular grid. During inference and training, patch-level positional features are obtained from this lookup table by bilinear interpolation. We refer to this implementation as WePE-LUT. The detailed construction of the lookup table, the interpolation procedure, and the associated approximation-error analysis are provided in Appendix~\ref{sec:wef_pe_lut}.
	
	We benchmark the computational efficiency of WePE using inference latency, training-step time, and peak GPU memory across three architectures, ViT-T/B/L, and two input resolutions, $224\times224$ and $384\times384$. We compare WePE-LUT with APE, 2D sinusoidal positional encoding (2D-Sin), RoPE-Mixed, and FoPE. Measurements are conducted on a single NVIDIA RTX 4090 GPU with 24GB memory, CUDA 12.1, and PyTorch 2.2.0. We use batch size 32 for inference and batch size 16 for training. Each latency measurement uses 20 warm-up iterations and 100 timed iterations, with GPU-side timing obtained through \texttt{torch.cuda.Event} to avoid host-synchronization overhead.
	
	Direct lattice summation is also reported for ViT-T. Without the lookup table, evaluating $\wp(z)$ through truncated lattice summation introduces inference latencies of about 532--540 ms on ViT-T. In contrast, WePE-LUT fixes the learned elliptic-function parameters after training, stores a $256\times256\times4$ lookup table as a model buffer, and retrieves patch positional features through hardware-accelerated bilinear interpolation using \texttt{F.grid\_sample}. The online cost is therefore comparable to other grid-based positional encodings.
	
	\begin{table*}[t]
		\centering
		\caption{Benchmark results at resolution $224\times224$. Inference uses batch size 32 and 100 timed iterations. Training step uses batch size 16 and includes forward, backward, and optimizer update. Peak GPU memory includes activations and gradients. WePE-Direct is reported only for ViT-T due to its high cost.}
		\label{tab:benchmark_224}
		\resizebox{\textwidth}{!}{%
			\begin{tabular}{llrrrrr}
				\toprule
				Arch & PE Method & Inf. (ms) & Inf. Mem (MB) & Trn. (ms) & Trn. Mem (MB) & PE params \\
				\midrule
				\multirow{6}{*}{6*ViT-T}
				& APE         & 5.83   & 113   & 22.57   & 669    & 38,016 \\
				& 2D-Sin      & 6.01   & 165   & 23.07   & 713    & 192 \\
				& RoPE-Mix    & 6.06   & 209   & 21.82   & 756    & 12,480 \\
				& FoPE        & 6.28   & 252   & 22.21   & 800    & 3,464 \\
				& WePE-Direct & 532.56 & 295   & 1476.41 & 857    & 1,731 \\
				& WePE-LUT    & 6.35   & 340   & 22.86   & 888    & 58,087 \\
				\midrule
				\multirow{5}{*}{5*ViT-B}
				& APE      & 27.21 & 612   & 48.38 & 3,340  & 152,064 \\
				& 2D-Sin   & 27.28 & 1,283 & 48.52 & 4,025  & 768 \\
				& RoPE-Mix & 27.25 & 1,958 & 48.84 & 4,703  & 49,920 \\
				& FoPE     & 27.30 & 2,628 & 48.90 & 5,373  & 38,424 \\
				& WePE-LUT & 27.40 & 3,304 & 49.41 & 6,045  & 895,879 \\
				\midrule
				\multirow{5}{*}{5*ViT-L}
				& APE      & 81.98 & 1,521  & 150.99 & 9,604  & 202,752 \\
				& 2D-Sin   & 82.10 & 3,851  & 151.22 & 11,932 & 1,024 \\
				& RoPE-Mix & 82.13 & 6,183  & 151.16 & 14,263 & 66,560 \\
				& FoPE     & 82.08 & 8,515  & 151.65 & 16,596 & 67,616 \\
				& WePE-LUT & 82.23 & 10,853 & 151.71 & 18,935 & 1,587,719 \\
				\bottomrule
		\end{tabular}}
	\end{table*}
	
	\begin{table*}[t]
		\centering
		\caption{Benchmark results at resolution $384\times384$. ViT-L training exceeds 24GB GPU memory for all methods, inference results are still reported.}
		\label{tab:benchmark_384}
		\resizebox{\textwidth}{!}{%
			\begin{tabular}{llrrrrr}
				\toprule
				Arch & PE Method & Inf. (ms) & Inf. Mem (MB) & Trn. (ms) & Trn. Mem (MB) & PE params \\
				\midrule
				\multirow{6}{*}{6*ViT-T}
				& APE         & 18.23  & 419 & 27.10   & 2,367 & 110,976 \\
				& 2D-Sin      & 18.24  & 463 & 26.93   & 2,410 & 192 \\
				& RoPE-Mix    & 18.23  & 507 & 27.37   & 2,454 & 12,480 \\
				& FoPE        & 18.30  & 550 & 27.45   & 2,497 & 3,464 \\
				& WePE-Direct & 540.48 & 594 & 1491.96 & 2,576 & 1,731 \\
				& WePE-LUT    & 18.33  & 638 & 27.39   & 2,585 & 58,087 \\
				\midrule
				\multirow{5}{*}{5*ViT-B}
				& APE      & 98.57 & 1,270 & 171.36 & 10,082 & 443,904 \\
				& 2D-Sin   & 98.62 & 1,936 & 170.54 & 10,751 & 768 \\
				& RoPE-Mix & 98.56 & 2,606 & 170.36 & 11,403 & 49,920 \\
				& FoPE     & 98.66 & 3,270 & 170.77 & 12,068 & 38,424 \\
				& WePE-LUT & 98.71 & 3,942 & 170.28 & 12,740 & 895,879 \\
				\midrule
				\multirow{5}{*}{5*ViT-L}
				& APE      & 294.52 & 2,389 & OOM & OOM & 591,872 \\
				& 2D-Sin   & 294.65 & 3,553 & OOM & OOM & 1,024 \\
				& RoPE-Mix & 294.56 & 4,713 & OOM & OOM & 66,560 \\
				& FoPE     & 294.68 & 5,875 & OOM & OOM & 67,616 \\
				& WePE-LUT & 294.93 & 7,042 & OOM & OOM & 1,587,719 \\
				\bottomrule
		\end{tabular}}
	\end{table*}
	
	The benchmark shows that WePE-LUT introduces negligible latency overhead. On ViT-T at $224\times224$, WePE-LUT requires 6.35 ms compared with 5.83 ms for APE; at $384\times384$, the difference is only 0.10 ms. On ViT-B and ViT-L, the latency difference remains within measurement noise. Training-step time is also comparable: on ViT-B at $224\times224$, WePE-LUT requires 49.41 ms compared with 48.38 ms for APE.
	
	The main overhead of WePE-LUT is memory rather than latency. This comes primarily from the lightweight projection head that maps the four-dimensional elliptic features to the model dimension and stores the associated activations during training. The static $256\times256\times4$ lookup table contributes only about 1 MB and is independent of model size or input resolution. Therefore, once the lookup table is precomputed, the online time complexity of WePE-LUT is $O(Nd)$, the same order as standard absolute and 2D sinusoidal positional encodings, while the memory cost is moderately larger due to the projection module.

	\subsection{pre-training}
	\label{sec:pre-training}
	
	We begin with CIFAR-100~\citep{krizhevsky2009learning} and ImageNet-1k~\citep{deng2009imagenet} benchmarks to evaluate WePE in 2D vision tasks. All models are based on the ViTs~\citep{dosovitskiy2020image} . We compare WePE with APE~\citep{dosovitskiy2020image}, RoPE‑Mixed~\citep{heo2024rotary}, Auto-Scaled 2D RoPE (AS2DRoPE)~\citep{chu2024visionllama}, Sinusoidal Position Encoding(SPE)~\citep{vaswani2017attention}, FoPE~\citep{hua2024fourier}, RoPE~\citep{su2024roformer}, and LieRE~\citep{ostmeier2025liere} under identical configurations. To adapt RoPE and FoPE to 2D inputs, an image $I\!\in\!\mathbb{R}^{H\times W\times C}$ is patchified by a strided convolution (kernel\,$=$\,stride\,$=P$), each $P{\times}P$ patch is linearly projected and the resulting $H/P\times W/P$ grid is serialized into a 1D token sequence, on which the original sequence-based formulations are applied to the query/key vectors in multi-head self-attention to encode periodic, relative spatial structure. Table \ref{tab:cifar100_full_120} and Table \ref{tab:cifar_imagenet_200} compare the performance of ViTs~\citep{dosovitskiy2020image} integrated with different position encodings, trained from scratch for a varying number of epochs on multiple datasets. In these comparisons, WePE consistently demonstrates superior performance.

	\begin{table}[t]
		\centering
		\footnotesize 
		
		\setlength{\belowcaptionskip}{2pt}
		\renewcommand{\thead}[1]{\begin{tabular}{@{}c@{}}#1\end{tabular}}
		
		\caption{Top-1 accuracy after training for 120 epochs on the full dataset. (\%).}
		\label{tab:cifar100_full_120}
		
		\begin{tabular*}{\columnwidth}{@{\extracolsep{\fill}}lccccc@{}}
			\toprule
			Method & \thead{\textbf{WePE}} & \thead{APE} & RoPE & FoPE & \thead{SPE} \\
			\midrule
			Accuracy & \textbf{63.78} & 56.46 & 57.29 & 57.70 & 51.99 \\
			\bottomrule
		\end{tabular*}
	\end{table}

	\begin{table}[t]
		\centering
		\caption{Top-1 accuracy (\%) on CIFAR-100 and ImageNet-1k, trained for 200 epochs.}
		\label{tab:cifar_imagenet_200}
		
		\resizebox{\linewidth}{!}{
			\begin{tabular}{lccccccc} 
				\toprule
				Dataset & Frac. & \textbf{WePE} & LieRE$_8$ & LieRE$_{64}$ & RoPE-Mix & AS2D & APE \\
				\midrule
				CIFAR-100 & 20\%  & \textbf{46.36} & 45.42 & 44.44 & 44.48 & 39.14 & 39.80 \\
				CIFAR-100 & 40\%  & \textbf{56.81} & 54.68 & 54.64 & 55.14 & 50.53 & 49.90 \\
				CIFAR-100 & 60\%  & \textbf{63.38} & 62.04 & 62.90 & 61.56 & 58.58 & 56.83 \\
				CIFAR-100 & 90\%  & \textbf{68.96} & 67.72 & 68.36 & 67.00 & 62.59 & 62.76 \\
				ImageNet-1k & 100\% & \textbf{70.10} & 69.60 & 69.30 & 68.80 & 64.40 & 66.10 \\
				\bottomrule
		\end{tabular}}
		
	\end{table}
	
	To further assess the proposed method, we integrated the WePE into a Dynamic Hybrid Vision Transformer Tiny (DHVT-Ti) model~\citep{lu2022bridging}, which is engineered for data efficiency on smaller datasets. The DHVT model is specifically engineered to enhance the inductive biases of ViTs~\citep{dosovitskiy2020image} for improved data efficiency on small-scale datasets by incorporating convolutional operations. This serves as an excellent baseline model for comparing the pre-training capabilities of various vision models. As shown in Table \ref{tab:224_resolution_results}, the model achieves a peak validation accuracy of 76.53\%, surpassing all baselines.

	\subsection{fine-tuning}
	
	To assess transferability and data efficiency, we fine-tuned an ImageNet-21k pre-trained ViT-L/16 on  Visual Task Adaptation Benchmark 1000 (VTAB-1k) tasks under the 1k-shot protocol~\citep{zhai2019large}. Inputs were resized to $384\times384$, requiring bilinear interpolation of the pre-trained positional encodings from a $14\times14$ to a $24\times24$ grid; rather than discarding these encodings, we formed a hybrid position module that blends the interpolated encodings with our WePE through a learnable gate $\lambda$ (see Appendix ~\ref{sec:supp_experimental_details}). The experimental procedure is identical to that described in~\cite{dosovitskiy2020image}. The experimental results, as shown in Table~\ref{tab:vtab_breakdown_selected}, demonstrate that our algorithm outperforms traditional methods on most datasets.
	
	\begin{table}[t]
		\centering
		\footnotesize 
		\caption{Results on 224$\times$224 resolution.}
		\label{tab:224_resolution_results}
		
		\setlength{\tabcolsep}{3.5pt} 
		
		\begin{tabular}{lrrr}
			\toprule
			\textbf{Method} & \textbf{\#Params} & \textbf{GFLOPs} & \textbf{Accuracy (\%)} \\
			\midrule
			ResNet-50+$\mathcal{L}_{\text{dr.loc}}$  & 21.2M & 3.8 & 72.94 \\
			SwinT+$\mathcal{L}_{\text{dr.loc}}$      & 24.1M & 4.3 & 66.23 \\
			CvT-13+$\mathcal{L}_{\text{dr.loc}}$     & 19.6M & 4.5 & 74.51 \\
			T2T-ViT+$\mathcal{L}_{\text{dr.loc}}$    & 21.2M & 4.8 & 68.03 \\
			DHVT-T                                  & 6.0M  & 1.2 & 74.78 \\
			\midrule
			\textbf{WePE}                    & \textbf{5.5M} & 1.6 & \textbf{76.53} \\
			\bottomrule
		\end{tabular}
	\end{table}

	On the full CIFAR-100~\citep{krizhevsky2009learning} dataset, we fine-tuned an ImageNet-21k~\citep{deng2009imagenet} pre-trained ViT-B/16 while directly replacing the original learnable positional encodings with WePE. This configuration attains a peak test accuracy of $93.28\%$, indicating that WePE’s continuous, doubly periodic spatial representation benefits fine-grained recognition, the learned parameters $\omega_{3}^{\prime}\!\approx\!1.085$ and $\beta\!\approx\!0.610$ further evidence successful geometric adaptation to the dataset.

	\begin{table*}[t]
		\centering
		\caption{Performance breakdown on selected VTAB-1k tasks.}
		\label{tab:vtab_breakdown_selected}
		
		\providecommand{\rot}[1]{\rotatebox{90}{#1}}
		\providecommand{\rdb}{\textcolor{red}{$\bullet$}}
		\providecommand{\grb}{\textcolor{green}{$\bullet$}}
		\providecommand{\blb}{\textcolor{blue}{$\bullet$}}
		\providecommand{\mad}{\textcolor{magenta}{$\bullet$}}
		
		\setlength{\tabcolsep}{3pt}
		\renewcommand{\arraystretch}{1.1}
		
		\scriptsize
		
		\resizebox{\textwidth}{!}{%
			\begin{tabular}{@{}l ccccccc|cccc|cccccccc|c@{}}
				\toprule
				& \rot{Caltech101} & \rot{CIFAR-100} & \rot{DTD} & \rot{Flowers102} & \rot{Pets} & \rot{Sun397} & \rot{SVHN} &
				\rot{Camelyon} & \rot{EuroSAT} & \rot{Resisc45} & \rot{Retinopathy} &
				\rot{Clevr-Count} & \rot{Clevr-Dist} & \rot{DMLab} & \rot{dSpr-Loc} & \rot{dSpr-Ori} & \rot{KITTI-Dist} & \rot{sNORB-Azim} & \rot{sNORB-Elev} &
				\rot{Mean} \\
				
				\\[-2.5ex]
				& \rdb & \rdb & \rdb & \rdb & \rdb & \rdb & \rdb &
				\grb & \grb & \grb & \grb &
				\blb & \blb & \blb & \blb & \blb & \blb & \blb & \blb &
				\mad \\
				\midrule
				
				APE &
				90.80 & 84.10 & 74.10 & 99.30 & 92.70 & 61.00 & 80.90 &
				82.50 & 95.60 & 85.20 & 75.30 &
				70.30 & 56.10 & 41.90 & 74.70 & 64.90 & 79.90 & 30.50 & 41.70 &
				72.70 \\
				
				\textbf{WePE} &
				\textbf{91.32} & \textbf{87.59} & \textbf{77.41} & 98.79 & \textbf{93.16} & \textbf{64.30} & \textbf{84.58} &
				\textbf{83.73} & 93.89 & \textbf{86.10} & \textbf{77.15} &
				\textbf{73.81} & \textbf{60.91} & \textbf{54.24} & \textbf{75.18} & \textbf{68.10} & \textbf{81.25} & \textbf{34.11} & \textbf{42.01} &
				\textbf{73.59} \\
				\bottomrule
			\end{tabular}
		}
	\end{table*}

	\subsection{ablation}
	
	To evaluate the contribution of each component in WePE, we conducted ablation studies. The baseline achieves 63.78\%. Removing $\wp'(z)$ and using only $\text{Re}(\wp(z)),\text{Im}(\wp(z))$ lowers accuracy to 63.08\% ($-0.70$ \%), confirming the derivative provides essential gradient cues. Fixing $\alpha_{\text{scale}}$ and $\alpha_{\text{learn}}$ yields 62.88\% ($-0.90$\%), showing the necessity of adaptive scaling and lattice adjustment. Using non-lemniscatic invariants results in 63.20\% ($-0.58$\%), indicating robustness but also the superiority of the lemniscatic square lattice. Fixing the global scaling parameter to unity produces the largest drop, 62.60\% ($-1.18$\%), highlighting the need for adaptive control of positional strength. Overall, the consistent yet modest degradations across all settings demonstrate that while each component enhances performance, the primary advantage arises from the holistic geometric prior imparted by the Weierstrass elliptic function, seamlessly integrated into the ViTs~\citep{dosovitskiy2020image} backbone.
	
	\subsection{Hyperparameter Sensitivity of WePE}
	\label{app:wepe_sensitivity}
	We evaluate the sensitivity of WePE to four groups of hyperparameters on a $14\times14$ patch grid, corresponding to $224\times224$ images with patch size $16$. For each configuration, we report three types of measurements: (i) the Pearson correlation $\rho$ between normalized patch distance and positional-encoding dissimilarity $1-\cos$ over all $\binom{196}{2}=19{,}110$ patch pairs, which quantifies preservation of the distance-decay property; (ii) the proxy training loss after 30 gradient steps on a depth-6 ViT-T model, which probes optimization stability; and (iii) feature statistics, including the mean absolute feature value, the saturation fraction $|f|>0.99$, and the near-zero fraction $|f|<0.01$. The results are summarized in Tables~\ref{tab:kappa_sensitivity}--\ref{tab:cv_summary}.
	
	\begin{table}[t]
		\centering
		\caption{Sensitivity to the pole-neighbourhood multiplier $\kappa$.}
		\label{tab:kappa_sensitivity}
		\begin{tabular}{ccccccc}
			\toprule
			$\kappa$ & $\rho$ & Proxy loss & Mean$|f|$ & Std & Sat.\% & $\approx0$\% \\
			\midrule
			1.0   & 0.6253 & 4.508 & 0.1063 & 0.225 & 1.91 & 11.61 \\
			5.0   & 0.6253 & 4.454 & 0.1063 & 0.225 & 1.91 & 11.61 \\
			15.0  & 0.6253 & 4.471 & 0.1063 & 0.225 & 1.91 & 11.61 \\
			30.0  & 0.6253 & 4.515 & 0.1063 & 0.225 & 1.91 & 11.61 \\
			60.0  & 0.6253 & 4.531 & 0.1063 & 0.225 & 1.91 & 11.61 \\
			120.0 & 0.6253 & 4.496 & 0.1063 & 0.225 & 1.91 & 11.61 \\
			\bottomrule
		\end{tabular}
	\end{table}
	
	The pole-neighbourhood threshold $\kappa$, which activates the substitution $C_{\mathrm{large}}$ when $|z|<\kappa\epsilon$, is structurally inactive under the standard normalized coordinate mapping. For the lemniscatic square lattice with $\omega_1\approx2.622$, normalized patch coordinates $(u,v)\in(0,1)^2$ are mapped to $|z|\geq 0.5\omega_1/H\approx0.094$, which is approximately $10^7\epsilon$ away from the origin for $\epsilon=10^{-8}$. Hence, no patch coordinate falls within the clipping radius for the tested range $\kappa\in[1,120]$. This explains why the distance correlation, proxy loss, and feature statistics remain invariant up to measurement noise.
	
	\begin{table}[t]
		\centering
		\caption{Sensitivity to the pole substitute value $C_{\mathrm{large}}$.}
		\label{tab:clarge_sensitivity}
		\begin{tabular}{cccc}
			\toprule
			$C_{\mathrm{large}}$ & $\rho$ & Proxy loss & Sat.\% \\
			\midrule
			$5\times10^1$ & 0.6253 & 4.459 & 1.91 \\
			$2\times10^2$ & 0.6253 & 4.500 & 1.91 \\
			$5\times10^2$ & 0.6253 & 4.476 & 1.91 \\
			$2\times10^3$ & 0.6253 & 4.459 & 1.91 \\
			$5\times10^3$ & 0.6253 & 4.488 & 1.91 \\
			$1\times10^5$ & 0.6253 & 4.519 & 1.91 \\
			\bottomrule
		\end{tabular}
	\end{table}
	
	Similarly, the substitute value $C_{\mathrm{large}}$ has no measurable effect on the positional statistics because the clipping branch is not activated by any valid patch coordinate. The scaled hyperbolic tangent compression further ensures that, even if an extreme value were assigned near a pole, it would be mapped to a bounded value before projection.
	
	\begin{table}[t]
		\centering
		\caption{Sensitivity to the lattice scaling factors $\alpha_u=\alpha_v$.}
		\label{tab:lattice_scale_sensitivity}
		\begin{tabular}{cccc}
			\toprule
			$\alpha_u=\alpha_v$ & $\rho$ & Proxy loss & Mean$|f|$ \\
			\midrule
			0.20 & 0.613 & 4.464 & 0.323 \\
			0.40 & 0.625 & 4.488 & 0.106 \\
			0.60 & 0.620 & 4.473 & 0.053 \\
			0.80 & 0.513 & 4.469 & 0.040 \\
			1.00 & 0.347 & 4.556 & 0.073 \\
			1.20 & 0.089 & 4.478 & 0.100 \\
			\bottomrule
		\end{tabular}
	\end{table}
	
	The lattice scaling factors are the only tested hyperparameters that substantially affect the distance-decay metric. Values in $\alpha\in[0.2,0.6]$ preserve strong distance decay, with $\rho\in[0.61,0.63]$, whereas larger values move patch coordinates into higher-frequency oscillatory regions of $\wp(z)$ and reduce monotonicity. This sensitivity is expected: $\alpha_u$ and $\alpha_v$ control the effective spatial frequency of the elliptic field, analogous to the choice of frequency bands in sinusoidal positional encodings. In all main experiments, we set $\alpha_u=\alpha_v=0.4$ by default, while the learnable half-period parameter $\omega_3'$ allows the vertical lattice geometry to adapt during training. The proxy loss varies by only about $1\%$ across the sweep, indicating that training stability is not sensitive to this choice even when the geometric correlation degrades at extreme values.
	
	\begin{table}[t]
		\centering
		\caption{Sensitivity to the tanh compression factor $\alpha_{\mathrm{scale}}$.}
		\label{tab:tanh_scale_sensitivity}
		\begin{tabular}{ccccc}
			\toprule
			$\alpha_{\mathrm{scale}}$ & $\rho$ & Proxy loss & Sat.\% & $\approx0$\% \\
			\midrule
			0.010 & 0.634 & 4.485 & 0.26 & 84.82 \\
			0.050 & 0.630 & 4.559 & 0.89 & 42.98 \\
			0.100 & 0.627 & 4.434 & 1.40 & 20.92 \\
			0.150 & 0.625 & 4.442 & 1.91 & 11.61 \\
			0.300 & 0.625 & 4.492 & 3.32 & 5.87 \\
			0.500 & 0.626 & 4.486 & 4.34 & 3.57 \\
			\bottomrule
		\end{tabular}
	\end{table}
	
	The compression factor $\alpha_{\mathrm{scale}}$ has negligible influence on the distance-decay metric and proxy loss. Its main effect is feature utilization: overly small values leave most features close to zero, whereas overly large values increase saturation. The default $\alpha_{\mathrm{scale}}=0.15$ balances these two effects. In the implementation, $\alpha_{\mathrm{scale}}$ is parameterized through a softplus transformation and can adapt during training.
	
	\begin{table}[t]
		\centering
		\caption{Coefficient of variation (CV = std/mean) of key metrics across each hyperparameter sweep. CV $<0.01$ indicates negligible sensitivity.}
		\label{tab:cv_summary}
		\begin{tabular}{lccc}
			\toprule
			Sweep & CV($\rho$) & CV(loss) & CV(mean$|f|$) \\
			\midrule
			A: threshold $\kappa$ & 0.0000 & 0.0059 & 0.0000 \\
			B: $C_{\mathrm{large}}$ & 0.0000 & 0.0049 & 0.0000 \\
			C: $\alpha_u=\alpha_v$ & 0.4175 & 0.0070 & 0.8230 \\
			D: $\alpha_{\mathrm{scale}}$ & 0.0049 & 0.0091 & 0.6610 \\
			\bottomrule
		\end{tabular}
	\end{table}
	
	Overall, two numerical-safety hyperparameters, $\kappa$ and $C_{\mathrm{large}}$, are structurally inert under standard coordinate normalization and can be set to any reasonable value. The tanh compression factor has negligible influence on both the distance-decay correlation and short-run optimization stability. Only the lattice scaling factors require care, and the interval $[0.2,0.6]$ consistently preserves strong distance decay. These results support the robustness of WePE and indicate that its main user-facing spatial-frequency control can be fixed to a stable default, with the learnable half-period parameter adapting the lattice aspect ratio during training.

	\section{Conclusion}
	In this work, we introduce Weierstrass elliptic Positional Encoding, a mathematically principled
	approach that leverages the rich structure of elliptic functions to address spatial representation
	limitations in Vision Transformers. Our method preserves 2D spatial relationships through a direct
	complex domain mapping and provides explicit spatial proximity priors via a theoretically guaranteed
	distance-decay property. We demonstrated the effectiveness of WePE through extensive experiments,
	achieving superior performance in most from-scratch training and fine-tuning scenarios across a
	variety of standard benchmarks. Furthermore, empirical analyses are conducted to investigate the
	underlying factors contributing to the superiority of WePE. Additionally, we provide a rigorous
	exposition and derivation of the core mathematical principles underpinning WePE. In summary, our
	proposed WePE offers a plug-and-play, resolution-agnostic positional module that restores the 2D
	geometric inductive bias with negligible computational and memory overhead, making it a practical
	drop-in replacement for existing encodings in ViTs.

	\bibliographystyle{IEEEtran}
	\bibliography{IEEEabrv,references}

\begin{thebibliography}{10}
\providecommand{\url}[1]{#1}
\csname url@samestyle\endcsname
\providecommand{\newblock}{\relax}
\providecommand{\bibinfo}[2]{#2}
\providecommand{\BIBentrySTDinterwordspacing}{\spaceskip=0pt\relax}
\providecommand{\BIBentryALTinterwordstretchfactor}{4}
\providecommand{\BIBentryALTinterwordspacing}{\spaceskip=\fontdimen2\font plus
\BIBentryALTinterwordstretchfactor\fontdimen3\font minus
  \fontdimen4\font\relax}
\providecommand{\BIBforeignlanguage}[2]{{%
\expandafter\ifx\csname l@#1\endcsname\relax
\typeout{** WARNING: IEEEtran.bst: No hyphenation pattern has been}%
\typeout{** loaded for the language `#1'. Using the pattern for}%
\typeout{** the default language instead.}%
\else
\language=\csname l@#1\endcsname
\fi
#2}}
\providecommand{\BIBdecl}{\relax}
\BIBdecl

\bibitem{dosovitskiy2020image}
A.~Dosovitskiy, L.~Beyer, A.~Kolesnikov, D.~Weissenborn, X.~Zhai,
  T.~Unterthiner, M.~Dehghani, M.~Minderer, G.~Heigold, S.~Gelly \emph{et~al.},
  ``An image is worth 16x16 words,'' \emph{arXiv preprint arXiv:2010.11929},
  vol.~7, p.~5, 2020.

\bibitem{lecun2002gradient}
Y.~LeCun, L.~B{\'e}on, Y.~Bengio, and P.~Haffner, ``Gradient-based learning
  applied to document recognition,'' \emph{Proceedings of the IEEE}, vol.~86,
  no.~11, pp. 2278--2324, 2002.

\bibitem{vaswani2017attention}
A.~Vaswani, N.~Shazeer, N.~Parmar, J.~Uszkoreit, L.~Jones, A.~N. Gomez,
  {\L}.~Kaiser, and I.~Polosukhin, ``Attention is all you need,''
  \emph{Advances in Neural Information Processing Systems}, vol.~30, 2017.

\bibitem{zeiler2014visualizing}
M.~D. Zeiler and R.~Fergus, ``Visualizing and understanding convolutional
  networks,'' in \emph{Proceedings of the European Conference on Computer
  Vision (ECCV)}.\hskip 1em plus 0.5em minus 0.4em\relax Springer, 2014, pp.
  818--833.

\bibitem{shaw2018self}
P.~Shaw, J.~Uszkoreit, and A.~Vaswani, ``Self-attention with relative position
  representations,'' in \emph{Proceedings of the 2018 Conference of the North
  American Chapter of the Association for Computational Linguistics: Human
  Language Technologies (NAACL-HLT)}, 2018, pp. 464--468.

\bibitem{parmar2018image}
N.~Parmar, A.~Vaswani, J.~Uszkoreit, L.~Kaiser, N.~Shazeer, A.~Ku, and D.~Tran,
  ``Image transformer,'' in \emph{Proceedings of the 35th International
  Conference on Machine Learning (ICML 2018)}.\hskip 1em plus 0.5em minus
  0.4em\relax PMLR, 2018, pp. 4055--4064.

\bibitem{hua2024fourier}
E.~Hua, C.~Jiang, X.~Lv, K.~Zhang, Y.~Sun, Y.~Fan, X.~Zhu, B.~Qi, N.~Ding, and
  B.~Zhou, ``Fourier position embedding: Enhancing attention's periodic
  extension for length generalization,'' \emph{arXiv preprint
  arXiv:2412.17739}, 2024.

\bibitem{su2024roformer}
J.~Su, M.~Ahmed, Y.~Lu, S.~Pan, W.~Bo, and Y.~Liu, ``Roformer: Enhanced
  transformer with rotary position embedding,'' \emph{Neurocomputing}, vol.
  568, p. 127063, 2024.

\bibitem{ostmeier2025liere}
S.~Ostmeier, B.~Axelrod, M.~Varma, M.~Moseley, A.~S. Chaudhari, and
  C.~Langlotz, ``Liere: Lie rotational positional encodings,'' in
  \emph{Proceedings of the 42nd International Conference on Machine Learning
  (ICML 2025)}, 2025.

\bibitem{heo2024rotary}
B.~Heo, S.~Park, D.~Han, and S.~Yun, ``Rotary position embedding for vision
  transformer,'' in \emph{European Conference on Computer Vision}.\hskip 1em
  plus 0.5em minus 0.4em\relax Springer, 2024, pp. 289--305.

\bibitem{wu2021rethinking}
K.~Wu, H.~P. amenities, M.~Chen, J.~Fu, and H.~Chao, ``Rethinking and improving
  relative position encoding for vision transformer,'' in \emph{Proceedings of
  the IEEE/CVF International Conference on Computer Vision}, 2021, pp.
  10\,033--10\,041.

\bibitem{cordonnier2019relationship}
J.~B. Cordonnier, A.~Loukas, and M.~Jaggi, ``On the relationship between
  self-attention and convolutional layers,'' \emph{arXiv preprint
  arXiv:1911.03584}, 2019.

\bibitem{weierstrass1854theorie}
K.~Weierstra{\ss}, ``Theorie der abel'schen functionen,'' \emph{Journal f{\"u}r
  die reine und angewandte Mathematik (Crelle's Journal)}, vol.~47, pp.
  289--306, 1854.

\bibitem{lozier2003nist}
D.~W. Lozier, ``Nist digital library of mathematical functions,'' \emph{Annals
  of Mathematics and Artificial Intelligence}, vol.~38, no.~1, pp. 105--119,
  2003.

\bibitem{higham2002accuracy}
N.~J. Higham, \emph{Accuracy and Stability of Numerical Algorithms}.\hskip 1em
  plus 0.5em minus 0.4em\relax SIAM, 2002.

\bibitem{ba2016layer}
J.~L. Ba, J.~R. Kiros, and G.~E. Hinton, ``Layer normalization,'' \emph{arXiv
  preprint arXiv:1607.06450}, 2016.

\bibitem{rumelhart1986learning}
D.~E. Rumelhart, G.~E. Hinton, and R.~J. Williams, ``Learning representations
  by back-propagating errors,'' \emph{Nature}, vol. 323, no. 6088, pp.
  533--536, 1986.

\bibitem{devlin2019bert}
J.~Devlin, M.~W. Chang, K.~Lee, and K.~Toutanova, ``Bert: Pre-training of deep
  bidirectional transformers for language understanding,'' in \emph{Proceedings
  of the 2019 Conference of the North American Chapter of the Association for
  Computational Linguistics: Human Language Technologies, Volume 1 (Long and
  Short Papers)}, 2019, pp. 4171--4186.

\bibitem{steiner2021train}
A.~Steiner, A.~Kolesnikov, X.~Zhai, R.~Wightman, J.~Uszkoreit, and L.~Beyer,
  ``How to train your vit? data, augmentation, and regularization in vision
  transformers,'' \emph{arXiv preprint arXiv:2106.10270}, 2021.

\bibitem{keys2003cubic}
R.~Keys, ``Cubic convolution interpolation for digital image processing,''
  \emph{IEEE Transactions on Acoustics, Speech, and Signal Processing},
  vol.~29, no.~6, pp. 1153--1160, 2003.

\bibitem{touvron2021training}
H.~Touvron, M.~Cord, M.~Douze, F.~Massa, A.~Sablayrolles, and H.~J{\'e}gou,
  ``Training data-efficient image transformers \& distillation through
  attention,'' in \emph{Proceedings of the 38th International Conference on
  Machine Learning (ICML 2021)}.\hskip 1em plus 0.5em minus 0.4em\relax PMLR,
  2021, pp. 10\,347--10\,357.

\bibitem{jiawei2006data}
H.~Jiawei and M.~Kamber, \emph{Data Mining: Concepts and Techniques}.\hskip 1em
  plus 0.5em minus 0.4em\relax Morgan Kaufmann, 2006.

\bibitem{krizhevsky2009learning}
A.~Krizhevsky, G.~Hinton \emph{et~al.}, ``Learning multiple layers of features
  from tiny images,'' 2009.

\bibitem{deng2009imagenet}
J.~Deng, W.~Dong, R.~Socher, L.~J. Li, K.~Li, and L.~Fei-Fei, ``Imagenet: A
  large-scale hierarchical image database,'' in \emph{2009 IEEE Conference on
  Computer Vision and Pattern Recognition}.\hskip 1em plus 0.5em minus
  0.4em\relax IEEE, 2009, pp. 248--255.

\bibitem{chu2024visionllama}
X.~Chu, J.~Su, B.~Zhang, and C.~Shen, ``Visionllama: A unified llama backbone
  for vision tasks,'' in \emph{European Conference on Computer Vision}.\hskip
  1em plus 0.5em minus 0.4em\relax Springer, 2024, pp. 1--18.

\bibitem{lu2022bridging}
Z.~Lu, H.~Xie, C.~Liu, and Y.~Zhang, ``Bridging the gap between vision
  transformers and convolutional neural networks on small datasets,''
  \emph{Advances in Neural Information Processing Systems}, vol.~35, pp.
  14\,663--14\,677, 2022.

\bibitem{zhai2019large}
X.~Zhai, J.~Puigcerver, A.~Kolesnikov, P.~Ruyssen, C.~Riquelme, M.~Lucic,
  J.~Djolonga, A.~S. Pinto, M.~Neumann, A.~Dosovitskiy \emph{et~al.}, ``A
  large-scale study of representation learning with the visual task adaptation
  benchmark,'' \emph{arXiv preprint arXiv:1910.04867}, 2019.

\bibitem{bai2023qwen}
J.~Bai, S.~Bai, Y.~Chu, Z.~Cui, K.~Dang, X.~Deng, Y.~Fan, W.~Ge, Y.~Han,
  F.~Huang \emph{et~al.}, ``Qwen technical report,'' \emph{arXiv preprint
  arXiv:2309.16609}, 2023.

\bibitem{wang2023geometric}
Y.~Wang, S.~Li, T.~Wang, B.~Shao, N.~Zheng, and T.~Liu, ``Geometric transformer
  with interatomic positional encoding,'' \emph{Advances in Neural Information
  Processing Systems}, vol.~36, pp. 55\,981--55\,994, 2023.

\bibitem{gray2006toeplitz}
R.~M. Gray, ``Toeplitz and circulant matrices: A review,'' 2006.

\bibitem{catmull1974subdivision}
E.~E. Catmull, \emph{A Subdivision Algorithm for Computer Display of Curved
  Surfaces}.\hskip 1em plus 0.5em minus 0.4em\relax The University of Utah,
  1974.

\bibitem{ridnik2021imagenet}
T.~Ridnik, E.~Ben-Baruch, A.~Noy, and L.~Zelnik-Manor, ``Imagenet-21k
  pretraining for the masses,'' \emph{arXiv preprint arXiv:2104.10972}, 2021.

\bibitem{he2016deep}
K.~He, X.~Zhang, S.~Ren, and J.~Sun, ``Deep residual learning for image
  recognition,'' in \emph{Proceedings of the IEEE Conference on Computer Vision
  and Pattern Recognition}, 2016, pp. 770--778.

\bibitem{bishop2006pattern}
C.~M. Bishop and N.~M. Nasrabadi, \emph{Pattern Recognition and Machine
  Learning}.\hskip 1em plus 0.5em minus 0.4em\relax Springer, 2006, vol.~4.

\bibitem{liu2021swin}
Z.~Liu, Y.~Lin, Y.~Cao, H.~Hu, Y.~Wei, Z.~Zhang, S.~Lin, and B.~Guo, ``Swin
  transformer: Hierarchical vision transformer using shifted windows,'' in
  \emph{Proceedings of the IEEE/CVF International Conference on Computer
  Vision}, 2021, pp. 10\,012--10\,022.

\bibitem{heo2021rethinking}
B.~Heo, S.~Yun, D.~Han, S.~Chun, J.~Choe, and S.~J. Oh, ``Rethinking spatial
  dimensions of vision transformers,'' in \emph{Proceedings of the IEEE/CVF
  International Conference on Computer Vision}, 2021, pp. 11\,936--11\,945.

\bibitem{li2022exploring}
Y.~Li, H.~Mao, R.~Girshick, and K.~He, ``Exploring plain vision transformer
  backbones for object detection,'' in \emph{European Conference on Computer
  Vision}.\hskip 1em plus 0.5em minus 0.4em\relax Springer, 2022, pp. 280--296.

\bibitem{he2022masked}
K.~He, X.~Chen, S.~Xie, Y.~Li, P.~Doll{\'a}r, and R.~Girshick, ``Masked
  autoencoders are scalable vision learners,'' in \emph{Proceedings of the
  IEEE/CVF Conference on Computer Vision and Pattern Recognition}, 2022, pp.
  16\,000--16\,009.

\end{thebibliography}

	\clearpage
	
	\appendices

	\section{Supplementary Background Knowledge}
	\label{sec:supplementary_background} 

	Mainstream explicit function-based positional
	encodings for Vision Transformers~\citep{dosovitskiy2020image} that are added
	to patch embeddings can be broadly grouped into three families:  (i) learnable absolute positional embeddings~\citep{dosovitskiy2020image}, (ii) Sinusoidal Position Encoding~\citep{vaswani2017attention} and their higher-order Fourier extensions (e.g., FoPE~\citep{hua2024fourier}), and (iii) RoPE~\citep{su2024roformer} and their complex-valued or group-theoretic extensions (e.g., LieRE~\citep{ostmeier2025liere}, RoPE‑Mixed~\citep{heo2024rotary}, Multimodal Rotary Position Embedding (M-RoPE)~\citep{bai2023qwen}. Indeed, according to Euler's formula, the foundation of RoPE~\citep{su2024roformer} still lies in Sinusoidal Position Encoding~\citep{vaswani2017attention}. Its essence resides in leveraging Euler's formula to map positional information into vector rotations, making RoPE~\citep{su2024roformer} fundamentally a geometric extension of the traditional sinusoidal encoding.
	However, WePE is conceptually distinct from all two or three families above, even though it is also built upon periodic and complex-valued functions. 
	To our knowledge, WePE is the first 2D positional encoding scheme designed for ViTs~\citep{dosovitskiy2020image} that is genuinely constructed on the complex plane. Rather than composing multiple 1D sinusoidal bands along the flattened token index or
	applying block-wise rotary phases in the hidden space, WePE is formulated as a
	genuinely 2D positional function on the complex plane: it maps the
	2D patch lattice to a complex lattice and evaluates a Weierstrass elliptic
	function with an intrinsic doubly periodic structure.  As a result, the 2D
	geometry of the image grid is an inherent
	property of the encoding itself, rather than an artifact of the 1D
	serialization or separable 1D Fourier bases. Through the elliptic addition formula, absolute
	positional codes in WePE are algebraically linked to relative displacements,
	providing a tight coupling between absolute and relative position at the
	function level rather than relying solely on phase differences in Fourier
	space. This makes WePE the "fourth approach", standing apart from the three mainstream positional encoding schemes currently available.

	In this part, we supplement with further necessary preliminaries and the corresponding proofs of theorems regarding the Weierstrass elliptic function~\citep{weierstrass1854theorie}.

	\begin{definition}[Meromorphic Function]
		Let $D \subset \mathbb{C}$ be an open set. A function $f: D \to \mathbb{C} \cup \{\infty\}$ is called meromorphic in $D$ if $f$ is analytic everywhere in $D$ except at finitely many isolated singularities, and each singularity is a pole.
	\end{definition}
	
	The Cauchy integral formula is one of the core tools in complex analysis:
	
	\begin{theorem}[Cauchy Integral Formula]
		Let $f(z)$ be analytic on a simple closed curve $C$ and its interior, and let $z_0$ be a point inside $C$. Then:
		\begin{equation}
			f^{(n)}(z_0) = \frac{n!}{2\pi i} \oint_C \frac{f(z)}{(z-z_0)^{n+1}} dz
			\label{eq:cauchy_integral}
		\end{equation}
	\end{theorem}
	
	Based on the Cauchy integral formula, we can derive Liouville's theorem:
	
	\begin{theorem}[Liouville's Theorem]
		Any bounded entire function must be constant. That is, if $f(z)$ is analytic everywhere on the complex plane $\mathbb{C}$ and there exists a constant $M > 0$ such that $|f(z)| \leq M$ for all $z \in \mathbb{C}$, then $f(z)$ is constant.
	\end{theorem}
	
	\begin{proof}
		By the Cauchy integral formula, for any $z_0 \in \mathbb{C}$ and $r > 0$:
		\begin{equation}
			|f'(z_0)| \leq \frac{1}{r} \sup_{|z-z_0|=r} |f(z)| \leq \frac{M}{r}
			\label{eq:derivative_bound}
		\end{equation}
		
		As $r \to \infty$, $\frac{M}{r} \to 0$, hence $|f'(z_0)| = 0$, which implies $f'(z_0) = 0$.
		
		Since $z_0$ is arbitrary, $f'(z) \equiv 0$ holds throughout $\mathbb{C}$. Let $f(z) = u(x,y) + iv(x,y)$. By the Cauchy-Riemann equations:
		\begin{align}
			\frac{\partial u}{\partial x} &= \frac{\partial v}{\partial y} = 0 \label{eq:cr1}\\
			\frac{\partial u}{\partial y} &= -\frac{\partial v}{\partial x} = 0 \label{eq:cr2}
		\end{align}
		
		This implies that all partial derivatives of $u(x,y)$ and $v(x,y)$ are zero, therefore $u$ and $v$ are both constants, and consequently $f(z)$ is constant.
		
	\end{proof}
	\begin{definition}[Period Lattice]
		Let $\omega_1, \omega_3 \in \mathbb{C}$ be linearly independent (i.e., $\frac{\omega_3}{\omega_1} \notin \mathbb{R}$). The period lattice is defined as:
		\begin{equation}
			\Lambda = \{2m\omega_1 + 2n\omega_3 : m,n \in \mathbb{Z}\}
			\label{eq:lattice_def}
		\end{equation}
		where $2\omega_1$ and $2\omega_3$ are called fundamental periods.
	\end{definition}
	
	The period lattice divides the complex plane into congruent parallelograms, with each fundamental parallelogram determined by vertices $\{0, 2\omega_1, 2\omega_3, 2\omega_1 + 2\omega_3\}$.

	\begin{definition}[Elliptic Function]
		\label{def:elliptic_function} 
		An elliptic function with period lattice $\Lambda$ is a meromorphic function $f: \mathbb{C} \to \mathbb{C} \cup \{\infty\}$ satisfying:
		\begin{enumerate}
			\item $f(z + \omega) = f(z)$ for all $z \in \mathbb{C}$ and $\omega \in \Lambda$
			\item $f$ has only finitely many poles in the fundamental parallelogram
			\item $f$ is not identically constant
		\end{enumerate}
	\end{definition}

	\begin{definition}[Weierstrass Elliptic Function]
		\label{def:weierstrass}
		For the period lattice $\Lambda = \{2m\omega_1 + 2n\omega_3 : m,n \in \mathbb{Z}\}$, the Weierstrass elliptic function is defined as:
		\begin{equation}
			\wp(z) = \frac{1}{z^2} + \sum_{\omega \in \Lambda \setminus \{0\}} \left( \frac{1}{(z-\omega)^2} - \frac{1}{\omega^2} \right)
			\label{eq:weierstrass_series}
		\end{equation}
	\end{definition}

	\begin{theorem}[Laurent Expansion of Weierstrass Function]
		In a neighborhood of the origin,$\wp(z)$ has a specific Laurent expansion:
		\begin{equation}
			\wp(z) = \frac{1}{z^2} + \frac{g_2}{20}z^2 + \frac{g_3}{28}z^4 + \frac{g_2^2}{1200}z^6 + \cdots
			\label{eq:laurent_expansion}
		\end{equation}
		where $g_2, g_3$ are elliptic invariants.
	\end{theorem}
	
	\begin{theorem}[Weierstrass Differential Equation]
		\label{thm:weierstrass_ode} 
		\begin{equation}
			(\wp'(z))^2 = 4(\wp(z))^3 - g_2\wp(z) - g_3
			\label{eq:weierstrass_ode}
		\end{equation}
	\end{theorem}
	
	\begin{proof}
		Define the auxiliary function:
		\begin{equation}
			f(z) = (\wp'(z))^2 - 4(\wp(z))^3 + g_2\wp(z) + g_3
			\label{eq:auxiliary_function}
		\end{equation}
		
		Through Laurent expansion analysis, we have:
		\begin{align}
			(\wp'(z))^2 &= \frac{4}{z^6} - \frac{2g_2}{5z^2} - \frac{4g_3}{7} + \cdots \label{eq:derivative_squared}\\
			4(\wp(z))^3 &= \frac{4}{z^6} + \frac{3g_2}{5z^2} + \frac{3g_3}{7} + \cdots \label{eq:function_cubed}\\
			g_2\wp(z) &= \frac{g_2}{z^2} + \cdots \label{eq:g2_function}
		\end{align}
		
		Substituting these expansions into $f(z)$:
		- $z^{-6}$ term: $\frac{4}{z^6} - \frac{4}{z^6} = 0$
		- $z^{-2}$ term: $-\frac{2g_2}{5z^2} - \frac{3g_2}{5z^2} + \frac{g_2}{z^2} = 0$
		- Constant term: $-\frac{4g_3}{7} - \frac{3g_3}{7} + g_3 = 0$
		
		Therefore, $f(z)$ has no singularity at $z = 0$. Similarly, $f(z)$ has no singularities at other lattice points, so $f(z)$ is holomorphic on $\mathbb{C}$.
		
		Since both $\wp(z)$ and $\wp'(z)$ are doubly periodic, $f(z)$ is also doubly periodic. In the fundamental parallelogram, $f(z)$ is continuous and has no poles, hence is bounded. By periodicity, $f(z)$ is bounded on the entire complex plane.
		
		By Liouville's theorem, $f(z) \equiv C$ (constant). Through analysis of special values, we can determine $C = 0$, therefore the differential equation holds.
	\end{proof}

	When the elliptic invariant $g_3 = 0$, the elliptic curve degenerates to the lemniscatic case:
	\begin{equation}
		y^2 = 4x^3 - g_2x = x(4x^2 - g_2)
		\label{eq:lemniscatic_curve}
	\end{equation}
	
	In this case, the elliptic curve has special symmetry properties, and the period lattice forms a square structure.

	\begin{theorem}[Half-Periods in Lemniscatic Case]
		When $g_2 = 1, g_3 = 0$, the real half-period is:
		\begin{equation}
			\omega_1 = \frac{\Gamma^2(1/4)}{2\sqrt{2\pi}} \approx 2.62205755429212
			\label{eq:exact_half_period}
		\end{equation}
		where $\Gamma$ is the gamma function.
	\end{theorem}

	\begin{definition}[Elliptic Curve Group Law]
		Let $P_1 = (x_1, y_1), P_2 = (x_2, y_2)$ be two points on the elliptic curve. If $x_1 \neq x_2$, then $P_3 = P_1 + P_2$ has coordinates:
		\begin{align}
			x_3 &= \left(\frac{y_2 - y_1}{x_2 - x_1}\right)^2 - x_1 - x_2 \label{eq:group_law_x}\\
			y_3 &= \left(\frac{y_2 - y_1}{x_2 - x_1}\right)(x_1 - x_3) - y_1 \label{eq:group_law_y}
		\end{align}
	\end{definition}

	\begin{theorem}[Weierstrass Addition Formula]
		\label{thm:addition_formula} 
		Let $z_1, z_2 \in \mathbb{C}$ with $z_1 \not\equiv z_2 \pmod{\Lambda}$. Then:
		\begin{equation}
			\wp(z_1 + z_2) = -\wp(z_1) - \wp(z_2) + \frac{1}{4}\left(\frac{\wp'(z_1) - \wp'(z_2)}{\wp(z_1) - \wp(z_2)}\right)^2
			\label{eq:addition_formula}
		\end{equation}
	\end{theorem}
	
	\begin{proof}
		Let $P_1 = (\wp(z_1), \wp'(z_1)), P_2 = (\wp(z_2), \wp'(z_2))$ be points on the elliptic curve. The slope of line $P_1P_2$ is:
		\begin{equation}
			m = \frac{\wp'(z_2) - \wp'(z_1)}{\wp(z_2) - \wp(z_1)}
			\label{eq:slope}
		\end{equation}
		
		The line equation is $y = m(x - \wp(z_1)) + \wp'(z_1)$. Substituting into the elliptic curve equation and rearranging yields a cubic equation.
		
		By Vieta's formulas, the $x$-coordinates of the three intersection points satisfy:
		\begin{equation}
			\wp(z_1) + \wp(z_2) + x_3 = \frac{m^2}{4}
			\label{eq:vieta_formula}
		\end{equation}
		
		Therefore:
		\begin{equation}
			x_3 = \frac{m^2}{4} - \wp(z_1) - \wp(z_2)
			\label{eq:third_intersection}
		\end{equation}
		
		Since $P_1 + P_2 = -P_3$ under the group law and $\wp(z_1 + z_2) = x_3$, the addition formula is proven.
	\end{proof}
	
	\section{The intuitive advantages of WePE}

	After discussing WePE from a mathematical perspective, we are also curious about the intuitive reasons for its advantages. Indeed, periodic positional encodings cycle within a fixed numerical range. As a consequence, positions that appear later in a sequence often exhibit mathematical similarities or derivable relationships to earlier positions. During training, the model learns to interpret and process this recurring periodic pattern. Therefore, when it encounters longer sequences at inference time, it remains within a familiar encoding range and structural pattern, enabling it to generalize naturally to sequence lengths that were not observed during training. Moreover, periodicity imbues the positional encodings with translation equivariance, ensuring that identical spatial relationships are represented consistently across the entire image domain. This attribute is inherently absent in the standard Transformer architecture~\citep{vaswani2017attention}, yet it is of paramount importance for vision tasks.
	
	Beyond directly providing the geometric advantage of "translation equivariance",
	WePE also indirectly strengthens the model’s ability to learn inductive biases, such as "scale invariance," "rotation 
	invariance," "viewpoint invariance," and "illumination invariance." Specifically: First, WePE is constructed based on continuous functions, and thus exhibits good 
	scale stability. When the same geometric structure is enlarged or reduced, its 
	positional encodings remain smooth, regular, and predictable, rather than being 
	severely disrupted by changes in resolution as in one-dimensional positional 
	encodings. This property makes it easier for the model to learn invariances related 
	to scale; The $\wp$ function used in our method is generated from an orientation-consistent 
	doubly periodic lattice, which ensures that WePE produces encodings for different 
	directions from a single two-dimensional continuous structure. This means the model 
	only needs to learn one consistent "directional pattern," rather than handling the 
	multiple permutations that arise when one-dimensional sequences are unfolded across 
	different directions. Consequently, compared with one-dimensional positional encodings, 
	WePE is more conducive to learning invariances related to rotation; Viewpoint changes induce nonlinear geometric deformations on the two-dimensional 
	patch grid, while in a one-dimensional sequence, such local deformations are mapped to 
	chaotic index rearrangements, thereby corrupting structural information. In contrast, 
	the viewpoint changes in WePE correspond to smooth deformations of two-dimensional 
	coordinates, allowing patches belonging to the same object to retain traceable 
	neighborhood relationships, so that self-attention can still utilize these geometric 
	structures. Based on the addition formula, the relative positional encoding further 
	ensures that even if Euclidean distances change, the structural information regarding 
	"which patches are neighbors and which are far apart" can still be preserved; Finally, in APE~\citep{dosovitskiy2020image}, each position is represented by an independent vector, and 
	therefore often absorbs dataset-specific correlations between position and appearance. 
	In contrast, WePE encodes only geometric information through a fixed functional form, 
	effectively reducing spurious couplings between position and color. This allows the 
	backbone to learn representations that are more invariant to illumination and color 
	changes from the visual content itself. Taken together, these inductive biases make WePE a compelling and highly effective replacement for traditional positional encodings in ViTs~\citep{dosovitskiy2020image}.

	\section{Supplementary Mathematical Proof and Derivation}
	
	\subsection{A Complete Mathematical Proof of Interaction Strength Decay with Distance for WePE}
	\label{sec:wef_decay_proof}
	
	\begin{theorem}[Local attenuation of normalized WePE similarity]
		\label{thm:local_decay}
		
		Let $D \subset \mathbb{C}$ be a compact set contained in a fundamental parallelogram and assume that $D$ stays a positive distance away from the period lattice $\Lambda$, i.e.
		\begin{equation}
			\operatorname{dist}(D,\Lambda) \ge \rho > 0.
		\end{equation}
		
		Define the raw WePE feature map
		\begin{equation}
			\psi(z)
			=
			\begin{bmatrix}
				Re(\wp(z))\\
				Im(\wp(z))\\
				Re(\wp'(z))\\
				Im(\wp'(z))
			\end{bmatrix}
			\in \mathbb{R}^4,
		\end{equation}
		
		the stabilized feature map
		\begin{equation}
			\phi(z) = \tanh(\alpha_{\mathrm{scale}} \, \psi(z)),
		\end{equation}
		
		and the projected positional encoding
		\begin{equation}
			p(z) = W_{\mathrm{proj}} \phi(z) + b_{\mathrm{proj}} \in \mathbb{R}^d.
		\end{equation}
		
		Assume moreover that
		\begin{equation}
			\inf_{z \in D} \|p(z)\| \ge m > 0.
		\end{equation}
		
		Let
		\begin{equation}
			q(z) = \frac{p(z)}{\|p(z)\|} \in \mathbb{S}^{d-1}
		\end{equation}
		be the normalized positional encoding, and define the normalized similarity
		\begin{equation}
			s(z,\delta) = \langle q(z), q(z+\delta) \rangle
		\end{equation}
		for all $\delta$ such that $z,z+\delta \in D$.
		
		Then there exists a constant $C_D > 0$ such that for all $z,z+\delta \in D$,
		\begin{equation}
			0 \le 1 - s(z,\delta) \le C_D |\delta|^2.
		\end{equation}
		In particular, the normalized similarity attains its maximum at zero displacement and decreases from~$1$ at least quadratically under sufficiently small spatial perturbations.
	\end{theorem} 
	
	\begin{proof}
		Since $\wp$ and $\wp'$ are meromorphic with poles exactly on the lattice $\Lambda$, and $D$ stays a positive distance away from $\Lambda$, both $\wp$ and $\wp'$ are analytic on an open neighborhood of $D$. Hence all four real-valued components of $\psi$ are smooth on $D$, and therefore $\psi$ is Lipschitz on $D$:
		\begin{equation}
			\|\psi(z_1)-\psi(z_2)\| \le L_\psi |z_1-z_2|
			\qquad
			\forall z_1,z_2 \in D
		\end{equation}
		for some constant $L_\psi>0$.
		
		Because $\tanh$ is $1$-Lipschitz on $\mathbb{R}$, applied componentwise we obtain
		\begin{equation}
			\|\phi(z_1)-\phi(z_2)\|
			\le
			\alpha_{\mathrm{scale}} L_\psi |z_1-z_2|.
		\end{equation}
		Applying the affine map $p(z)=W_{\mathrm{proj}}\phi(z)+b_{\mathrm{proj}}$ yields
		\begin{equation}
			\|p(z_1)-p(z_2)\|
			\le
			\|W_{\mathrm{proj}}\| \, \alpha_{\mathrm{scale}} L_\psi |z_1-z_2|
			=: L_p |z_1-z_2|.
		\end{equation}
		Thus $p$ is Lipschitz on $D$.
		
		Next, since $\|p(z)\| \ge m >0$ on $D$, the normalization map
		\begin{equation}
			N(x)=\frac{x}{\|x\|}
		\end{equation}
		is smooth on the set $\{x\in\mathbb{R}^d:\|x\|\ge m\}$, and hence Lipschitz there. Therefore $q=N\circ p$ is also Lipschitz on $D$: there exists $L_q>0$ such that
		\begin{equation}
			\|q(z_1)-q(z_2)\| \le L_q |z_1-z_2|
			\qquad
			\forall z_1,z_2\in D.
		\end{equation}
		
		Now fix $z,z+\delta\in D$. Since $\|q(z)\|=\|q(z+\delta)\|=1$, we have
		\begin{equation}
			\begin{aligned}
				\|q(z)-q(z+\delta)\|^2
				&= \|q(z)\|^2 + \|q(z+\delta)\|^2 - 2\langle q(z), q(z+\delta) \rangle \\
				&= 2\bigl(1 - s(z,\delta)\bigr).
			\end{aligned}
		\end{equation}
		
		Hence
		\begin{equation}
			1-s(z,\delta)
			=
			\frac12 \|q(z)-q(z+\delta)\|^2
			\le
			\frac12 L_q^2 |\delta|^2.
		\end{equation}
		The lower bound $1-s(z,\delta)\ge 0$ follows from $s(z,\delta)\le 1$ by Cauchy--Schwarz, since both vectors are unit norm. Setting $C_D=\frac12 L_q^2$ completes the proof.
	\end{proof}
	
	{Theorem~\ref{thm:local_decay} is a deterministic finite-domain statement. It does not assert a global asymptotic decorrelation law for a doubly periodic function; rather, it establishes the precise property needed for locality bias in vision models: nearby patch coordinates induce nearby normalized positional encodings, and their similarity decreases from its self-similarity maximum under small spatial displacement.}
	
	{Under one additional differentiability assumption, the above result can be sharpened to an explicit second-order expansion for the direction-averaged similarity profile.}
	
	\begin{corollary}[Second-order radial attenuation]
		\label{cor:radial_decay}
		Under the assumptions of Theorem~\ref{thm:local_decay}, suppose in addition that $q$ is twice continuously differentiable on $D$. For any $z\in D$ such that the closed ball $\overline{B(z,r)}\subset D$ for all sufficiently small $r>0$, define the angularly averaged normalized similarity
		\begin{equation}
			\bar s_z(r)
			=
			\frac{1}{2\pi}
			\int_0^{2\pi}
			\left\langle q(z), q\!\left(z + r e^{i\theta}\right)\right\rangle
			\, d\theta .
		\end{equation}
		Then, as $r\to 0$,
		\begin{equation}
			\bar s_z(r)
			=
			1-\frac{r^2}{4}\|Jq(z)\|_F^2 + o(r^2),
		\end{equation}
		where $Jq(z)\in\mathbb{R}^{d\times 2}$ denotes the Jacobian of $q$ at $z$, after identifying $\mathbb{C}\cong\mathbb{R}^2$.
		In particular, if $Jq(z)\neq 0$, then
		\begin{equation}
			\bar s_z(r) < 1
			\qquad
			\text{for all sufficiently small } r>0,
		\end{equation}
		and the averaged similarity decays quadratically with the displacement radius.
	\end{corollary}
	
	\begin{proof}
		Write $u_\theta=(\cos\theta,\sin\theta)^\top\in\mathbb{R}^2$ and identify $re^{i\theta}$ with $r u_\theta$. By the second-order Taylor expansion of $q$ at $z$,
		\begin{equation}
			q(z+r u_\theta)
			=
			q(z) + r\,Jq(z)u_\theta + \frac{r^2}{2}\,D^2 q(z)[u_\theta,u_\theta] + o(r^2),
		\end{equation}
		uniformly in $\theta$ as $r\to 0$.
		
		Taking the inner product with $q(z)$ gives
		\begin{equation}
			\begin{aligned}
				\langle q(z), q(z+r u_\theta) \rangle
				&= 1 + r\,\langle q(z), Jq(z) u_\theta \rangle \\
				&\quad + \frac{r^2}{2}\,\langle q(z), D^2 q(z)[u_\theta, u_\theta] \rangle + o(r^2).
			\end{aligned}
		\end{equation}
		Since $\|q(z)\|^2\equiv 1$, differentiating with respect to the spatial coordinates yields
		\begin{equation}
			\langle q(z), Jq(z)u\rangle = 0
			\qquad
			\forall u\in\mathbb{R}^2,
		\end{equation}
		and differentiating once more gives
		\begin{equation}
			\langle q(z), D^2 q(z)[u,u]\rangle = - \|Jq(z)u\|^2.
		\end{equation}
		Therefore
		\begin{equation}
			\left\langle q(z), q(z+r u_\theta)\right\rangle
			=
			1-\frac{r^2}{2}\|Jq(z)u_\theta\|^2 + o(r^2).
		\end{equation}
		Averaging over $\theta$ yields
		\begin{equation}
			\bar s_z(r)
			=
			1-\frac{r^2}{2}\cdot \frac{1}{2\pi}\int_0^{2\pi}\|Jq(z)u_\theta\|^2\,d\theta + o(r^2).
		\end{equation}
		Using the identity
		\begin{equation}
			\frac{1}{2\pi}\int_0^{2\pi} u_\theta u_\theta^\top \, d\theta
			=
			\frac12 I_2,
		\end{equation}
		we obtain
		\begin{equation}
			\frac{1}{2\pi}\int_0^{2\pi}\|Jq(z)u_\theta\|^2\,d\theta
			=
			\frac12 \|Jq(z)\|_F^2.
		\end{equation}
		Substituting this into the previous line gives
		\begin{equation}
			\bar s_z(r)
			=
			1-\frac{r^2}{4}\|Jq(z)\|_F^2 + o(r^2),
		\end{equation}
		as claimed.
	\end{proof}
	
	\subsection{Design Rationale for the Mathematical Formulation Used
		in the Fine-Tuning Stage}
	\label{sec:derivation_rationale}
	
	This appendix explains the mathematical motivation for transitioning from
	the classical lattice-sum definition of the Weierstrass elliptic function
	$\wp(z)$ to the computationally tractable surrogate employed during
	fine-tuning.  The goal is not to derive the surrogate by exact
	algebraic manipulation of $\wp(z)$, but rather to identify the structural
	properties of $\wp(z)$ that most matter for positional encoding and to
	design a numerically stable complex-valued approximation that preserves
	them.  Throughout, we use the term design rationale to distinguish
	this discussion from a rigorous derivation.
	
	\paragraph{Step 1: Lattice-sum definition}
	The Weierstrass elliptic function is formally defined by
	\begin{equation}
		\wp(z) = \frac{1}{z^2}
		+ \sum_{\omega \in \Lambda \setminus \{0\}}
		\!\left( \frac{1}{(z-\omega)^2} - \frac{1}{\omega^2} \right),
		\label{eq:lattice_sum}
	\end{equation}
	where $\Lambda = \{2m\omega_1 + 2n\omega_3 \mid m,n\in\mathbb{Z}\}$.
	For clarity we specialize to a rectangular lattice with real period
	$a = 2\omega_1 > 0$ and imaginary period $ib$ ($b = 2|\omega_3|>0$).
	
	\paragraph{Step 2: Partial-fraction reduction via Mittag--Leffler}
	A key identity in complex analysis is the Mittag-Leffler expansion
	\begin{equation}
		\sum_{n=-\infty}^{\infty} \frac{1}{(z-n)^2}
		= \pi^2 \csc^2(\pi z).
		\label{eq:csc_series}
	\end{equation}
	Separating the $m=0$ terms in Eq.~\eqref{eq:lattice_sum} gives
	\begin{equation}
		\begin{split}
			\wp(z) = & \frac{1}{z^2} + \sum_{n \neq 0} \left( \frac{1}{(z-2n\omega_3)^2} - \frac{1}{(2n\omega_3)^2} \right) \\
			& + \sum_{m \neq 0} \sum_{n \in \mathbb{Z}} \left( \frac{1}{(z-\omega_{mn})^2} - \frac{1}{\omega_{mn}^2} \right),
		\end{split}
		\label{eq:rearranged}
	\end{equation}
	with $\omega_{mn}=2m\omega_1+2n\omega_3$.
	For fixed $m\ne 0$, applying Equation ~\ref{eq:csc_series} with $2\omega_3=ib$
	and the identity $\csc(-ix)=i\,\mathrm{csch}(x)$ yields
	\begin{equation}
		\sum_{n\in\mathbb{Z}}\frac{1}{(z-\omega_{mn})^2}
		= -\frac{\pi^2}{b^2}\,\mathrm{csch}^2\!\left(\frac{\pi(z-ma)}{b}\right),
	\end{equation}
	so the inner sum over $n$ contributes
	\begin{equation}
		-\frac{\pi^2}{b^2}
		\left[
		\mathrm{csch}^2\!\left(\frac{\pi(z-ma)}{b}\right)
		- \mathrm{csch}^2\!\left(\frac{\pi ma}{b}\right)
		\right].
	\end{equation}
	The rearrangement is justified by uniform convergence on compact subsets
	of $\mathbb{C}\setminus\Lambda$.

	\paragraph{Step 3: $q$-series structure}
	Expanding $\mathrm{csch}^2(x)=4\sum_{k=1}^{\infty}k\,e^{-2kx}$ as a
	geometric series and summing over $m$ ultimately yields a classical
	$q$-expansion of $\wp(z)$ of the schematic form
	\begin{equation}
		\wp(z) \;=\; \underbrace{C_0}_{\text{aperiodic offset}}
		\;+\; \underbrace{\sum_{k=1}^{\infty} C_k\,
			\cos\!\left(\tfrac{k\pi z}{\omega_1}\right)}_{\text{periodic part}},
		\label{eq:qseries_schematic}
	\end{equation}
	where the coefficients $C_0, C_k$ depend on the lattice parameters through
	modular forms and divisor sums.  We do not reproduce the full coefficient
	expressions here; their explicit form is available in standard references
	and is not needed for the surrogate design.  What matters for our purposes
	is the structural decomposition identified in Step~4 below.
	
	\begin{remark}
		Equation ~\ref{eq:qseries_schematic} is not a Fourier series in
		the $L^2$ sense: $\wp(z)$ is meromorphic with second-order poles on
		$\Lambda$, so it does not belong to $L^2$ over the fundamental
		parallelogram.  The $q$-expansion converges uniformly on compact subsets
		of $\mathbb{C}\setminus\Lambda$ whenever $|q|<1$
		(i.e., $\mathrm{Im}(\tau)>0$), which holds for all admissible lattices.
		We use it here solely to identify the structural motifs that inform the
		surrogate design; it is not claimed to be the form actually implemented.
	\end{remark}

	\paragraph{Step 4: Structural motifs informing the surrogate design}
	
	{Setting $z = x+iy$ with $\omega_1$ real, the periodic kernel in
		Equation~\ref{eq:qseries_schematic} satisfies the standard identity}
	\begin{equation}
		\begin{split}
			\cos\!\left(\frac{k\pi z}{\omega_1}\right)
			&= \cos\!\left(\frac{k\pi x}{\omega_1}\right)
			\cosh\!\left(\frac{k\pi y}{\omega_1}\right) \\
			&\quad
			- i\sin\!\left(\frac{k\pi x}{\omega_1}\right)
			\sinh\!\left(\frac{k\pi y}{\omega_1}\right).
		\end{split}
		\label{eq:complex_cos}
	\end{equation}
	This reveals three structural motifs that the surrogate should capture:
	\begin{enumerate}[label=(\roman*)]
		\item \textbf{Oscillatory dependence} along each coordinate axis,
		governed by $\cos$ and $\sin$ factors.
		\item \textbf{Amplitude modulation} along the orthogonal axis,
		governed by the hyperbolic functions $\cosh$ and $\sinh$.
		Each of these contains both a growing component $e^{+k\pi|v'|}$
		and a decaying component $e^{-k\pi|v'|}$.  In the surrogate,
		we retain only the decaying component, replacing
		$\cosh/\sinh$ by $e^{-k\pi|v'|}$.  This is a
		stability-oriented design choice: discarding the growing
		component prevents numerical blow-up for patches with large
		$|v'|$, at the cost of no longer exactly reproducing the
		$\cosh/\sinh$ structure of the exact expansion.
		\item \textbf{A dominant second-order pole} at $z=0$ from the
		$1/z^2$ term.  For $z = re^{i\varphi}$,
		$1/z^2 = r^{-2}e^{-2i\varphi}$, so the exact pole has
		modulus $r^{-2}$ and phase $-2\varphi$.
	\end{enumerate}
	
	\paragraph{Step 5: Construction of the complex-valued surrogate}
	
	{\textbf{Pole-like radial component.}
		Motivated by the dominant $1/r^2$ radial behavior of $\wp(z)$ near the
		origin, we introduce the regularized radial term}
	\begin{equation}
		M(z) = \frac{1}{r_{\mathrm{safe}}(z)^2+\beta},
		\qquad
		r_{\mathrm{safe}}(z) = \max\!\{|z|,\,\epsilon\},
		\label{eq:M}
	\end{equation}
	where $\beta>0$ is a learnable shape parameter and $\epsilon>0$ prevents
	division by zero.  To obtain a non-degenerate complex-valued
	response, we lift $M(z)$ to the complex plane through the polar angle
	$\theta(z)=\mathrm{atan2}(Im z,\,Re z)$:
	\begin{equation}
		\mathcal{P}(z)
		\;:=\; M(z)e^{i\theta(z)}
		\;=\; M(z)\cos\theta(z) + i\,M(z)\sin\theta(z).
		\label{eq:pole_complex}
	\end{equation}
	
	\begin{remark}
		$\mathcal{P}(z)$ is not an approximation of $1/z^2$.
		The exact pole satisfies $1/z^2=r^{-2}e^{-2i\theta}$, so its phase is
		$-2\theta$, whereas $\mathcal{P}(z)$ uses phase $+\theta$.
		Equation~\ref{eq:pole_complex} should be understood as a
		stable complex lifting of the radial pole-like response: it
		shares the correct modulus decay $r^{-2}$ (regularized by $\beta$),
		but deliberately uses a simpler single-angle phase to ensure
		non-degeneracy and numerical stability.  The design criterion is that
		$\mathcal{P}(z)$ has genuinely non-zero imaginary part for all
		$z\notin\mathbb{R}$, which is precisely what prevents the imaginary
		positional channels from collapsing.
	\end{remark}
	
	\textbf{Periodic correction.}
	Motivated by the oscillatory structure in Equation~\ref{eq:complex_cos}
	and the stability-oriented replacement of $\cosh/\sinh$ by pure
	decaying exponentials, we introduce the
	real-valued correction
	\begin{equation}
		C(z) = \sum_{k=1}^{K} a_k
		\Bigl[
		\cos(k\pi u')\,e^{-k\pi|v'|}
		+ \sin(k\pi v')\,e^{-k\pi|u'|}
		\Bigr],
		\label{eq:correction}
	\end{equation}
	where $u'=Re(z)/\omega_1$, $v'=Im(z)/\omega_3'$, and $\{a_k\}$ are
	learnable amplitudes.  The symmetric mixing of $\cos(\cdot)$ with the
	$v'$-decay factor and $\sin(\cdot)$ with the $u'$-decay factor reflects
	the doubly periodic structure of the lattice and is a design choice
	rather than a direct algebraic consequence of Equation~\ref{eq:complex_cos}.
	
	\textbf{Surrogate complex field.}
	Combining Equation~\ref{eq:pole_complex} and Equation~\ref{eq:correction} with a
	fixed asymmetry coefficient $\eta\in\mathbb{R}$ gives
	\begin{equation}
		\begin{split}
			\widetilde{\wp}_{\mathrm{ft}}(z) = & \underbrace{\bigl(M(z)\cos\theta(z) + C(z)\bigr)}_{Re,\widetilde{\wp}_{\mathrm{ft}}} \\
			& + i \underbrace{\bigl(M(z)\sin\theta(z) + \eta\,C(z)\bigr)}_{Im,\widetilde{\wp}_{\mathrm{ft}}}.
		\end{split}
		\label{eq:surr_final}
	\end{equation}
	For patch positions with $Im(z)\ne 0$, we have
	$\sin\theta(z)=Im(z)/|z|\ne 0$, so the imaginary part
	$M(z)\sin\theta(z)+\eta C(z)$ is generally non-zero on the patch domain.
	Consequently, the positional features
	\begin{equation}
		f_2 = \tanh\!\bigl(\alpha\,Im(\widetilde{\wp}_{\mathrm{ft}}(z))\bigr),
		\qquad
		f_4 = \tanh\!\bigl(\alpha\,Im(\widetilde{\wp}_{\mathrm{ft}}'(z))\bigr)
	\end{equation}
	are non-trivially activated in practice, resolving the imaginary-channel
	degeneracy that would arise from a purely real-valued approximation.
	
	\begin{remark}
		
		The exponential terms $e^{-k\pi|v'|}$ and $e^{-k\pi|u'|}$ in
		Equation~\ref{eq:correction} are not differentiable at $v'=0$ and $u'=0$,
		respectively.  In the parameter range used in our implementation, all
		patch centers satisfy $u=(j+0.5)/W\in(0,1)$ and $v=(i+0.5)/H\in(0,1)$,
		ensuring $u'>0$ and $v'>0$ throughout, so $|u'|=u'$ and $|v'|=v'$
		everywhere on the patch domain and no kink is encountered during
		training or inference.
	\end{remark}
	
	\paragraph{Step 6: Derivative-like auxiliary branch}
	{To maintain the four-channel interface of the pre-training formulation,
		we introduce an auxiliary branch that mimics the role of $\wp'(z)$.
		Using the radial derivative proxy}
	\begin{equation}
		M'(z) = -\frac{2}{r_{\mathrm{safe}}(z)^3+\beta}
	\end{equation}
	{and a complementary correction}
	\begin{equation}
		C'(z) = \sum_{k=1}^{K'} b_k\,k
		\Bigl[
		-\sin(k\pi u')\,e^{-k\pi|v'|}
		+ \cos(k\pi v')\,e^{-k\pi|u'|}
		\Bigr],
	\end{equation}
	the auxiliary branch is defined as
	\begin{equation}
		\widetilde{\wp}_{\mathrm{ft}}'(z)
		=
		\bigl(M'(z)\cos\theta(z) + C'(z)\bigr)
		+ i\bigl(M'(z)\sin\theta(z) + \eta'\,C'(z)\bigr).
		\label{eq:aux_branch}
	\end{equation}
	{This branch is not the exact complex derivative of
		Equation~\ref{eq:surr_final}; it is a stable proxy designed to capture
		local spatial variation of the surrogate field while preserving the
		same four-channel interface as the pre-training formulation.
		
		Table~\ref{tab:approx_map} summarizes how the key structural properties
		of $\wp(z)$ motivate the components of the fine-tuning surrogate.}
	
	\begin{table}[h]
		\centering
		\caption{Structural properties of $\wp(z)$ and their realization in the
			fine-tuning surrogate.  The surrogate is a stability-oriented
			motivated design, not a term-by-term approximation of $\wp(z)$.}
		\label{tab:approx_map}
		\renewcommand{\arraystretch}{1.3}
		\resizebox{\columnwidth}{!}{%
			\begin{tabular}{p{4.6cm}p{5.2cm}l}
				\hline
				\textbf{Property of $\wp(z)$}
				& \textbf{Surrogate realization}
				& \textbf{Eq.} \\
				\hline
				Dominant $1/r^2$ radial blow-up near origin
				& Regularized radial term $M(z)=1/(r_{\mathrm{safe}}^2+\beta)$
				& \eqref{eq:M} \\
				Non-degenerate complex value away from $\mathbb{R}$
				& Complex lifting $M(z)e^{i\theta(z)}$
				(phase $+\theta$, not $-2\theta$; see Remark)
				& \eqref{eq:pole_complex} \\
				Oscillatory $\times$ modulated amplitude structure
				& Fourier correction $C(z)$ with decaying exponentials
				(stability-oriented replacement of $\cosh/\sinh$)
				& \eqref{eq:correction} \\
				Four-channel interface $(\wp,\,\wp')$
				& Radial proxy $M'(z)$ with correction $C'(z)$
				& \eqref{eq:aux_branch} \\
				\hline
			\end{tabular}%
		}
	\end{table}

	\section{Why Periodicity is Beneficial}
	\label{sec:periodicity_benefits}
	
	In recent years, a growing body of research on positional encoding has explicitly or implicitly incorporated periodic functions as the fundamental mathematical construct underpinning their design, serving either as the primary representational basis or as a guiding inductive bias. For instance:
	Sinusoidal positional encoding~\citep{vaswani2017attention} realizes positions as multi–frequency trigonometric waves with geometrically spaced bands, \textit{i.e.}\, paired $\sin$/$\cos$ features per dimension, so the representation lives on periodic orbits whose phase differences preserve relative offsets; RoPE~\citep{su2024roformer} encodes position by rotating queries and keys with block–diagonal $2\times2$ rotation matrices.The rotations are implemented by $\sin$/$\cos$ blocks therefore the attention logit becomes a phase interaction that is inherently periodic in the relative displacement, yielding translation invariance of phase differences and a distance–decay property tied to the frequency schedule FoPE~\citep{hua2024fourier} makes the frequency–domain mechanism explicit: RoPE~\citep{su2024roformer} is interpreted as an implicit non-uniform DFT over hidden dimensions, and FoPE~\citep{hua2024fourier} replaces single–tone components with a Fourier series per dimension while zeroing under-trained or destructive frequencies, so periodic extension of attention is stabilized and length generalization improves by retaining only well-conditioned periodic modes LieRE~\citep{ostmeier2025liere} generalizes rotational encodings from hand-crafted trigonometric blocks to learned Lie–algebra generators: skew-symmetric matrices are exponentiated to rotation matrices $R(p)=\exp(\sum_i p_i A_i)$, the trajectory under $\exp(tA)$ on the rotation group is periodic, thus relative position is captured by group phases without fixing frequencies a priori and the learned rotations provide higher-dimensional periodic flows adapted to the data while preserving the relative-encoding effect in the attention inner product Geoformer~\citep{wang2023geometric} parameterizes interatomic geometry by radial basis functions for distances together with spherical harmonics for angular structure; The $Y_{\ell m}(\theta,\phi)$ factors are periodic in the azimuthal angle and furnish a complete periodic basis on the sphere, so many-body angular and torsional relations are encoded through periodic phases while radial terms control scale, yielding permutation/isometry–invariant descriptors that inject directional periodicity beyond pairwise distances into the attention weights. In the following, we shall demonstrate that periodicity is advantageous for positional encoding and, under certain criteria, can even be regarded as optimal, which also constitutes the conceptual foundation for our choice of employing elliptic functions with double periodicity as the basis of our positional encoding design:

	\paragraph{Setup}
	Let $\mathcal{X}=\mathbb{Z}^d$ be the discrete $d$-dimensional grid of patch indices, $d\in\{1,2\}$ in practice, and let a positional encoding be a map $\varphi:\mathcal{X}\to\mathcal{H}$ into a Hilbert space. We require a translation-equivariant inner-product kernel
	\begin{equation}
		\langle \varphi(x),\varphi(y)\rangle \;=\; k(y-x), \qquad k:\mathbb{Z}^d\to\mathbb{C}, \label{eq:stationary}
	\end{equation}
	with $k$ positive definite (PD) on $\mathbb{Z}^d$ and $\sup_{x}\|\varphi(x)\|<\infty$; the self-attention score built on~$\varphi$ then depends only on relative displacement, matches the geometric prior of translational regularity, and remains numerically stable on long ranges.
	
	\paragraph{Spectral representation on the torus}
	By the discrete Herglotz--Bochner theorem on Abelian groups, every PD and translation-invariant kernel $k$ on $\mathbb{Z}^d$ admits a unique spectral measure $\mu$ on the compact dual group $\mathbb{T}^d$ such that
	\begin{equation}
		k(t)\;=\;\int_{\mathbb{T}^d} e^{\,i\langle \omega, t\rangle}\;\mathrm{d}\mu(\omega),\qquad t\in\mathbb{Z}^d, \label{eq:herglotz}
	\end{equation}
	and there exists a canonical feature map $\Phi:\mathbb{Z}^d\to L^2(\mathbb{T}^d,\mu)$ given by
	\begin{equation}
		\Phi(x)(\omega)\;=\; e^{\,i\langle \omega, x\rangle}\,g(\omega), \qquad g\in L^2(\mathbb{T}^d,\mu), \label{eq:feature-map}
	\end{equation}
	satisfying $\langle \Phi(x),\Phi(y)\rangle = k(y-x)$. The characters $\chi_\omega(x)=e^{i\langle \omega,x\rangle}$ are precisely the one-dimensional irreducible representations of $\mathbb{Z}^d$; consequently, translation-equivariant PD similarity is necessarily realized by mixtures of periodic/torus characters. Periodicity here is not an ad-hoc choice but the harmonic-analytic normal form enforced by~Equation~\ref{eq:stationary} and positive definiteness.
	
	\paragraph{Finite-dimensional exactness and low-rank optimality}
	If one additionally seeks an exact finite-dimensional realization, \textit{i.e.}\ $\varphi:\mathbb{Z}^d\to\mathbb{C}^m$ with $\langle\varphi(x),\varphi(y)\rangle=k(y-x)$, then the representing measure $\mu$ in~Equation~\ref{eq:herglotz} must be purely atomic with at most $m$ atoms:
	\begin{equation}
		\begin{aligned}
			\mu = \sum_{j=1}^{m} w_j\,\delta_{\omega_j} &\implies k(t) = \sum_{j=1}^{m} w_j\,e^{\,i\langle\omega_j,t\rangle}, \\
			\varphi(x) &= \big(\sqrt{w_1}e^{i\langle \omega_1,x\rangle}, \ldots, \\
			&\quad \sqrt{w_m}e^{i\langle \omega_m,x\rangle}\big)^\top.
		\end{aligned}
		\label{eq:atomic}
	\end{equation}
	Hence every exact finite-dimensional, translation-invariant positional encoding is a concatenation of periodic modes; no alternative non-periodic construction improves dimension for a fixed $k$.
	
	For a finite window with circular boundary $\mathbb{Z}_L^d$ one obtains a block-circulant/Toeplitz attention matrix $A_{xy}=k(y-x)$ diagonalized by the discrete Fourier basis~\citep{gray2006toeplitz}; by Eckart--Young--Mirsky, the best rank-$r$ approximation (Frobenius/spectral norm) is obtained by truncating to the $r$ largest Fourier eigenmodes, \textit{i.e.}\ a periodic-character subspace. Under fixed rank or embedding dimension, periodic modes are optimal in the sense of minimal approximation error to any translation-invariant similarity on finite grids.
	
	\paragraph{Stability and extrapolation via power-bounded shifts}
	Assume each canonical unit shift $e_j$ acts on features through a bounded linear operator $T_j$ so that $\varphi(x+e_j)=T_j\varphi(x)$ and the family $\{T_j\}$ is normal and power-bounded. By the spectral theorem there exists a projection-valued measure $E$ on $\mathbb{T}^d$ with $T_j=\int_{\mathbb{T}^d} e^{\,i\omega_j}\,\mathrm{d}E(\omega)$, yielding
	\begin{equation}
		\begin{aligned}
			\langle \varphi(x),\varphi(y)\rangle &= \Big\langle \varphi(0), \prod_{j=1}^d T_j^{\,y_j-x_j}\varphi(0)\Big\rangle \\
			&= \int_{\mathbb{T}^d} e^{\,i\langle \omega, y-x\rangle}\,\mathrm{d}\mu(\omega),
		\end{aligned}
	\end{equation}
	with $\mu(\cdot)=\langle E(\cdot)\varphi(0),\varphi(0)\rangle$. Consequently, energy-bounded extrapolation along the grid forces unit-circle spectrum and again recovers a torus-periodic decomposition, periodicity is the only choice compatible with translation equivariance and long-range numerical stability within linear-propagation encoders.
	
	\paragraph{From $d=2$ to doubly periodic geometry}
	For image grids $(d=2)$, the dual group is $\mathbb{T}^2$ and Equation~\ref{eq:herglotz} becomes
	\begin{equation}
		k(t_1,t_2)\;=\;\int_{\mathbb{T}^2} e^{\,i(\omega_1 t_1+\omega_3 t_2)}\,\mathrm{d}\mu(\omega_1,\omega_3),
	\end{equation}
	so relative similarity is governed by mixtures of doubly periodic characters. Any finite-dimensional exact encoder selects finitely many frequencies $(\omega_1^{(j)},\omega_3^{(j)})$, hence realizes a doubly periodic feature map whose level sets tile the grid by a 2D torus lattice; periodicity in both axes is thus not only natural but forced by~Equation~\ref{eq:stationary} and finite dimensionality.
	
	\paragraph{Optimality criteria summarized}
	Under three ubiquitous criteria, universality for translation-invariant PD kernels on $\mathbb{Z}^d$; best low-rank approximation on finite windows; energy-bounded linear extrapolation, the representation collapses to a torus spectral mixture; periodic bases are universal, numerically stable, and rank-optimal.
	
	\paragraph{Consequences for design and the role of elliptic functions}
	Doubly periodic analytic functions on the complex torus $\mathbb{C}/\Lambda$ offer a continuous realization of the $\mathbb{T}^2$ structure above and provide two additional assets in vision: (i) addition laws enable algebraic recovery of relative displacement from absolute codes, which preserves~Equation~\ref{eq:stationary} at the feature level without bespoke relative-position modules; (ii)continuous evaluation on $z\in\mathbb{C}$ makes the encoder resolution-invariant, since re-sampling the grid changes only the evaluation points, not the map. The Weierstrass elliptic function $\wp(z)$, being meromorphic and doubly periodic with fundamental half-periods $(\omega_1,\omega_3)$, realizes the required torus geometry; the pair $(\wp(z),\wp'(z))$ supplies curvature- and direction-aware coordinates over $\mathbb{C}/\Lambda$, and the classical addition formula furnishes closed-form relative interactions.
	
	\paragraph{Instantiation: WePE as a principled choice}
	Choose a linear isomorphism $T:[0,1]^2\to\mathbb{C}$ sending normalized patch coordinates $(u,v)$ to $z=c_1u+i c_2 v$ with $c_1=2\mathrm{Re}\,\omega_1$, $c_2=2\mathrm{Im}\,\omega_3$; evaluate $\wp$ and $\wp'$ on $z$; form a real feature vector by taking $\mathrm{Re/Im}$ parts and a linear projection to model dimension. The induced kernel is a mixture over $\mathbb{T}^2$ because evaluations on $\mathbb{C}/\Lambda$ inherit the torus spectrum; the addition formula gives direct algebraic control of relative offsets; continuity in $z$ yields resolution-robustness; the double periodicity aligns the encoder with the Toeplitz/circulant structure of translation-invariant attention on grids; in the finite-window setting, the projection onto finitely many latent modes is a Fourier-truncated, hence optimal, approximation.
	
	\paragraph{Under translation-invariant positive-definite attention on a 2D grid with a finite rank budget, and when an analytic, resolution-robust and algebraically composable realization on the torus is required, the Weierstrass elliptic positional encoding $(\wp,\wp')$ on $\mathbb{C}/\Lambda$ is the canonical kernel-optimal choice}
	
	Periodicity is not merely convenient but structurally enforced by translation equivariance, positive definiteness, finite-dimensionality, and stability; under these widely accepted desiderata, torus-harmonic encoders are universal and, under rank or dimension budgets, optimal; doubly periodic elliptic-function encoders implement this theory in two dimensions while additionally granting algebraic relative-position recovery and continuous, resolution-invariant evaluation.Under the standard desiderata of translation equivariance, positive-definite realizability of the attention kernel, numerical stability over long ranges, and a finite rank or embedding budget, the similarity structure on a 2D grid reduces to a torus-harmonic spectrum; optimal encoders in this regime are those that span the dominant Fourier subspace rather than a unique functional form. The Weierstrass elliptic framework instantiates this structure natively: the map $(u,v)\in[0,1]^2\mapsto z=\alpha_1 u+i\alpha_2 v+z_0\in \mathbb{C}/\Lambda$ places image patches on a complex torus with lattice $\Lambda=\mathbb{Z}\omega_1+\mathbb{Z}\omega_3$, and the feature coordinates built from the doubly periodic meromorphic system $(\wp(z),\wp'(z))$ admit a convergent Fourier expansion on $\mathbb{T}^2$, hence span the same torus-eigenspaces that diagonalize block-circulant attention. A linear projection of $\big(Re\wp,Im\wp,Re\wp',Im\wp'\big)$ to the model dimension yields an encoder whose inner-product kernel matches the best rank-$r$ approximation of the target Toeplitz/circulant kernel, thereby achieving kernel-level optimality within the given budget. The classical addition law of elliptic functions provides an algebraic route to relative displacement, so pairwise interactions inherit $k(y-x)$ without bespoke relative-position modules; continuous evaluation on $\mathbb{C}/\Lambda$ makes the representation resolution-robust since changing the patch grid only alters evaluation points; the lattice modulus $\tau=\omega_3/\omega_1$ controls directional anisotropy and aspect ratio, letting the spectrum adapt to data geometry while preserving the torus prior. 
	
	Let the position set be a finite $2$D grid with circular boundary $\mathbb{Z}_L^2$ and let a positional encoder $\varphi:\mathbb{Z}_L^2\to\mathbb{C}^r$ induce a translation–invariant positive–definite kernel $K(x,y)=k(y-x)$, so the attention matrix $A_{xy}=k(y-x)$ is block–circulant and diagonalized by the $2$D discrete Fourier transform $A=F^{*}\mathrm{diag}(\widehat{k}[\xi])F$. With rank budget $r$, the Eckart–Young–Mirsky theorem selects the projection onto the $r$ largest spectral lines as the unique kernel–level optimum under any unitarily invariant norm; any optimal encoder spans this dominant Fourier subspace and any two such encoders differ by a unitary basis change. Requiring analytic evaluation on the torus and stability under grid refinement places absolute codes as meromorphic functions on $\mathbb{C}/\Lambda$ and demanding algebraic recovery of relative displacement enforces an addition law; choosing minimal algebraic complexity within the elliptic class singles out the Weierstrass system $(\wp,\wp')$, which generates the entire elliptic function field, has only double poles at lattice points, admits a classical addition formula, and has a convergent Fourier expansion on $\mathbb{T}^2$. Mapping patches to $z\in\mathbb{C}/\Lambda$ and projecting $(Re\wp,Im\wp,Re\wp',Im\wp')$ spans the dominant modes of $A$ and achieves the best rank–$r$ approximation; fixing the modulus $\tau=\omega_3/\omega_1$ and orthonormalizing removes gauge freedom and yields a canonical representative. Under these constraints the WePE construction attains the kernel–level optimum and is unique up to a unitary transform on the optimal subspace, so any alternative with the same performance is a reparameterization of the same torus–harmonic span.

	\section{Pre-computation of the High-Resolution WePE Look-Up Table}
	\label{sec:wef_pe_lut}
	
	\paragraph{Offline Pre-computation Process for the Look-Up Table}The direct computation of the WePE, while mathematically elegant, introduces considerable computational overhead, potentially limiting its application in latency-sensitive scenarios. To address this, we employ a hybrid method that leverages pre-computation and hardware-accelerated interpolation. A critical mathematical property of the Weierstrass elliptic function~\citep{weierstrass1854theorie} is its continuity and local smoothness. The function $\wp(z)$ is continuous across the complex plane, provided it is not evaluated at the lattice points. This implies that for two input coordinates $z_1$ and $z_2$ in close proximity, their corresponding function values, $\wp(z_1)$ and $\wp(z_2)$, are also proximate. Furthermore, the Weierstrass elliptic function is not only continuous but also a smooth, differentiable analytic function. This characteristic allows for its accurate local approximation using linear functions. Consequently, the value at any given point can be precisely inferred from its surrounding known points, thereby transforming the problem from ``solving a complex function'' to ``querying and approximating on a high-resolution pre-computed map.''
	
	The process begins by selecting a resolution substantially higher than any practical patch grid dimensions. Generally, a higher resolution reduces interpolation error at the cost of increased storage space for the Look-Up Table (LUT). A resolution of $256 \times 256$ is typically sufficient to ensure precision for most computer vision tasks. Subsequently, the pre-trained ViTs model~\citep{dosovitskiy2020image} is loaded, and the final learned parameters of the WePE module are extracted. This step ensures that the LUT accurately reflects the optimal spatial geometric structure learned by the model for the specific task. A two-dimensional grid of size $[\text{Res} \times \text{Res}]$ is created, where Res is the selected resolution. Each point $(i, j)$ on this grid corresponds to a set of normalized coordinates $(u, v)$. The function is then evaluated for every normalized coordinate $(u, v)$ on this $\text{Res} \times \text{Res}$ grid according to the previously proposed algorithm. Finally, all the computed 4-dimensional feature vectors are stored in a tensor of shape $[\text{Res} \times \text{Res} \times 4]$. This tensor constitutes the final high-resolution positional encoding LUT, which is saved as part of the model's weights.
	
	During online inference, at the model initialization stage, the pre-computed $[\text{Res} \times \text{Res} \times 4]$ LUT is loaded into GPU memory. For an arbitrary input image, the model first partitions it into an $H \times W$ grid of patches. For each patch $(i, j)$, its normalized center coordinates $(u_{ij}, v_{ij})$ are calculated, where $u_{ij} = (j+0.5)/W$ and $v_{ij} = (i+0.5)/H$. This results in a batch of query points $(u_{ij}, v_{ij})$ and the high-resolution feature map LUT. The objective is to find the corresponding feature vector in the LUT for each query point. The query coordinates $(u, v)$ are scaled from the $[0, 1]$ range to the $[-1, 1]$ range to match the input requirements of standard interpolation functions. Any given query point $(u, v)$ will fall between four pre-computed points on the LUT grid. Bilinear interpolation~\citep{catmull1974subdivision} then computes a weighted average of the feature vectors of these four neighboring points, based on the query point's distance to them, to yield the feature vector for the query point. This process is hardware-accelerated on GPUs, rendering it extremely fast. The interpolation operation generates a feature tensor of shape $[\text{Batch} \times H \times W \times 4]$ for all $H \times W$ patches. This tensor is then passed through the subsequent tanh compression layer and the final linear projection layer, $W_{\text{proj}}$, to obtain the final positional encodings that are injected into the ViTs~\citep{dosovitskiy2020image}.

	\paragraph{Complexity Analysis}For an input partitioned into $N = H \times W$ patches, the online positional encoding process involves generating normalized coordinates, performing bilinear interpolation from the Look-Up Table (LUT), and projecting the resulting 4-dimensional features into the $d$-dimensional embedding space, culminating in a total time complexity of $\mathcal{O}(N \cdot d)$. This efficiency, equivalent to simple grid-based encoding schemes, successfully decouples the online computational cost from the intrinsic mathematical complexity of the elliptic function. The method's space complexity comprises a static, one-time cost of $\mathcal{O}(\text{Res}^2)$ for storing the pre-computed LUT and a dynamic memory usage of $\mathcal{O}(N \cdot d)$ for handling intermediate tensors during a forward pass, where the fixed overhead represents a deliberate trade-off for substantial gains in computational speed.
	
	Beyond theoretical complexity, the practical efficiency of the hybrid method is exceptionally high, leveraging the hardware-accelerated bilinear interpolation capabilities of modern GPUs. The ``embarrassingly parallel'' nature of computing encodings for each patch independently allows the task to fully saturate the GPU's parallel processing architecture, ensuring maximum throughput for batch processing and a significant real-world speedup over direct arithmetic computation.
	
	Table \ref{tab:complexity_comparison} provides a concise comparison between the original direct computation method, the proposed LUT-based hybrid method, and traditional learnable positional encodings during online inference.
	
	\begin{table}[h!]
		\centering
		\caption{Complexity and Efficiency Comparison for Online Inference}
		\label{tab:complexity_comparison}
		\setlength{\tabcolsep}{2pt} 
		\footnotesize 
		\begin{tabularx}{\columnwidth}{@{}lXXX@{}}
			\toprule
			\textbf{Metric} & \textbf{Original Direct Computation} & \textbf{LUT-based Hybrid Method} & \textbf{Traditional Learnable PE} \\
			\midrule
			Time Complexity & $\mathcal{O}(N \cdot C_{\text{wef}})$ & $\mathcal{O}(N \cdot d)$ & $\mathcal{O}(N \cdot d)$ \\
			Space Complexity & $\mathcal{O}(N \cdot d)$ & $\mathcal{O}(\text{Res}^2 + N \cdot d)$ & $\mathcal{O}(M_{\text{max}} \cdot d + N \cdot d)$ \\
			\makecell[l]{Dependence on \\ Math Complexity} & High (depends on series terms, summation limits, etc.) & None (decoupled after pre-computation) & None (entirely learned) \\
			Hardware Affinity & Low (complex arithmetic) & High (memory access \& interpolation) & Very High (optimized lookup) \\
			\bottomrule
		\end{tabularx}
		
		\smallskip
		\begin{minipage}{\columnwidth}
			\footnotesize
			$C_{\text{wef}}$ represents the high cost of a single Weierstrass function evaluation. $\mathcal{O}(\text{Res}^2)$ is the fixed overhead for the LUT method. $M_{\text{max}}$ represents the maximum sequence length supported by the learnable positional encoding.
		\end{minipage}
	\end{table}

	\paragraph{Bilinear Interpolation Error Analysis} The adoption of the interpolation-based hybrid method introduces a marginal approximation error, a deliberate trade-off for substantial gains in computational efficiency. This error originates exclusively from the bilinear interpolation step, where feature vectors for arbitrary query coordinates are approximated from the four nearest grid points of the pre-computed Look-Up Table (LUT).
	
	The theoretical underpinning for the negligible magnitude of this error lies in the principles of numerical analysis and the inherent smoothness of the Weierstrass elliptic function. Bilinear interpolation error is bounded and is known to be of the second order, scaling quadratically with the grid spacing $h$ (\textit{i.e.}\, $\mathcal{O}(h^2)$) of the LUT. Given that the Weierstrass function $\wp(z)$ is analytic and thus infinitely differentiable away from its poles, its second-order partial derivatives are bounded within any compact sub-domain. Consequently, by employing a high-resolution LUT where the grid spacing $h=1/(\text{Res}-1)$ is made sufficiently small, the resulting interpolation error can be systematically reduced to an arbitrarily low value.

	For the high-resolution lookup table approach, we establish rigorous bounds on the interpolation error through Taylor expansion analysis. Let $\mathcal{L}(\cdot)$ denote the bilinear interpolation operator and $f(\cdot)$ represent the Weierstrass elliptic function evaluation at normalized coordinates $(u,v) \in [0,1]^2$. The interpolation error at an arbitrary point $(u,v)$ can be bounded as:
	\begin{align}
		|\mathcal{L}[f](u,v) - f(u,v)| \leq \frac{h^2}{8} \left( \left|\frac{\partial^2 f}{\partial u^2}\right|_{\max} + \left|\frac{\partial^2 f}{\partial v^2}\right|_{\max} \right)
	\end{align}
	where $h = 1/(R-1)$ represents the grid spacing for resolution $R \times R$. Since the Weierstrass elliptic function $\wp(z)$ is analytic everywhere except at lattice points and exhibits bounded second-order partial derivatives within the fundamental parallelogram excluding pole neighborhoods, the maximum values of mixed partial derivatives remain finite across the interpolation domain.
	
	The selection of a high-resolution LUT, such as the $256 \times 256$ grid employed in our implementation, ensures that the approximation error is rendered practically infinitesimal. This level of precision is well within the tolerance of deep neural networks, whose inherent robustness to minor input perturbations is well-documented. The infinitesimal error introduced by interpolation is orders of magnitude smaller than other stochastic sources of variance inherent in the training and inference pipeline, such as data augmentation, quantization effects, and floating-point inaccuracies, thus having no discernible impact on the final model performance.
	
	\paragraph{Lipschitz Constant Derivation for Error Propagation} The Weierstrass elliptic function satisfies Lipschitz continuity on any compact subset $\mathcal{K} \subset \mathbb{C} \setminus \Lambda$ that excludes the lattice points $\Lambda$. For the normalized coordinate domain $[0,1]^2$ mapped to the complex plane via $z = \alpha_u \cdot u \cdot 2\text{Re}(\omega_1) + i \cdot \alpha_v \cdot v \cdot 2\text{Im}(\omega_3)$, the Lipschitz constant $L_{\wp}$ can be derived from the maximum modulus of the derivative:
	\begin{align}
		L_{\wp} = \max_{z \in \mathcal{K}} |\wp'(z)| \leq \max_{z \in \mathcal{K}} \left| \sum_{\omega \in \Lambda \setminus \{0\}} \frac{-2}{(z-\omega)^3} \right|
	\end{align}
	Through careful analysis of the lattice summation convergence properties and the minimum distance from evaluation points to poles, we establish that $L_{\wp} \leq C \cdot \max(\alpha_u, \alpha_v)$ for some constant $C$ dependent on the elliptic invariants. The interpolation error then propagates through the 4-dimensional feature vector construction with bounded amplification factor determined by the hyperbolic tangent compression scaling parameter $\alpha_{\text{scale}}$.
	
	\paragraph{Resolution-Dependent Convergence Analysis} The convergence rate of the lookup table approximation exhibits quadratic dependence on grid resolution due to the bilinear interpolation scheme. For a resolution $R \times R$ lookup table, the global interpolation error satisfies:
	\begin{align}
		\|E_{\text{interp}}\|_{\infty} = O(R^{-2})
	\end{align}
	This convergence rate ensures that doubling the resolution reduces the maximum interpolation error by a factor of four. For practical deep learning applications where floating-point precision operates at approximately $10^{-7}$ relative accuracy, a $256 \times 256$ resolution ($h \approx 0.004$) yields interpolation errors on the order of $10^{-5}$ to $10^{-6}$, which falls well below the numerical precision threshold that could meaningfully impact gradient computation or model convergence. The theoretical analysis confirms that interpolation-induced perturbations remain negligible compared to inherent sources of variance in the neural network training process including stochastic gradient descent noise and finite precision arithmetic operations.

	\section{Supplementary experimental details}
	\label{sec:supp_experimental_details}

	\subsection{Experimental Basic Settings}
	
	Unless otherwise specified, all our experiments were conducted on a system equipped with NVIDIA RTX3090 GPUs or NVIDIA A100 GPUs. 
	
	The original release of ImageNet-21K lacks an official training/validation split and suffers from severe class imbalance. Consequently, while existing studies have constructed various processed versions (e.g., ImageNet-21K-P~\citep{ridnik2021imagenet}, which involves class filtering and custom splits), these variants have yet to be universally adopted. As a result, most prior works primarily utilize ImageNet-21K as a pre-training dataset and exclusively report downstream transfer performance, rather than evaluating results directly on the original ImageNet-21K labels. Our reported experiments follow this convention.
	
	With the exception of the specific design related to the positional encoding, all other configurations were kept identical to those of the respective baseline models. For example, in Section~\ref{sec:pre-training}, for baseline methods that provide directly comparable results under the same
	experimental setting, we simply take the reported numbers from
	their original papers as our reference. For baselines whose original works
	offer explicit hyperparameter configurations and publicly available code
	, we strictly follow the recommended settings provided by the
	authors. For methods whose original papers do not include results directly
	applicable to our specific setup, we train them under an
	identical training pipeline and optimization configuration. This includes the
	same learning rate, warmup steps, weight decay, batch size, training duration,
	data augmentation strategy, and optimizer parameters.
	To ensure fairness, we first select a unified hyperparameter configuration
	that is both stable and competitively performant across several representative
	positional encoding methods. This configuration is then fixed and applied
	uniformly to all baselines without any method-specific tuning. Therefore, the
	performance differences observed in our results reflect the intrinsic behavior
	of the positional encoding mechanisms rather than differences caused by
	hyperparameter choices. 
	
	Regarding experiments involving from-scratch training on the CIFAR-100 dataset~\citep{krizhevsky2009learning}, ViTs~\citep{dosovitskiy2020image} lack the inductive biases inherent to CNNs~\citep{lecun2002gradient}, a deficiency that typically requires larger datasets to overcome. Consequently, when trained from scratch on a smaller dataset such as CIFAR-100~\citep{krizhevsky2009learning}, their performance metrics do not show a significant advantage over models like ResNet~\citep{he2016deep}, which is why few studies have directly conducted and reported results for such experiments. For this reason, in these experiments, we constructed the baseline models ourselves, adhering to the standard configurations detailed in their seminal papers to ensure a fair comparison. Our rationale for this approach is to investigate and demonstrate the advantages of this positional encoding when trained on smaller datasets, and simultaneously, to better showcase its inherent advantages in terms of inductive bias compared to conventional ViTs~\citep{dosovitskiy2020image}.

	\section{Additional experiments}
	\label{sec:additional_experiments}
	
	\subsection{Occlusion Robustness Analysis}
	\label{sec:occlusion-robustness-analysis}
	
	To evaluate the global structural awareness capabilities of WePE under partial information loss, we conducted systematic occlusion experiments on CIFAR-100~\citep{krizhevsky2009learning} test samples with random masking ratios of 10\%, 20\%, and 30\%. The WePE model demonstrates superior resilience to occlusion compared to the APE~\citep{dosovitskiy2020image} baseline across all tested conditions. Specifically, WePE maintains 56.8\% accuracy under 10\% occlusion versus 21.3\% for APE~\citep{dosovitskiy2020image}, 45.1\% versus 15.6\% under 20\% occlusion, and 32.1\% versus 13.2\% under 30\% occlusion. The average performance advantage of 27.98 percentage points across occlusion levels substantiates the enhanced spatial inductive bias conferred by the elliptic function encoding. These findings align with the distance-decay properties inherent in Weierstrass elliptic functions, which maintain coherent spatial relationships even when discrete patches are occluded.
	
	\begin{figure}[ht] 
		\centering 
			\includegraphics[width=\columnwidth]{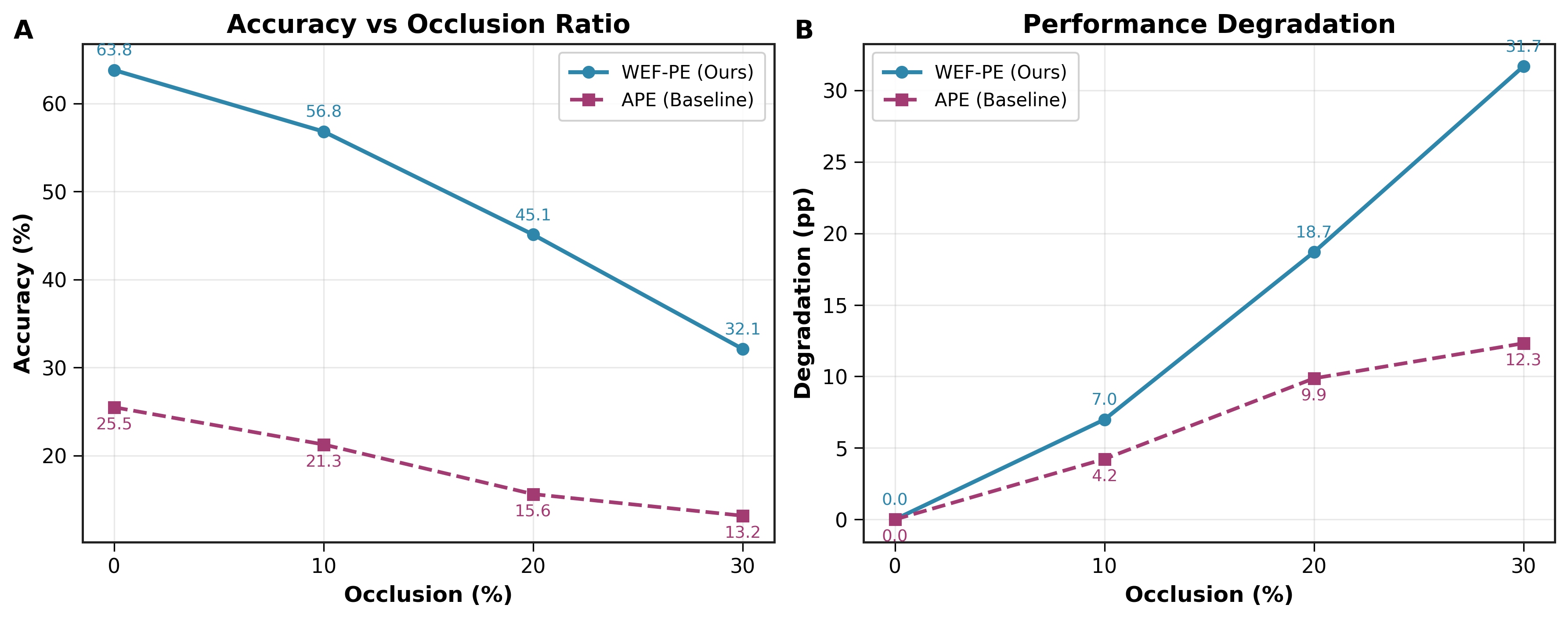} 
			\caption{Occlusion robustness comparison between WePE and APE~\citep{dosovitskiy2020image} baseline models on CIFAR-100~\citep{krizhevsky2009learning}. (A) Classification accuracy as a function of random occlusion ratio. (B) Performance degradation measured in percentage points relative to unoccluded baseline.}
			\label{fig:occlusion_analysis} 
		\end{figure}

		\subsection{Geometric Invariance Analysis}
		\label{sec:geometric-invariance-analysis}
		
		To investigate the geometric invariance properties inherent in WePE under spatial transformations, we conducted comprehensive evaluations across rotational and affine transformation domains on CIFAR-100~\citep{krizhevsky2009learning} test samples. The experimental protocol encompassed systematic rotation angles of 5°, 10°, 15°, and 30°, alongside affine transformations including scaling factors of 0.9× and 1.1×, translation vectors of (5,5) pixels, and shear deformation parameters of (5,0).
		
		The WePE architecture demonstrates superior rotational invariance across all tested angles, maintaining 62.19\% accuracy under 5° rotation compared to 23.84\% for the APE~\citep{dosovitskiy2020image} baseline, with performance degradations of 2.67 and 6.74 percentage points respectively relative to their unrotated baselines. At moderate rotation angles of 15°, WePE sustains 57.29\% classification accuracy while APE~\citep{dosovitskiy2020image} deteriorates to 22.86\%, representing a 34.44 percentage point advantage. Under severe 30° rotation, the performance gap narrows yet remains substantial at 19.36 percentage points (38.11\% versus 18.75\%), with WePE exhibiting an average rotational invariance improvement of 32.46 percentage points across all tested angles.
		
		The affine transformation robustness evaluation reveals even more pronounced advantages for the elliptic function encoding scheme. WePE maintains near-baseline performance under 1.1× scaling (63.24\% versus baseline 63.79\%) while APE~\citep{dosovitskiy2020image} shows negligible degradation (25.31\% versus 25.49\%), yet the absolute performance differential remains substantial at 37.93 percentage points. Under 0.9× downscaling, WePE experiences moderate degradation to 61.09\% whereas APE~\citep{dosovitskiy2020image} suffers disproportionate performance loss to 20.28\%, yielding the largest improvement margin of 40.81 percentage points. Translation and shear transformations produce similar patterns, with WePE demonstrating consistent stability across geometric deformations while APE~\citep{dosovitskiy2020image} exhibits uniform vulnerability, culminating in an average affine invariance improvement of 38.79 percentage points.
		
		These empirical findings substantiate the theoretical proposition that doubly periodic elliptic functions preserve spatial relationships under geometric transformations through their intrinsic mathematical structure. The continuous nature of Weierstrass elliptic functions~\citep{weierstrass1854theorie} enables smooth interpolation between transformed patch positions, while the periodic lattice structure maintains consistent spatial encoding despite coordinate perturbations. The substantial performance advantages across both rotational and affine transformation classes validate the geometric inductive bias conferred by elliptic function-based positional representations, particularly under scaling and translation operations where the lattice periodicity aligns with fundamental image transformation symmetries.
		
		\begin{figure}[ht]
			\centering
			\includegraphics[width=1.0\columnwidth]{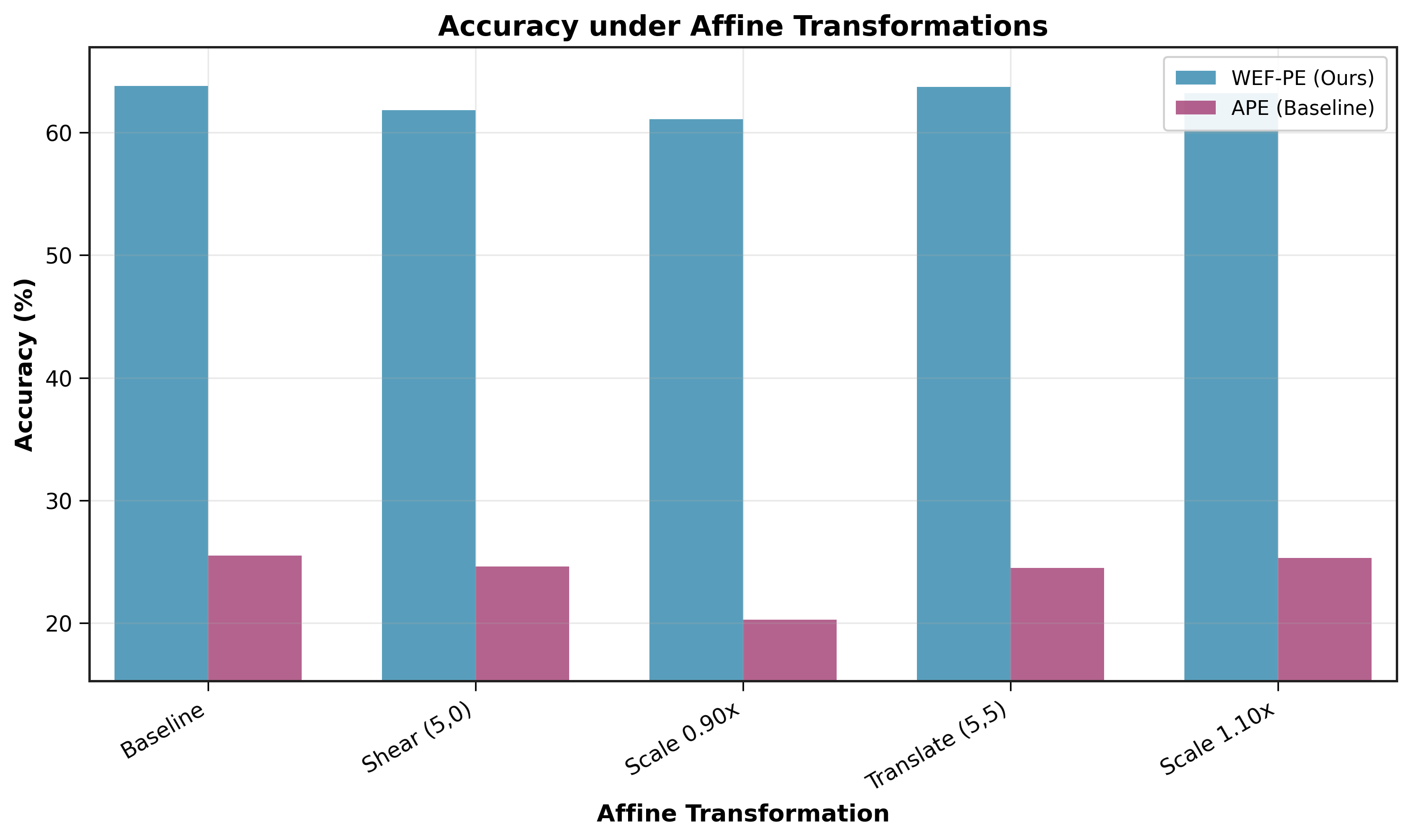} 
			\caption{Classification accuracy comparison between WePE and APE baseline under affine transformations on CIFAR-100~\citep{krizhevsky2009learning}.}
			\label{fig:affine_invariance} 
		\end{figure}
		
		\subsection{Relative Position Awareness Validation}
		\label{sec:relative-position-awareness}
		
		To empirically validate the theoretical proposition that elliptic function addition formulas endow models with enhanced relative position awareness capabilities, we designed a dedicated auxiliary task that directly probes the spatial relationship encoding within learned patch representations. The experimental framework leverages the mathematical property that for any two spatial positions $z_i$ and $z_j$ mapped to the complex plane, their relative displacement $\sigma_z = z_j - z_i$ can be algebraically derived from the Weierstrass elliptic function values through the addition formula $\wp(z_i + \sigma_z) = f(\wp(z_i), \wp(\sigma_z), \wp'(z_i), \wp'(\sigma_z))$.
		
		The experimental protocol extracts patch embeddings $e_i, e_j$ from pre-trained ViT-Tiny models without positional information, subsequently combining them with their corresponding positional encodings $p_i, p_j$ to form complete representations $e_i + p_i$ and $e_j + p_j$. A lightweight MLP predictor~\citep{rumelhart1986learning} with two hidden layers (192→128→64→2 dimensions) receives these concatenated patch representations as input and predicts the relative coordinate displacement $(\Delta x, \Delta y) = (x_j - x_i, y_j - y_i)$ in the original patch grid. We generated 8,000 training samples for each encoding scheme by randomly sampling patch pairs from CIFAR-100 test images~\citep{krizhevsky2009learning}, ensuring balanced coverage across different spatial separations within the 14×14 patch grid.
		
		The quantitative results demonstrate substantial superiority of WePE over conventional APE~\citep{dosovitskiy2020image} across all evaluation metrics. The mean squared error reduces from 6.69 to 0.90 (86.6\% improvement), mean absolute error decreases from 2.83 to 0.90 (68.1\% reduction), and root mean squared error diminishes from 2.59 to 0.95 (63.4\% improvement). Training dynamics reveal markedly different convergence behaviors, with WePE achieving stable convergence from an initial loss of 21.65 to 1.14 within 30 epochs, while APE converges more slowly from 30.42 to 4.14 under identical optimization settings. Error distribution analysis reveals that WePE concentrates prediction errors within the 0-1 unit range with peaked distribution characteristics, whereas APE~\citep{dosovitskiy2020image} exhibits broader error dispersion with substantial tail probability mass extending beyond 3 units.
		
		These empirical findings provide direct quantitative evidence that the continuous mathematical structure of elliptic functions facilitates superior spatial relationship encoding compared to discrete lookup table approaches. The substantial performance advantages validate the theoretical assertion that elliptic function addition formulas naturally embed relative positional information within absolute encodings, enabling models to extract geometric relationships through direct algebraic manipulation rather than requiring explicit learning of pairwise spatial dependencies. The 86.6\% reduction in prediction error magnitude demonstrates that WePE representations inherently preserve spatial geometric properties that prove essential for precise coordinate regression tasks, supporting the broader claim that mathematically principled positional encodings provide superior inductive biases for vision transformer architectures.
		
		\begin{figure}[ht]
			\centering
			\includegraphics[width=\columnwidth]{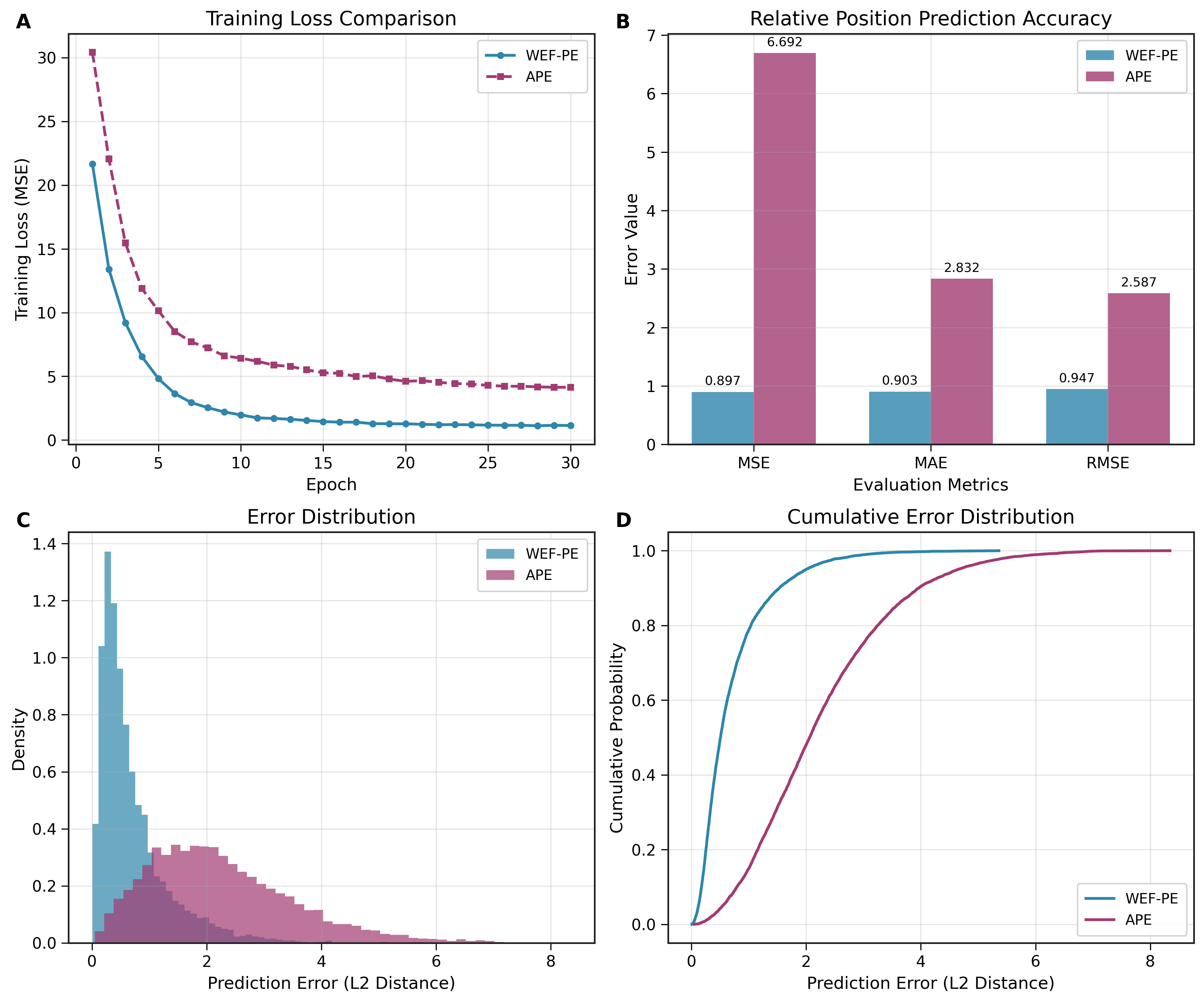}
			\caption{Quantitative validation of relative position awareness capabilities in WePE versus APE~\citep{dosovitskiy2020image} positional encodings. (A) Training loss convergence comparison showing faster and more stable learning dynamics for WePE. (B) Evaluation metrics demonstrating substantial improvements in prediction accuracy, with 86.6\%, 68.1\%, and 63.4\% reductions in MSE, MAE, and RMSE respectively. (C) Error distribution histograms revealing concentrated low-error predictions for WePE versus dispersed error patterns in APE~\citep{dosovitskiy2020image}. (D) Cumulative error distribution confirming superior prediction precision, with approximately 80\% of WePE predictions achieving sub-unit accuracy compared to broader error dispersion in APE-based representations.}
			\label{fig:relative_position_awareness}
		\end{figure}
		
		\subsection{Comparison with Alternative 2D Positional Encoding Schemes}
		\label{sec:comparison-2d-pe-schemes}
		
		To isolate the contribution of advanced mathematical structure from the fundamental advantage of preserving two-dimensional spatial relationships, we conducted systematic comparisons between WePE and alternative 2D positional encoding approaches that avoid the spatial flattening inherent in conventional APE methods~\citep{dosovitskiy2020image}. The experimental framework evaluated four distinct positional encoding schemes: our proposed WePE, traditional 1D flattened APE~\citep{dosovitskiy2020image}, 2D sinusoidal positional encoding extending classical Transformer sinusoidal patterns to two dimensions through independent coordinate-wise encoding, and 2D learnable grid positional encoding implementing direct parameter lookup based on spatial coordinates without intermediate flattening operations.
		
		The 2D sinusoidal approach generates positional representations by applying sine and cosine functions independently to horizontal and vertical coordinates, interleaving the resulting values as $\text{PE}[h,w,0::4] = \sin(w \cdot \text{div\_term}_x)$, $\text{PE}[h,w,1::4] = \cos(w \cdot \text{div\_term}_x)$, $\text{PE}[h,w,2::4] = \sin(h \cdot \text{div\_term}_y)$, and $\text{PE}[h,w,3::4] = \cos(h \cdot \text{div\_term}_y)$ where $\text{div\_term}$ follows the standard inverse frequency scaling. The 2D learnable grid method maintains a trainable parameter matrix of dimensions $H \times W \times D$ enabling direct coordinate-based lookup without spatial serialization, thereby preserving explicit two-dimensional indexing throughout the encoding process.
		
		Under identical training conditions with ViT-Tiny architecture on CIFAR-100~\citep{krizhevsky2009learning} for 60 epochs, the quantitative results demonstrate that WePE achieves superior performance with 60.03\% final validation accuracy and 60.28\% peak accuracy, followed closely by 2D sinusoidal encoding at 59.94\% final and 60.19\% peak accuracy. The 2D learnable grid approach yields 58.40\% accuracy for both final and peak measurements, while conventional 1D flattened APE produces 58.28\% final and 58.34\% peak accuracy. Training dynamics reveal that WePE maintains consistently lower training loss throughout the optimization process with smoother convergence characteristics compared to alternative approaches.

		\begin{figure}[ht]
			\centering
			\includegraphics[width=\columnwidth]{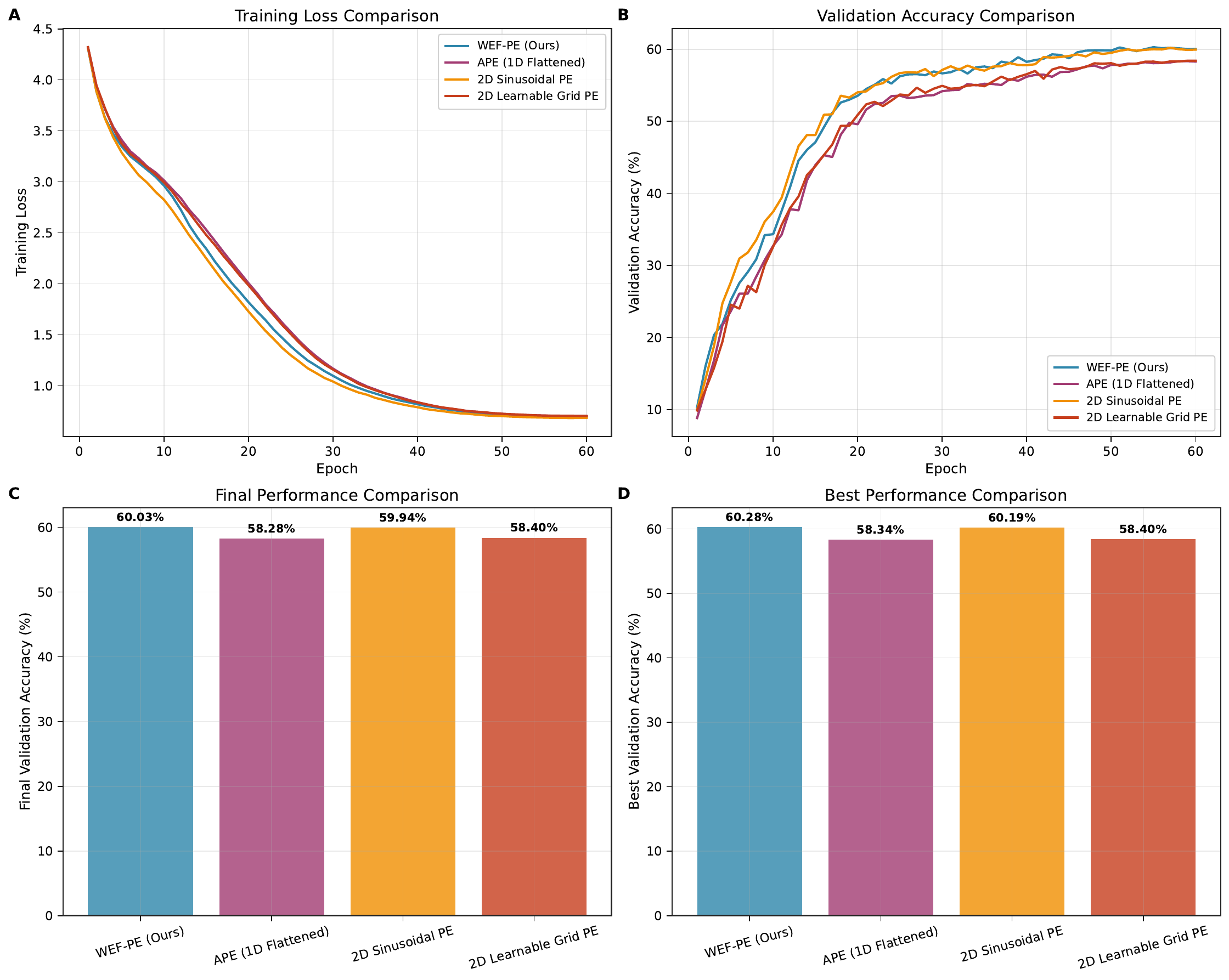}
			\caption{Comparative analysis of 2D positional encoding schemes. (A) Training loss dynamics for WePE, 1D flattened APE~\citep{dosovitskiy2020image}, 2D Sinusoidal PE, and 2D Learnable Grid PE over 60 epochs on CIFAR-100~\citep{krizhevsky2009learning}. (B) Corresponding validation accuracy curves. (C) Final validation accuracy comparison, where WePE achieves 60.03\%. (D) Best validation accuracy comparison, with WePE peaking at 60.28\%, demonstrating its superior performance over alternative 2D encoding strategies.}
			\label{fig:2d_pe_comparison}
		\end{figure}
		The experimental findings reveal that while preservation of two-dimensional spatial structure provides measurable advantages over traditional flattening approaches, the mathematical sophistication embedded within elliptic function-based encoding yields additional performance gains beyond those attributable solely to dimensionality considerations. The modest but consistent superiority of WePE over 2D sinusoidal encoding (0.09 percentage points final accuracy improvement) validates the hypothesis that continuous mathematical structure and inherent geometric properties contribute meaningfully to positional representation quality. The relatively strong performance of 2D sinusoidal encoding compared to learnable grid methods suggests that mathematical regularity and interpretability provide advantages over pure parameter optimization in spatial encoding tasks, supporting the broader principle that principled mathematical foundations enhance neural network architectural design for vision applications.

		\subsection{WePE exhibits better geometric inductive bias}
		\label{sec:wepe-geometric-bias}
		
		To further evaluate the structural properties of our proposed WePE, we visualized the positional encodings prior to any model training, as illustrated in ~\Cref{fig:structural_analysis}. The analysis comes from two aspects: a PCA~\citep{bishop2006pattern} to reveal the embedding manifold, and a cosine similarity matrix to expose their relational structure. In PCA~\citep{bishop2006pattern} space, WePE forms a highly structured, spiral-like manifold, that faithfully preserves the original 2D spatial arrangement, as confirmed by the color gradient-encoded patch coordinates. By contrast,  the APE~\citep{dosovitskiy2020image} projects into an unstructured, Gaussian-like cloud, demonstrating a complete lack of inherent spatial organization. Furthermore, the cosine similarity matrix of WePE displays a distinct, periodic, and grid-like pattern, indicating that the relationships between encodings are systematically governed by their relative spatial distances. The APE~\citep{dosovitskiy2020image} matrix, however, resembles random noise aside from the identity diagonal, confirming the absence of any pre-defined relational structure.
		
		\begin{figure}[h]
			\centering 
			\includegraphics[width=\columnwidth]{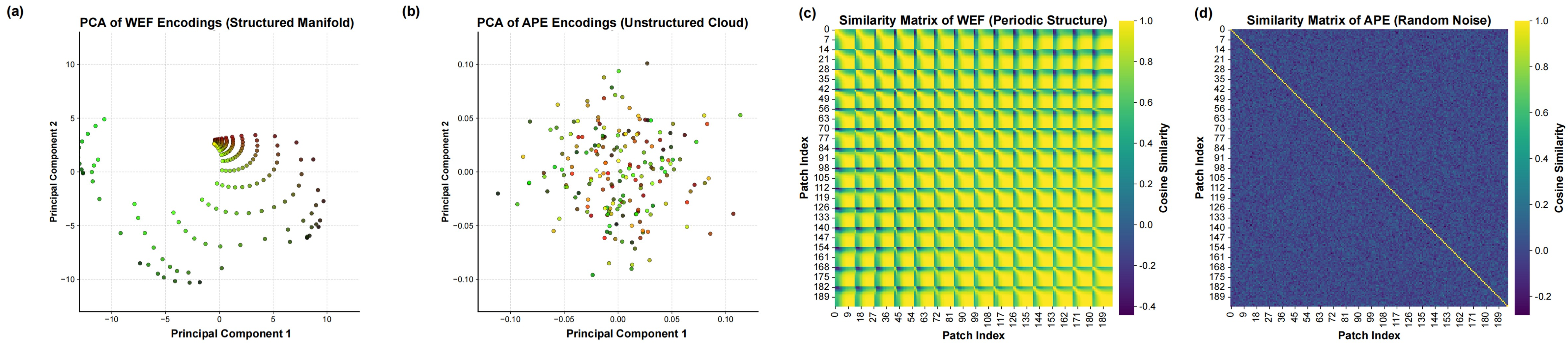} 
			\caption{Structural properties of positional encodings revealed through PCA~\citep{bishop2006pattern} and cosine similarity. (a, b) PCA~\citep{bishop2006pattern} projections demonstrate WePE forms structured spiral manifolds that preserving spatial topology, while APE~\citep{dosovitskiy2020image} appears as unstructured Gaussian-like clouds. (c, d) Cosine similarity matrices reveal that WePE displays periodic, grid-like patterns reflecting systematic spatial relationships, contrasting with largely random patterns of APE~\citep{dosovitskiy2020image}.}
			\label{fig:structural_analysis}
		\end{figure}

		\subsection{Supplement to the experiment in Section~\ref{sec:long_term_attenuation}}	
		\label{sec:supplement_long_term}
		
		\begin{figure}[t]
			\centering
			\includegraphics[width=\columnwidth]{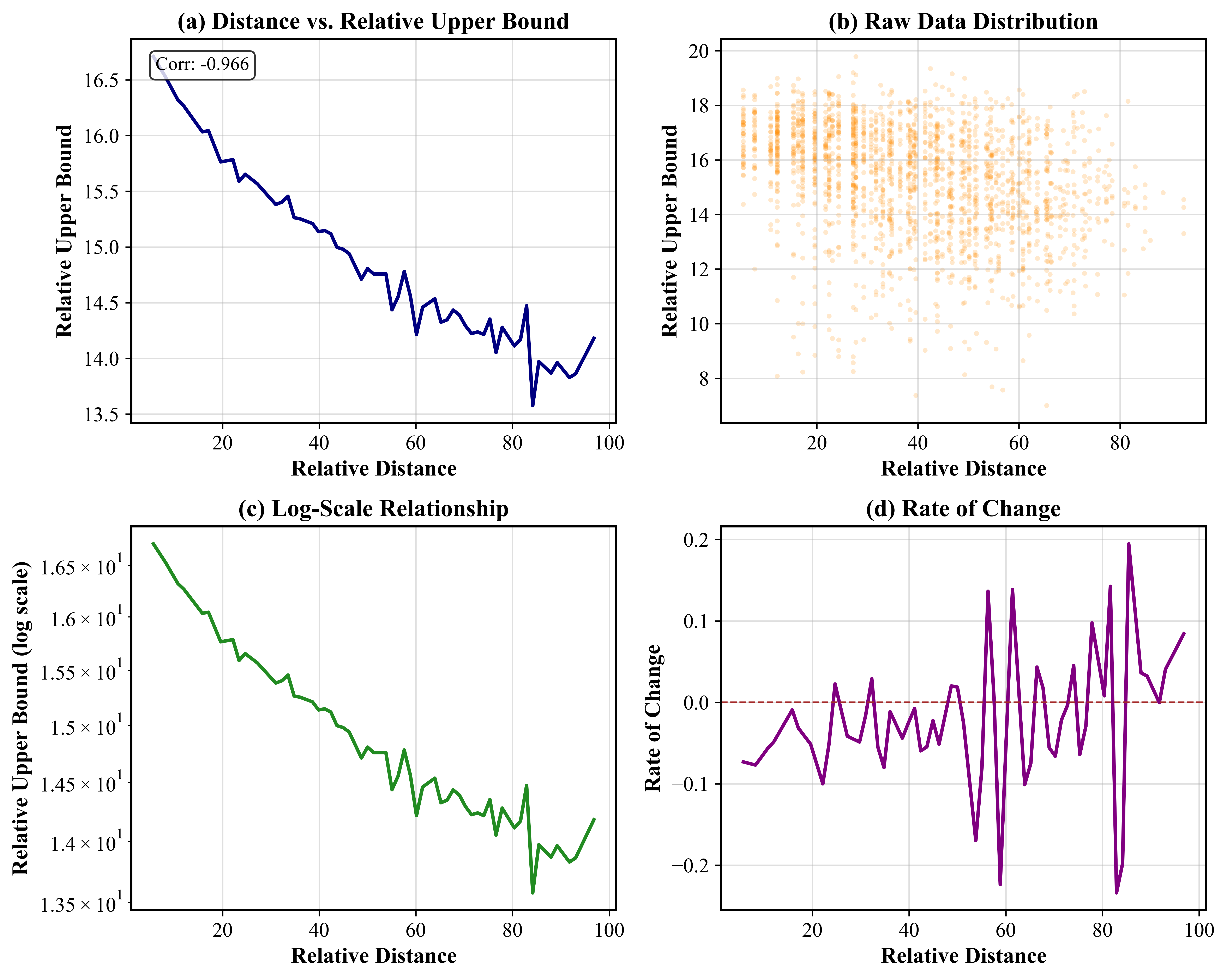}
			\caption{Quantitative analysis of distance-decay properties in WePE.
				(a) Scatter plot and fitted curve demonstrate strong negative correlation between relative patch distance and interaction strength. 
				(b) Raw data distribution across 19,110 patch pairs. 
				(c) Log-scale relationship confirming exponential decay characteristics. 
				(d) Rate of change analysis revealing monotonic decrease in similarity with increasing spatial separation.}
			\label{fig:app-distance-decay-quad}
		\end{figure}

		To validate the distance, decay property of the WePE, we designed a systematic experiment to quantitatively analyze the relationship between positional encoding interaction strength and spatial distance. 
		The experiment was based on the ViT-Ti architecture (embed\_dim=192, patch\_size=16), processing input images of size $224 \times 224$, which results in a $14 \times 14$ grid of patches, totaling 196 spatial locations.
		
		The core methodology involves establishing a quantitative relationship between the relative distance of all patch pairs and the interaction strength of their corresponding positional encodings. 
		Specifically, for any two patch locations $(i_1, j_1)$ and $(i_2, j_2)$, we first compute their Euclidean distance. 
		Subsequently, we extract their corresponding Weierstrass elliptic function positional encodings, $\mathbf{p}_{i_1,j_1}$ and $\mathbf{p}_{i_2,j_2}$, and compute their cosine similarity as a metric for interaction strength:
		\begin{equation}
			S_{i,j} = \frac{\mathbf{p}_i^\top \mathbf{p}_j}{\|\mathbf{p}_i\| \, \|\mathbf{p}_j\|}.
			\label{eq:cosine_similarity}
		\end{equation}
		
		Cosine similarity, which normalizes out the influence of vector magnitudes, provides a purer reflection of directional correlation and serves as a key indicator for evaluating the quality of positional encodings.
		
		To eliminate the dependency on image size, the distance is normalized into a relative distance:
		\begin{equation}
			d_{\mathrm{relative}} = \frac{d_{\mathrm{euclidean}}}{d_{\max}} \times 100,
			\label{eq:relative_distance}
		\end{equation}
		where $d_{\max}$ is the maximum possible distance between any two patches in the image.
		
		To enhance the visual interpretability of the relationship between inter-patch distance and interaction strength, we applied a linear transformation to the raw cosine similarity scores. 
		The primary motivation for this transformation is to normalize the similarity values, which may originally occupy a narrow numerical range, into a standardized and more visually dynamic scale. 
		This ensures consistent and comparable graphical representation across different experimental settings:
		\begin{equation}
			s_{\mathrm{rel}} = C_{\mathrm{base}} + \frac{s - s_{\min}}{s_{\max} - s_{\min}} \times C_{\mathrm{range}},
			\label{eq:similarity_normalization}
		\end{equation}
		where $s_{\min}$ and $s_{\max}$ represent the minimum and maximum observed cosine similarity values across the entire dataset of patch pairs, respectively. 
		The term $\tfrac{s-s_{\min}}{s_{\max}-s_{\min}}$ performs a min--max normalization, scaling the similarity scores to the range $[0,1]$. 
		In our specific analysis, we set the base constant $C_{\mathrm{base}}=6$ and the range constant $C_{\mathrm{range}}=14$, thereby mapping the original similarity scores to a new, standardized interval of $[6,20]$. 
		This procedure facilitates a clearer visualization of the decay trend by amplifying the dynamic range of the dependent variable.
		
		\begin{table}[t]
			\centering
			\caption{Quantitative Analysis Results of the Distance--Interaction Strength Relationship}
			\label{tab:distance_decay_results}
			\setlength{\tabcolsep}{4pt} 
			\begin{tabularx}{\columnwidth}{@{}Xl@{}}
				\toprule
				\textbf{Metric} & \textbf{Value} \\
				\midrule
				Pearson Correlation Coefficient $\rho$ & -0.966 \\
				Relative Distance Range & [0, 100] \\
				Relative Upper Bound Range & [13.5, 16.5] \\
				Initial Interaction Strength & 16.5 \\
				Final Interaction Strength & 13.8 \\
				Decay Magnitude $\Delta_{\mathrm{decay}}$ & 16.4\% \\
				Monotonicity Metric $\mathcal{M}$ & 87.5\% \\
				\bottomrule
			\end{tabularx}
		\end{table}

		Finally, a statistical summary is generated to distill the underlying trend from the point cloud of data. 
		We bin the data into 80 equi-width intervals based on distance. 
		For each of the 80 distance bins, indexed from $k=1$ to $80$, we aggregate all data points whose relative distance $d_{\mathrm{rel}}$ falls within the bin’s range $[d_k, d_{k+1}]$. 
		We then compute the arithmetic mean of the corresponding relative upper bound values within this bin, denoted as $\bar{s}_k$. 
		To represent the distance for each bin, we use its midpoint, calculated as $d_k = (d_k+d_{k+1})/2$. 
		This procedure culminates in two corresponding sequences: a sequence of bin centers $\{d_k\}$ and a sequence of average relative upper bounds $\{\bar{s}_k\}$. 
		These sequences effectively constitute a discrete function that quantitatively describes the relationship between relative distance and average interaction strength.The data of this experiment are presented in Figure~\ref{fig:app-distance-decay-quad} and Table~\ref{tab:distance_decay_results}.

		\subsection{Further supplementary attention visualizations}
		
		\begin{figure}[ht]
			\centering
			\begin{minipage}{\columnwidth}
				\centering
				\setlength{\tabcolsep}{0pt} 
				\renewcommand{\arraystretch}{0}
				
				\begin{tabular}{@{}cccccc@{}}
					\tile{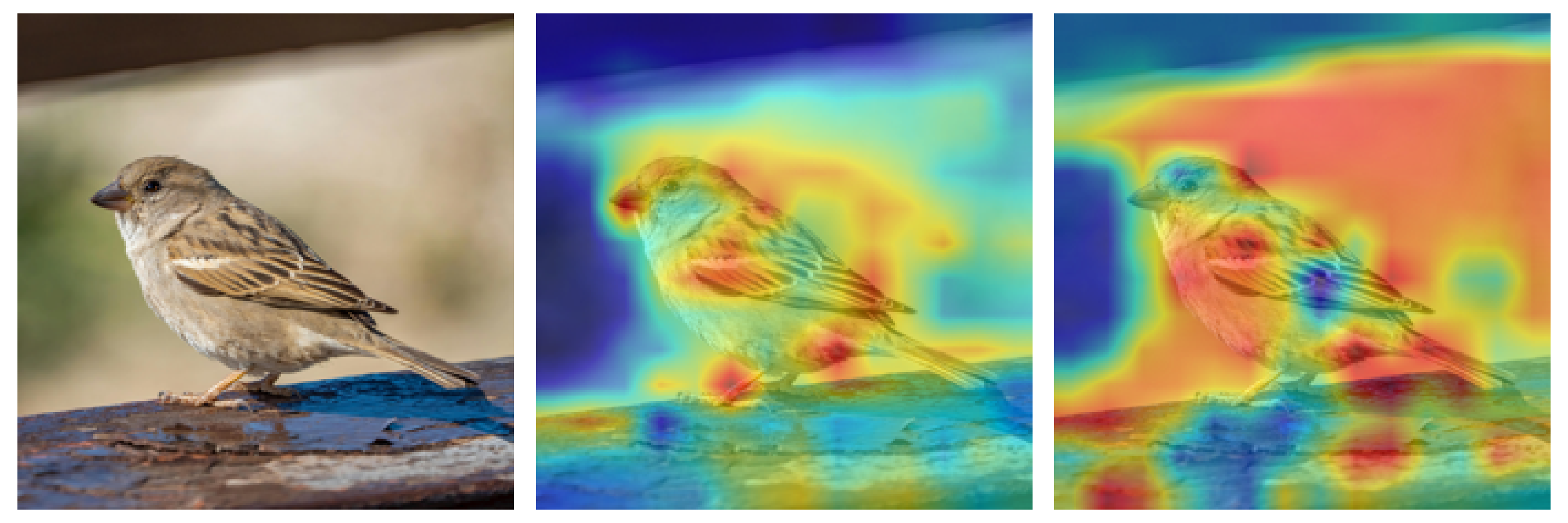}{1} &
					\tile{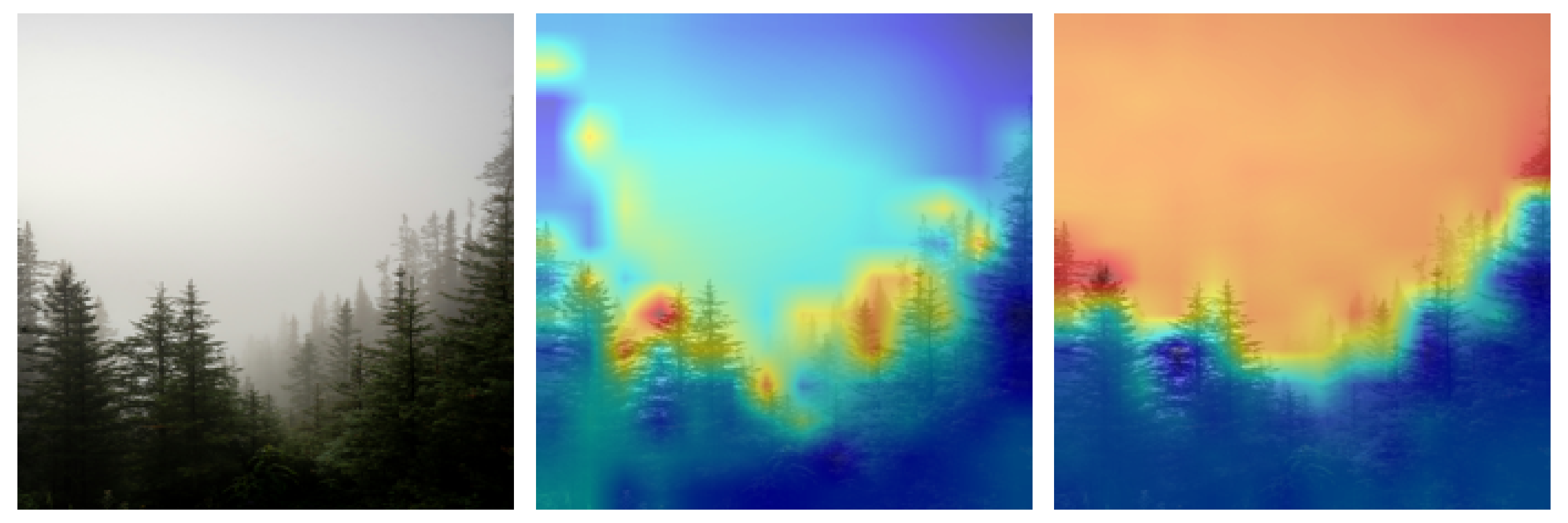}{2} &
					\tile{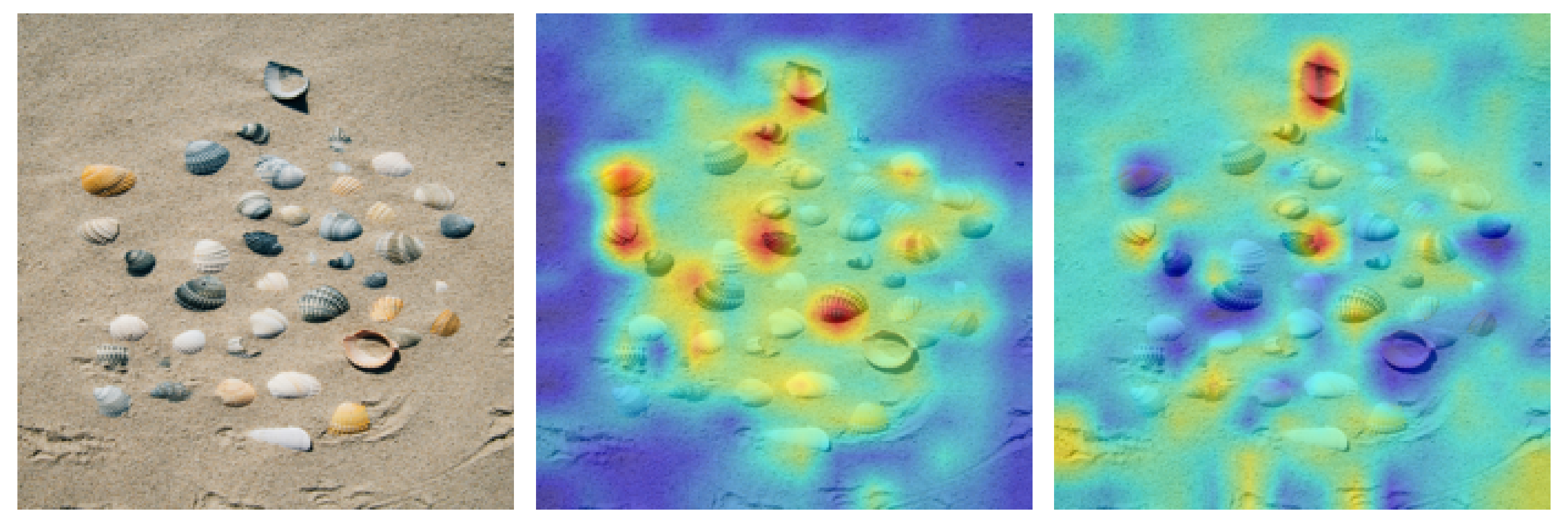}{3} &
					\tile{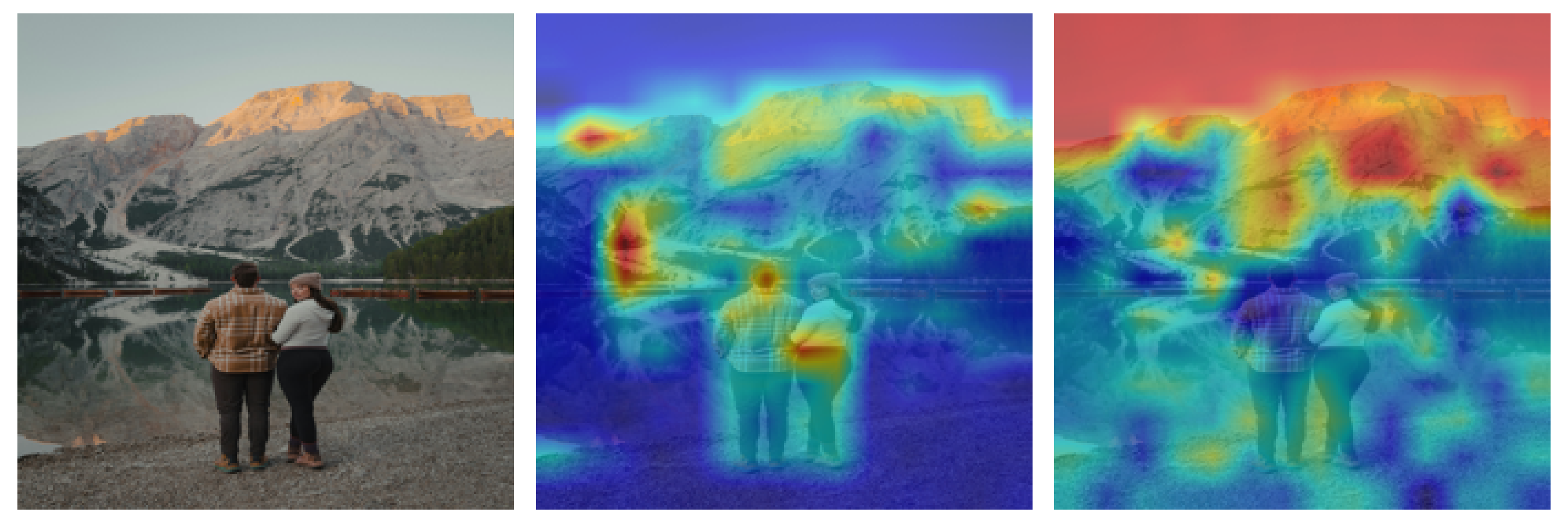}{4} &
					\tile{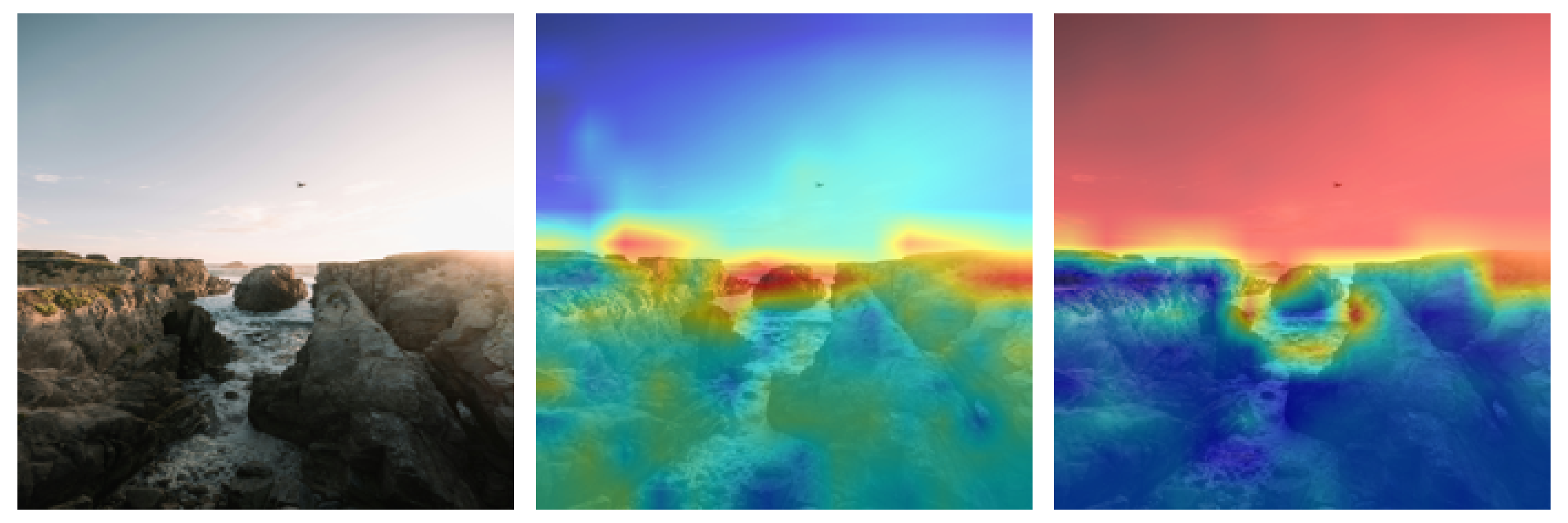}{5} &
					\tile{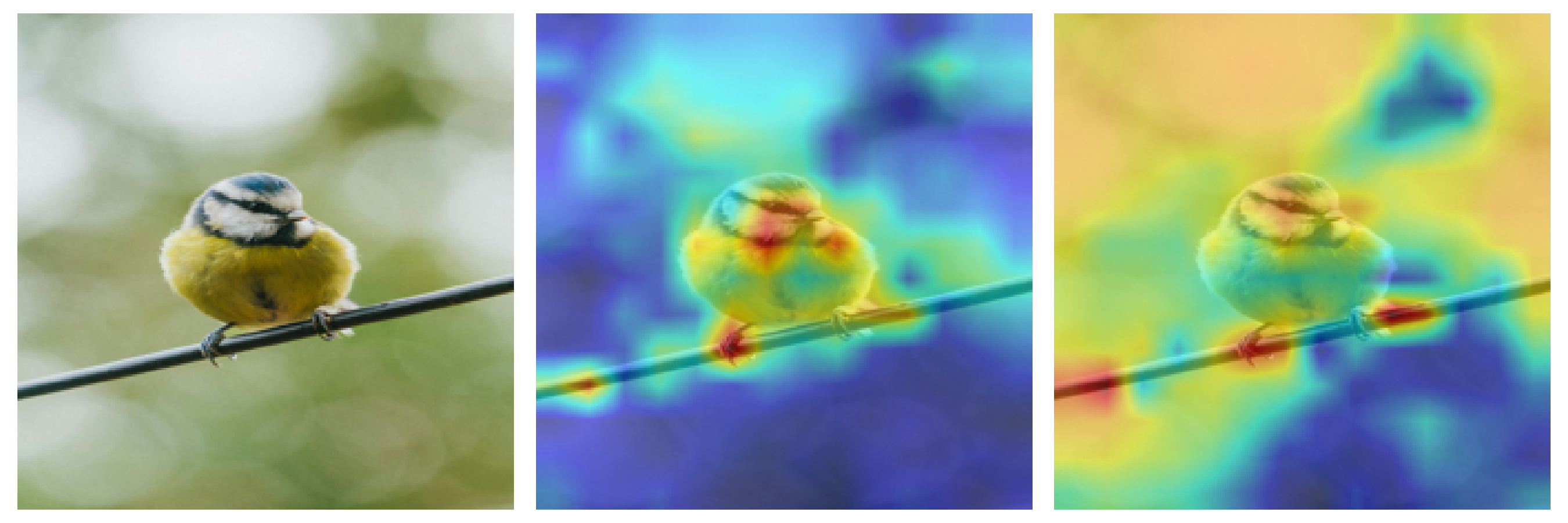}{6} \\
					\tile{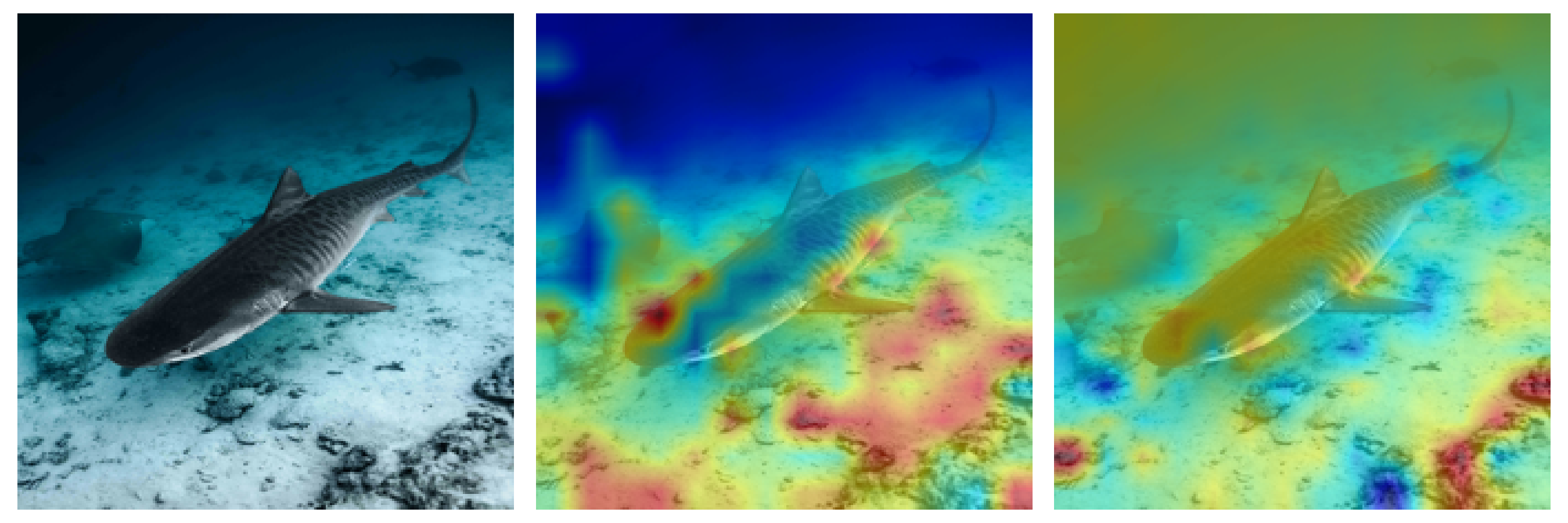}{7} &
					\tile{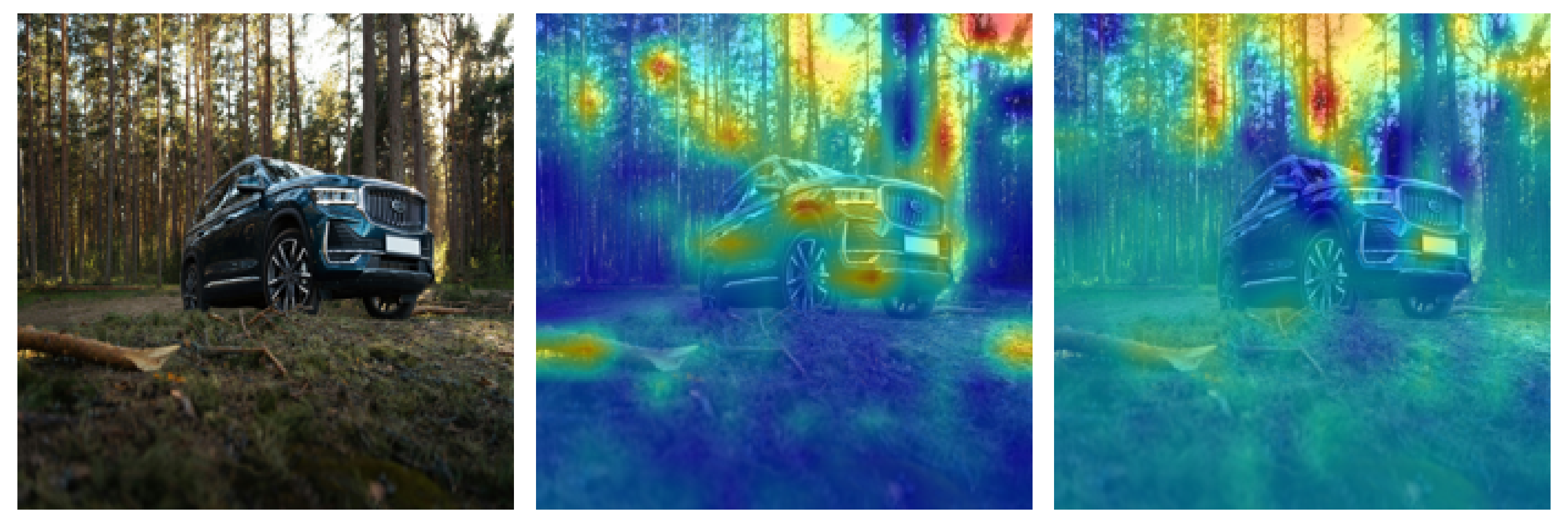}{8} &
					\tile{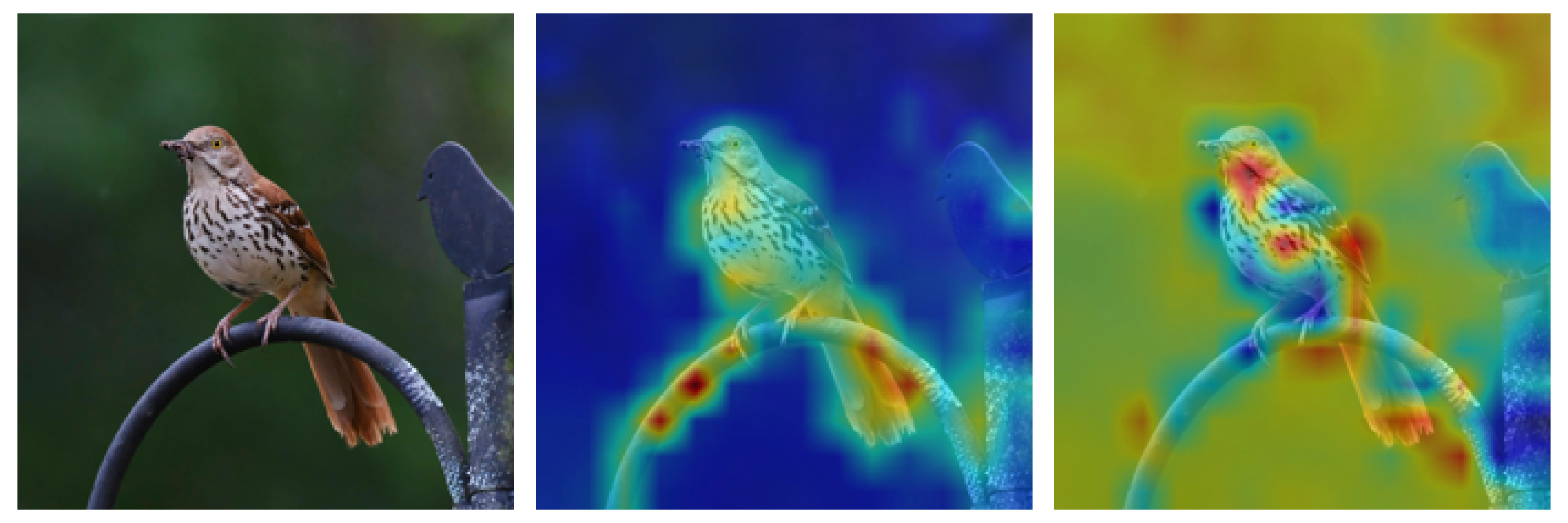}{9} &
					\tile{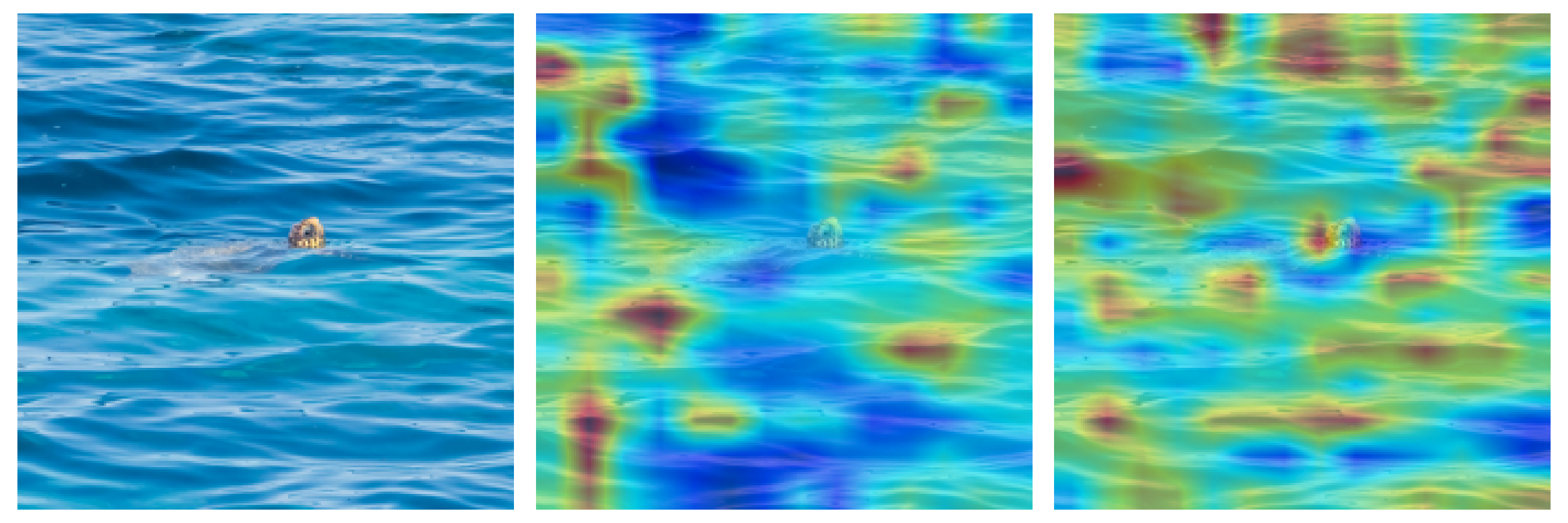}{10} &
					\tile{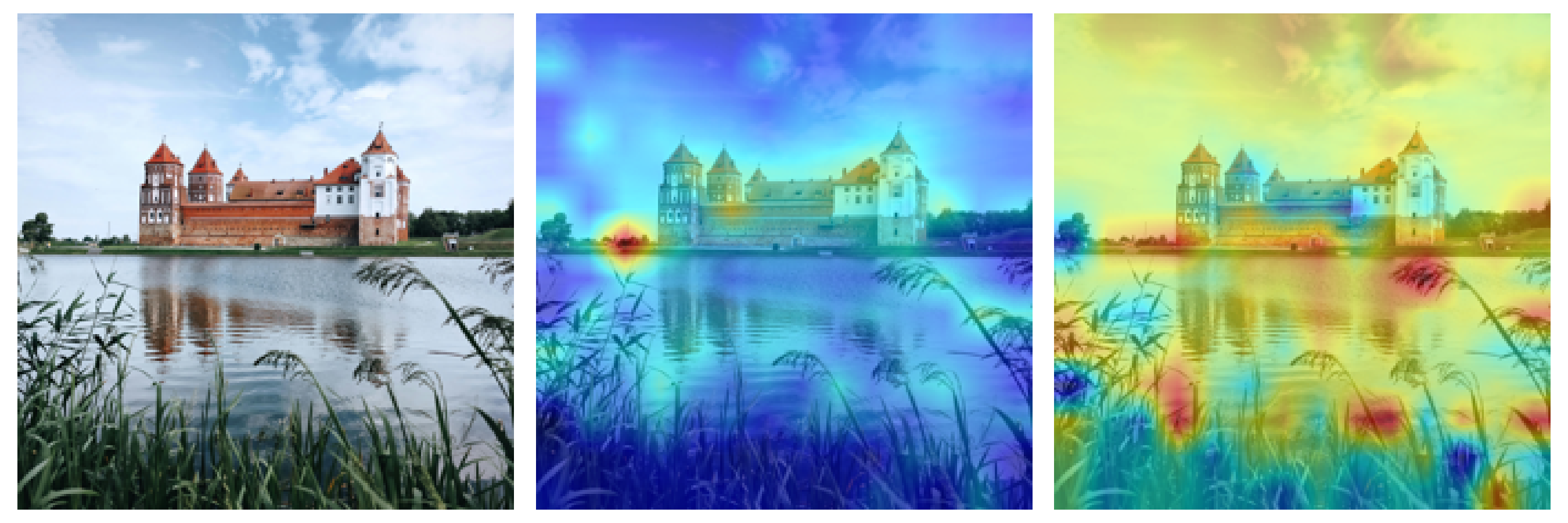}{11} &
					\tile{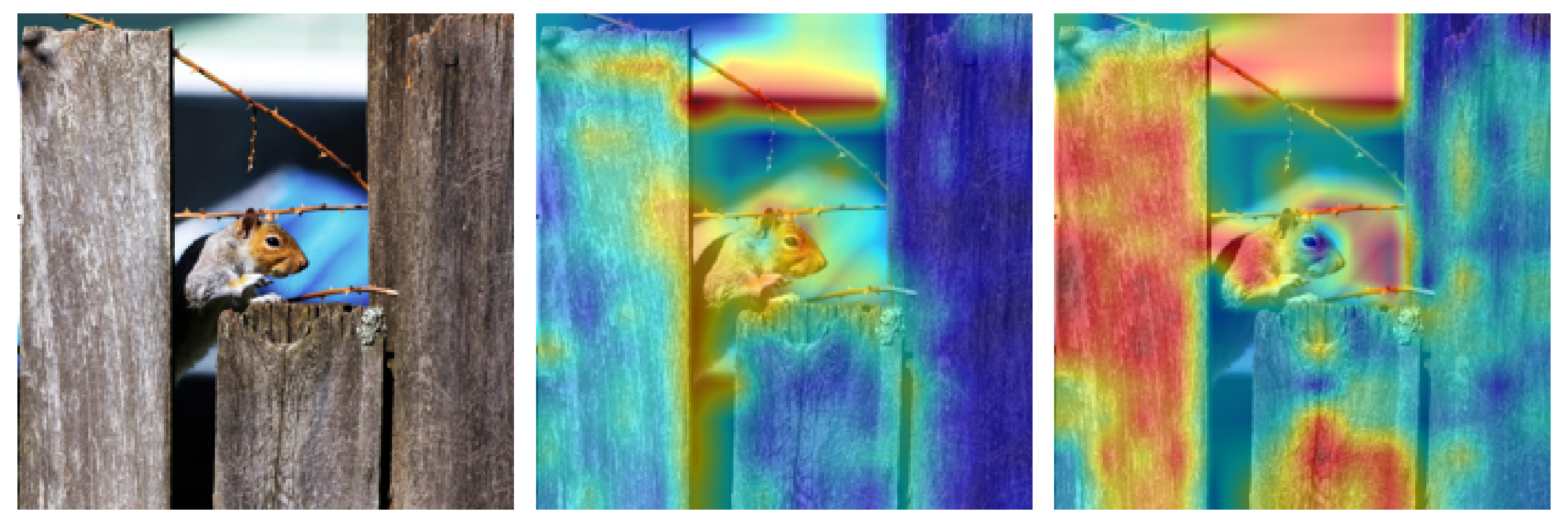}{12} \\
					\tile{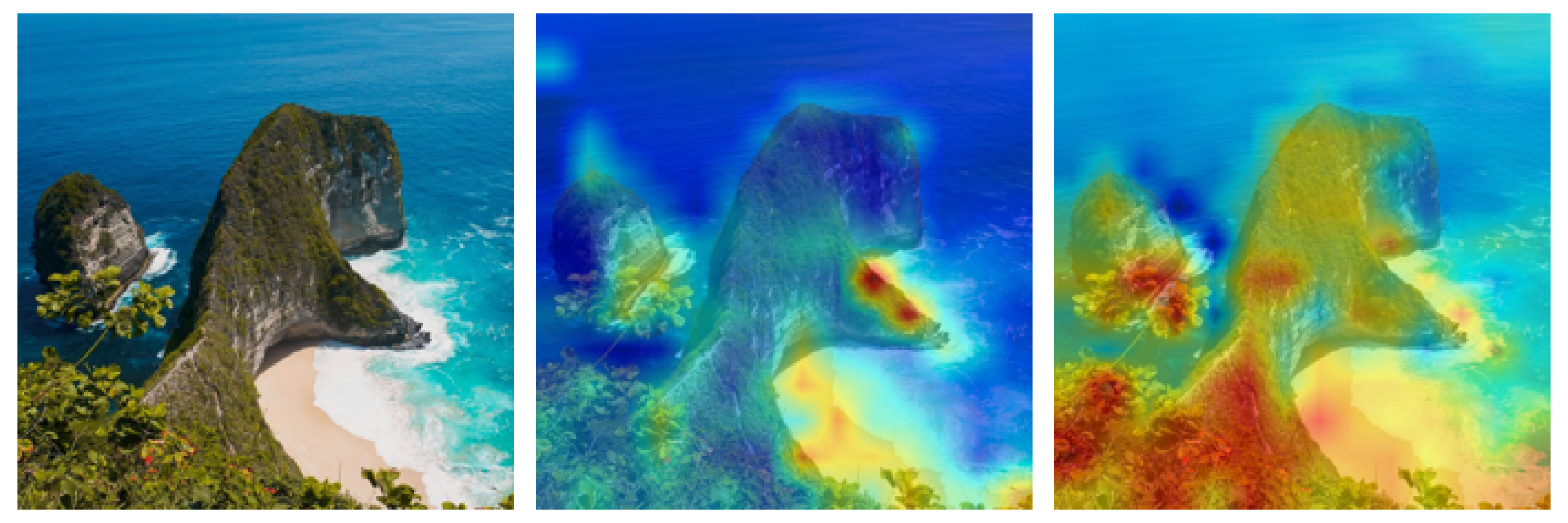}{13} &
					\tile{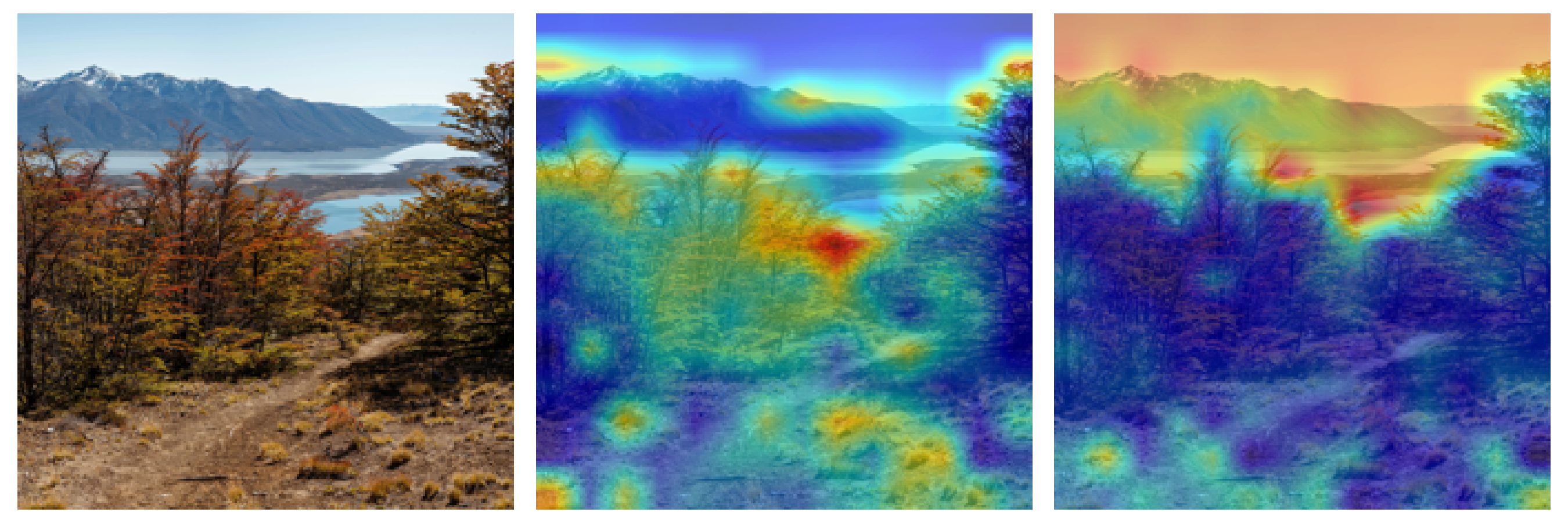}{14} &
					\tile{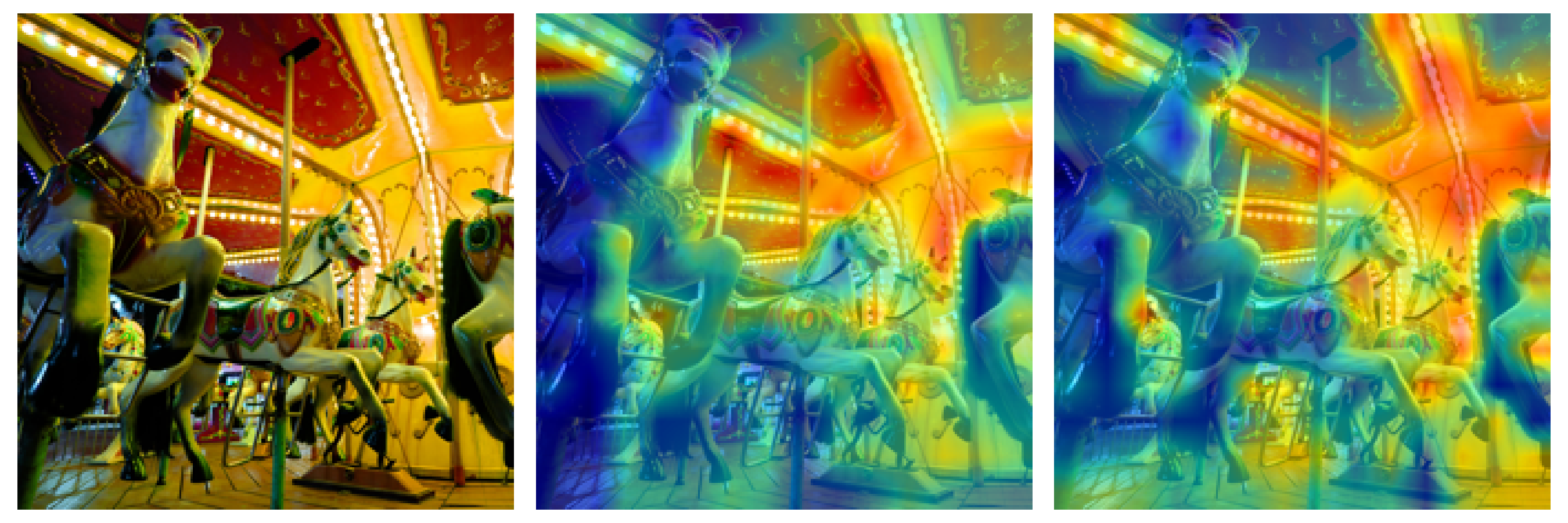}{15} &
					\tile{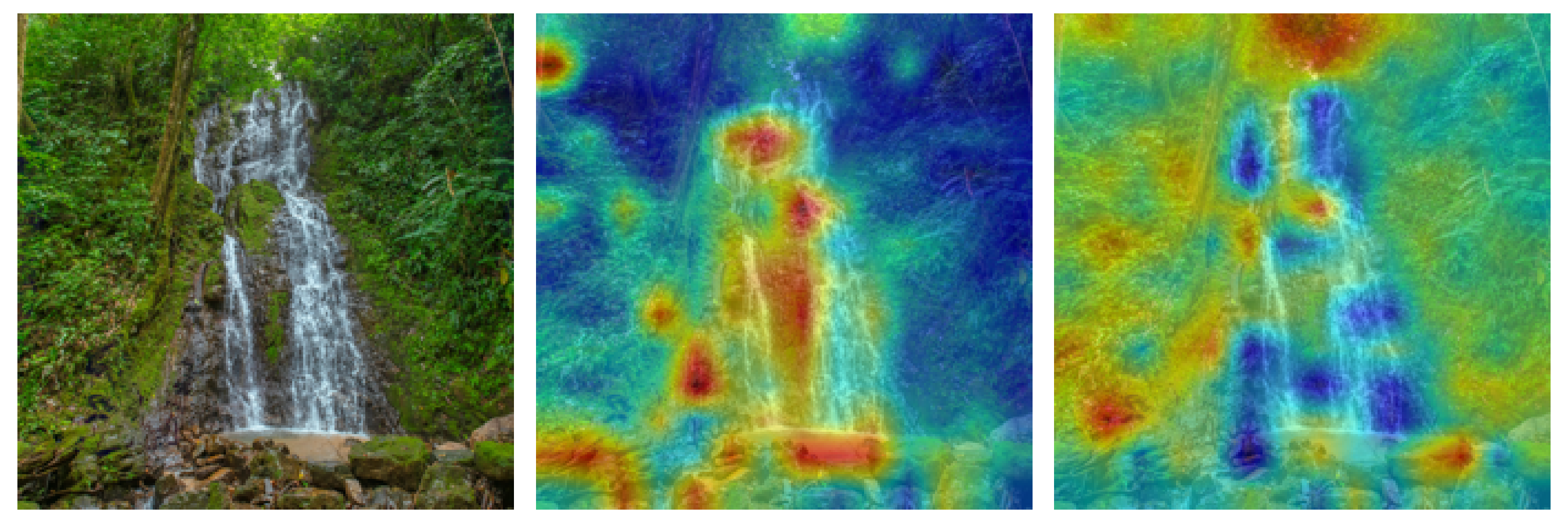}{16} &
					\tile{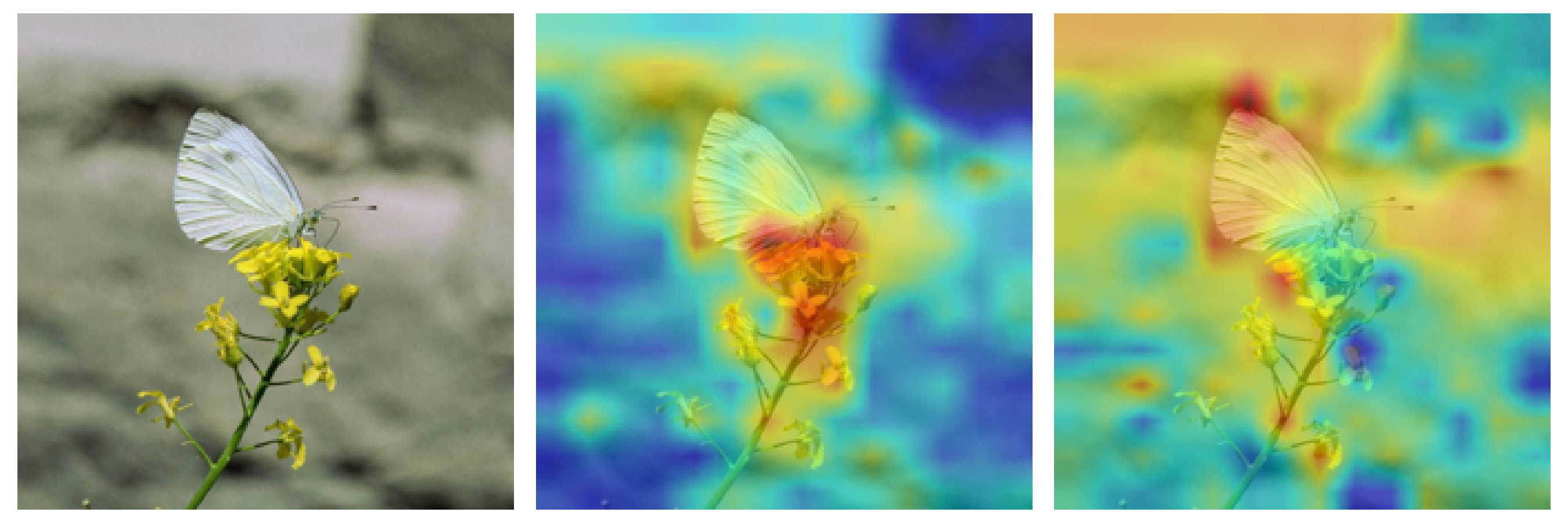}{17} &
					\tile{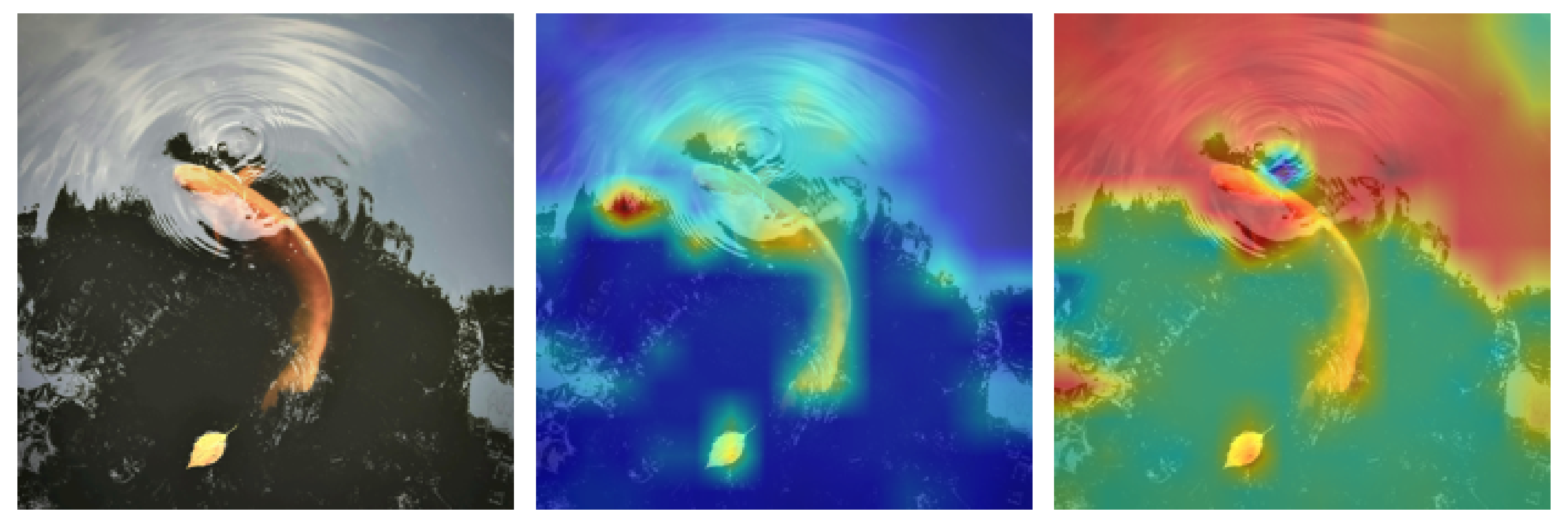}{18} \\
					\tile{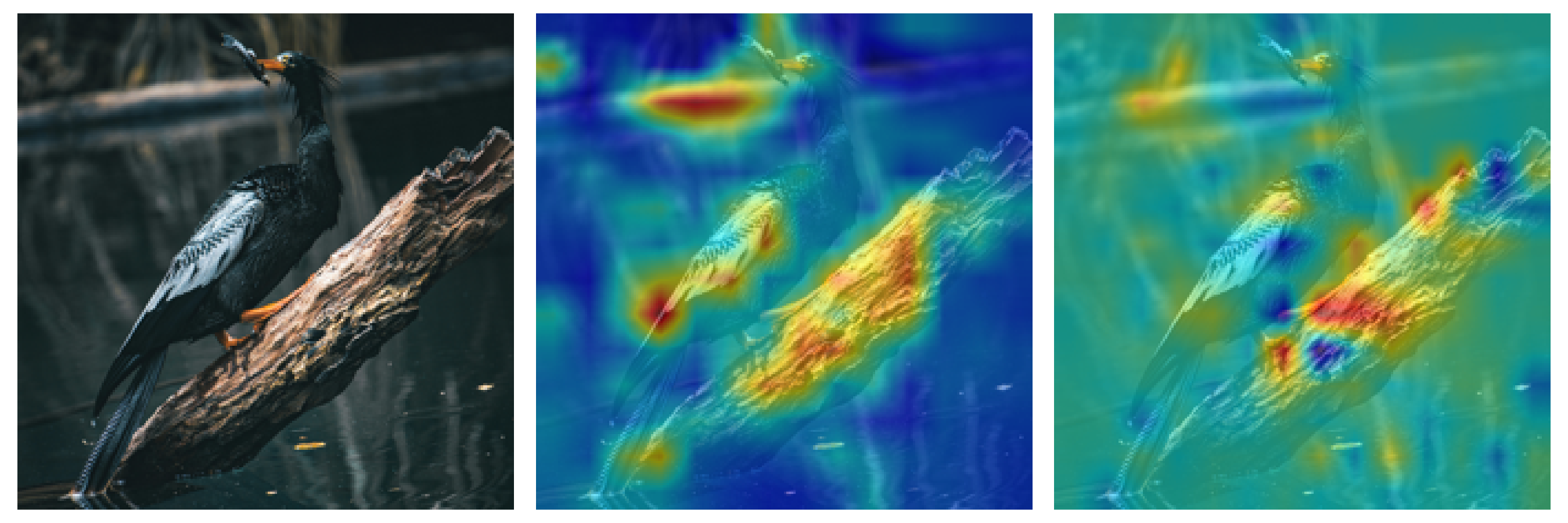}{19} &
					\tile{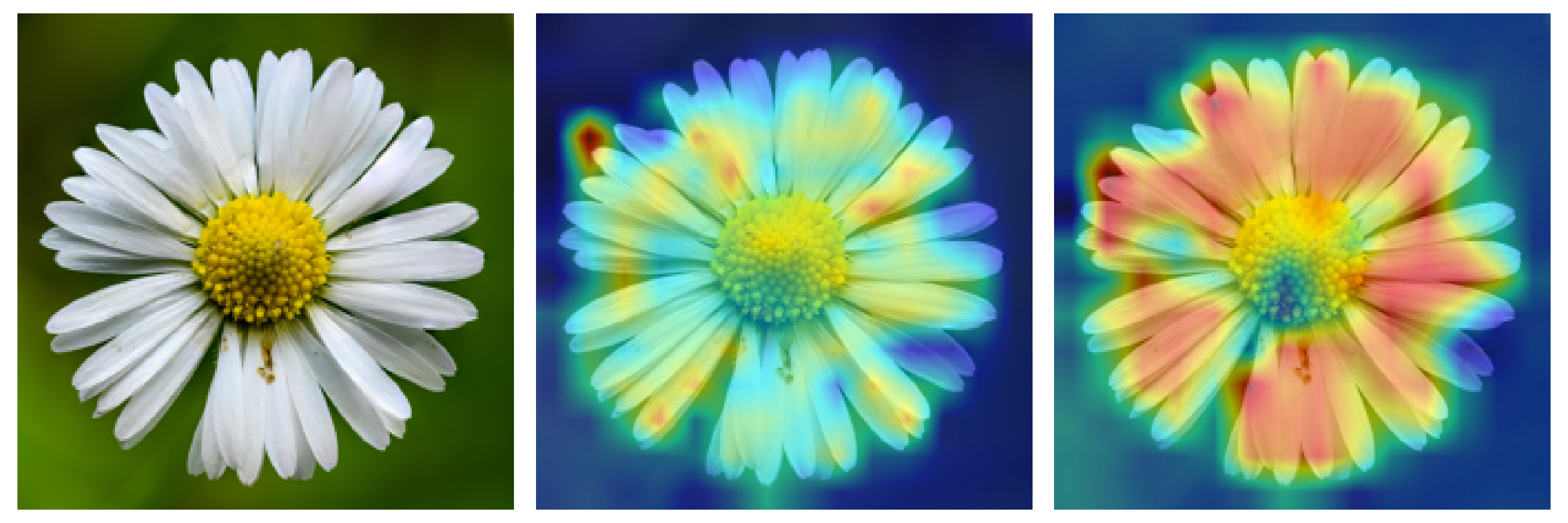}{20} &
					\tile{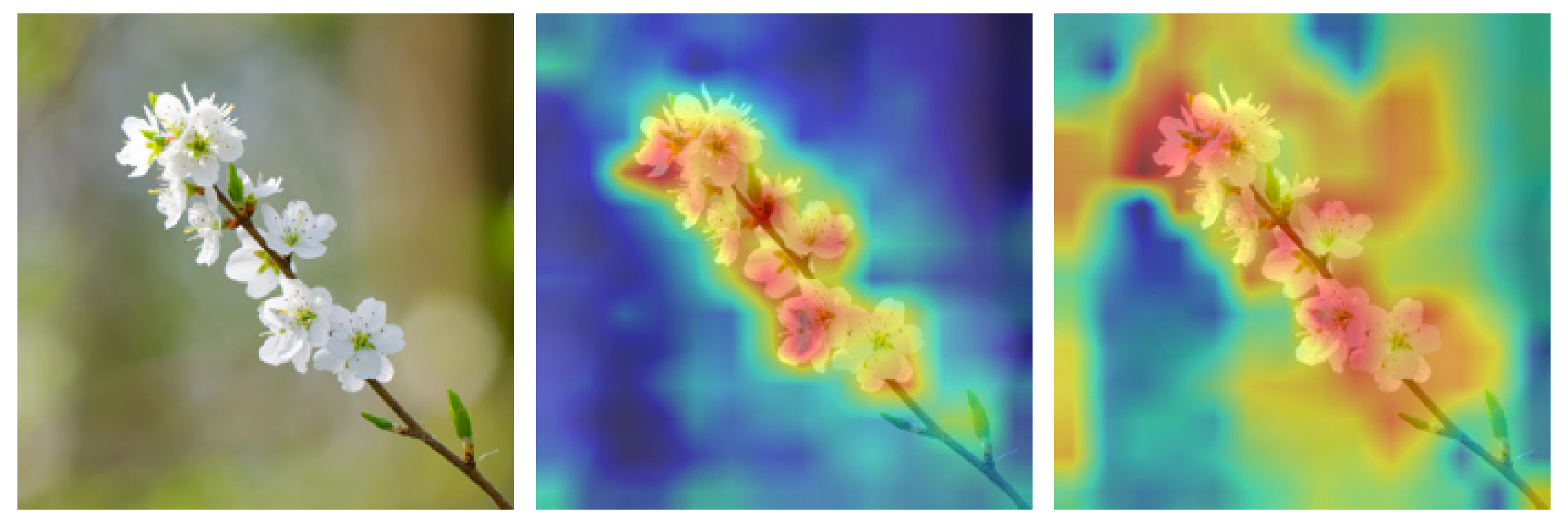}{21} &
					\tile{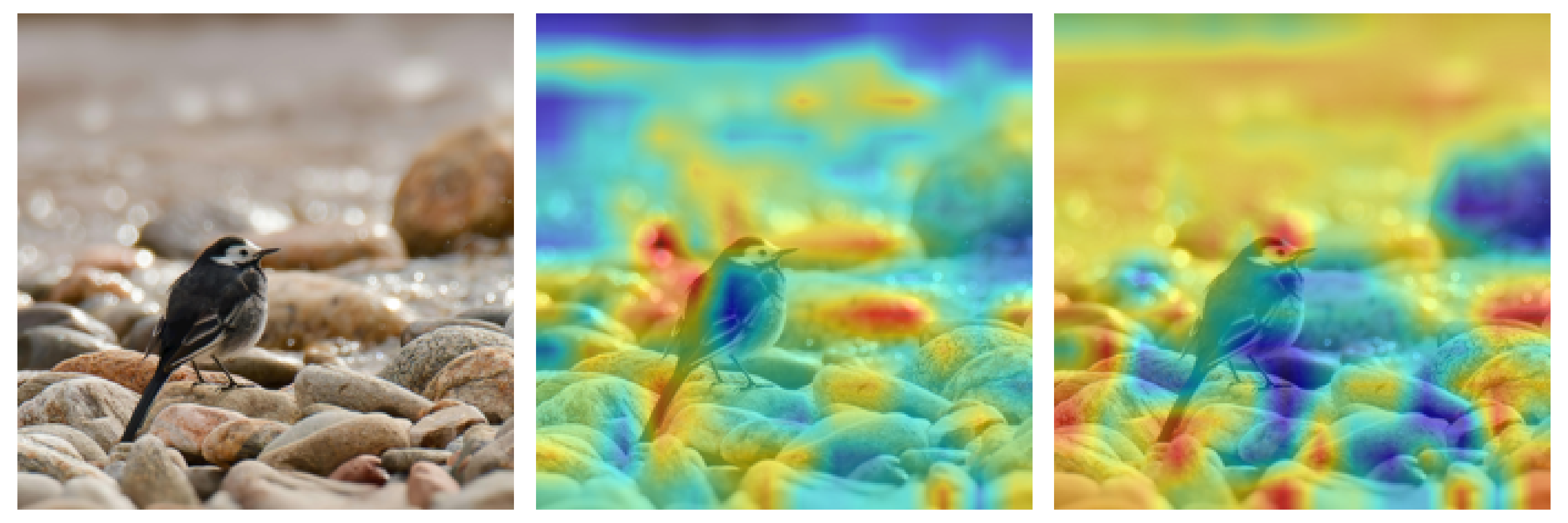}{22} &
					\tile{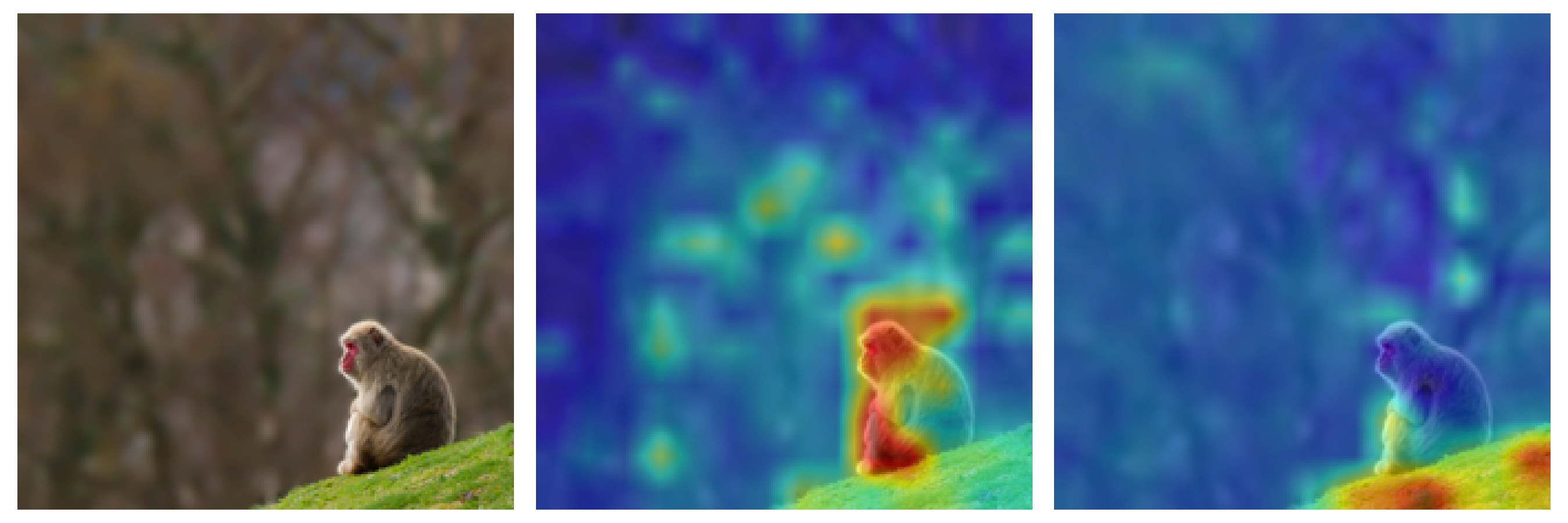}{23} &
					\tile{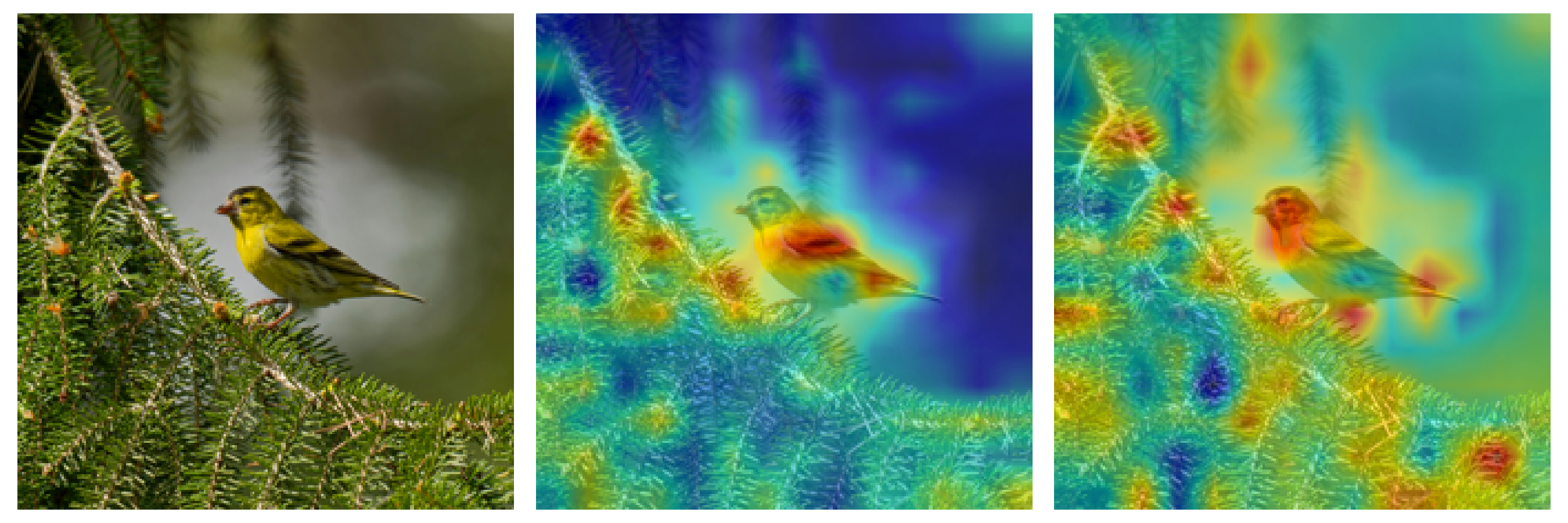}{24} \\
					\tile{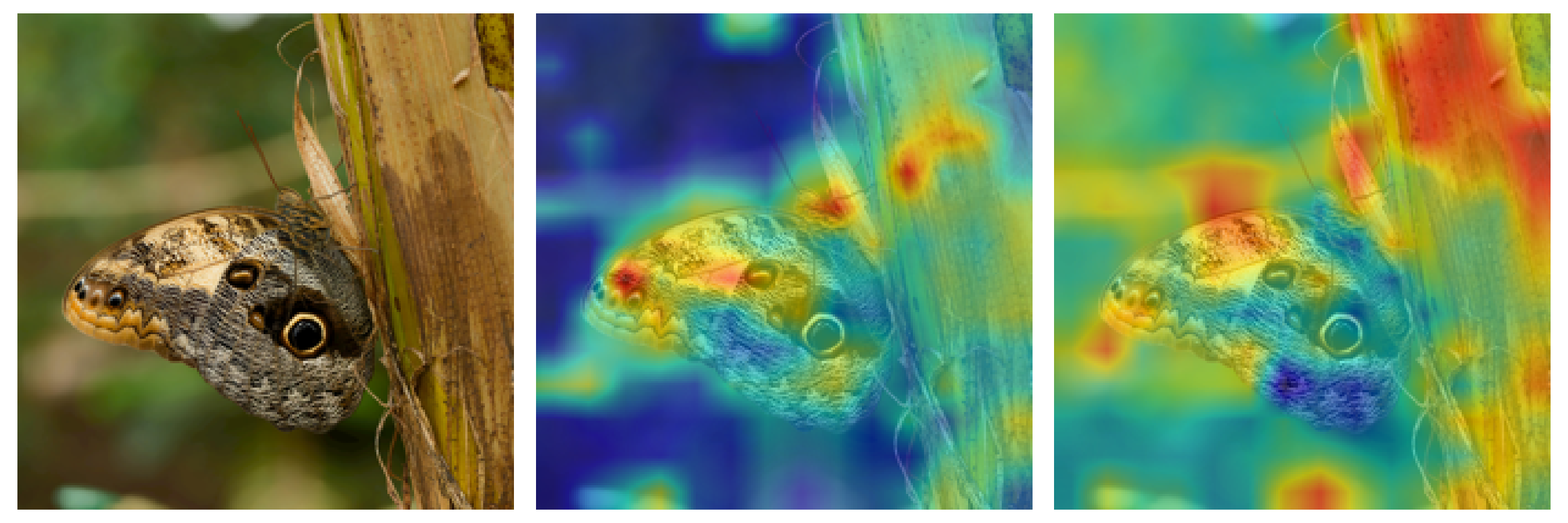}{25} &
					\tile{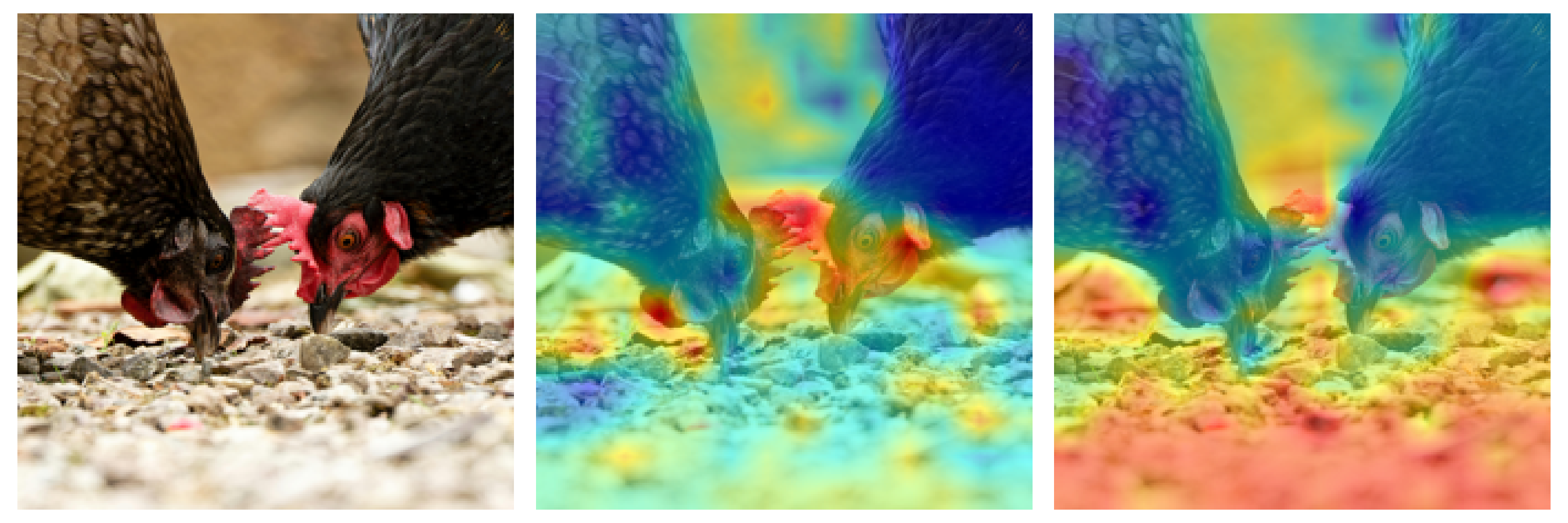}{26} &
					\tile{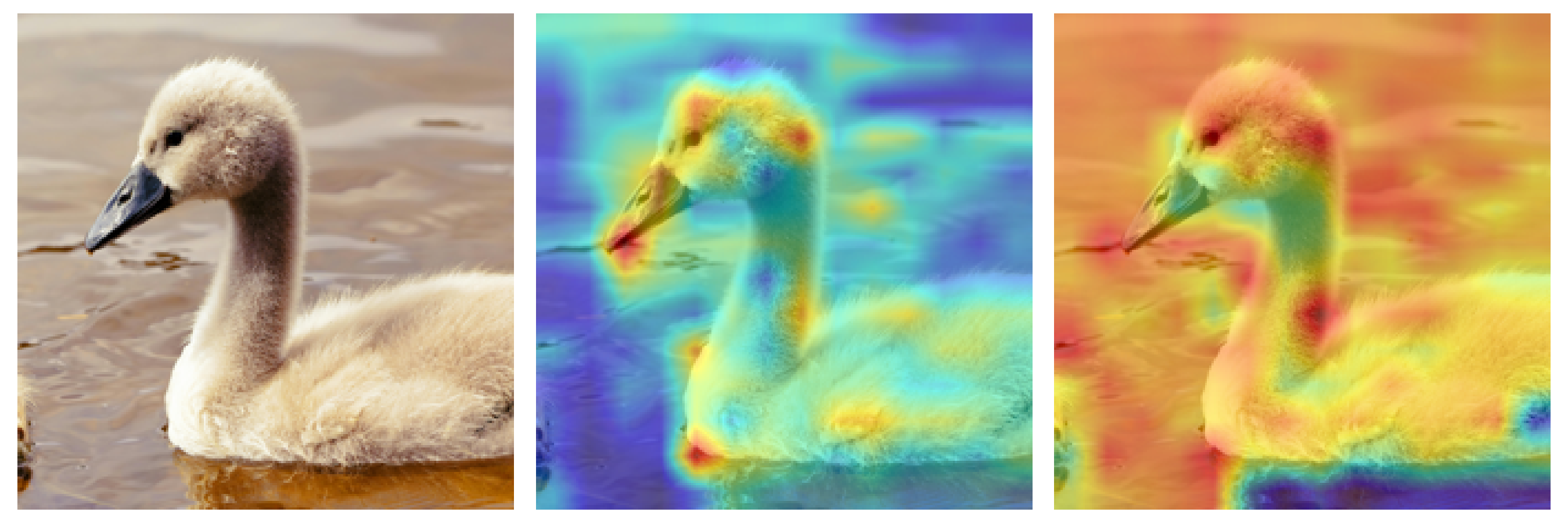}{27} &
					\tile{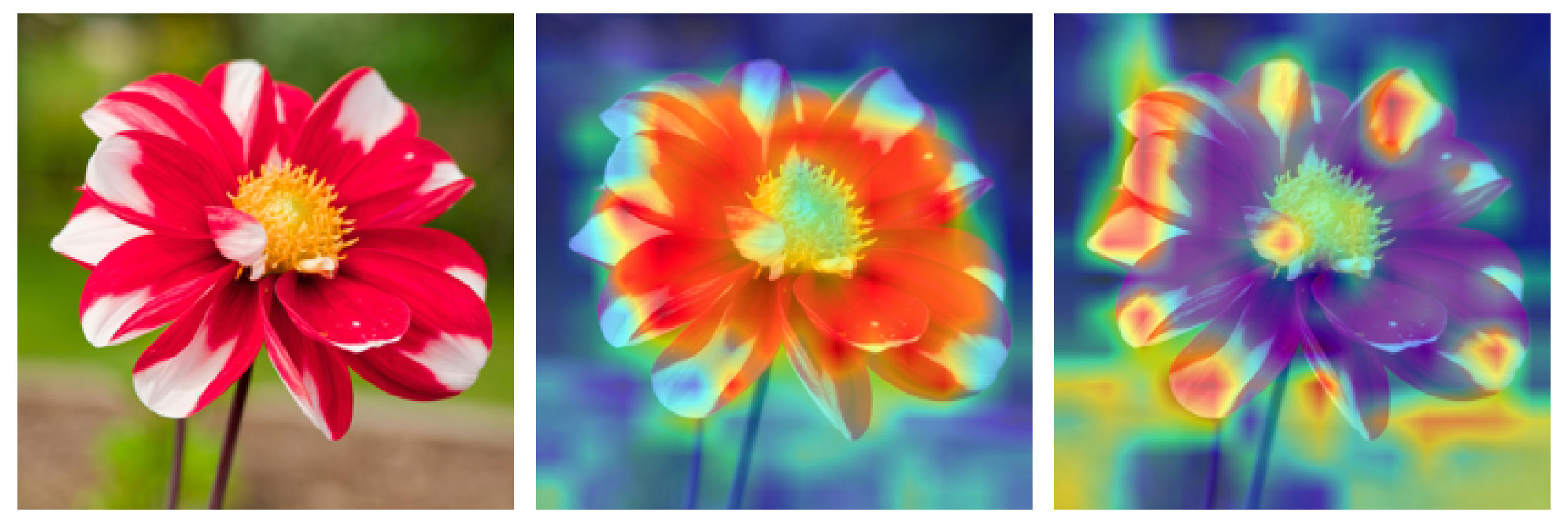}{28} &
					\tile{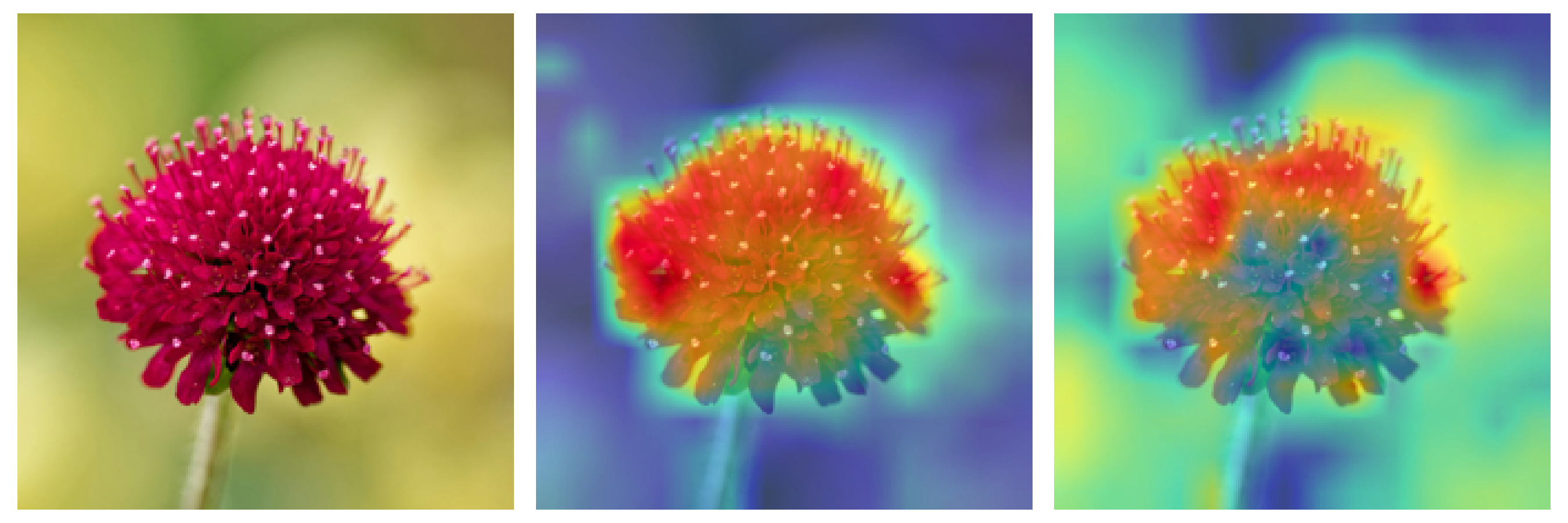}{29} &
					\tile{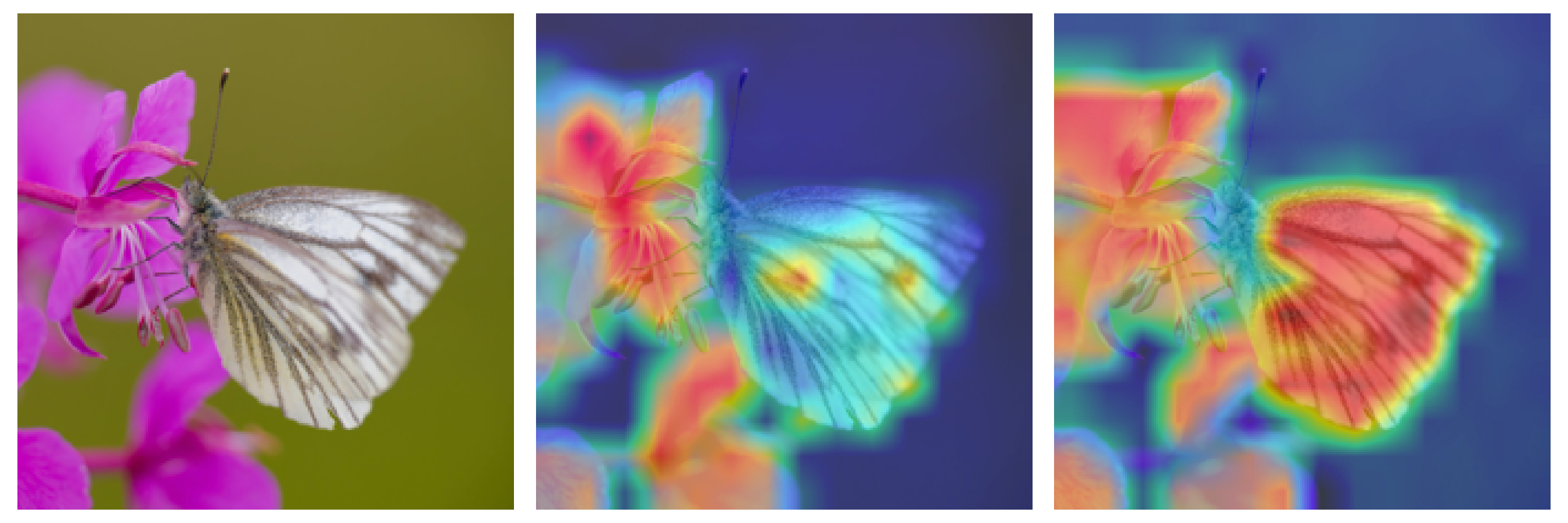}{30} \\
					\tile{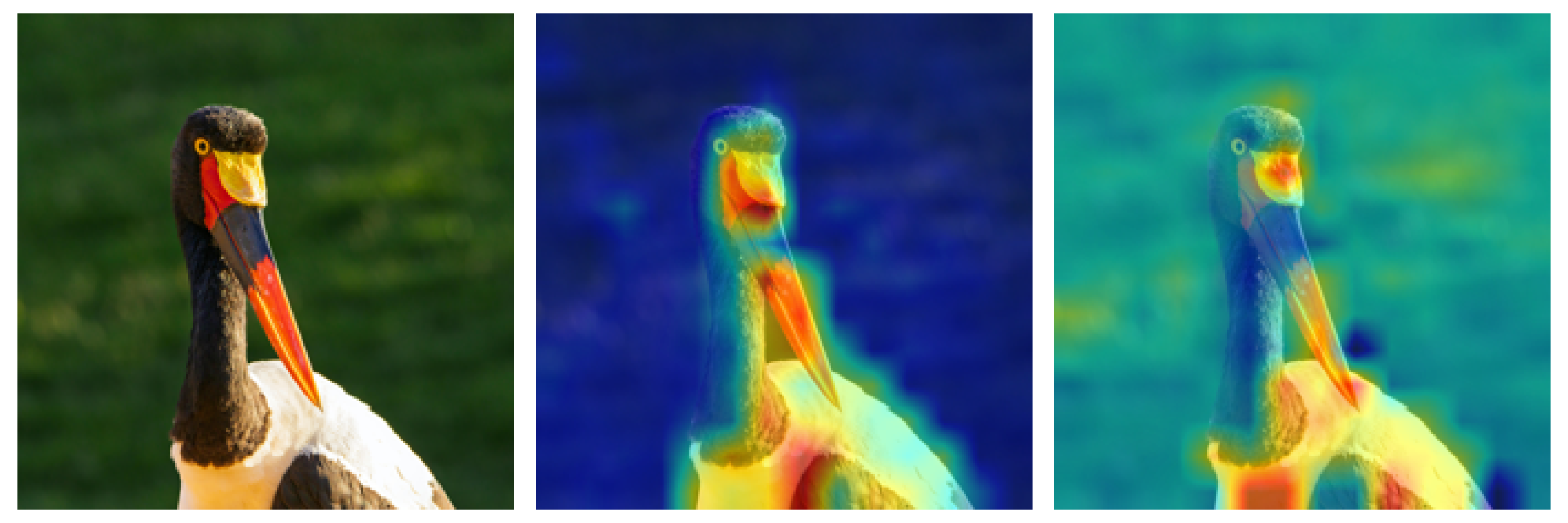}{31} &
					\tile{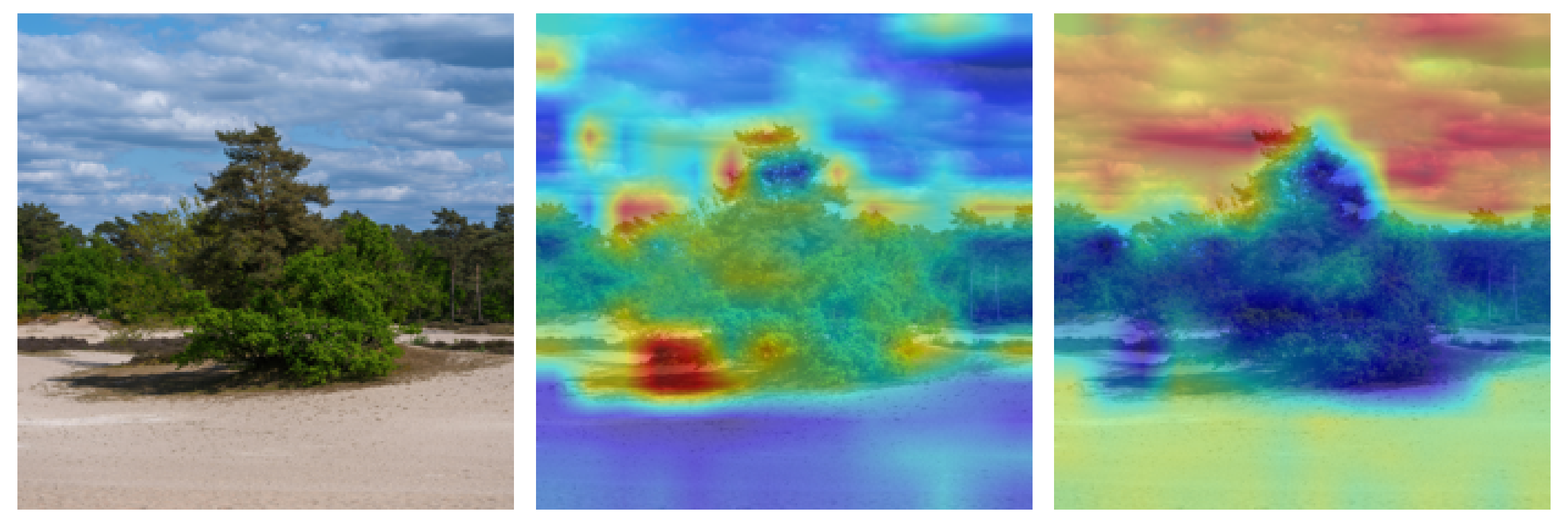}{32} &
					\tile{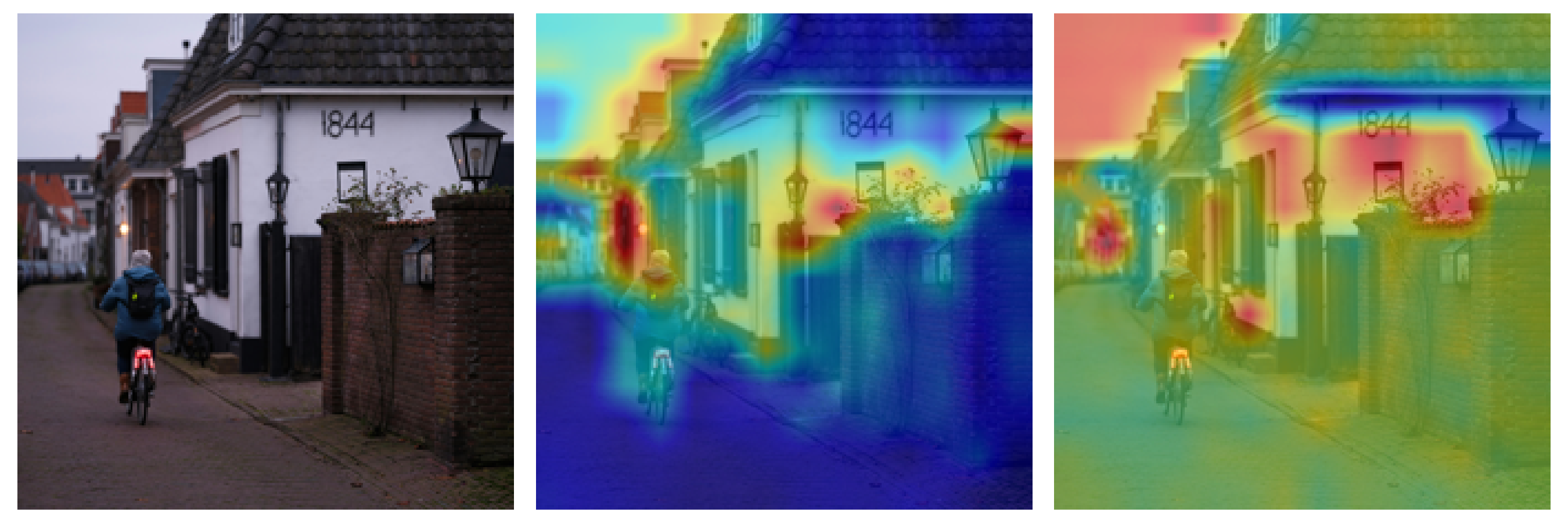}{33} &
					\tile{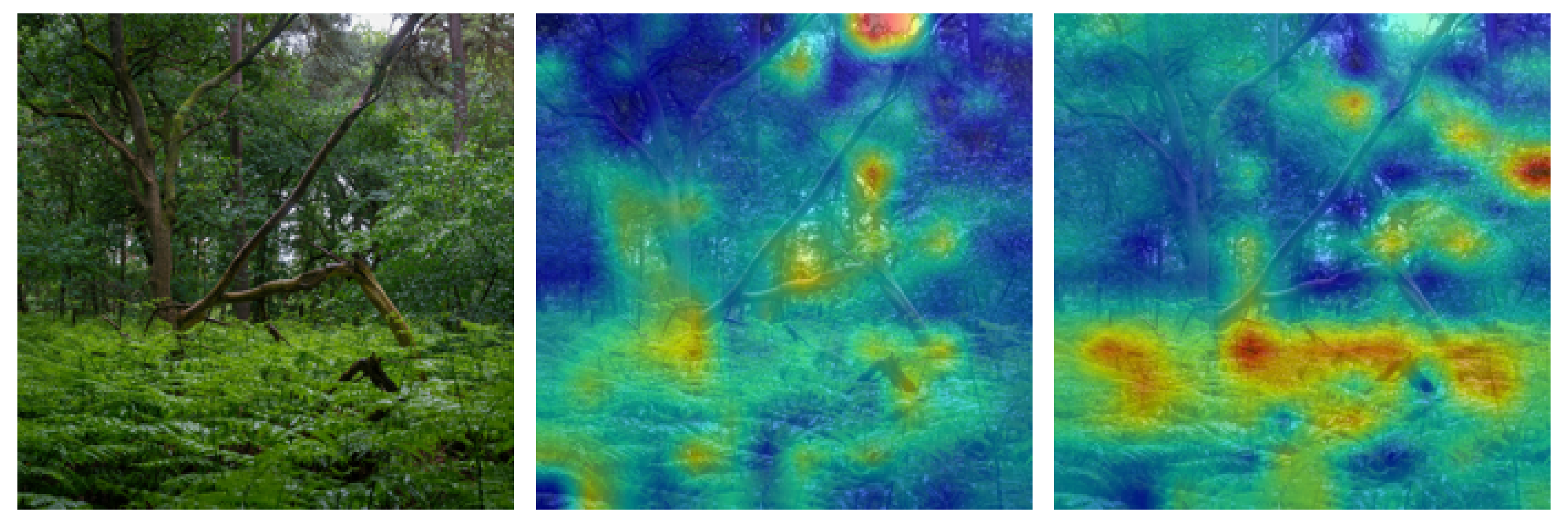}{34} &
					\tile{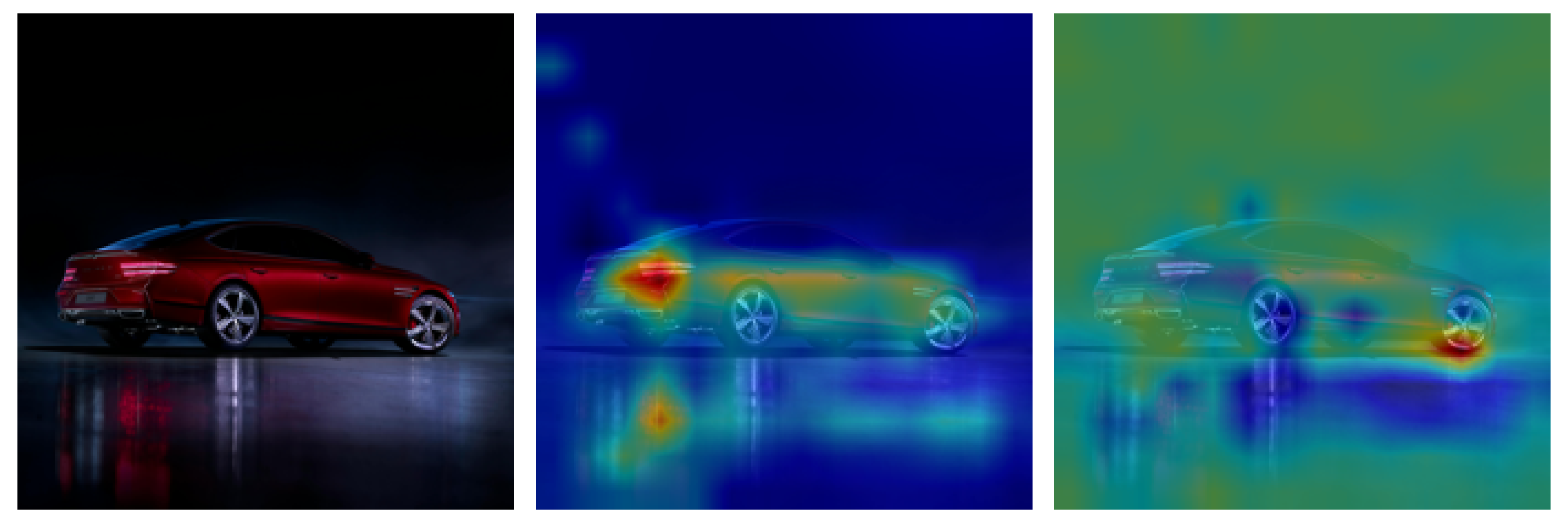}{35} &
					\tile{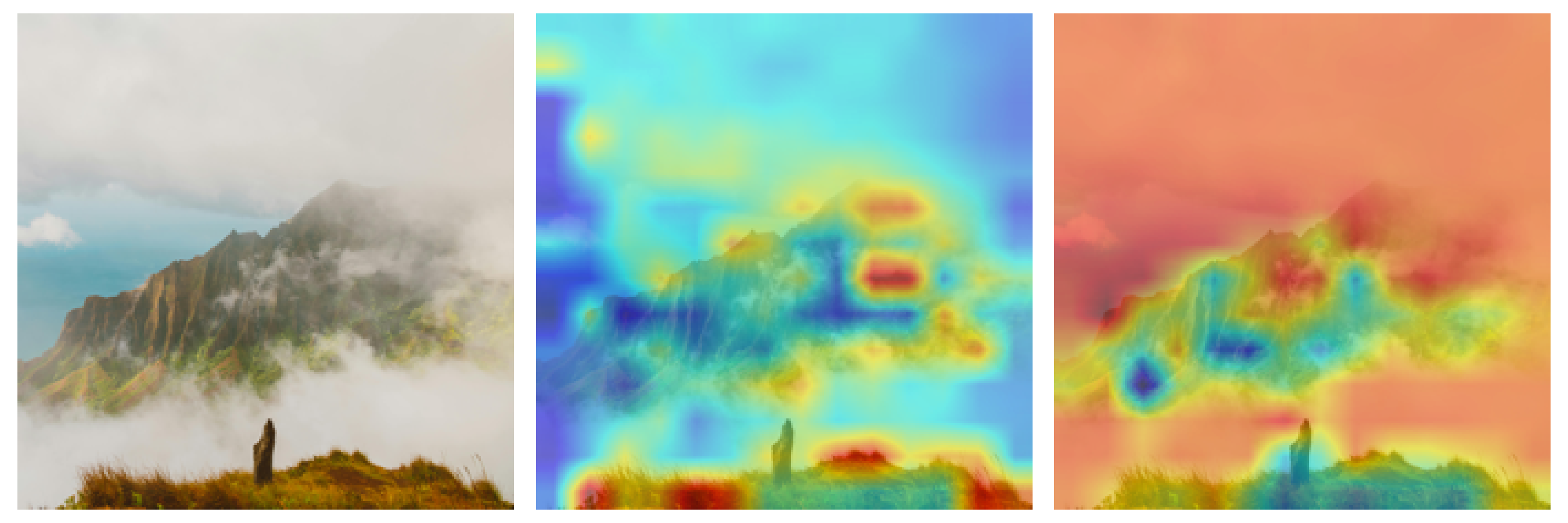}{36} \\
					\tile{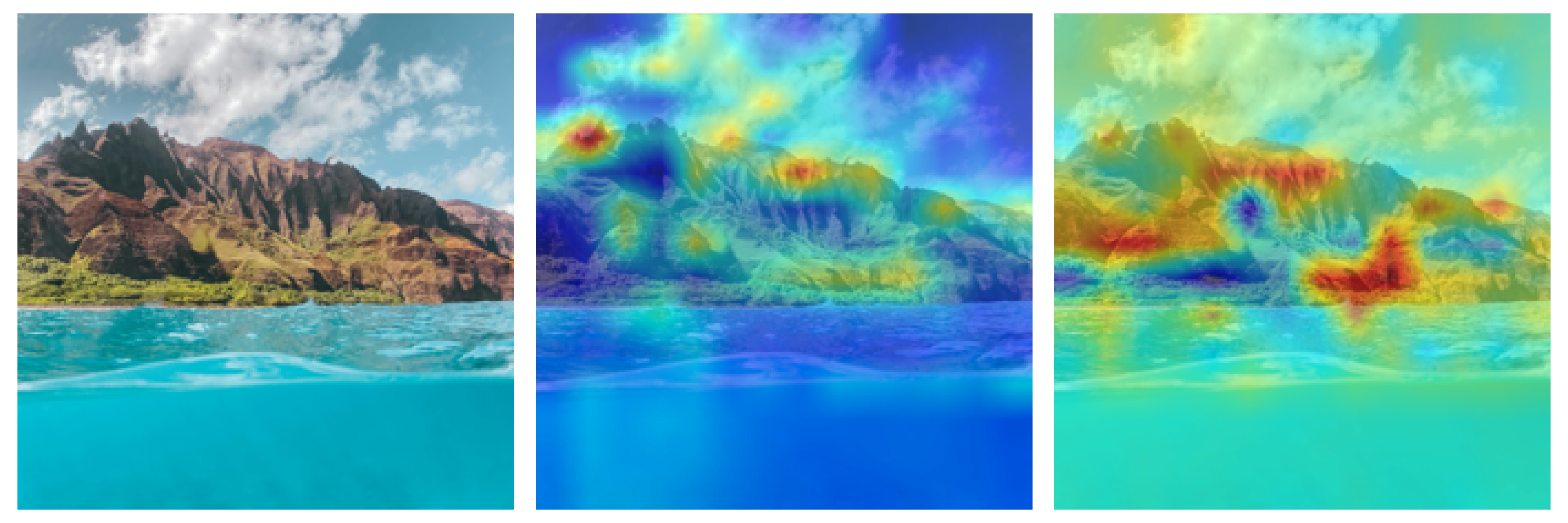}{37} &
					\tile{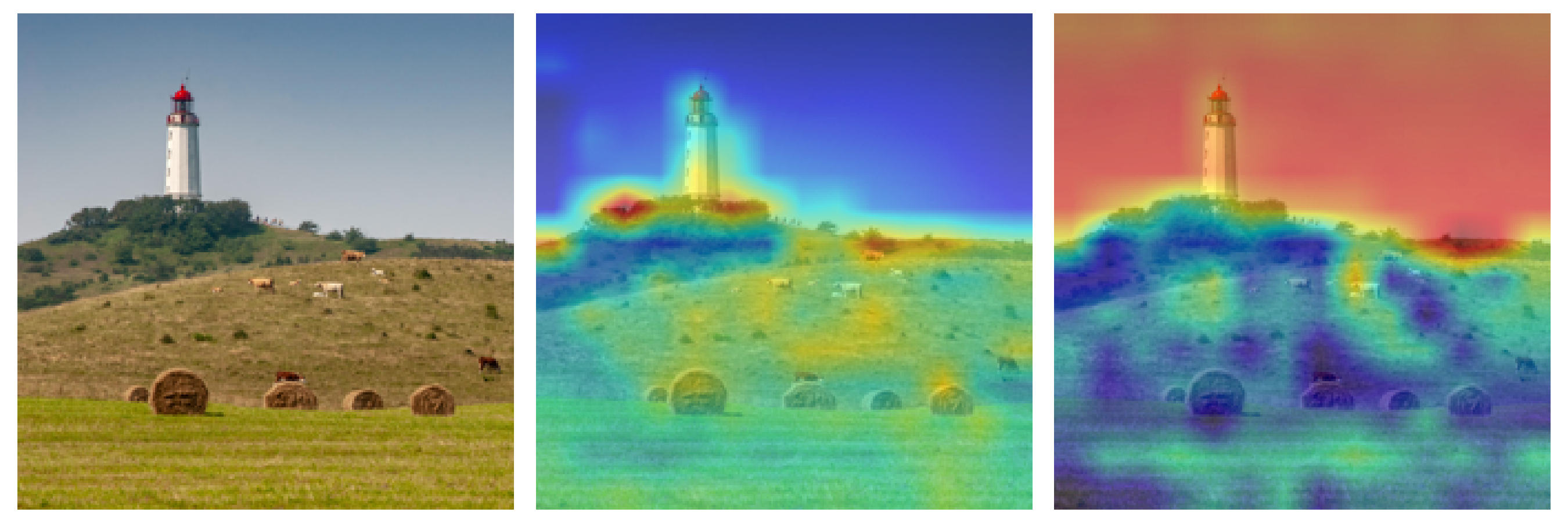}{38} &
					\tile{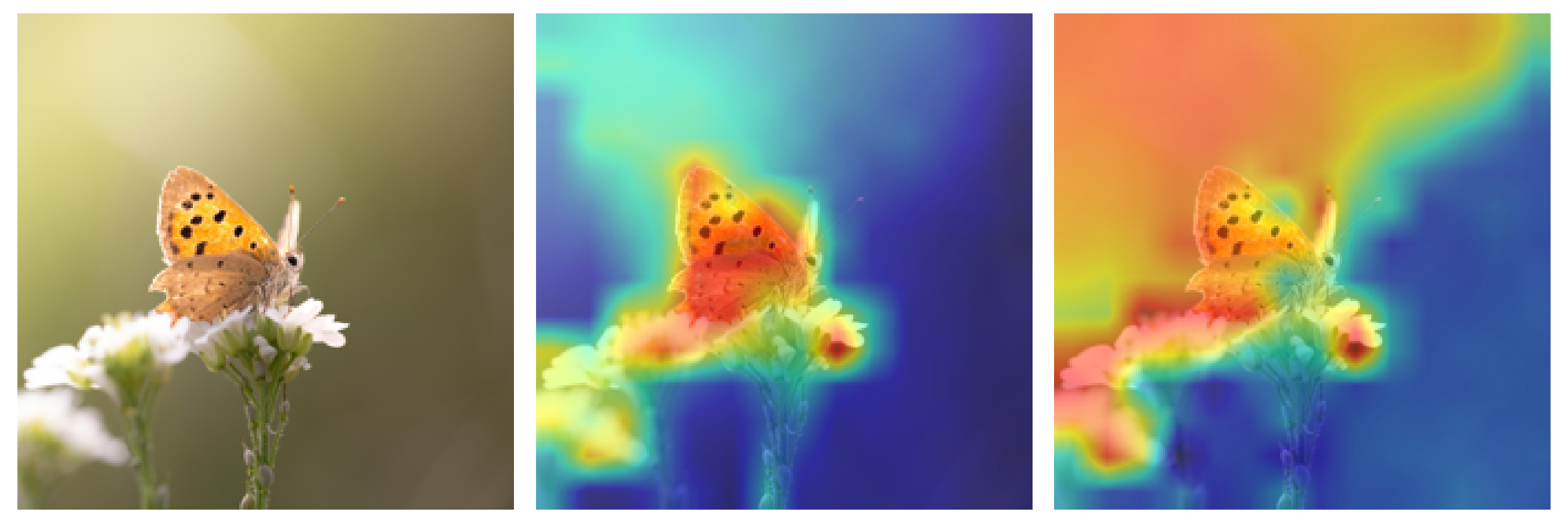}{39} &
					\tile{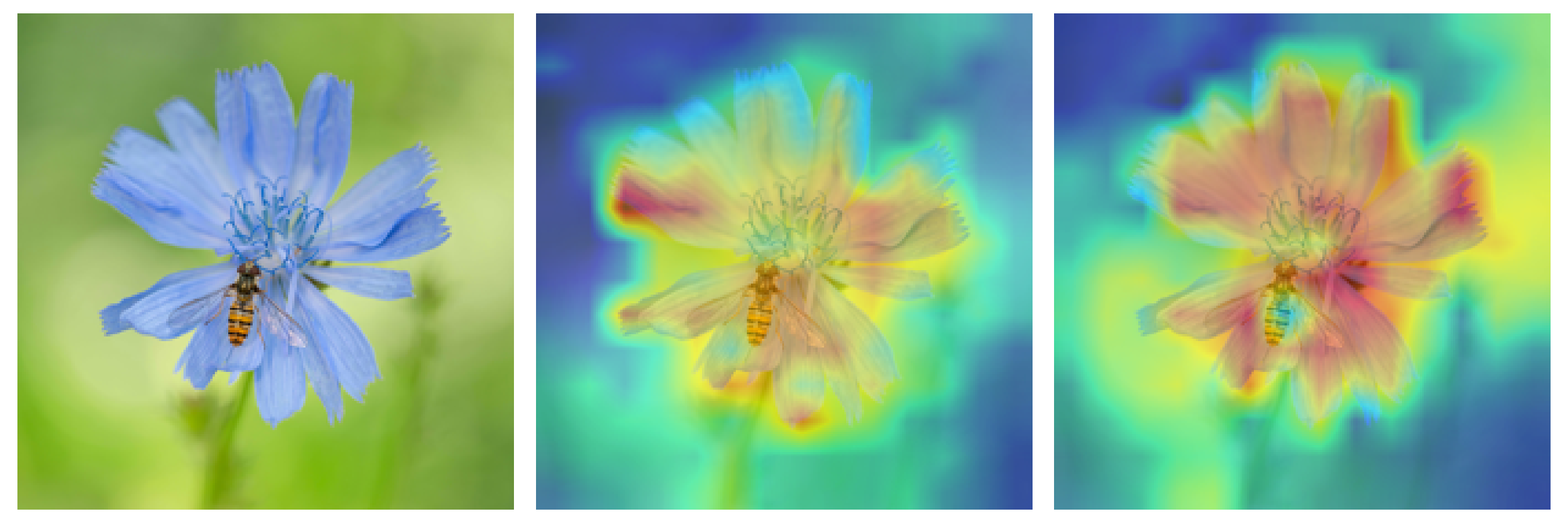}{40} &
					\tile{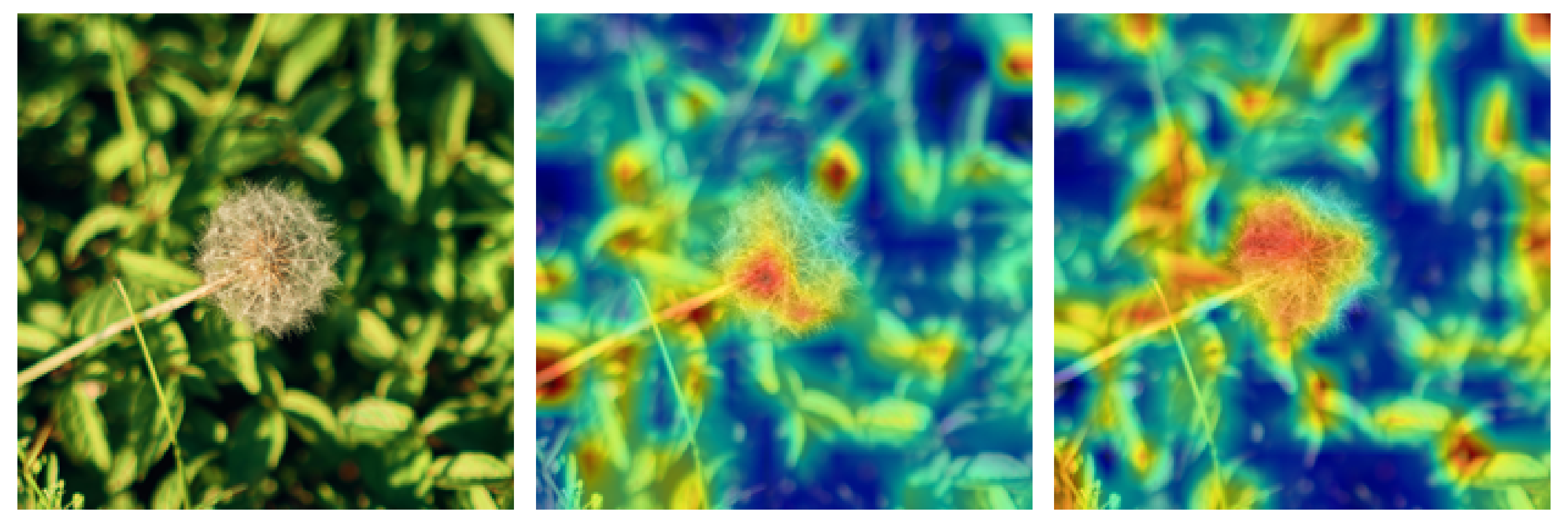}{41} &
					\tile{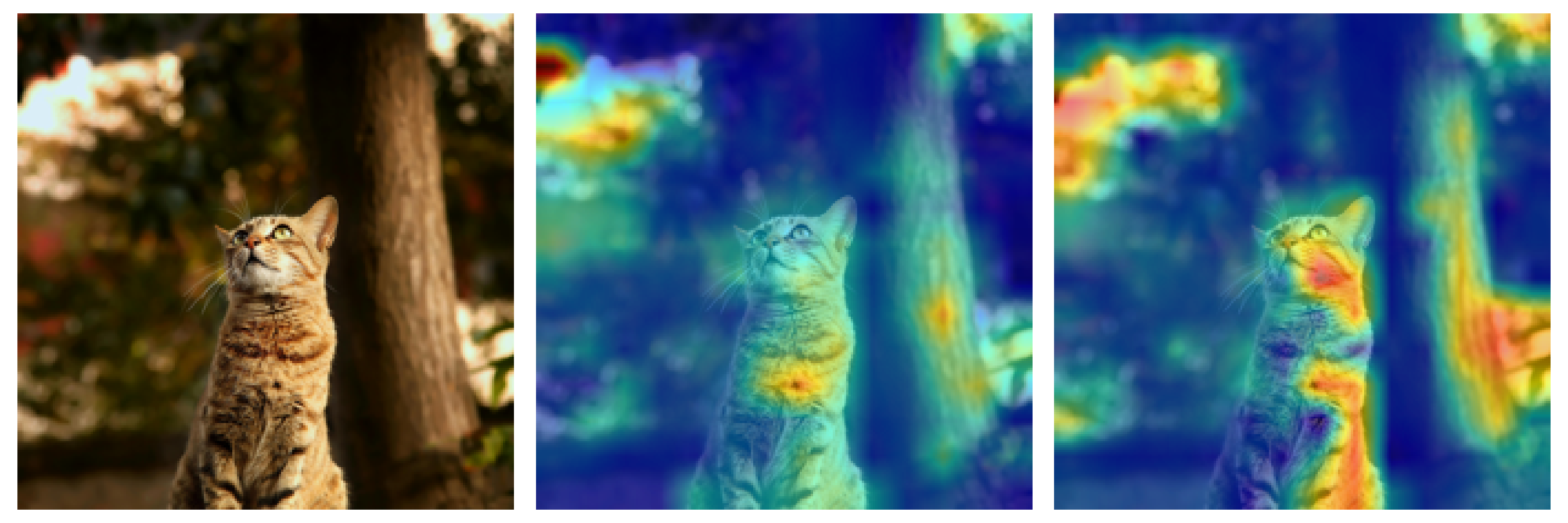}{42} \\
					\tile{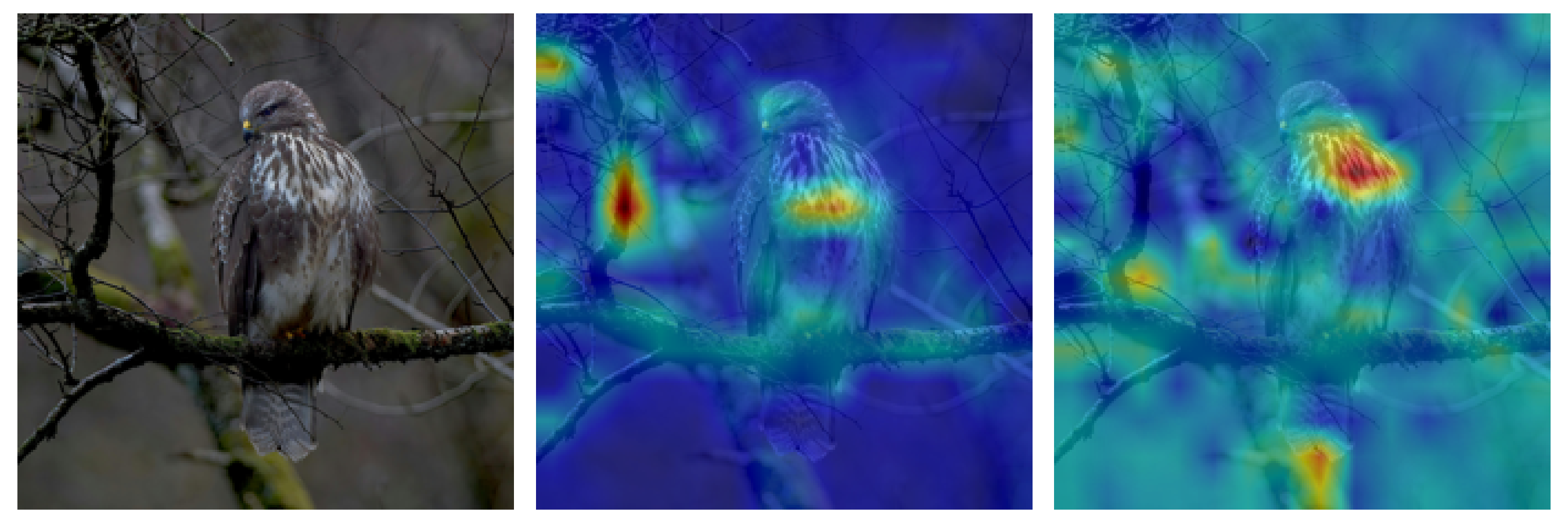}{43} &
					\tile{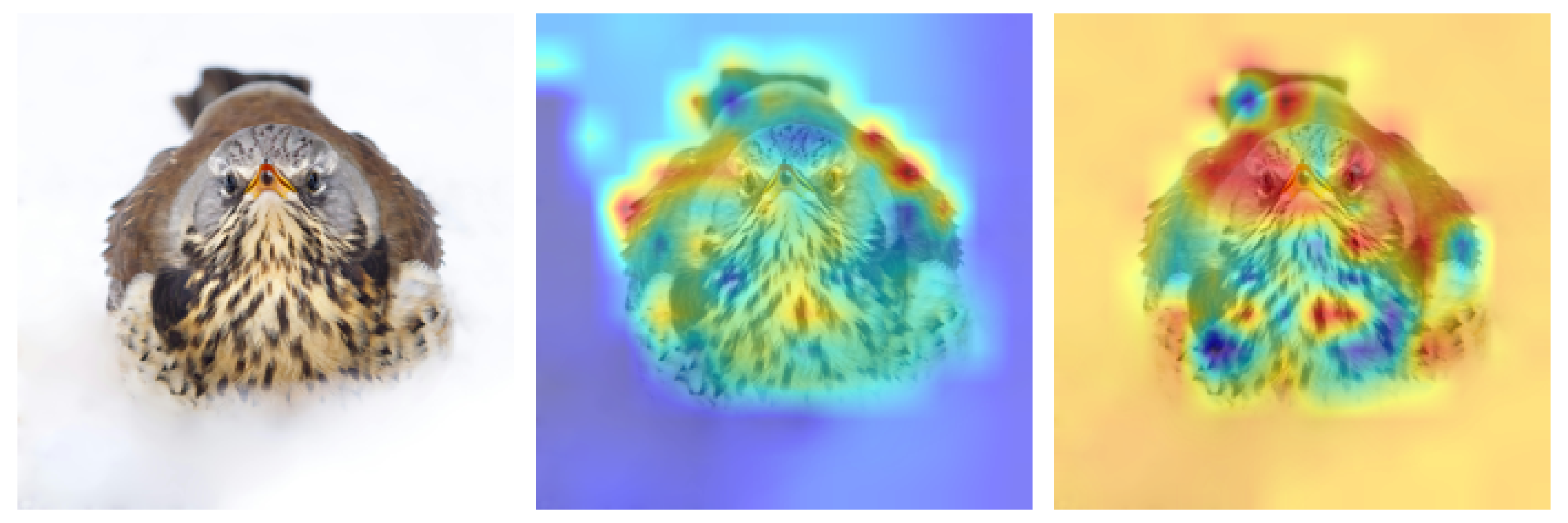}{44} &
					\tile{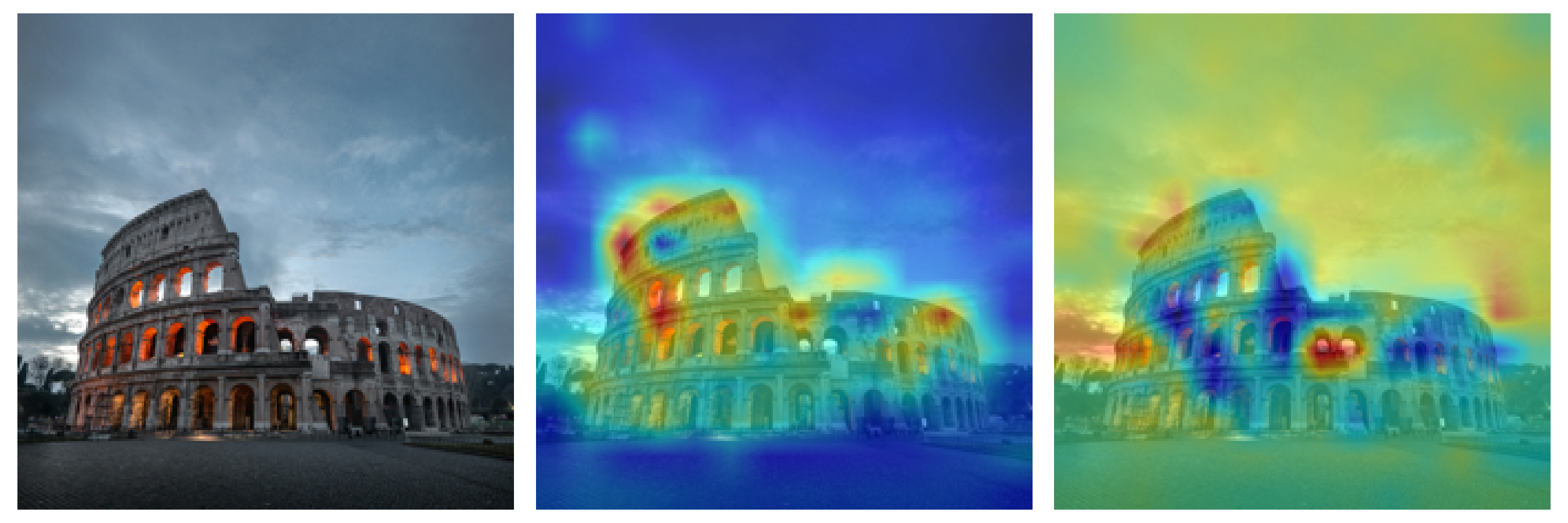}{45} &
					\tile{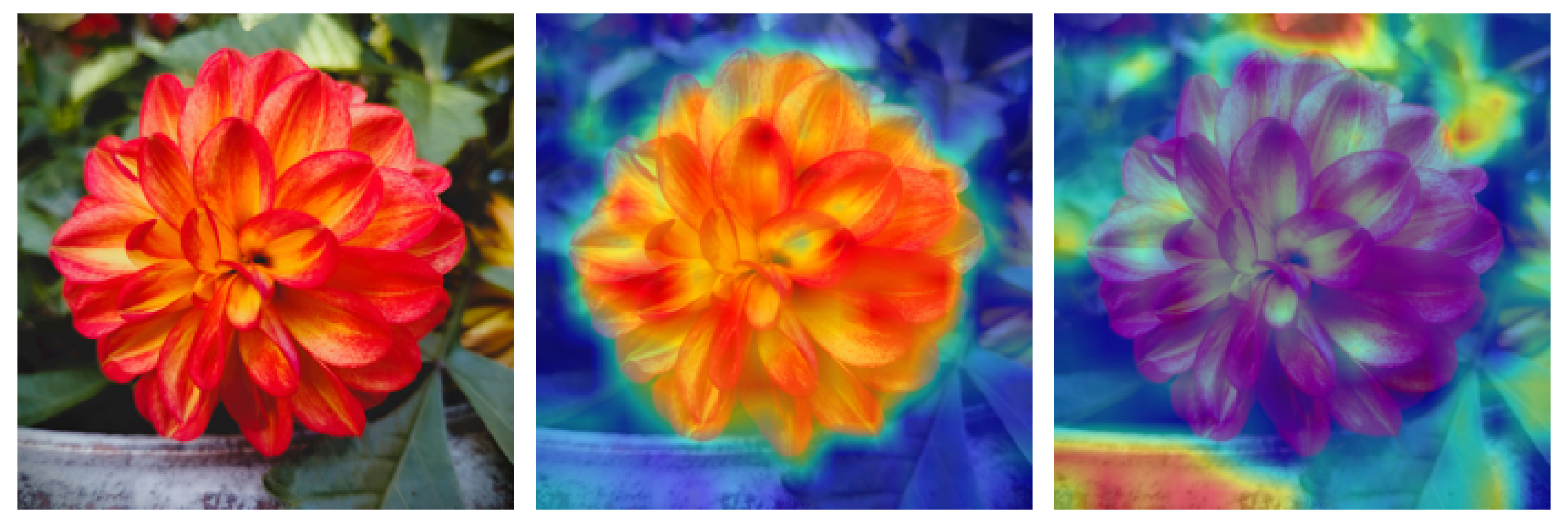}{46} &
					\tile{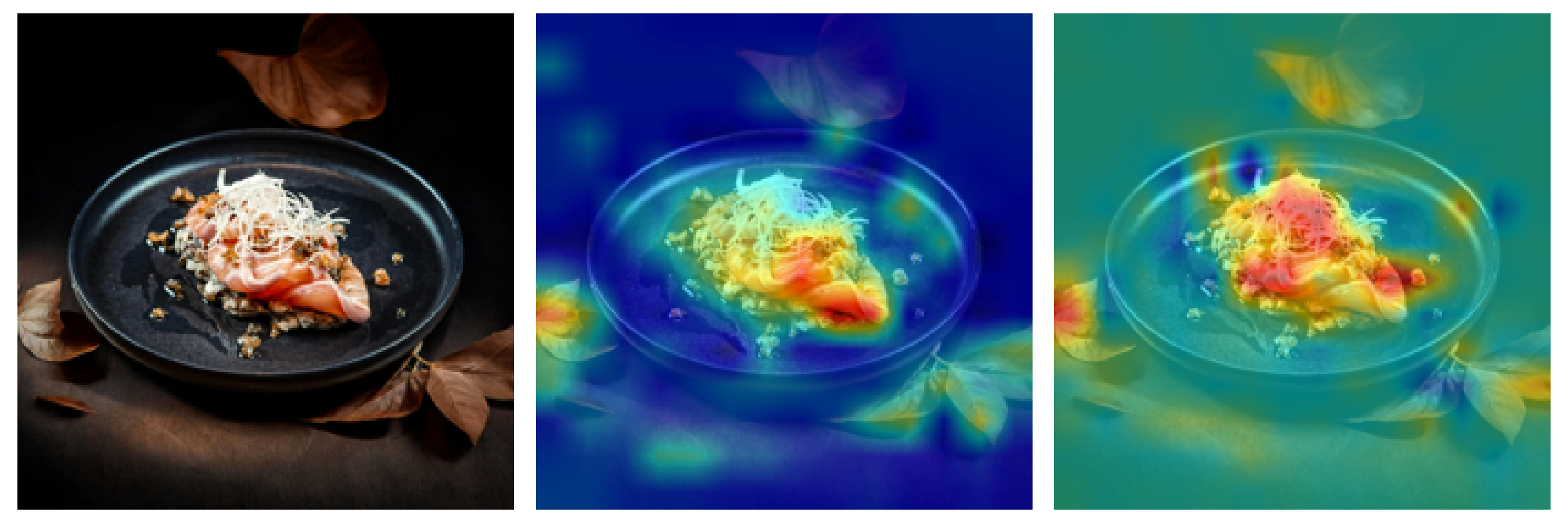}{47} &
					\tile{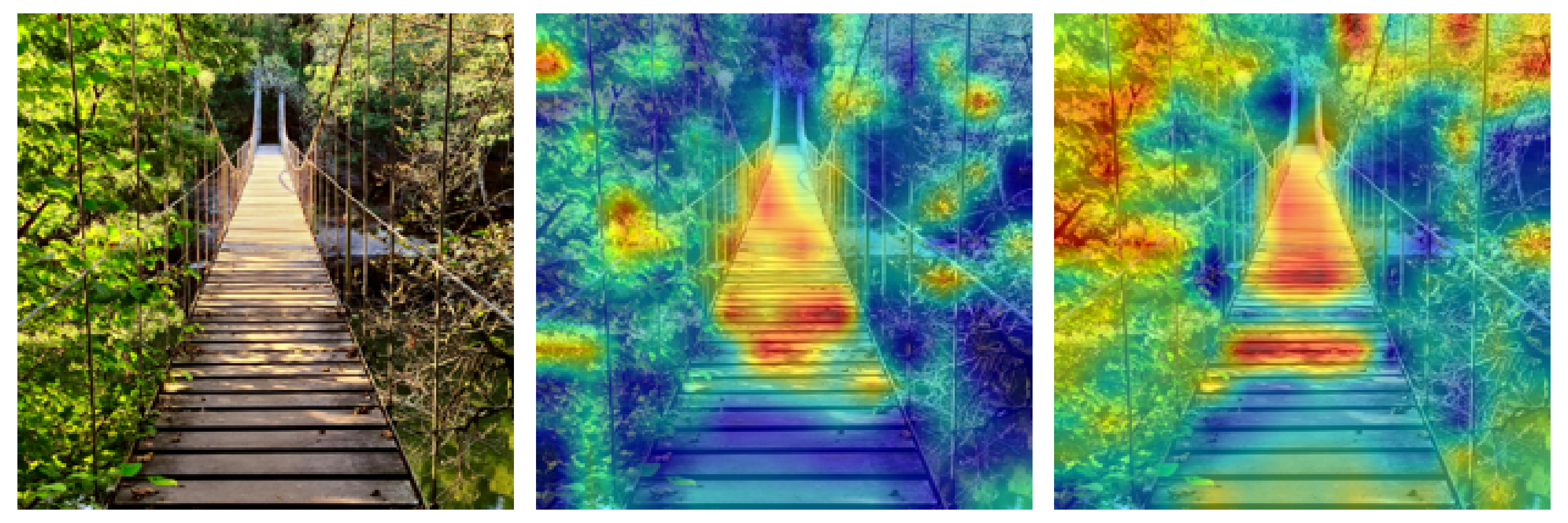}{48} \\
					\tile{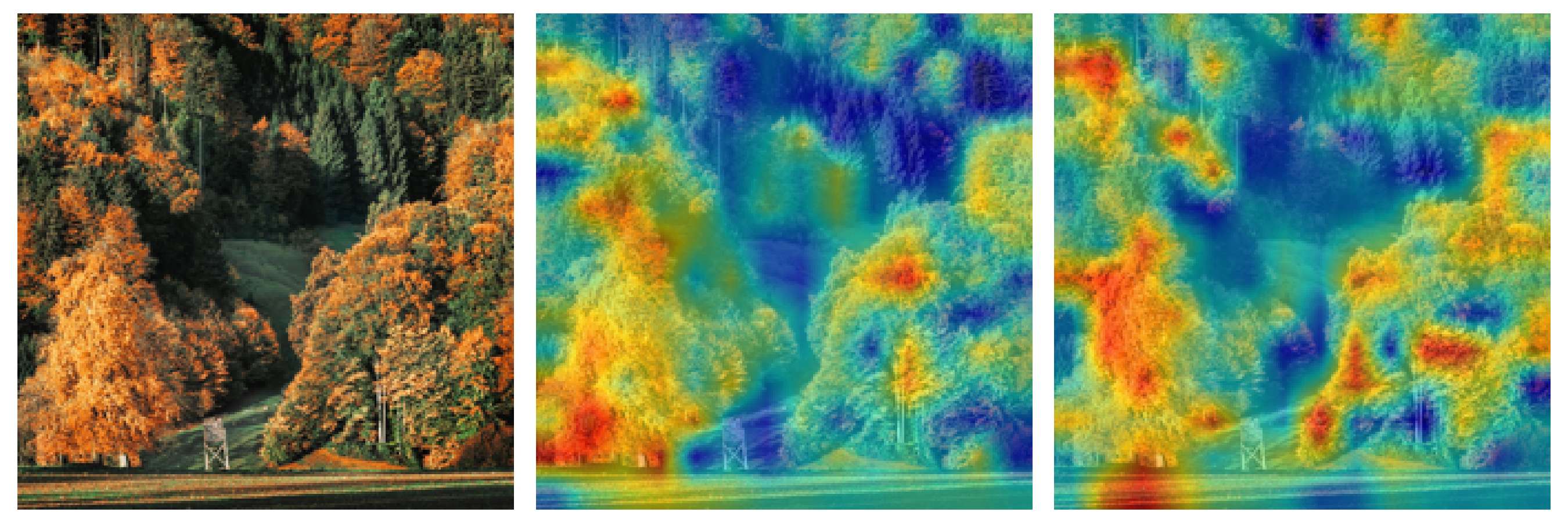}{49} &
					\tile{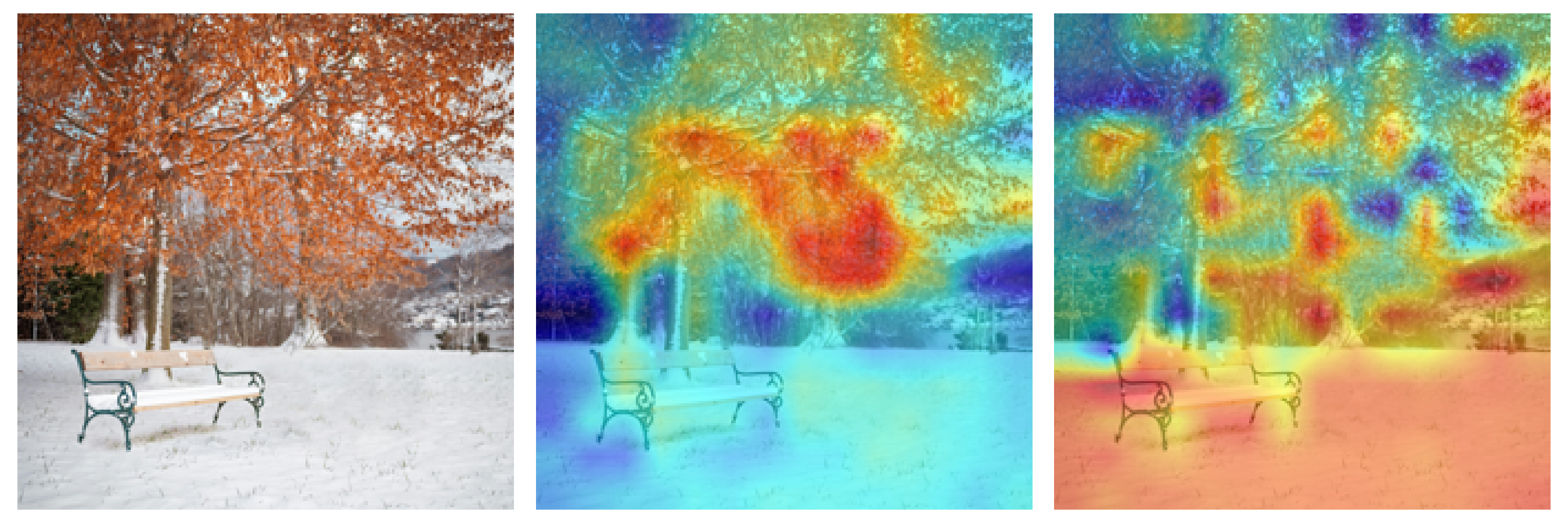}{50} &
					\tile{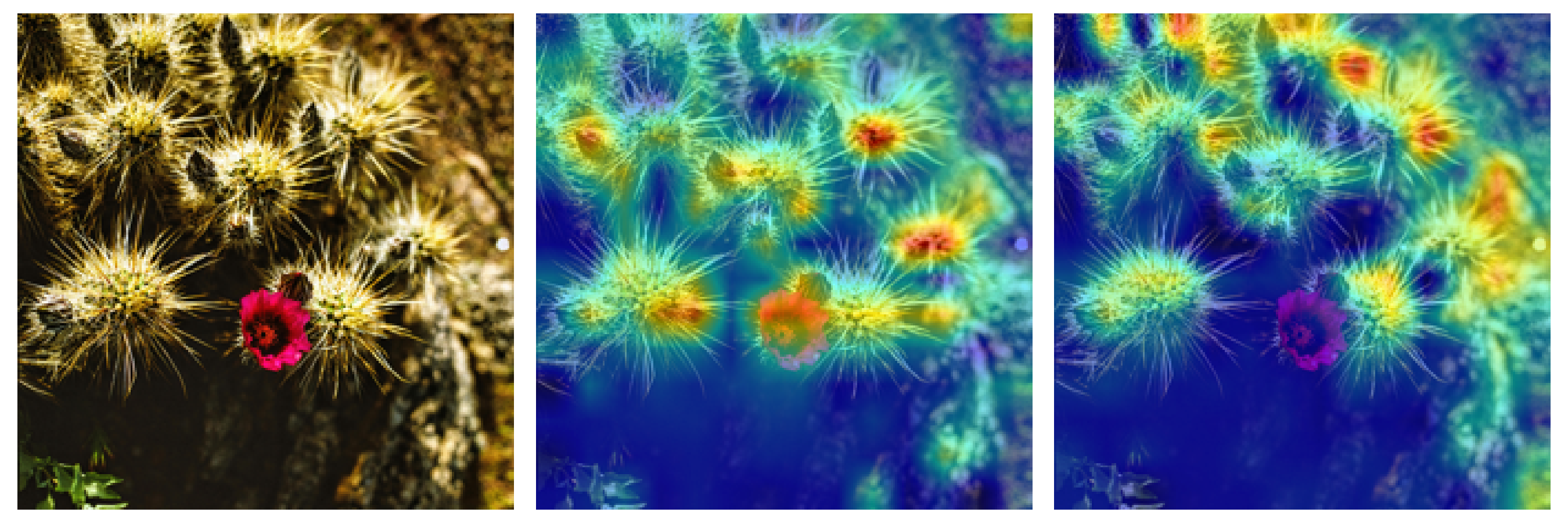}{51} &
					\tile{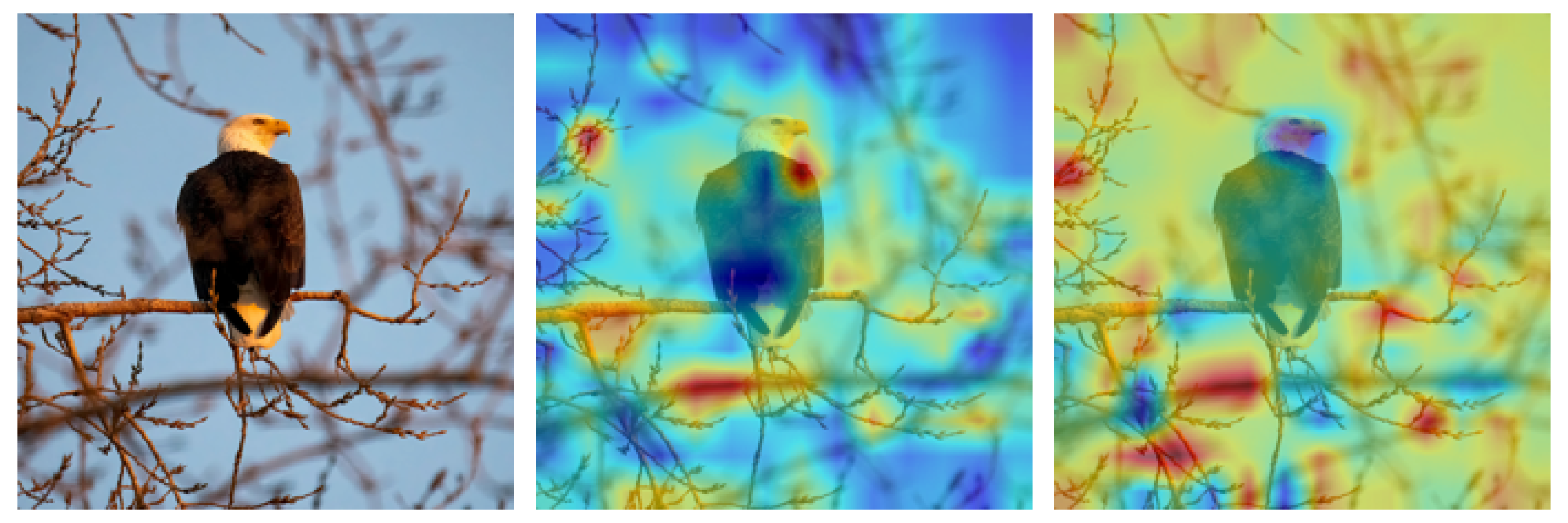}{52} &
					\tile{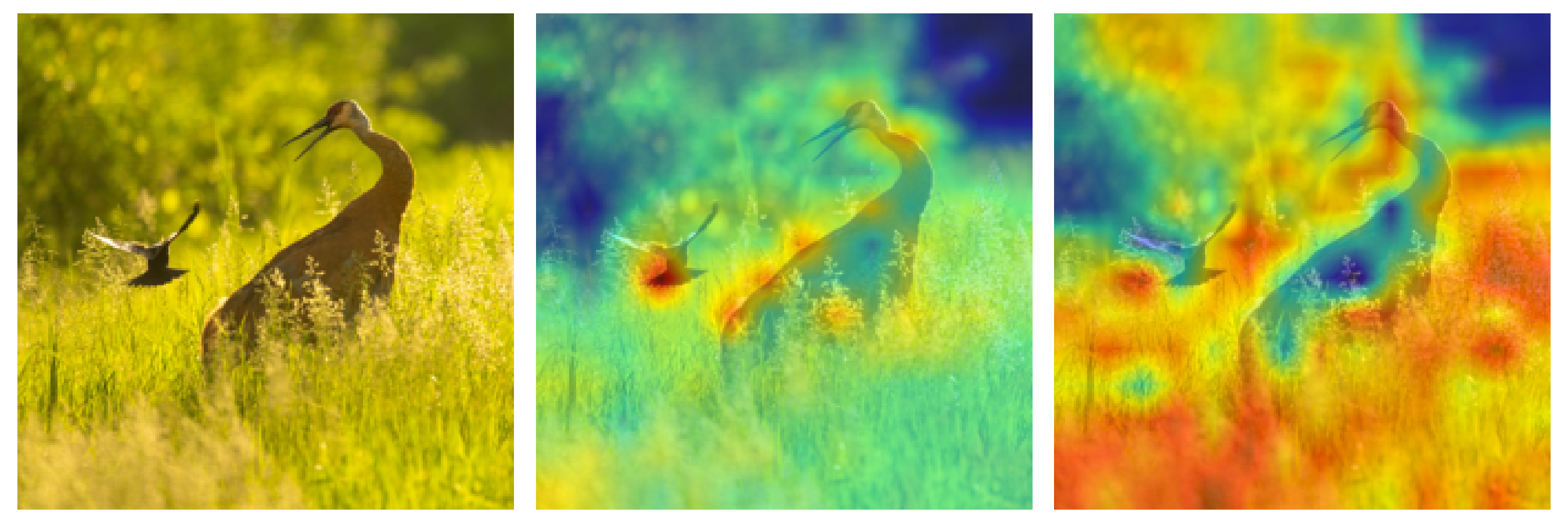}{53} &
					\tile{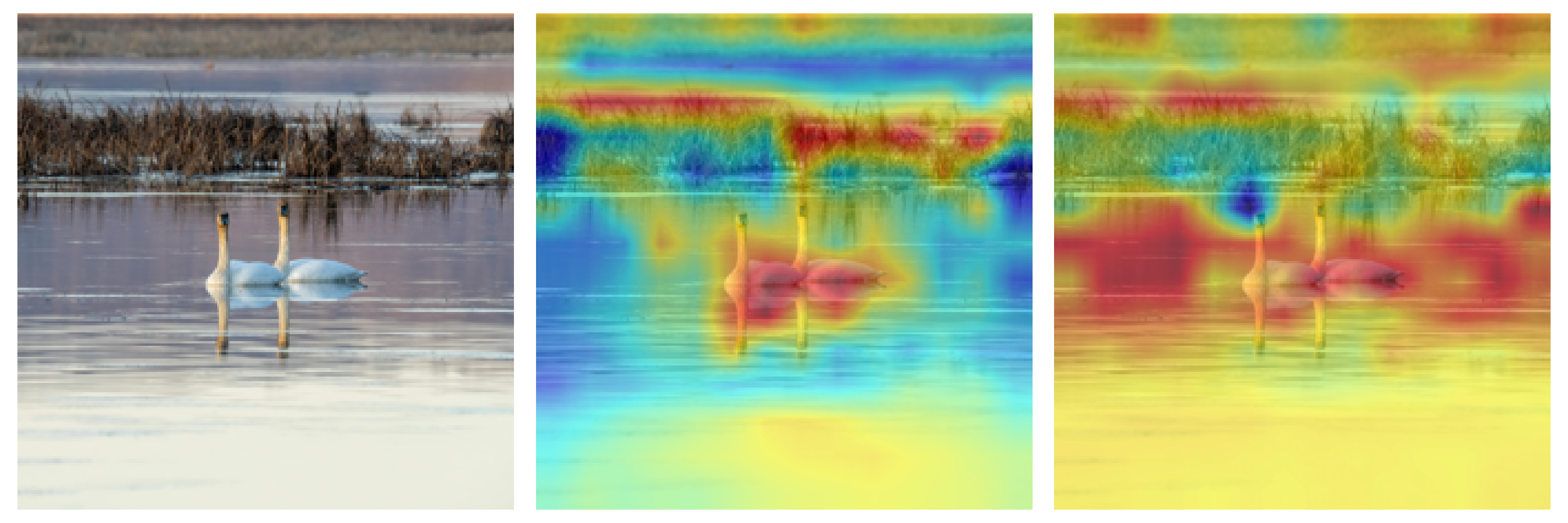}{54} \\
				\end{tabular}
			\end{minipage}
			
			\caption{Further example attention maps as in Figure~\ref{fig:semantic_attention}.}
		\end{figure}
		
		\subsection{Attention Behavior Analysis}		
		\label{sec:attention-behavior-analysis}
		
		To better understand the mechanisms underlying WePE's performance gains, we
		conduct a comprehensive analysis of attention behavior across all 12
		Transformer layers.

		Figure~\ref{fig:attention_maps} visualizes
		self-attention maps from Layer~6, Head~0 for several representative query
		locations. The WePE model (top row) exhibits clearly stronger spatial locality
		than the baseline with standard learnable absolute position embeddings (bottom
		row). For example, when the query lies on a bottom-right patch, WePE
		concentrates attention on a compact set of neighboring patches, whereas the
		baseline distributes attention much more uniformly over the image. This
		qualitative difference indicates that WePE injects an explicit geometric
		inductive bias, encouraging the model to favor local interactions.

		Figure~\ref{fig:positional_similarity} reveals that this locality originates
		from the structure of the positional encodings themselves. The positional
		similarity matrix of WePE displays a pronounced checkerboard pattern that
		directly reflects the doubly periodic lattice induced by the Weierstrass
		$\wp$-function, providing a highly regular spatial prior. Neighboring patches
		exhibit strongly correlated encodings, while distant patches are
		systematically decorrelated. In contrast, the baseline encodings produce a
		much more irregular similarity matrix, lacking such geometric regularity.

		We next quantitatively analyze how attention strength varies with spatial
		distance. Figure~\ref{fig:attention_distance} plots the mean attention
		weight as a function of patch distance for several layers. In early layers,
		WePE shows a much steeper decay of attention with distance than the baseline,
		indicating a significantly stronger locality bias. This distance--decay effect
		gradually relaxes in deeper layers, where both models become more global,
		allowing long-range integration while still preserving the structured
		initialization provided by WePE. These results confirm that the geometric
		properties of WePE are indeed manifested in the learned attention patterns.

		Entropy analysis in Figure~\ref{fig:attention_entropy} further
		quantifies attention focus. We compute the Shannon entropy of each attention
		distribution and average across heads and tokens. In the early layers, WePE
		consistently yields lower entropy than the baseline, corresponding to more
		concentrated and less noisy attention maps that can accelerate feature
		learning. Entropy gradually increases with depth for both models, reflecting
		the expected expansion of the effective receptive field, but the WePE model
		remains consistently more focused.

		Figure~\ref{fig:attention_range} investigates the effective attention
		range. For each query token, we sort keys by spatial distance and find the
		radius required to accumulate a fixed proportion (e.g., 90\%) of the total
		attention mass. In the first layer, WePE attains this threshold within a
		smaller radius than the baseline, indicating a more compact local receptive
		field. In deeper layers, the effective range gradually expands, yielding a
		multi-scale behavior reminiscent of CNNs~\citep{lecun2002gradient}' progressively enlarging receptive
		fields, but achieved here through learned geometric priors rather than
		architectural constraints.

		Finally, Figure~\ref{fig:feature_tsne} visualizes the learned patch
		representations using t-SNE. With WePE, patches that are spatially close in
		the image tend to form coherent clusters in the feature space, demonstrating
		that the spatial priors provided by WePE propagate into the learned
		representations. The baseline model exhibits substantially weaker clustering
		with respect to spatial location, indicating a less structured organization of
		features.
		
		Taken together, these analyses show that WePE's effectiveness stems from a
		mathematically grounded geometric inductive bias that consistently guides
		attention behavior across layers: it induces strong local bias in early
		stages, gradually relaxes to support global integration, and ultimately leads
		to spatially coherent feature representations. This behavior is particularly
		beneficial in limited-data regimes, where strong and well-structured priors
		are crucial for generalization.
		
		\begin{figure}[t]
			\centering
			\includegraphics[width=\linewidth]{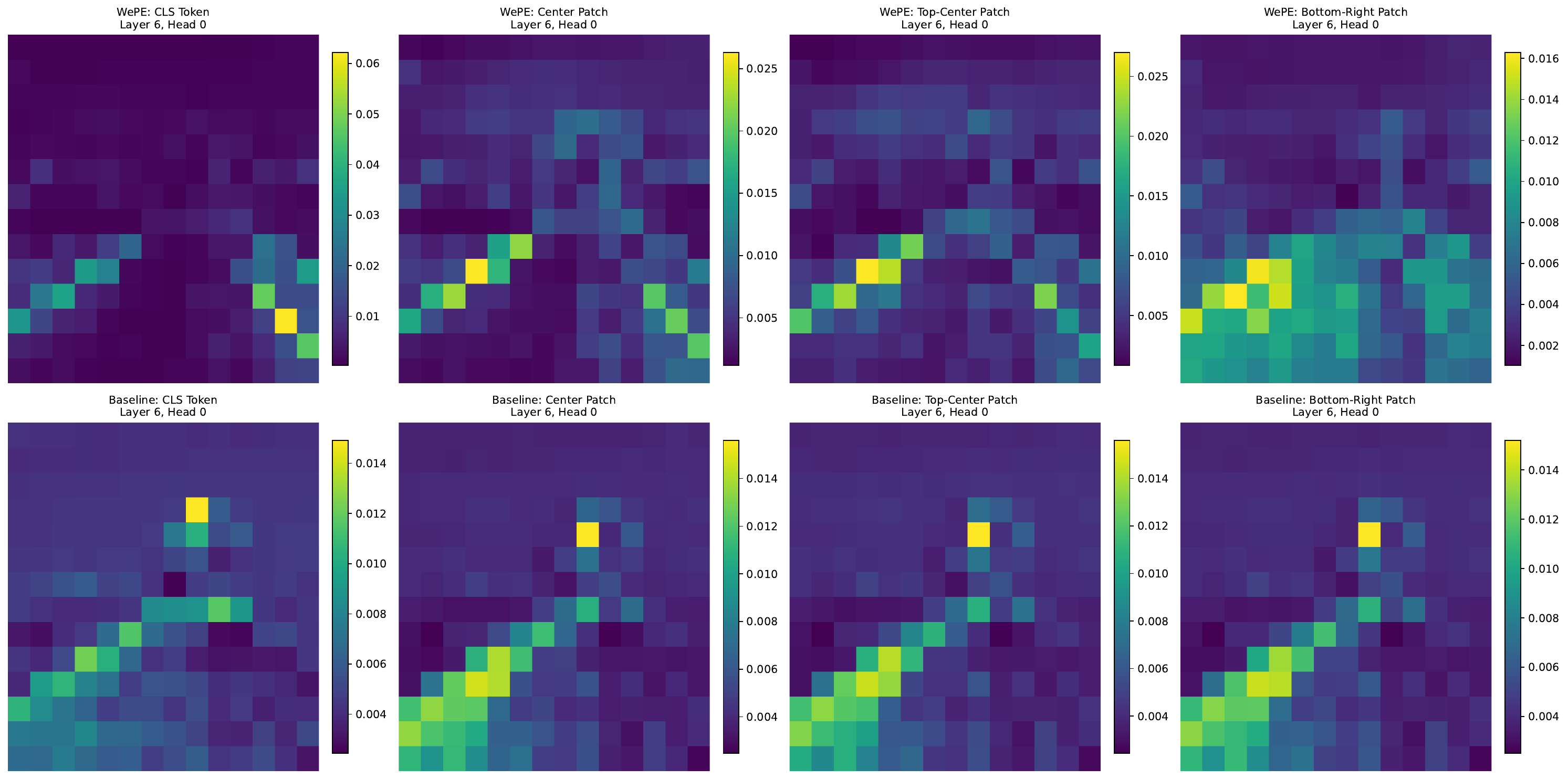}
			\caption{Attention maps comparing WePE (top) and baseline (bottom) at Layer 6, Head 0. WePE shows stronger spatial locality. }
			\label{fig:attention_maps}
		\end{figure}
		
		\begin{figure}[t]
			\centering
			\includegraphics[width=\linewidth]{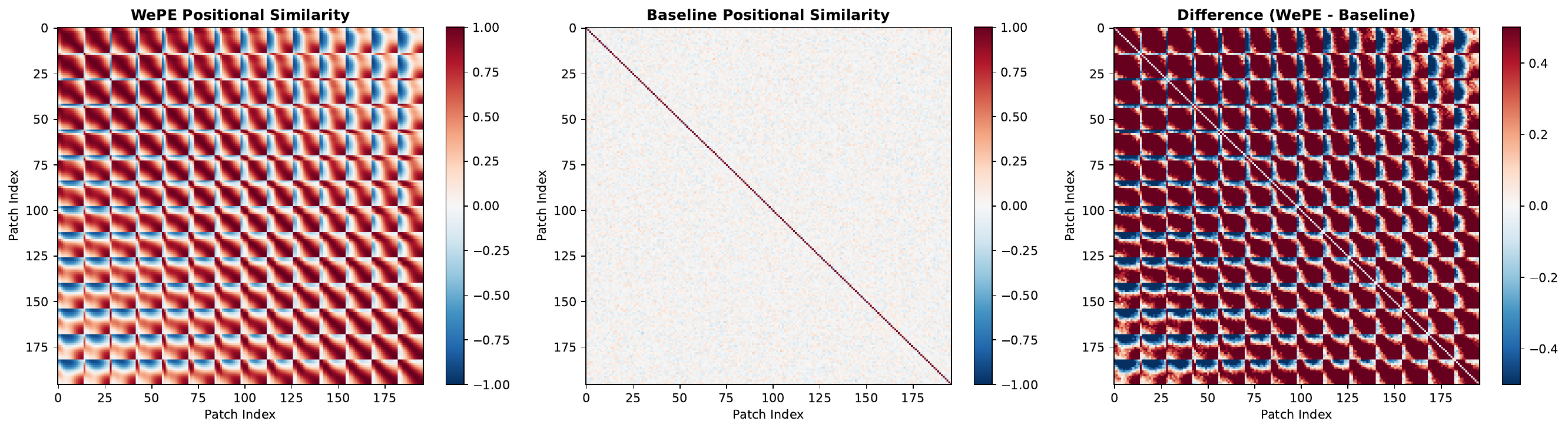}
			\caption{Positional similarity matrices. WePE exhibits checkerboard pattern encoding spatial topology; baseline lacks geometric structure. }
			\label{fig:positional_similarity}
		\end{figure}
		
		\begin{figure}[t]
			\centering
			\includegraphics[width=\linewidth]{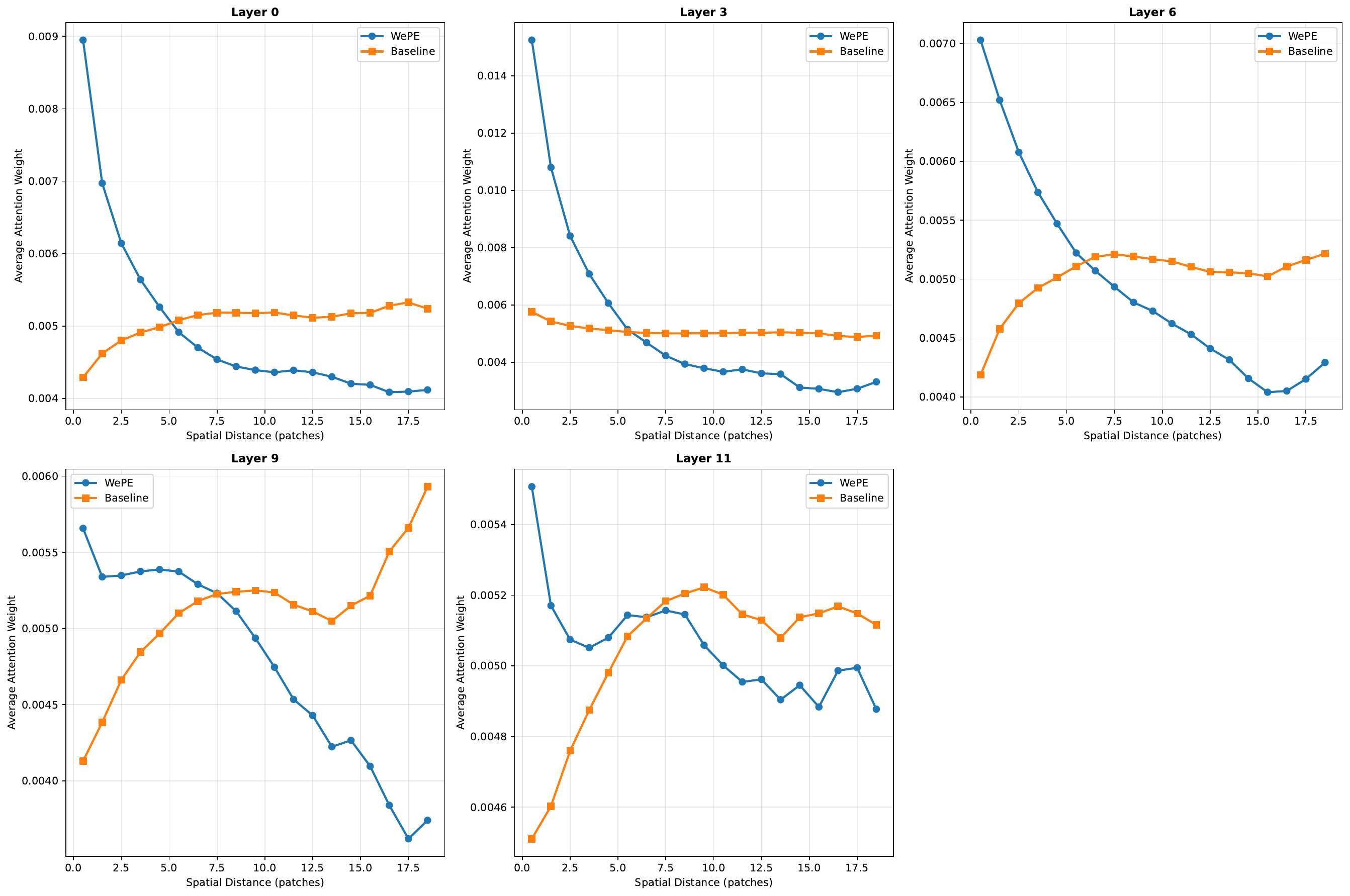}
			\caption{Attention weight vs.\ spatial distance. WePE shows stronger decay in early layers, converging with baseline in deeper layers. }
			\label{fig:attention_distance}
		\end{figure}
		
		\begin{figure}[t]
			\centering
			\includegraphics[width=\linewidth]{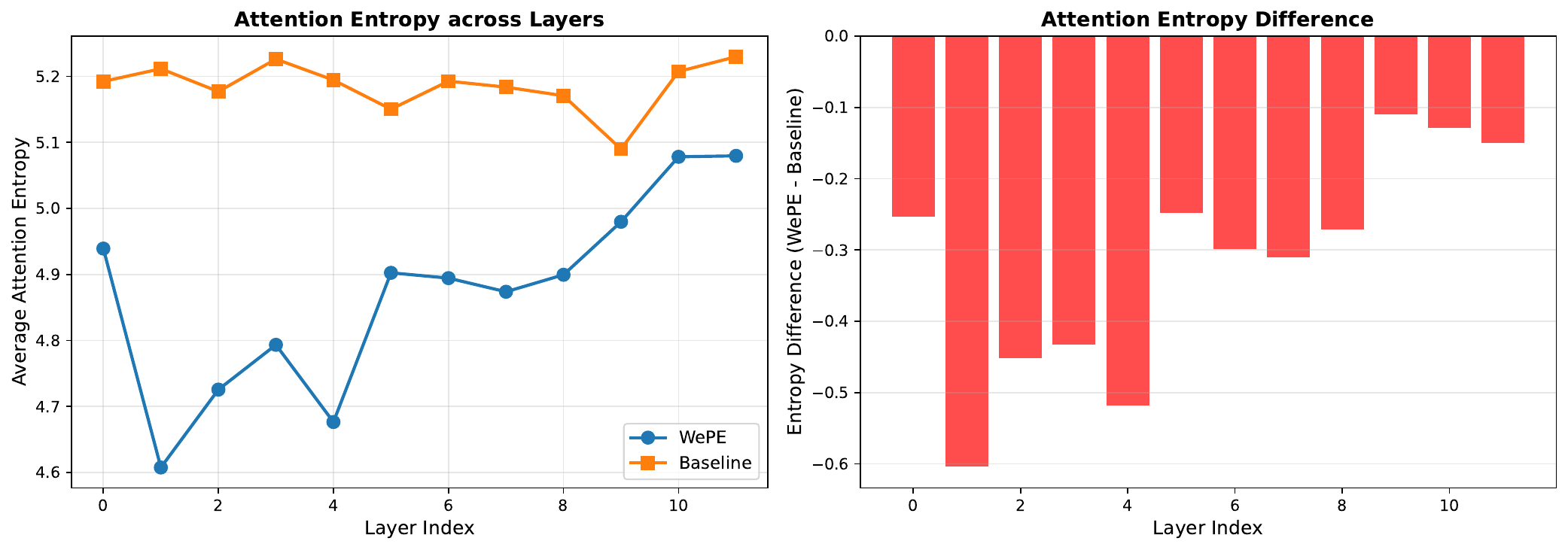}
			\caption{Attention entropy across layers. WePE maintains lower entropy in early layers (4.6-4.9 vs.\ 5.1-5.3), indicating focused attention. }
			\label{fig:attention_entropy}
		\end{figure}

		\begin{figure}[t]
			\centering
			\includegraphics[width=\linewidth]{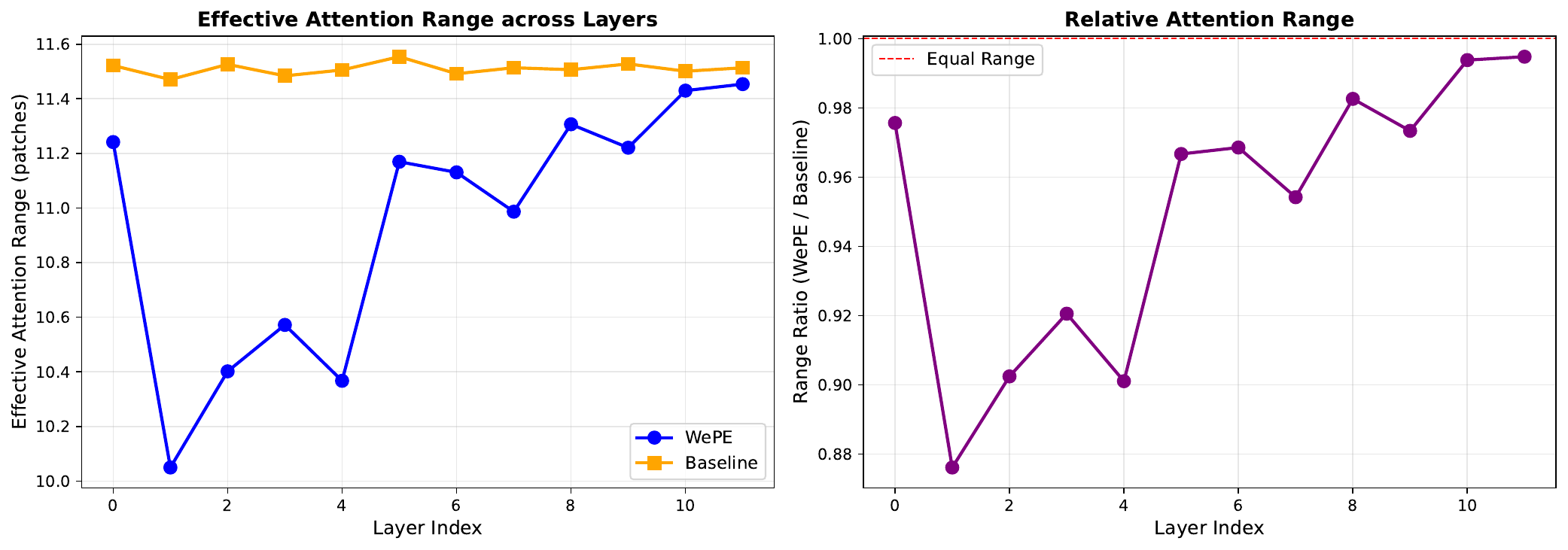}
			\caption{Effective attention range evolution. WePE contracts to 10.0 patches in Layer 1, then expands; baseline remains constant at 11.5 patches. }
			\label{fig:attention_range}
		\end{figure}

		\begin{figure}[t]
			\centering
			\includegraphics[width=\linewidth]{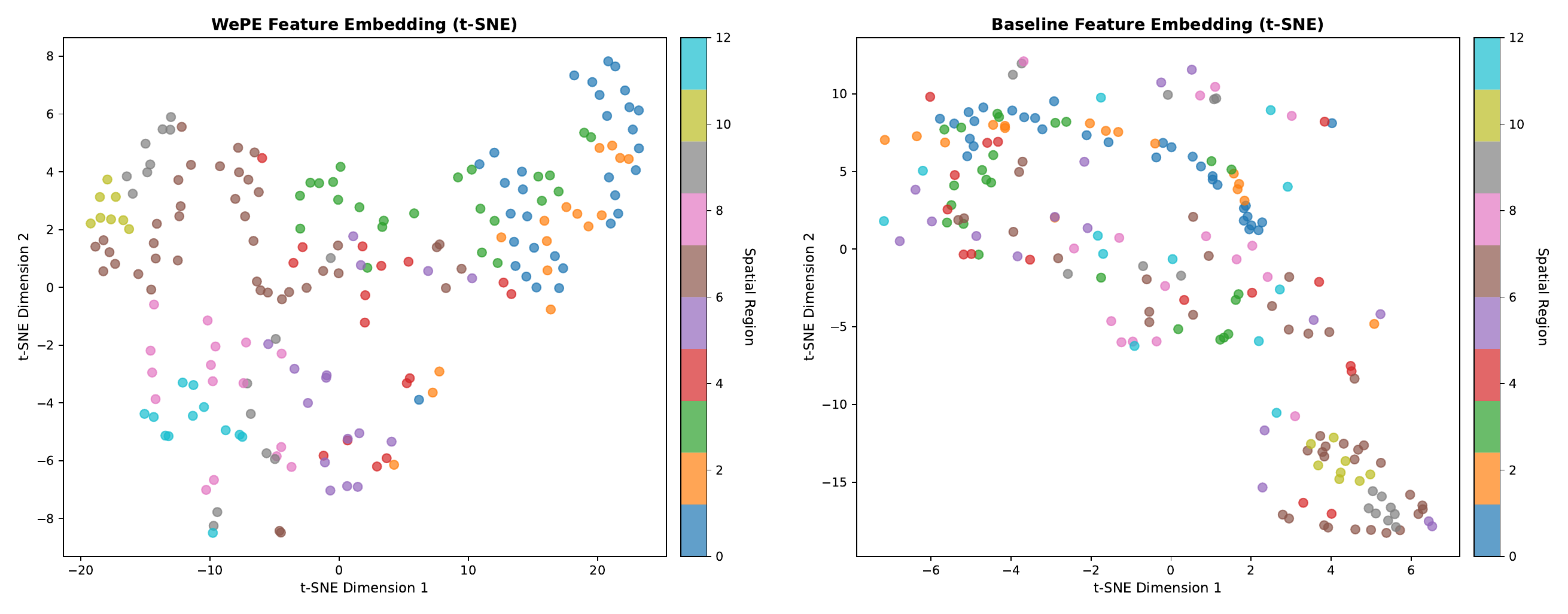}
			\caption{t-SNE of Layer 6 features colored by spatial region. WePE shows clearer spatial clustering than baseline. }
			\label{fig:feature_tsne}
		\end{figure}
		
		\subsection{Extreme aspect ratio input experiment}
		\label{sec:extreme-aspect-ratio}

		To assess how the proposed method behaves when the input aspect ratio departs from the standard 1:1 setting, we retrain the ViT model~\citep{dosovitskiy2020image} with WePE from scratch using several non square resolutions while keeping the architecture, patch size, optimizer, and training schedule identical to the baseline trained on $224 \times 224$ images.

		We consider $224 \times 448$ (wide, 2:1), $448 \times 224$ (tall, 1:2), $672 \times 224$ (very tall, 1:3) and $112 \times 448$ (very wide, 4:1), which change height and width by factors of two to four and induce highly anisotropic patch grids. The aspect ratios evaluated here arguably approach or even surpass the most extreme conditions that models are likely to encounter in real-world applications. We first conducted full 120-epoch training runs for the standard square input 
		($224\times224$, $1{:}1$) and the most extreme aspect ratio ($112\times448$, $4{:}1$). As shown in Figure~\ref{fig:aspect_ratio_curves}, both the training and validation accuracies are already very close to their final converged values by epoch 20, with subsequent training bringing only minor fluctuations and marginal improvements. This indicates that extending the training schedule does not alter the qualitative trends or the relative performance across different aspect ratios. Therefore, for each configuration, we report the best validation accuracy obtained at 10 and 20 epochs, together with the learned lattice parameter~$\alpha_{\text{learn}}$.

		As shown in the table~\ref{tab:aspect_ratio_performance}, After 10 epochs all non square settings already outperform the square baseline (32.82\%), achieving between 36.61\% and 42.72\%. After 20 epochs the models trained with moderate aspect ratios reach 51.86\%, 52.28\%, and 52.83\% accuracy for 2:1, 1:2, and 1:3 respectively, which is on par with or slightly higher than the 51.26\% obtained with 1:1 inputs. Only the extreme 4:1 configuration shows a degradation, which we attribute to the ViT architecture's~\citep{dosovitskiy2020image} inherent sensitivity to extreme patch grids~\citep{touvron2021training, liu2021swin, heo2021rethinking}. This indicates that WePE remains numerically stable for non-square inputs. Performance is negligibly affected by moderate aspect ratios, and even under extreme ratios, the model exhibits a graceful degradation rather than a collapse.

		Crucially, the learned lattice parameter remains in a narrow range across all aspect ratios, with $\alpha_{\text{learn}} \in [0.84, 0.90]$ for both training budgets, indicating that the elliptic lattice adapts automatically to changes in input shape and does not require any manual retuning even when the height to width ratio varies by a factor of four.
		
		\begin{table}[ht] 
			\centering
			\caption{Performance of WePE under different aspect ratios. Baseline uses square input resolution.}
			\label{tab:aspect_ratio_performance}
			\setlength{\tabcolsep}{1.5pt} 
			\tiny 
			\begin{tabularx}{\columnwidth}{@{}l l c c c c c@{}}
				\hline
				\textbf{Experiment} & \textbf{Size} & \textbf{Aspect Ratio} &
				\textbf{Acc (10 ep)} & \textbf{$\alpha_{\text{learn}}$ (10 ep)} &
				\textbf{Acc (20 ep)} & \textbf{$\alpha_{\text{learn}}$ (20 ep)} \\
				\hline
				baseline     & 224$\times$224 & 1:1         & 32.82\% & --     & 51.26\% & --     \\
				wide   & 224$\times$448 & 2:1           & 42.03\% & 0.8647 & 51.86\% & 0.8730 \\
				tall   & 448$\times$224 & 1:2           & 42.61\% & 0.8792 & 52.28\% & 0.9029 \\
				very tall   & 672$\times$224 & 1:3     & 42.72\% & 0.8605 & 52.83\% & 0.8533 \\
				very wide   & 112$\times$448 & 4:1     & 36.61\% & 0.8449 & 46.77\% & 0.8379 \\
				\hline
			\end{tabularx}
		\end{table}
		
		\begin{figure}[t]
			\centering
			\includegraphics[width=\linewidth]{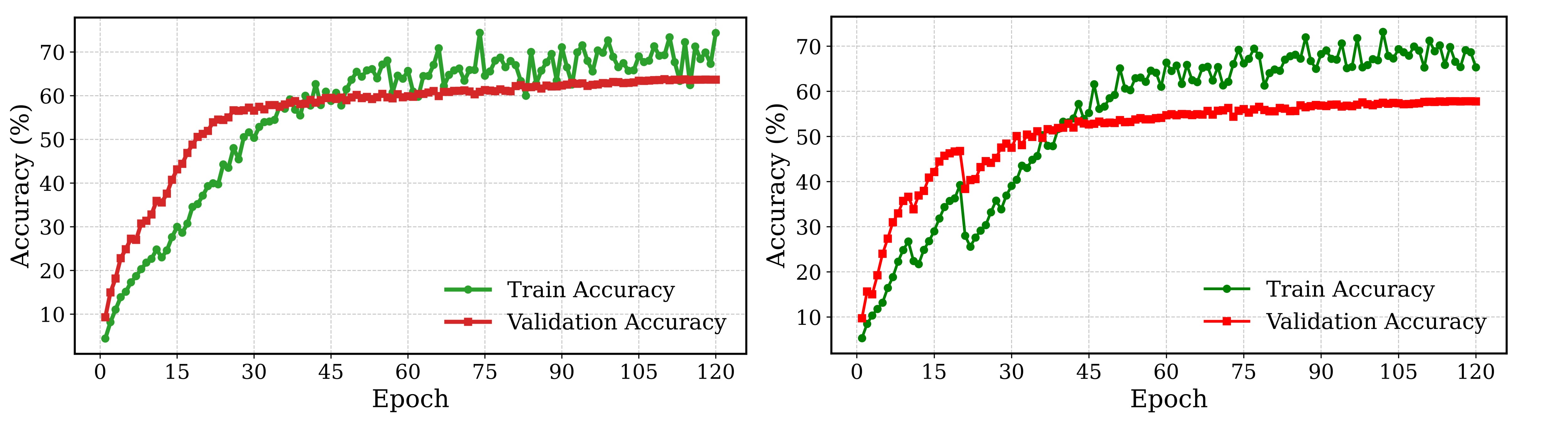}
			\caption{Training and validation accuracy curves for the model trained for the full 120 epochs under standard 1:1 input resolution (left) and extreme 4:1 resolution (right). In both cases, the validation accuracy is already close to convergence by epoch 20, with subsequent improvements being only marginal. }
			\label{fig:aspect_ratio_curves}
		\end{figure}

		\subsection{Sensitivity of the Fourier-like Parameters $\beta$ and $\gamma$}
		\label{sec:fourier-params-sensitivity}
		
		In the fine-tuning adaptation of WePE (Section~\ref{subsec:wefpe-finetune}), the Fourier-like approximation introduces two scalar parameters, $\beta$ and $\gamma$, which are learnable in all our main experiments. To explicitly verify that the method does not rely on careful hand-tuning of these parameters, we conduct an additional sensitivity study in which we fix $\beta$ and $\gamma$ to different values and measure the impact on downstream performance.

		We use a ViT-B/16 backbone pretrained on ImageNet-21k and fine-tune it on
		CIFAR-100 with our Weierstrass positional encoding. The image resolution is
		$224\times 224$ and the patch size is $16$. For the sensitivity experiment, we replace the learnable $\beta$ and $\gamma$ with fixed scalars and perform a $5\times5$ grid search over
		\begin{align*}
			\beta &\in \{0.01, 0.05, 0.10, 0.20, 0.50\}, \\
			\gamma_{\text{scale}} &\in \{0.01, 0.05, 0.10, 0.20, 0.50\}.
		\end{align*}
		where the $k$-th Fourier coefficient is parameterized as
		$\gamma_k = \gamma_{\text{scale}} / k^2$ with $K=3$ terms. For each of the $25$ configurations, we fine-tune the model for $5$ epochs using AdamW with learning rate $10^{-3}$ and weight decay $0.05$, and report the best test accuracy over epochs. All other settings (data augmentations, batch size, random seed, etc.) are kept fixed across the grid.

		Across all $25$ configurations, the mean best test accuracy is
		$49.93\%$ with a standard deviation of $2.81$ and a coefficient of
		variation of $5.6\%$. The minimum and maximum accuracies are $40.29\%$ and $52.27\%$, respectively. The full accuracy surface is visualized as a heatmap in Fig.~\ref{fig:wepe-sensitivity-heatmap}, and two sets of slices highlighting the effect of varying $\beta$ or $\gamma_{\text{scale}}$ independently are shown in Fig.~\ref{fig:wepe-sensitivity-curves}.

		We are primarily interested in the practically relevant range $\beta \ge 0.05$ and $\gamma_{\text{scale}} \ge 0.05$, which contains the settings used in our main experiments. In this $4\times4$ sub-grid ($16$ configurations), the best test accuracy lies between $50.28\%$ and
		$51.59\%$, \textit{i.e.}\, within a narrow band of only $1.31$ percentage points.
		Both Fig.~\ref{fig:wepe-sensitivity-heatmap} and Fig.~\ref{fig:wepe-sensitivity-curves} exhibit a broad plateau over this
		region: for any fixed $\gamma_{\text{scale}}\in\{0.05,0.10,0.20,0.50\}$,
		sweeping $\beta$ from $0.05$ to $0.50$ produces only mild fluctuations, and
		vice versa for sweeping $\gamma_{\text{scale}}$ at fixed
		$\beta\in\{0.05,0.10,0.20,0.50\}$. These variations are comparable to the
		typical run-to-run noise observed in short fine-tuning.

		Noticeable degradation occurs only in the extreme corner cases with very
		small $\beta = 0.01$ and very small
		$\gamma_{\text{scale}} \in \{0.01, 0.05\}$, where the best accuracy drops to
		$41$--$40\%$. In this regime, the stabilizing term $1/(\lVert z\rVert^2 +
		\beta)$ approaches the original singular behavior of $1/\lVert z\rVert^2$,
		and the Fourier correction is too weak to compensate, which is precisely what
		the learnable $\beta$ and $\gamma$ in our main method are designed to avoid.
		Importantly, none of our main experiments operate in this pathological
		parameter regime.

		Although the Fourier-like approximation introduces the additional scalars
		$\beta$ and $\gamma$, this experiment shows that WePE enjoys a wide region of
		insensitivity with respect to these parameters. In the broad practical range
		$\beta\in[0.05,0.50]$ and $\gamma_{\text{scale}}\in[0.05,0.50]$, the
		downstream accuracy varies by at most about $1.3$ percentage points, while
		larger drops are confined to extreme settings that are not used in practice.
		Combined with the fact that $\beta$ and $\gamma$ are learned automatically in our main models, this indicates that the Fourier-like adaptation does not
		introduce a significant tuning burden and that the proposed WePE is robust to the specific choices of these parameters. 
		
		\begin{figure}[t]
			\centering
			\includegraphics[width=0.7\linewidth]{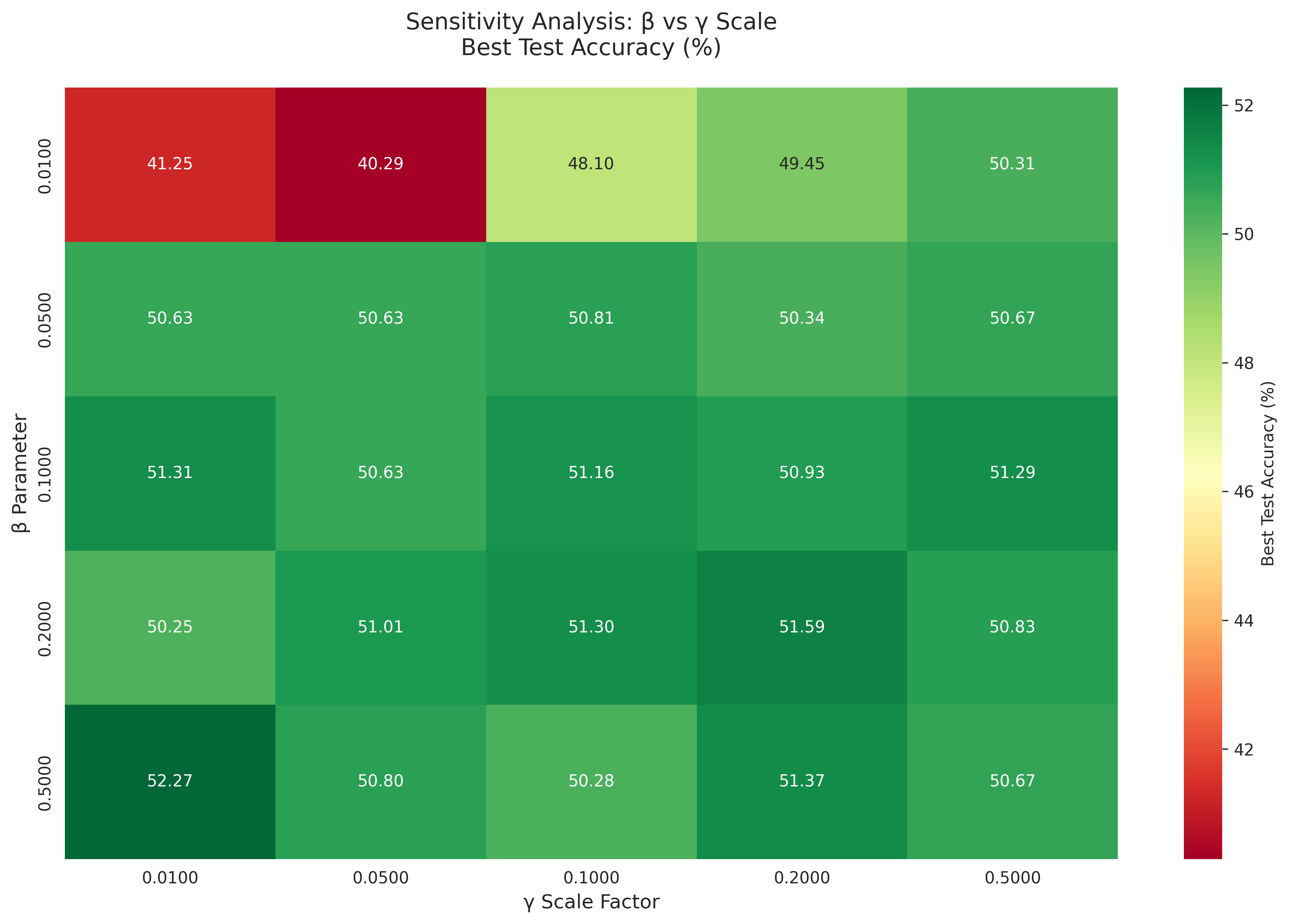}
			\caption{Sensitivity of WePE to the Fourier-like parameters $\beta$ and
				$\gamma_{\text{scale}}$ on CIFAR-100 with a ViT-B/16 backbone. Each cell
				shows the best test accuracy (\%) over $5$ epochs of fine-tuning. }
			\label{fig:wepe-sensitivity-heatmap}
		\end{figure}
		
		\begin{figure}[t]
			\centering
			\includegraphics[width=0.48\linewidth]{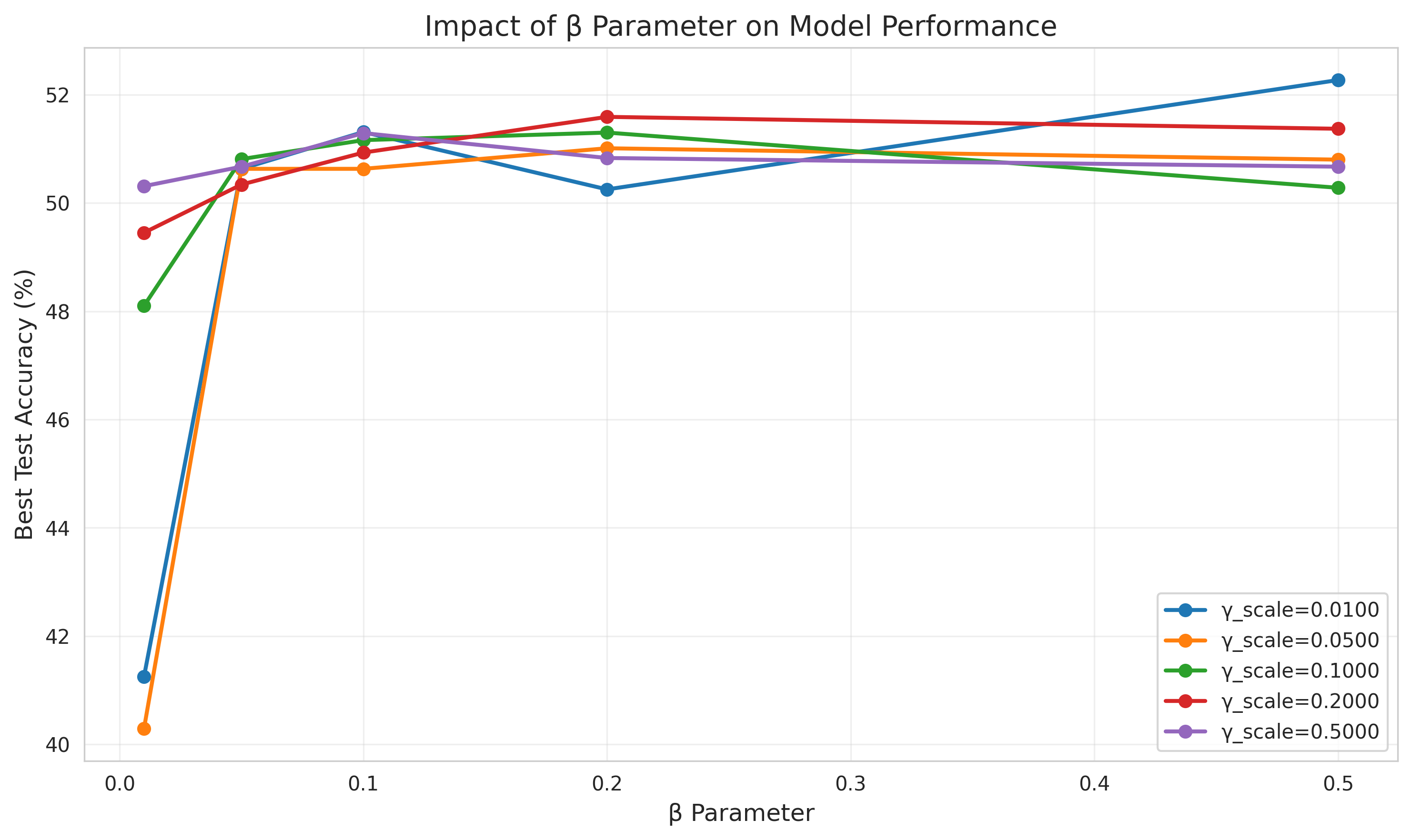}
			\includegraphics[width=0.48\linewidth]{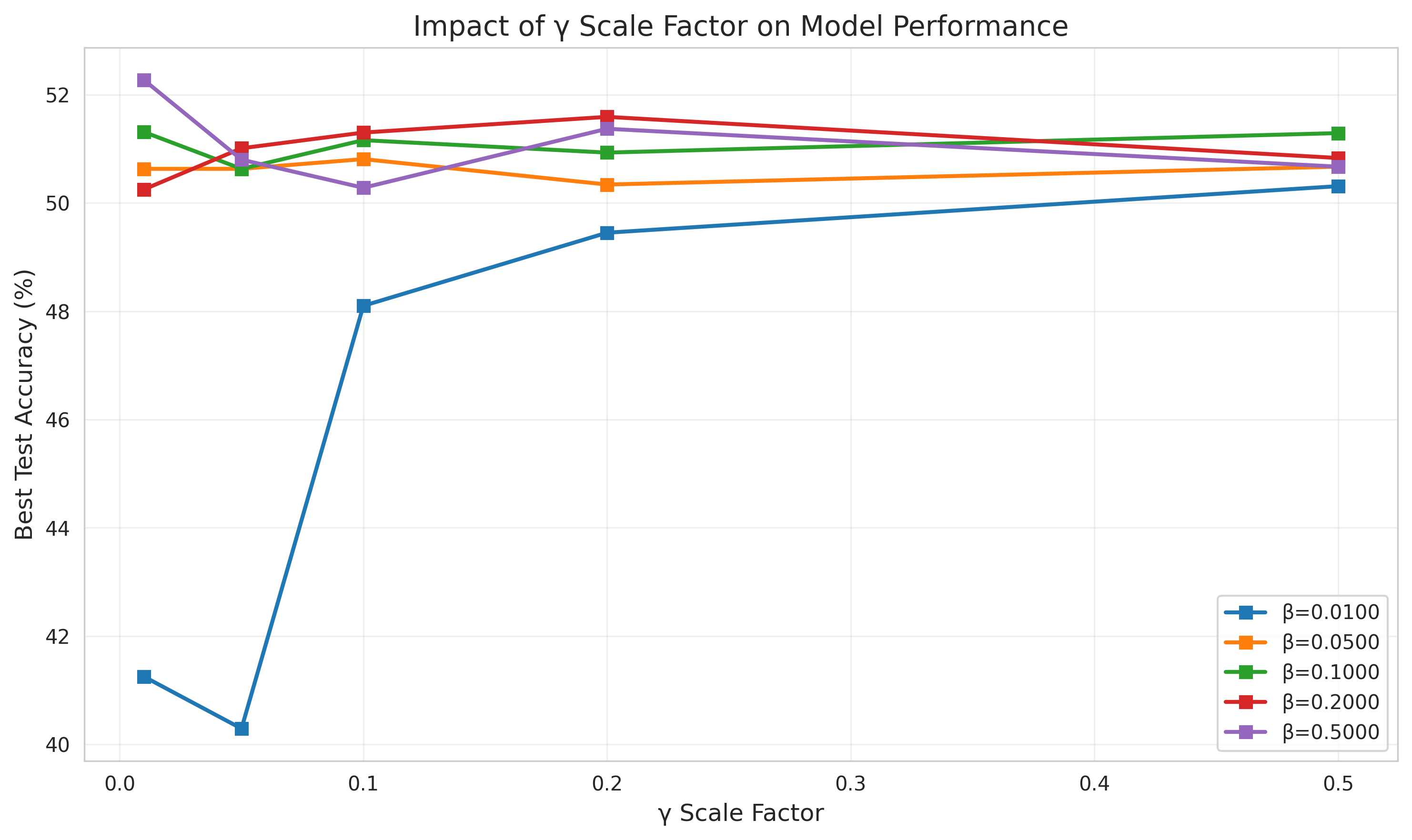}
			\caption{Slices of the sensitivity surface in
				Fig.~\ref{fig:wepe-sensitivity-heatmap}. Varying $\beta$ while keeping $\gamma_{\text{scale}}$ fixed, and
				varying $\gamma_{\text{scale}}$ while keeping $\beta$ fixed. In the practical range $\beta\ge 0.05,\ \gamma_{\text{scale}}\ge 0.05$ the curves are relatively flat, indicating low sensitivity. }
			\label{fig:wepe-sensitivity-curves}
		\end{figure}

		\subsection{Object detection and image segmentation tasks integrated with WePE}
		\label{app:coco_attention_seg}

		To further investigate whether the geometric inductive bias introduced by WePE is beneficial for dense prediction tasks such as detection and segmentation, we first conduct a zero-shot attention--segmentation evaluation on the COCO~2017 dataset. This experiment uses only the frozen image classifier and does not perform any task-specific fine-tuning, thereby isolating the effect of the positional encoding itself.

		We use the same ViT backbone as in the main experiments, instantiated with either standard APE or our WePE. The models are trained only on the classification task (no segmentation supervision) and are kept frozen during this evaluation. For the dataset, we randomly sample $50$ images from the COCO~2017 validation split, together with their official instance-level annotations. All images are resized to $224\times 224$. For each image, we convert the COCO instance masks into a semantic
		segmentation mask by taking the union of all annotated foreground instances (category IDs $>0$) and treating all remaining pixels as background.\footnote{
			We only require a binary foreground/background separation in this evaluation,
			so the exact instance identity is not used.} This yields a binary ground-truth
		mask $M \in \{0,1\}^{H\times W}$ for each image. Given a frozen ViT, we compute attention-based localization maps. Concretely, for each
		transformer layer we average the multi-head attention matrices, add the
		identity connection, renormalize along the last dimension, and then multiply
		the resulting matrices across layers to obtain a single attention matrix
		linking the \texttt{[CLS]} token to all patch tokens. The resulting vector is
		reshaped into a $h\times w$ grid and bilinearly upsampled to the input
		resolution, yielding a dense attention map
		$A \in [0,1]^{H\times W}$.\footnote{In practice we apply a small discard
			ratio to extremely low attention values before renormalization, this has
			negligible impact on the conclusions.} The same pipeline is applied to both
		the APE-based and WePE-based models, without any further tuning. We evaluate how well the attention map $A$ aligns with the ground-truth mask
		$M$ using two complementary metrics: 
		
		\begin{itemize}
			
			\item IoU. We threshold the attention map at $0.5$ to obtain a
			predicted binary mask $\hat{M}$ and compute the standard intersection-over-union
			(IoU) between $\hat{M}$ and $M$:
			\(
			\mathrm{IoU} = 
			\frac{|\hat{M} \cap M|}{|\hat{M} \cup M|}.
			\) 
			
			\item Point-biserial correlation. We treat the attention values
			$A$ as continuous scores and the mask $M$ as a binary variable, and compute
			the point-biserial correlation coefficient between them. This measures
			how strongly high-attention pixels correlate with foreground regions. 
			
		\end{itemize} 
		
		Both metrics are computed per image. we report the mean and standard deviation
		over the $50$ images, and perform a paired $t$-test between WePE and APE. Table~\ref{tab:coco_attention_seg} summarizes the quantitative results.
		WePE improves the mean IoU from $0.2213$ to $0.2408$, a gain of
		$+0.0195$, and increases the mean point-biserial correlation from $0.0255$
		to $0.0734$. The IoU improvement is statistically significant under a paired
		$t$-test ($t=2.41$, $p=0.0197$), while the correlation improvement shows a
		consistent positive trend but is not statistically significant at the $5\%$
		level ($t=1.20$, $p=0.23$). Figure~\ref{fig:coco_attention_hist}
		visualizes the distribution of IoU and correlation scores for both methods,
		and Figure~\ref{fig:coco_attention_qual} provides qualitative examples.

		Overall, these results indicate that WePE produces attention maps that are
		more tightly aligned with object regions in COCO images, even without any
		segmentation or detection training. Since spatially aligned attention is a
		key prerequisite for successful detection and segmentation systems, this
		zero-shot evaluation provides additional evidence that the geometric prior
		introduced by WePE is beneficial beyond image classification. 
		
		\begin{table}[t]
			\centering
			\small
			\begin{tabular}{lcc}
				\toprule
				Method & IoU $\uparrow$ & Corr. $\uparrow$ \\
				\midrule
				APE & $0.2213 \pm 0.1652$ & $0.0255 \pm 0.2261$ \\
				WePE (ours) & $\mathbf{0.2408} \pm 0.1835$ & $\mathbf{0.0734} \pm 0.1929$ \\
				\bottomrule
			\end{tabular}
			\caption{Zero-shot attention--segmentation evaluation on COCO~2017.
				WePE yields better alignment between attention maps and ground-truth
				masks than standard APE, despite no task-specific training. }
			\label{tab:coco_attention_seg}
		\end{table}
		
		\begin{figure}[t]
			\centering
			
			\includegraphics[width=\linewidth]{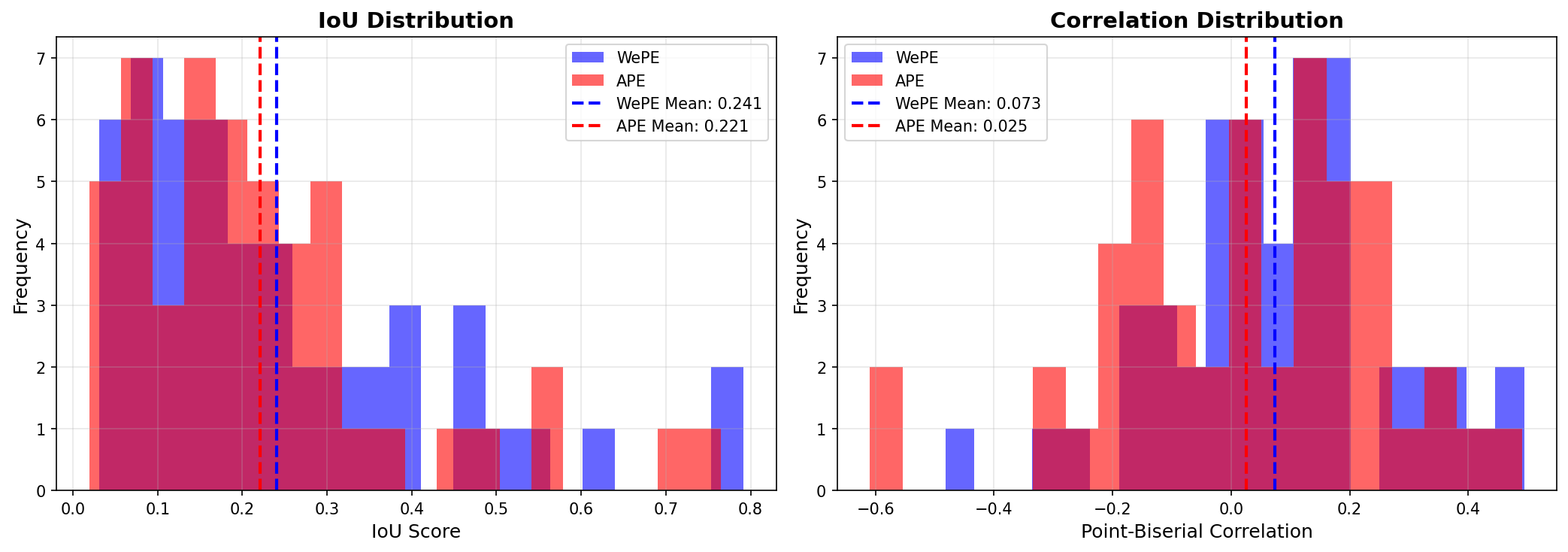}
			\caption{Distributions of IoU (left) and point-biserial correlation
				(right) between attention maps and COCO~2017 segmentation masks,
				comparing WePE and APE. Dashed lines denote the mean of each method. }
			\label{fig:coco_attention_hist}
		\end{figure}

		\begin{figure}[ht]
			\centering
			\newcommand{\cocoimg}[1]{%
				\begin{minipage}{0.31\columnwidth}
					\centering
					\includegraphics[width=\linewidth]{#1.png}%
			\end{minipage}}
			
			\cocoimg{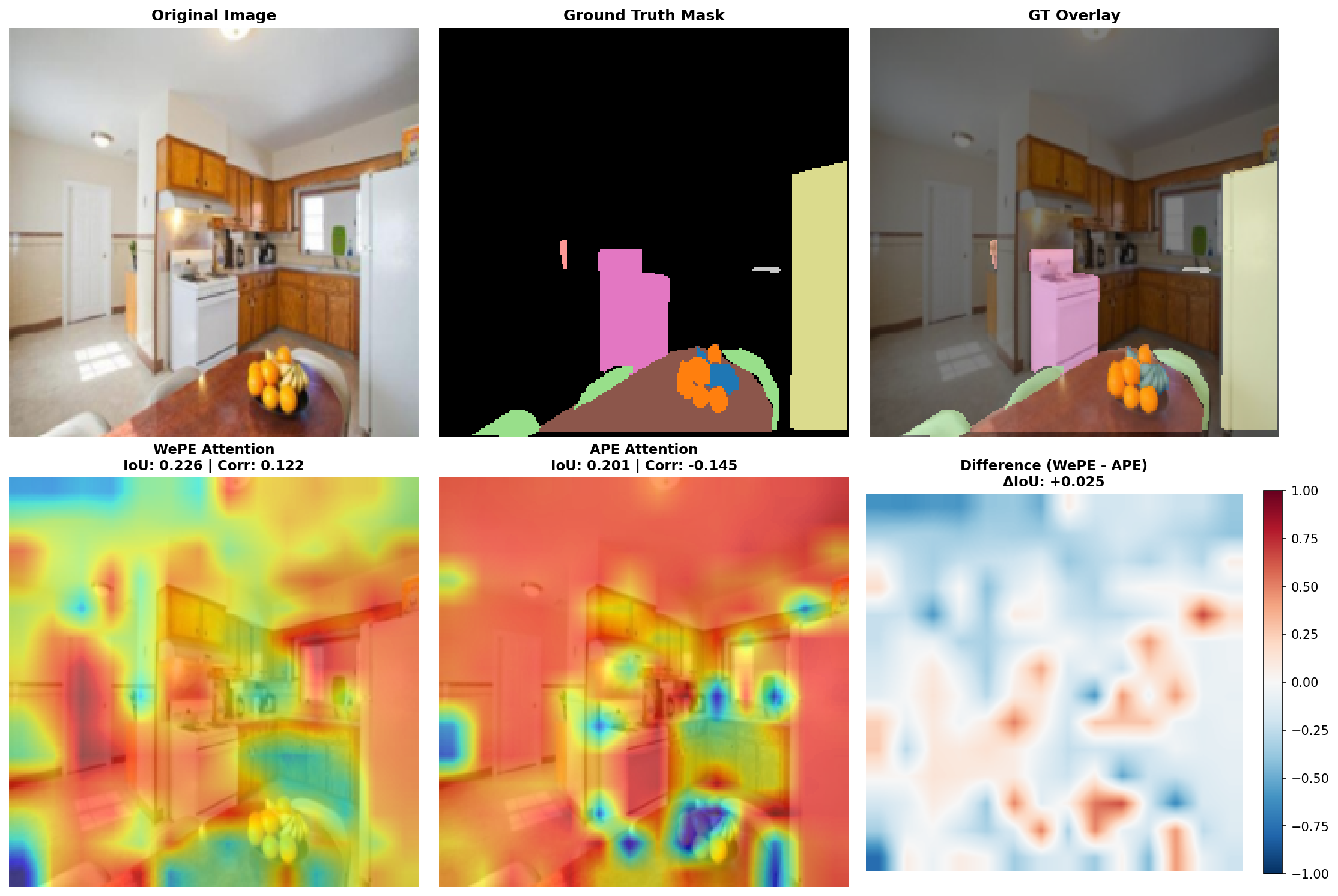}\hfill\cocoimg{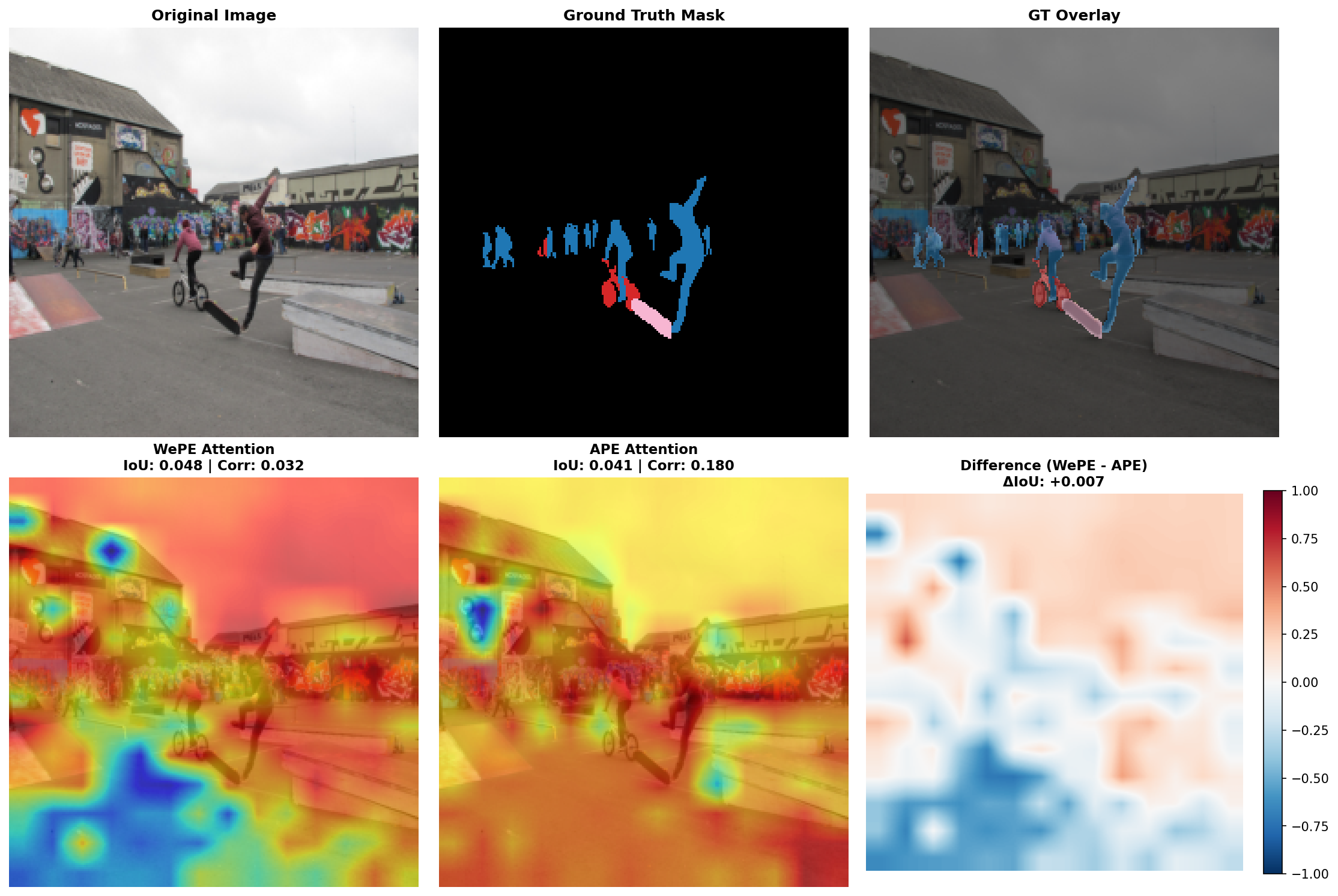}\hfill\cocoimg{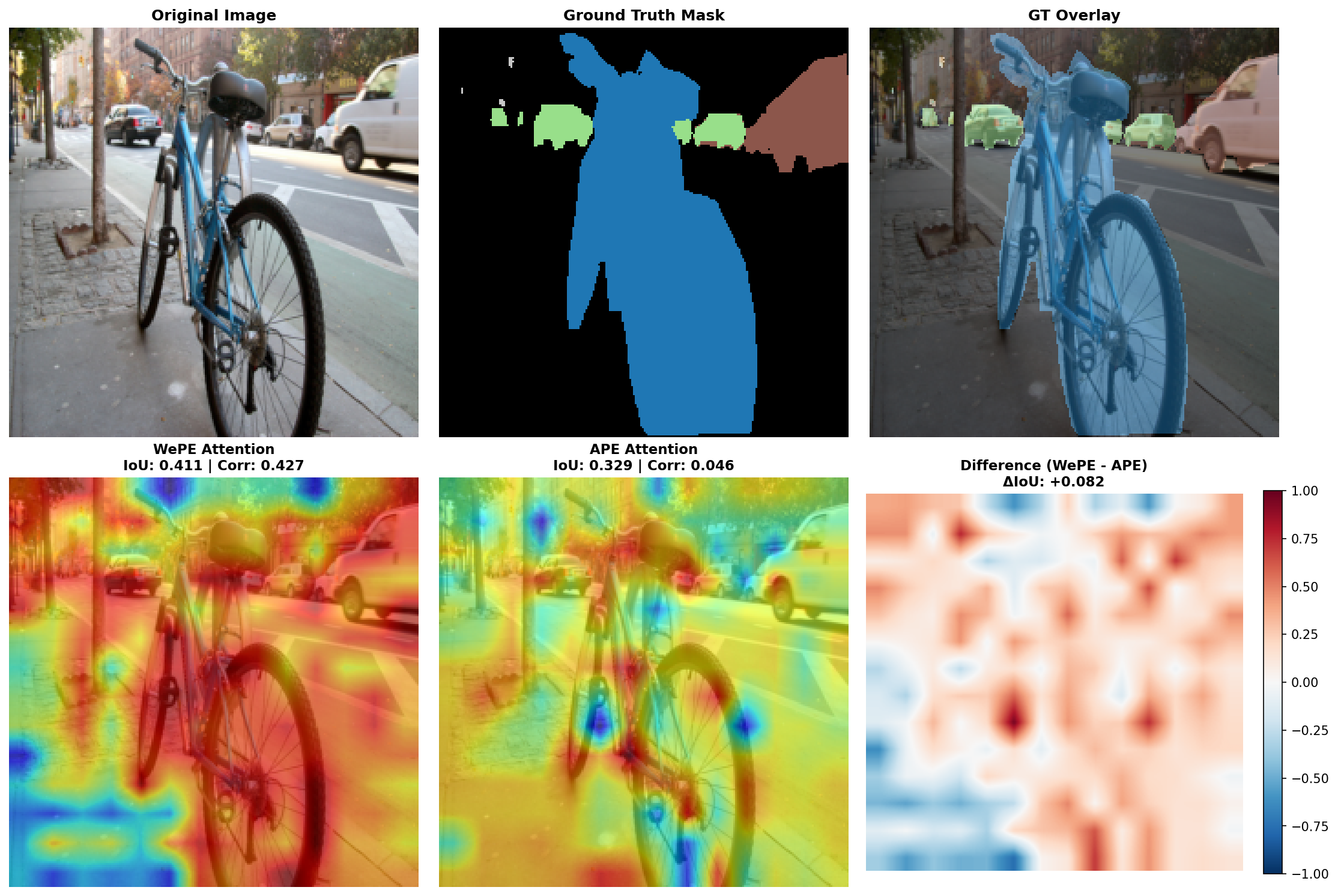}\\[2pt]
			\cocoimg{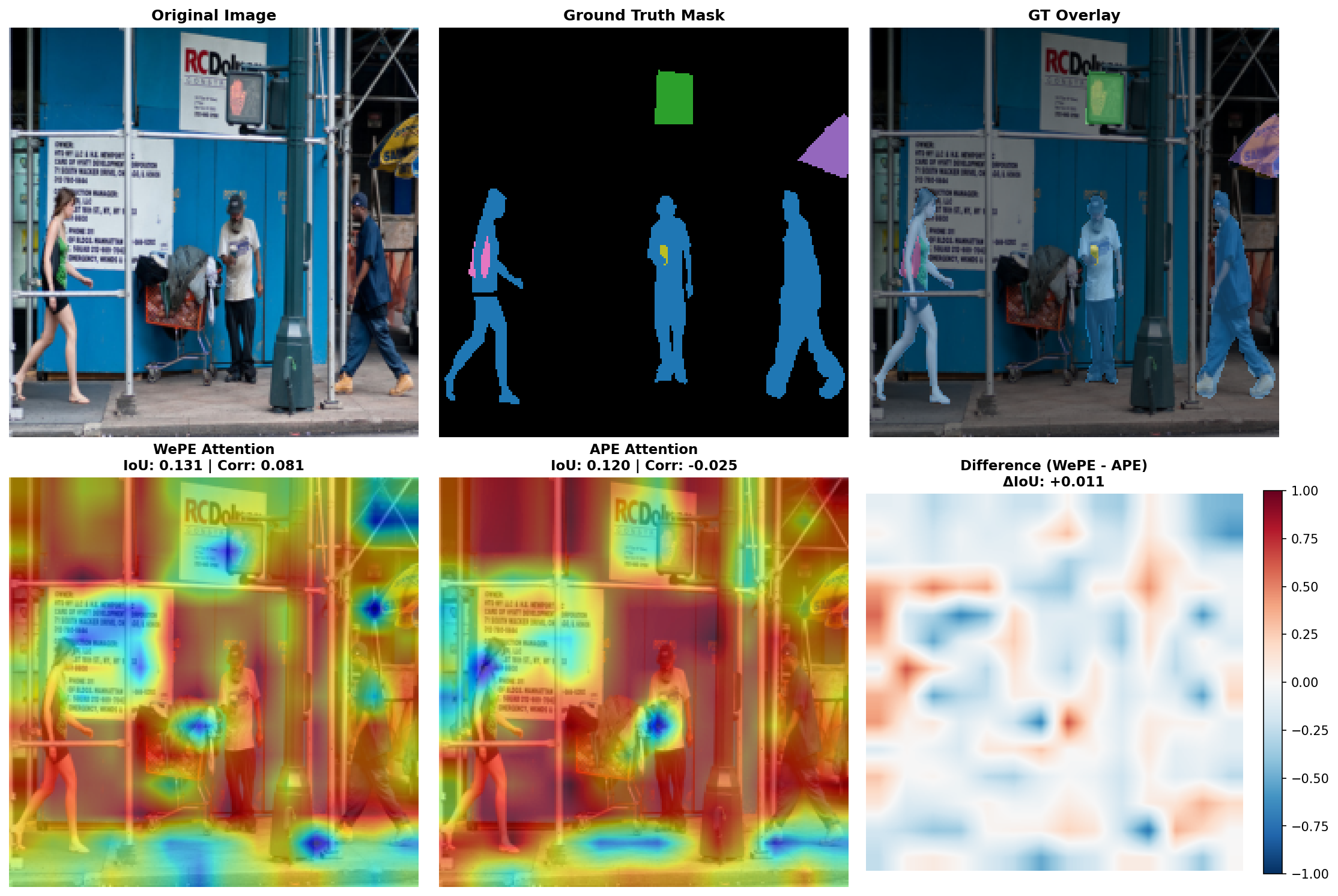}\hfill\cocoimg{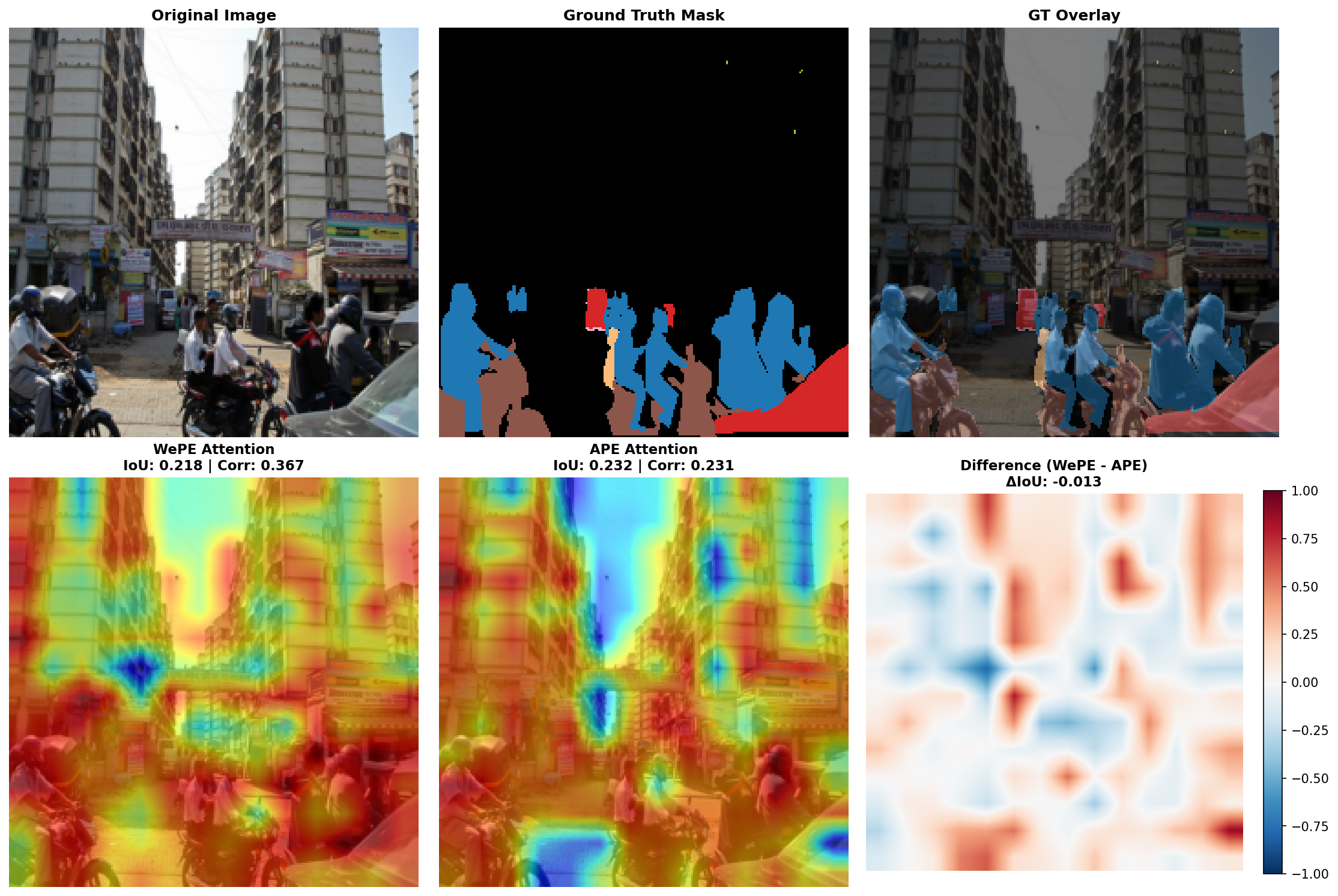}\hfill\cocoimg{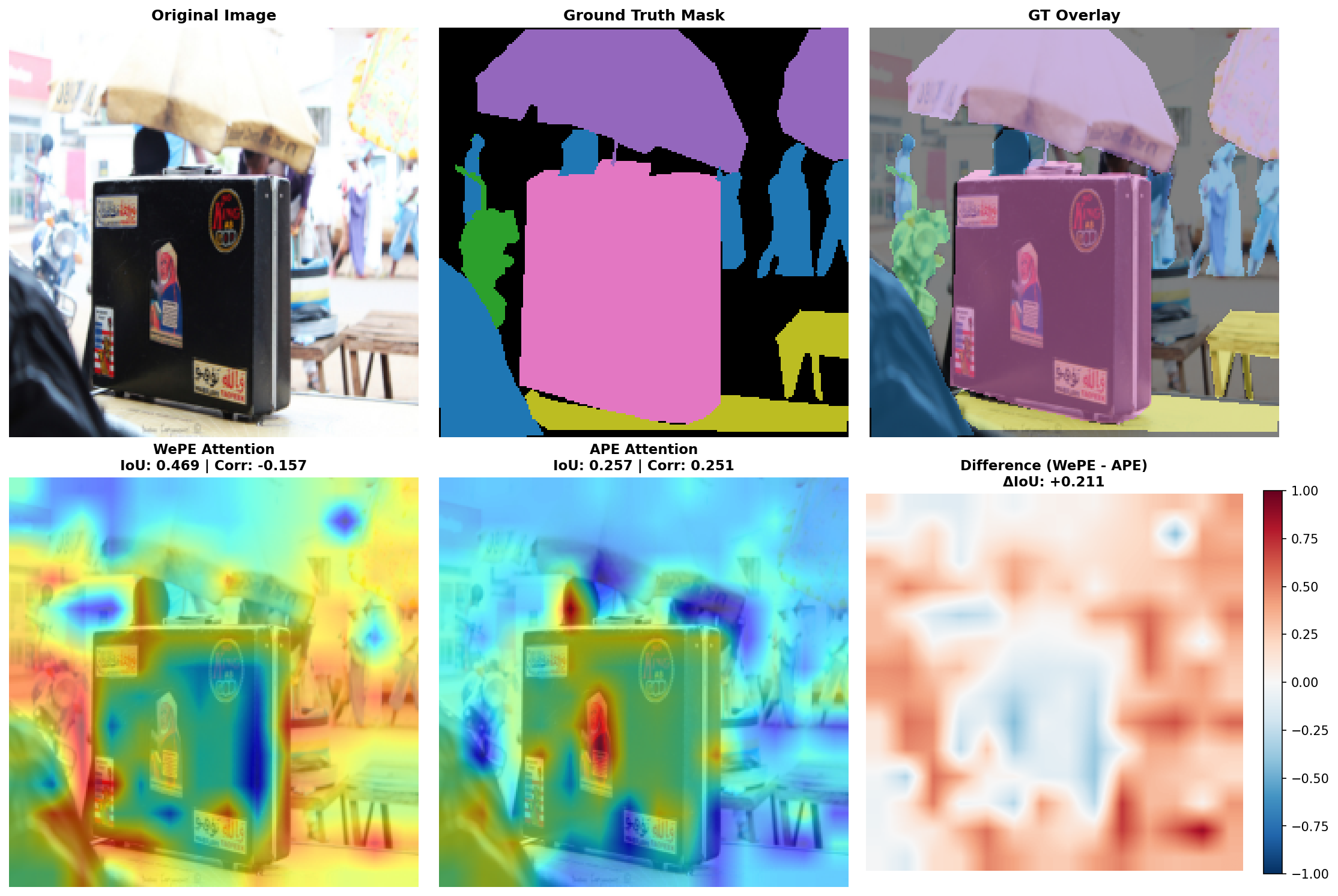}\\[2pt]
			\cocoimg{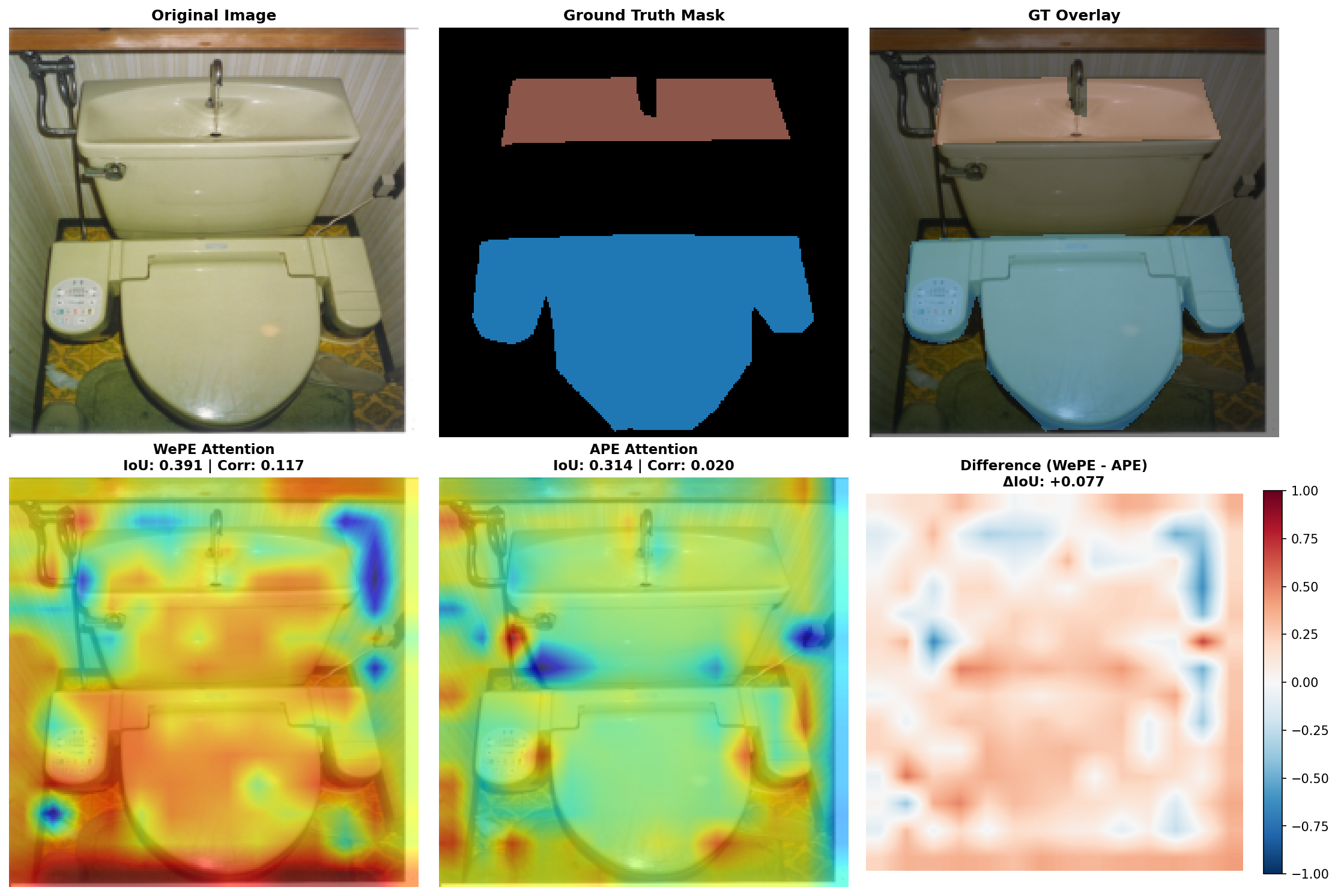}\hfill\cocoimg{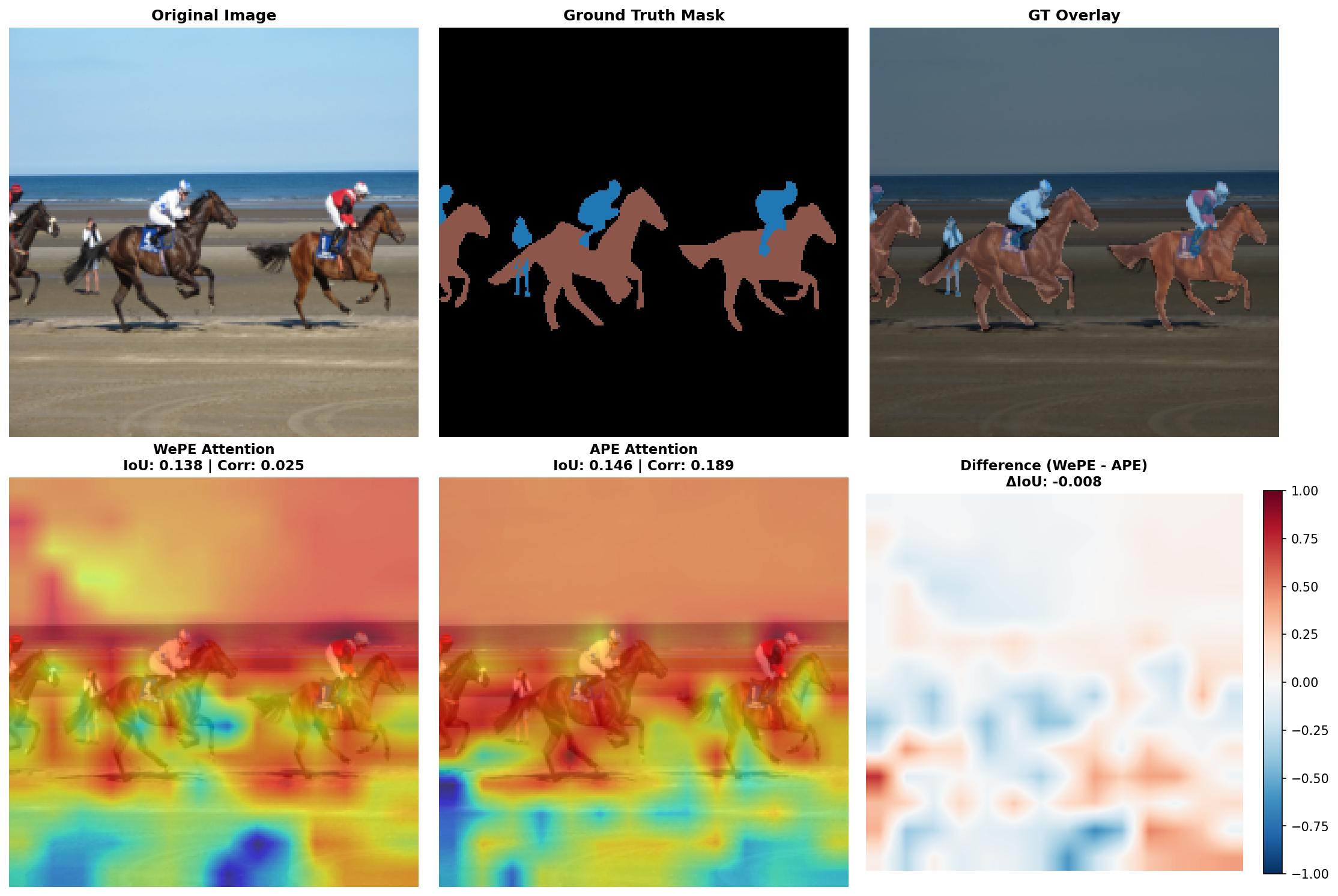}\hfill\cocoimg{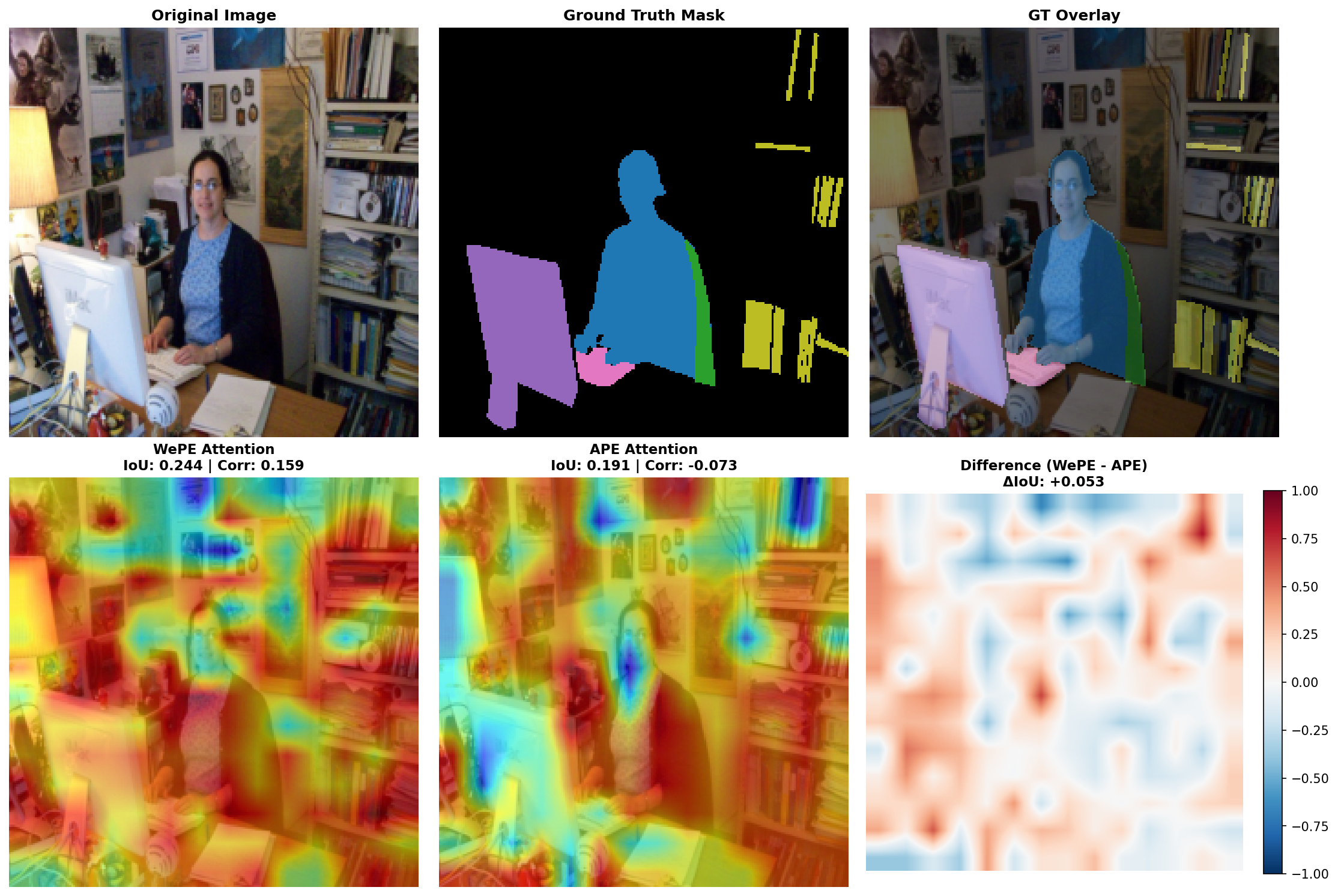}\\[2pt]
			\cocoimg{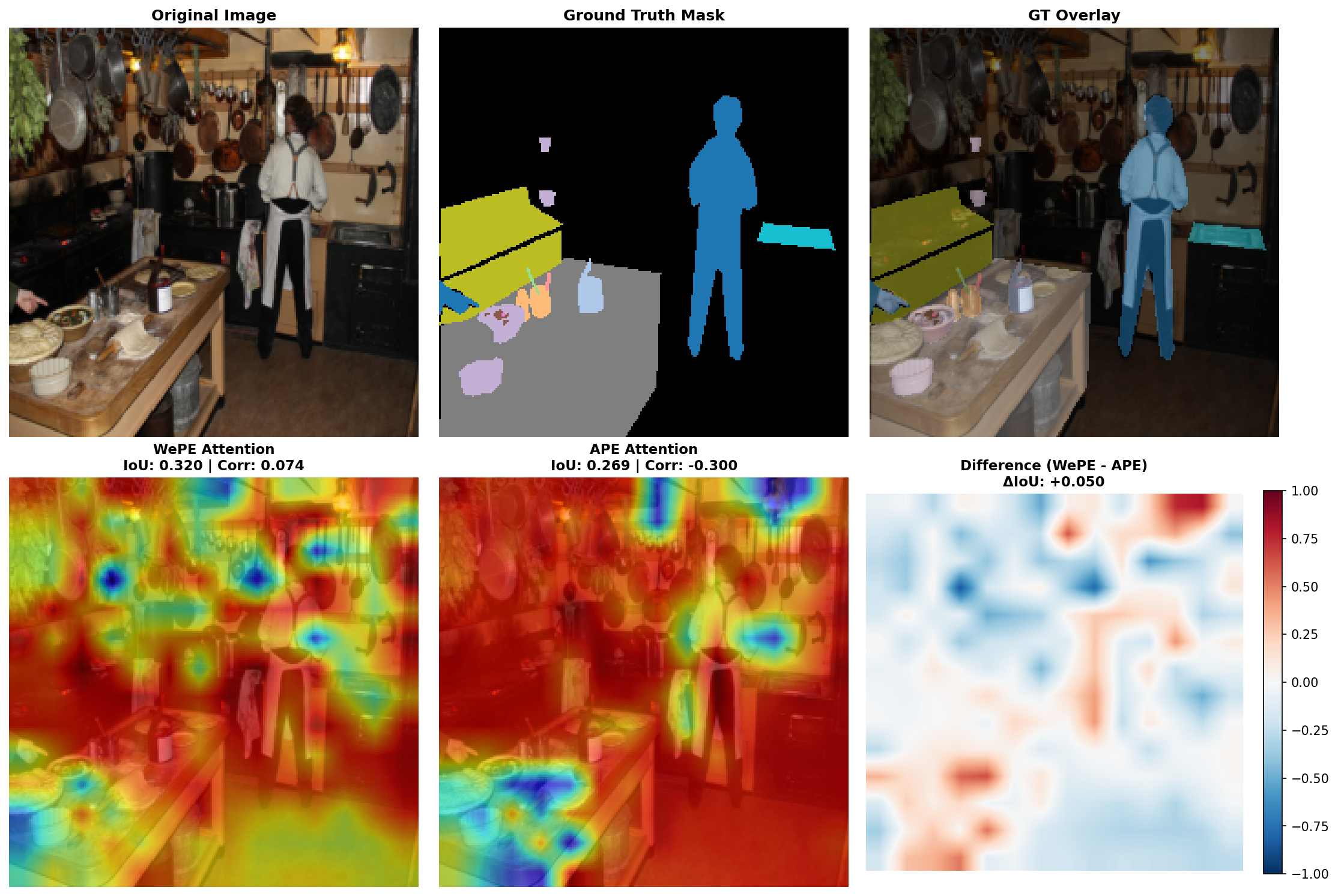}\hfill\cocoimg{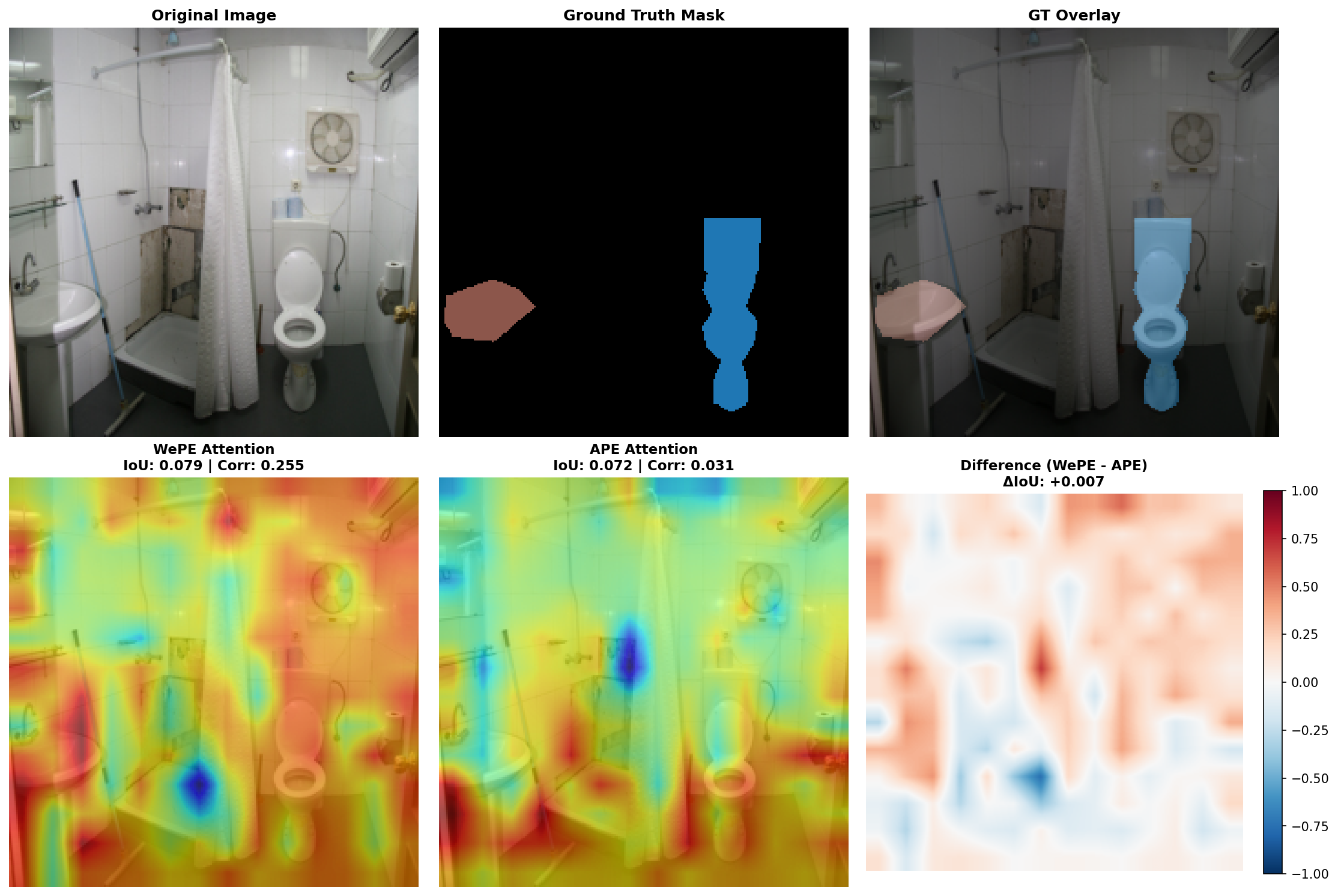}\hfill\cocoimg{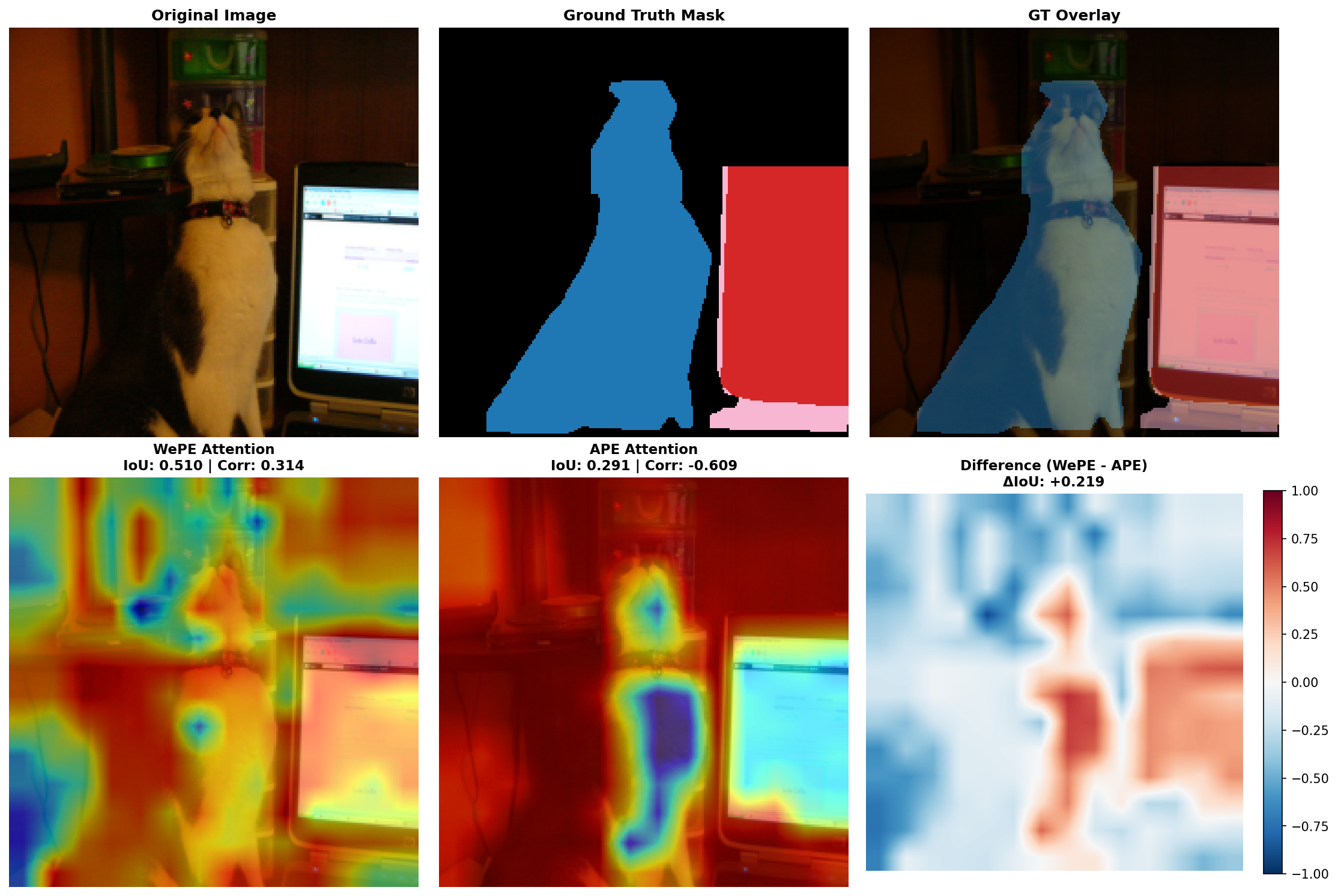}\\[2pt]
			\cocoimg{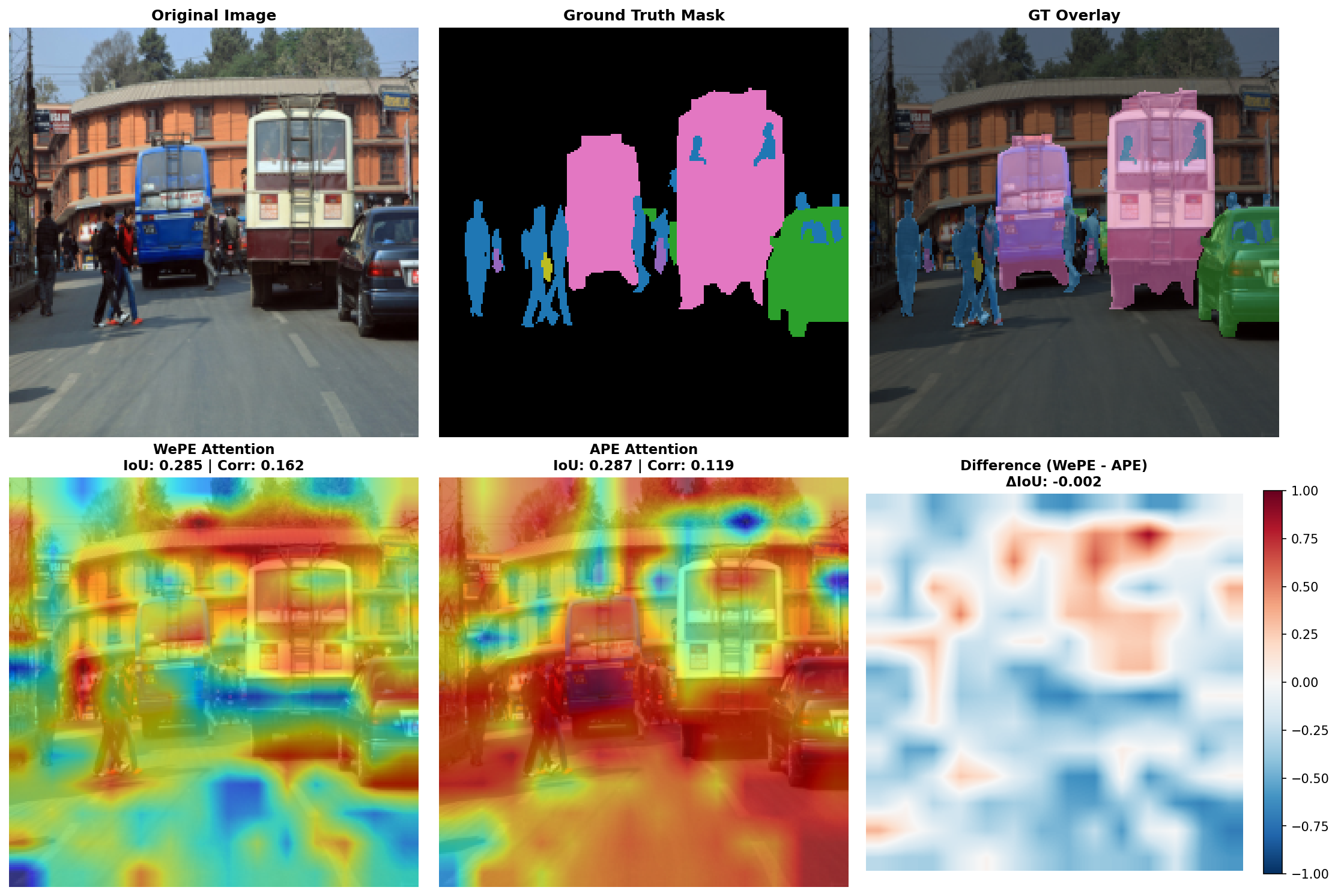}\hfill\cocoimg{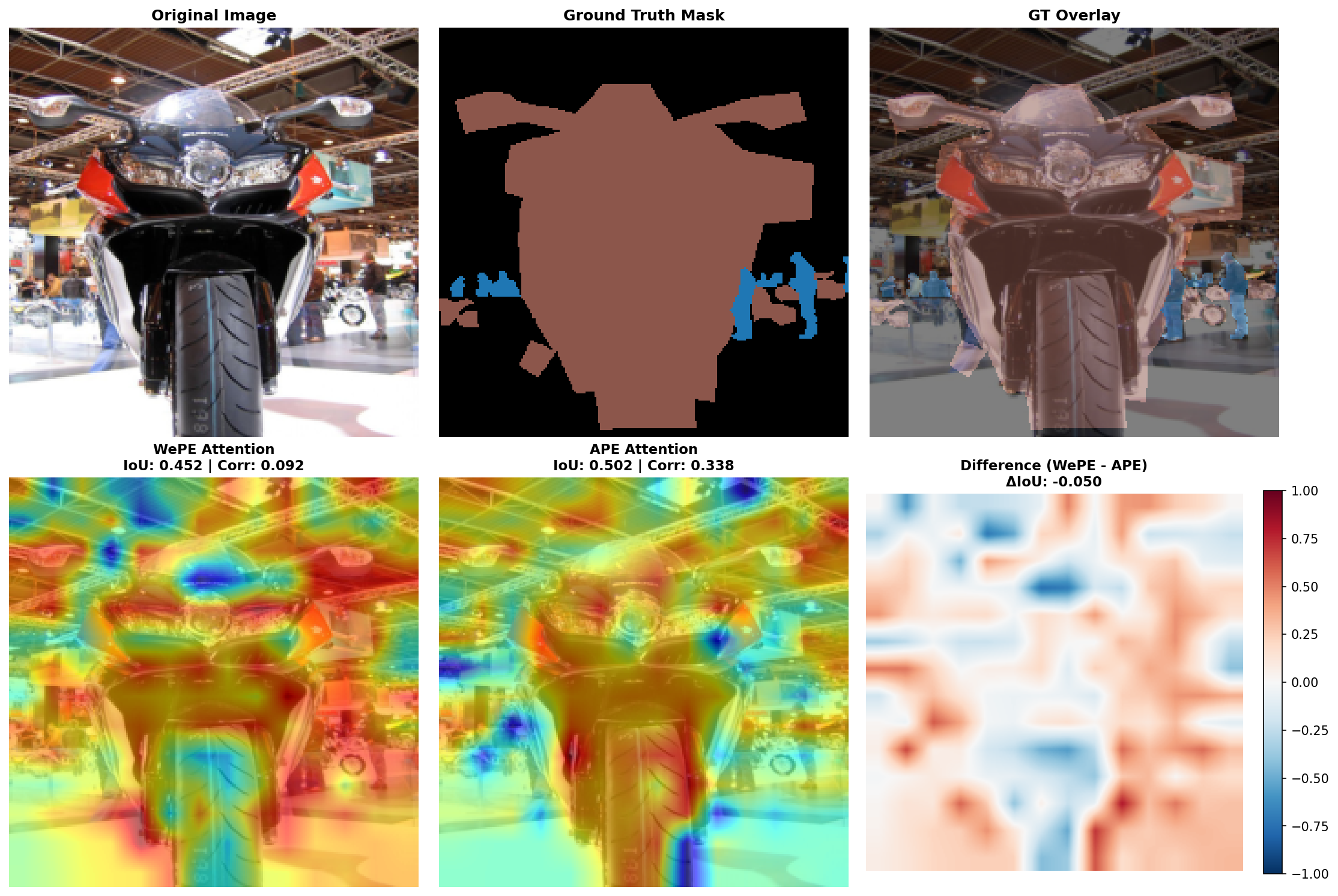}\hfill\cocoimg{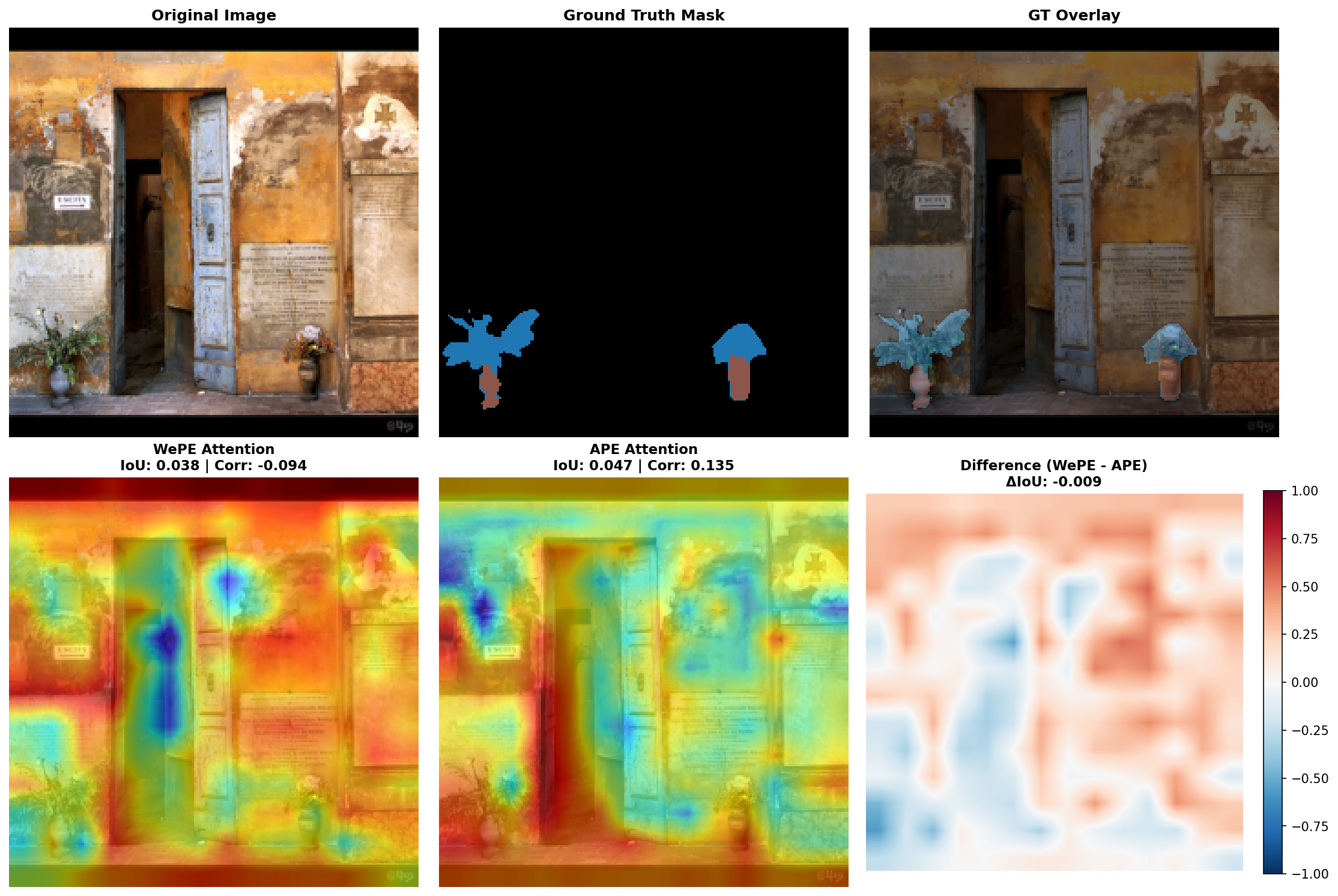}\\[2pt]
			\cocoimg{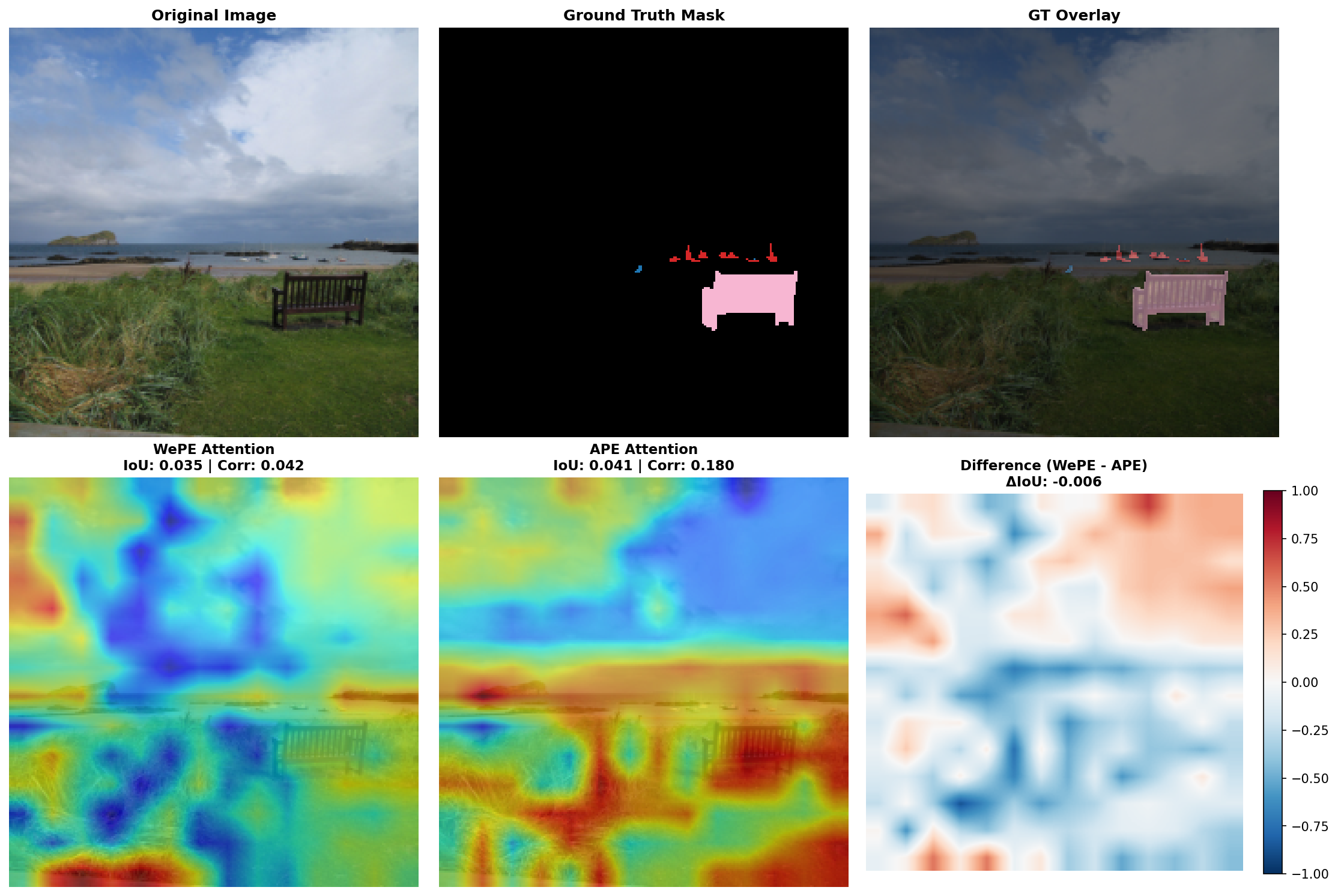}\hfill\cocoimg{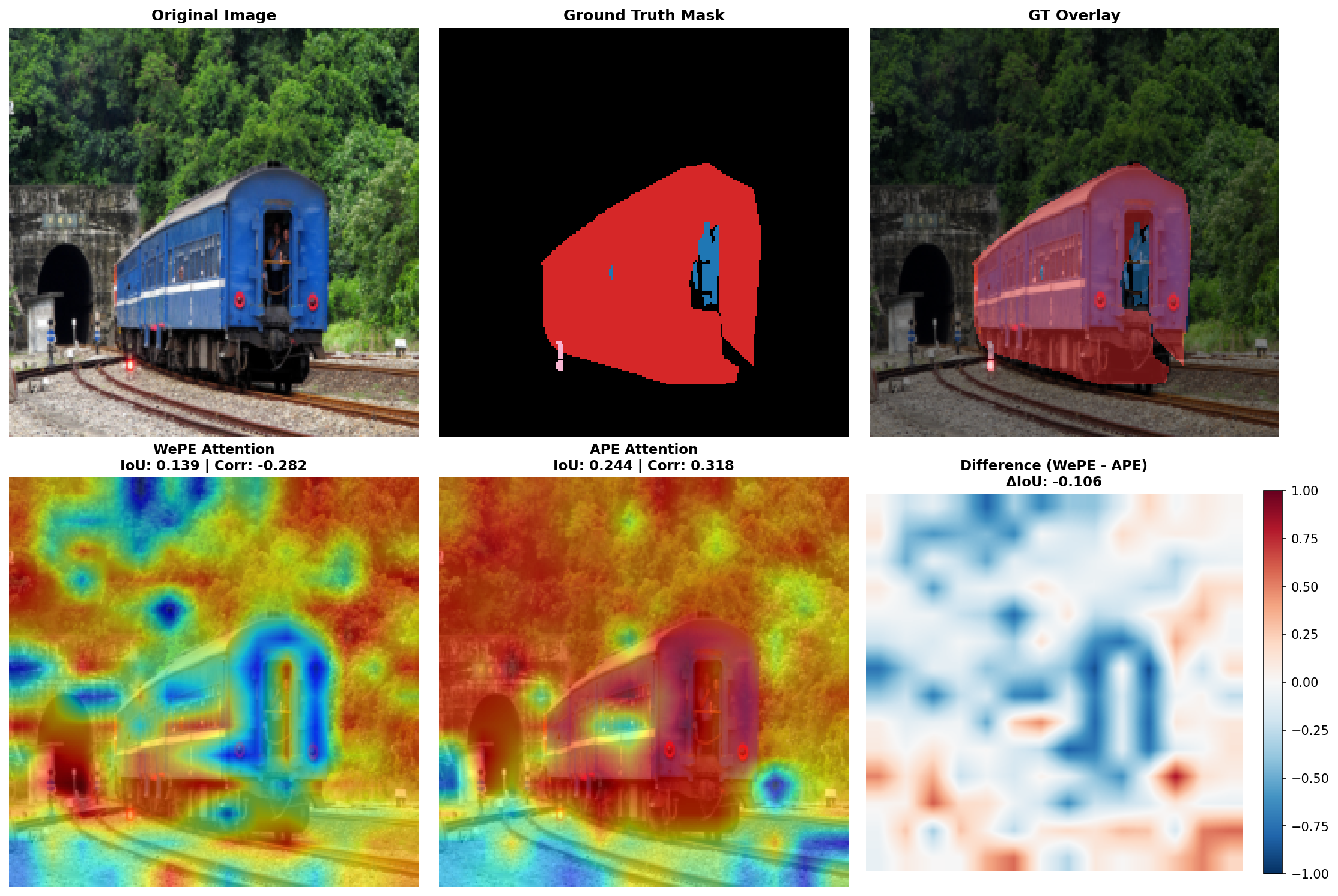}\hfill\cocoimg{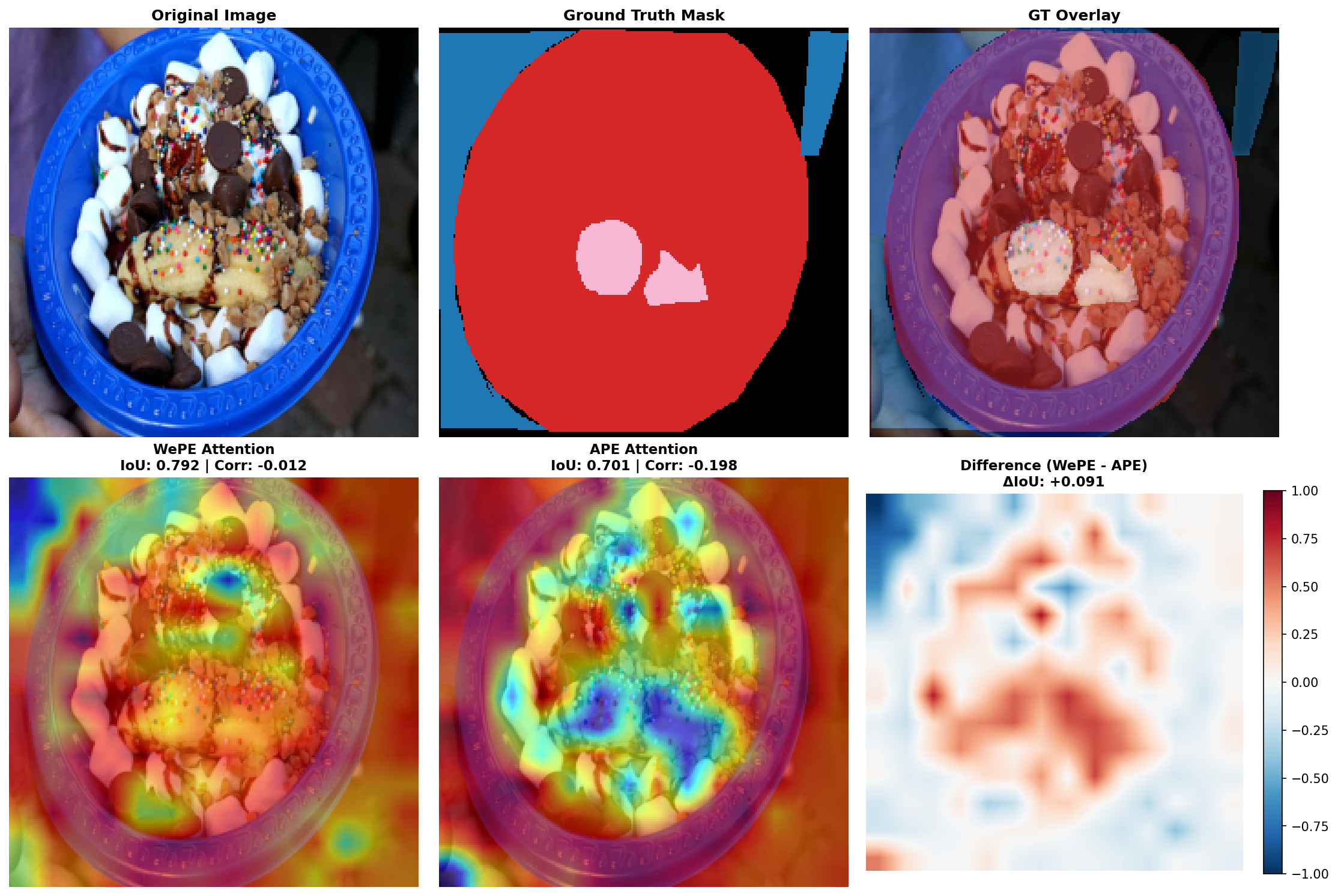}
			
			\caption{Qualitative examples from COCO~2017 showing the original image,
				ground-truth mask, and attention overlays from WePE and APE~\citep{dosovitskiy2020image}. WePE
				consistently produces attention patterns that better cover the foreground
				objects and suppress background clutter.} 
			\label{fig:coco_attention_qual}
		\end{figure}

		Furthermore, to verify that the geometric inductive bias of WePE also transfers to large-scale
		object detection, we integrate WePE into the famous ViTDet framework~\cite{li2022exploring}, a strong and widely adopted baseline for plain ViT backbones~\citep{dosovitskiy2020image}, 
		and evaluate on the COCO~2017 dataset. This experiment keeps all components of
		ViTDet unchanged except for the positional encoding inside the ViT backbone~\citep{dosovitskiy2020image},
		providing a controlled comparison between standard APE and WePE. We follow the official ViTDet recipe on COCO~2017.
		The detector is Mask R-CNN or Cascade Mask R-CNN with a simple feature
		pyramid built on top of a plain ViT backbone~\citep{dosovitskiy2020image}, as in~\cite{li2022exploring}.
		We use the default $1024\times 1024$ large-scale
		jittering (LSJ) data augmentation~\cite{li2022exploring}, 100-epoch training schedule,
		AdamW optimizer, and stochastic depth, exactly matching the public
		ViTDet configuration. For the backbone we consider ViT-B and ViT-L~\citep{dosovitskiy2020image} pretrained with MAE~\citep{he2022masked} on
		ImageNet-1K, using the official checkpoints.
		In the WePE configuration, we replace the APE module in the
		ViT backbone by our WePE parameterization while leaving all other architecture
		and optimization hyperparameters identical.
		In particular, the feature pyramid, RPN, ROI heads, and loss functions are
		unchanged. This ensures that any difference in detection performance can be
		attributed to the positional encoding. We train all models on \texttt{train2017} and report
		bounding-box AP ($\text{AP}_{\text{box}}$) and instance
		segmentation AP ($\text{AP}_{\text{mask}}$) on \texttt{val2017} using the
		standard COCO evaluation protocol. Table~\ref{tab:wepe_vitdet_coco} summarizes the experimental results. Crucially, WePE consistently outperforms the baseline, indicating that the inductive bias afforded by WePE successfully translates to the challenging object detection task, thereby further validating its broad generalization capability.

		\begin{table}[t]
			\centering
			\footnotesize 
			\setlength{\tabcolsep}{3.5pt} 
			\begin{tabular}{@{}lcccc@{}}
				\toprule
				\multirow{2}{*}{Backbone \& PE} &
				\multicolumn{2}{c}{Mask R-CNN} &
				\multicolumn{2}{c}{Cascade Mask} \\ 
				\cmidrule(lr){2-3} \cmidrule(lr){4-5}
				& AP$_{\text{box}}$ & AP$_{\text{mask}}$ &
				AP$_{\text{box}}$ & AP$_{\text{mask}}$ \\
				\midrule
				ViT-B + APE (base) & 51.6 & 45.9 & 54.0 & 46.7 \\
				ViT-B + \textbf{WePE} & \textbf{52.9} & \textbf{46.1} & \textbf{54.7} & \textbf{47.3} \\
				\addlinespace[0.2em]
				ViT-L + APE (base) & 55.6 & 49.2 & 57.6 & 49.8 \\
				ViT-L + \textbf{WePE} & \textbf{56.2} & \textbf{49.5} & \textbf{58.2} & \textbf{50.1} \\
				\bottomrule
			\end{tabular}
			\caption{COCO~2017 detection results with ViTDet-style plain backbones.
				All models use MAE-pretrained~\citep{he2022masked} ViT-B/L backbones~\citep{dosovitskiy2020image} and follow the official
				ViTDet configuration~\cite{li2022exploring}.}
			\label{tab:wepe_vitdet_coco}
		\end{table}

		\subsection{Model shortcomings and future improvement directions}
		\label{sec:limitations-future-work}

		At the end of this paper, we conduct an explicit failure analysis on WePE to identify potential deficiencies in the algorithm.

		We first quantify how classification errors relate to the spatial
		distribution of the most attended patches in the last transformer layer. We use the ViT-Tiny + WePE model at a resolution of
		$224\times224$ with a $14\times14$ patch grid.
		For each test image $x$, we extract the attention weights from the last multi-head
		self-attention block and average over heads to obtain an attention matrix
		$A \in \mathbb{R}^{L\times L}$, where $L=1+N$ and $N=14\times14$ is the number of patch tokens.
		We focus on the attention from the class token to patch tokens, \textit{i.e.}\, the row
		$A_{\text{cls}\rightarrow\text{patch}} \in \mathbb{R}^{N}$.
		We select the top-$K$ patches with the largest weights (we use $K=5$ throughout)
		and treat their grid coordinates $\{p_1,\dots,p_K\}$ on the $14\times14$ lattice as the
		discriminative region set for this image. We then define a scalar spatial separation score $S(x)$ as the mean pairwise Euclidean
		distance between these $K$ patches on the lattice:
		\( S(x) = \frac{2}{K(K-1)} \sum_{1 \le i < j \le K} d(p_i, p_j)\), $d(p_i,p_j)$  is the Euclidean distance on the $14\times14$ grid.  
		
		Intuitively, small $S(x)$ indicates that the model concentrates its attention on a compact region,
		while large $S(x)$ corresponds to highly dispersed evidence that requires long-range reasoning. For every test image we record the pair $(S(x), \mathbbm{1}[\hat{y}(x)\neq y])$, where
		$\hat{y}(x)$ is the model prediction and $y$ is the ground-truth label.
		We partition the range of $S$ into $B$ bins (we use $B=20$ with robust percentiles to avoid outliers)
		and compute, for each bin $b$, the empirical error rate
		\(
		\mathrm{Err}(b) = \frac{1}{|{\mathcal{D}}_b|}
		\sum_{x\in {\mathcal{D}}_b} \mathbbm{1}[\hat{y}(x)\neq y],
		\)
		together with the binomial standard deviation as an error bar.
		We also report the number of samples $|{\mathcal{D}}_b|$ per bin to distinguish reliable statistics from low-sample regimes.

		Figure~\ref{fig:wepe-error-vs-separation} shows the resulting
		classification error rate on the test set vs.\ spatial separation curve. Most test samples ($\approx 98.8\%$) fall into the range $S\in[2,10]$, corresponding to moderately dispersed discriminative regions. Within this regime, the error rate remains low and stable between $0.31$ and $0.38$, and the fitted trend is nearly flat. This indicates that WePE handles the vast majority of natural images well, even when the evidence spans multiple mid-range patches. In contrast, only a small fraction of samples ($\approx 1.2\%$) exhibit very large separation scores ($S>10$), meaning that the attended patches are spread across distant parts of the image. In this high-separation regime the error rate increases sharply, and several bins approach an error rate close to $1.0$. These cases correspond to images that require unusually long-range integration of visual cues.
		
		\begin{figure}[t]
			\centering
			\includegraphics[width=0.9\linewidth]{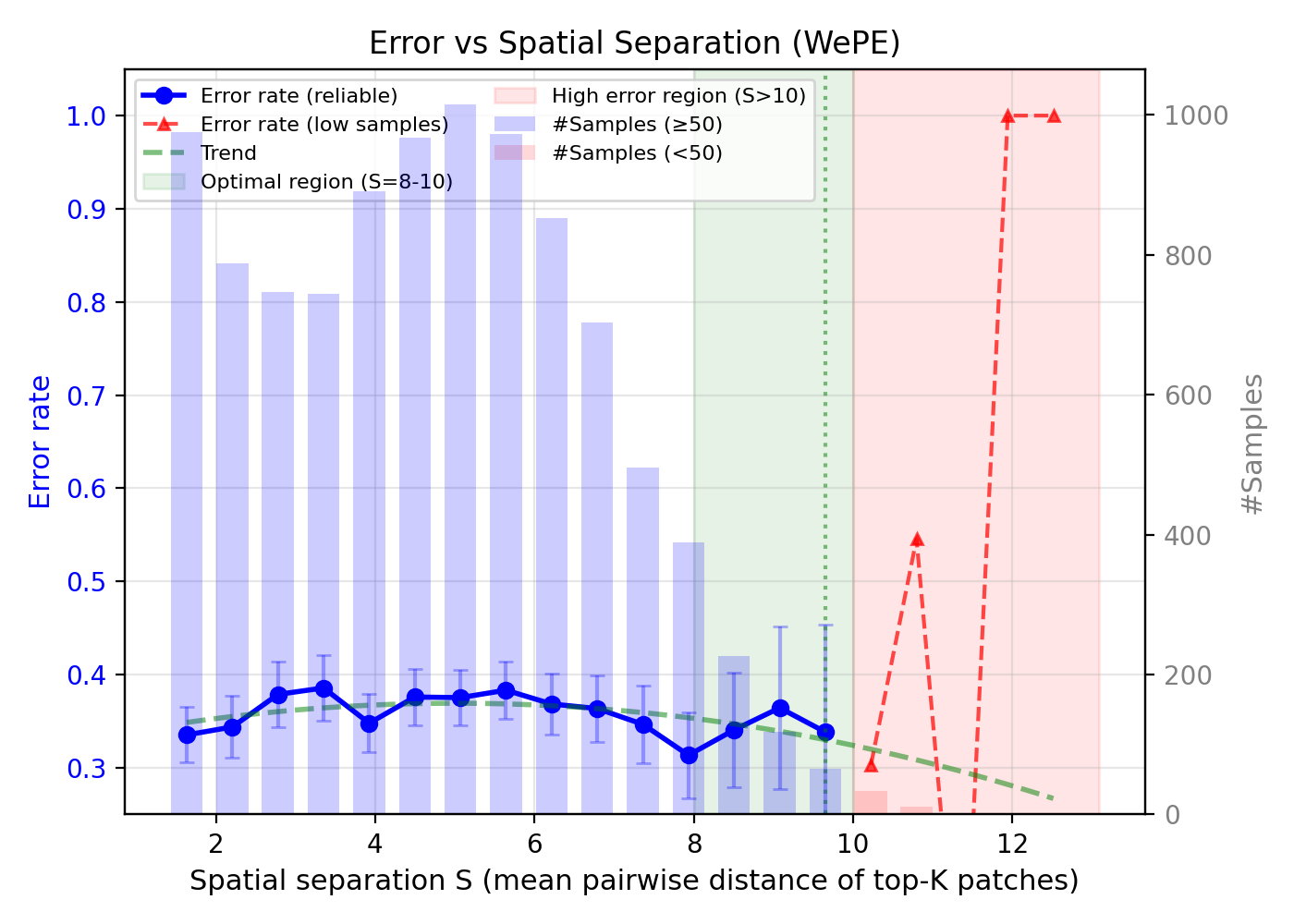}
			\caption{
				Error vs.\ spatial separation for WePE.
				The blue curve shows the classification error rate on the test set as a function of the spatial separation
				score $S$, with error bars
				indicating the binomial standard deviation.
				Bars on the secondary axis denote the number of samples per bin, and the green/red shaded
				regions highlight respectively the mid-range regime with stable low error and the rare
				high-separation regime where errors increase.}
			
			\label{fig:wepe-error-vs-separation}
		\end{figure}

		Immediately after, To visually inspect how WePE allocates attention across the image and to provide explicit
		failure diagnostics, we generate qualitative heatmaps from the last transformer layer. 
		For each test image $x$, we register a forward hook on the last multi-head
		self-attention block and obtain the attention tensor
		$A \in \mathbb{R}^{B\times H\times L\times L}$, where $B$ is the batch size,
		$H$ the number of heads, and $L=1+N$ the sequence length
		(class token plus $N=14\times14$ patch tokens).
		We average over heads and select the row corresponding to the class token,
		yielding a vector $a_{\mathrm{cls}\rightarrow\mathrm{patch}} \in \mathbb{R}^{N}$. We reshape this vector into a $14\times14$ attention map on the patch lattice and
		upsample it to the input resolution using bilinear interpolation.
		The original image is de-normalized to $[0,1]$ and the attention map is overlaid as a
		semi-transparent heatmap.
		We collect both correctly classified samples and failure cases
		(misclassified samples), spanning diverse object categories and spatial layouts. Representative examples are shown in Figure~\ref{fig:wepe-heatmaps}.
		For correctly classified images, WePE consistently produces
		structured and spatially coherent attention patterns:
		high responses concentrate on semantically meaningful parts of the object
		(e.g., the head and torso of an animal, the body of a vehicle, or the base of a lamp),
		while background regions receive very low attention.
		Moreover, attention rarely collapses to a single patch; instead, it forms
		several mid-range clusters covering multiple object parts.
		This behavior matches the intended inductive bias of WePE, which encourages
		locality and moderate-range coupling rather than purely global or purely
		nearest-neighbor interactions. Failure cases reveal complementary behavior.
		When the discriminative evidence is either extremely fragmented or heavily entangled
		with cluttered background, the attention maps sometimes highlight only a subset of the
		relevant regions or partially drift to salient but class-irrelevant structures.
		These patterns are consistent with the quantitative error--vs.--separation analysis: WePE is most reliable when the key cues form one or a few
		coherent regions, and struggles when correct classification requires jointly reasoning
		over many distant or weakly localized cues. 
		
		\begin{figure}[ht]
			\centering
			
			{\small \textbf{Correct Predictions}}
			
			\smallskip
			
			\setlength{\tabcolsep}{1pt}
			\begin{tabular}{cccccccccc}
				\includegraphics[width=0.09\columnwidth]{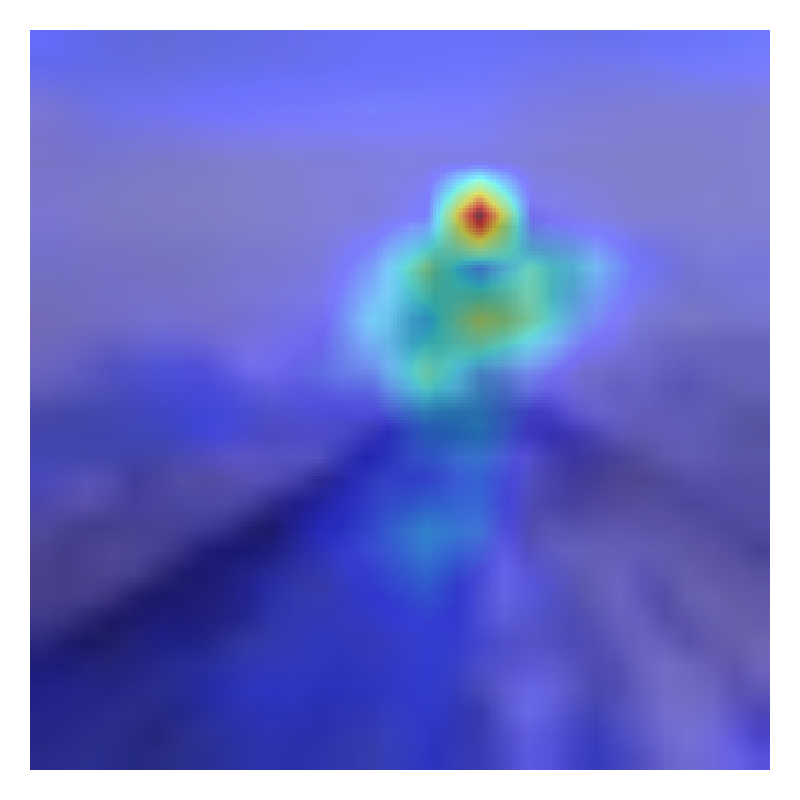} &
				\includegraphics[width=0.09\columnwidth]{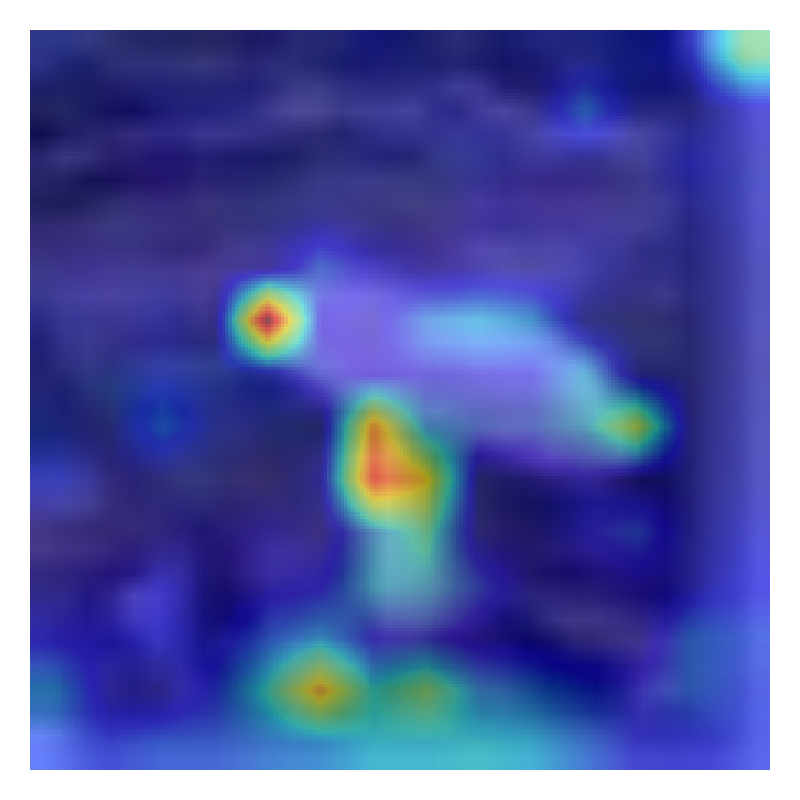} &
				\includegraphics[width=0.09\columnwidth]{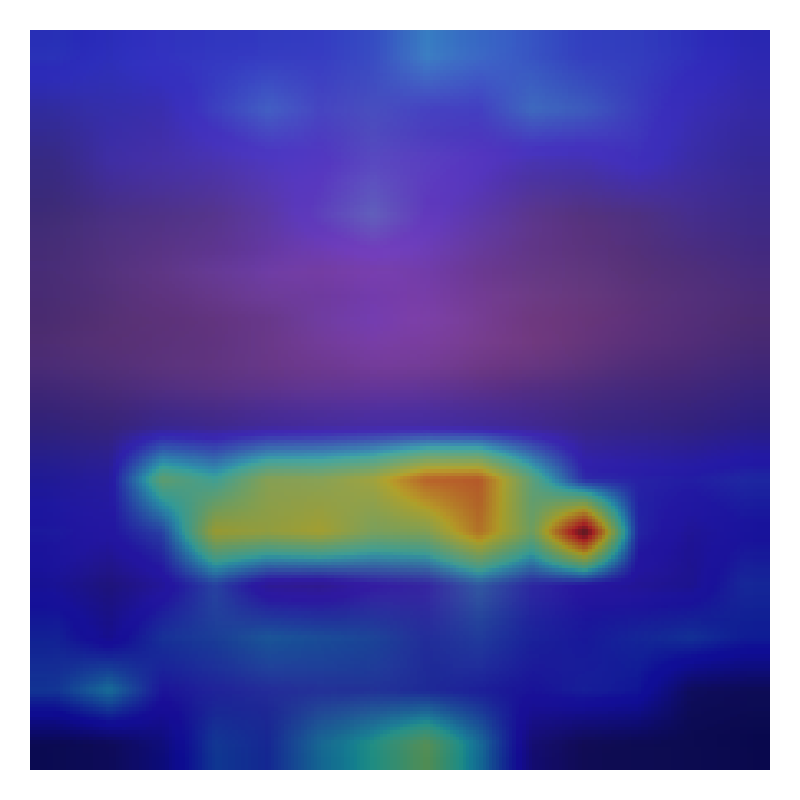} &
				\includegraphics[width=0.09\columnwidth]{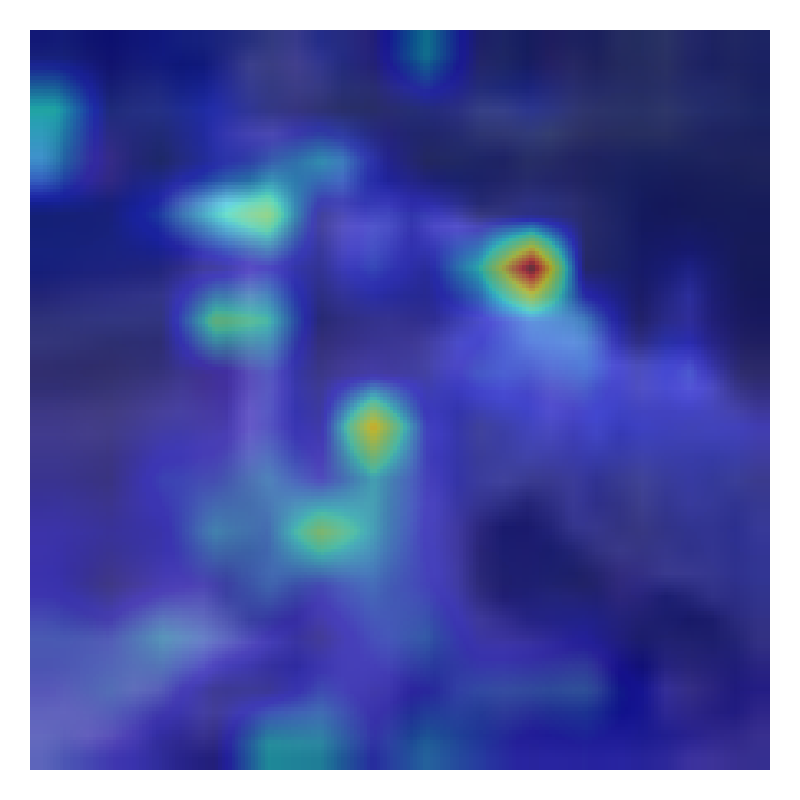} &
				\includegraphics[width=0.09\columnwidth]{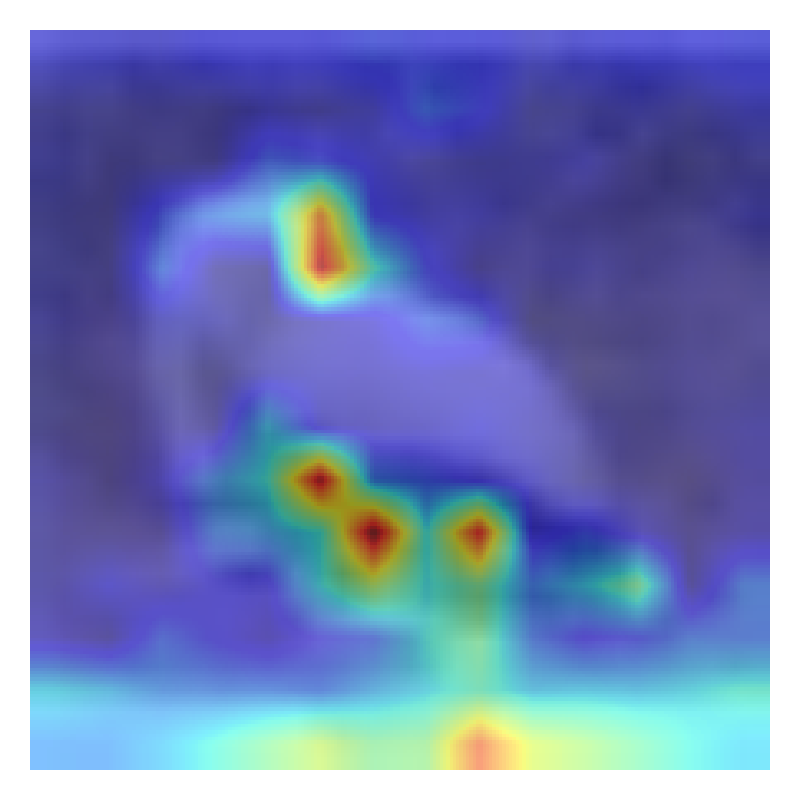} &
				\includegraphics[width=0.09\columnwidth]{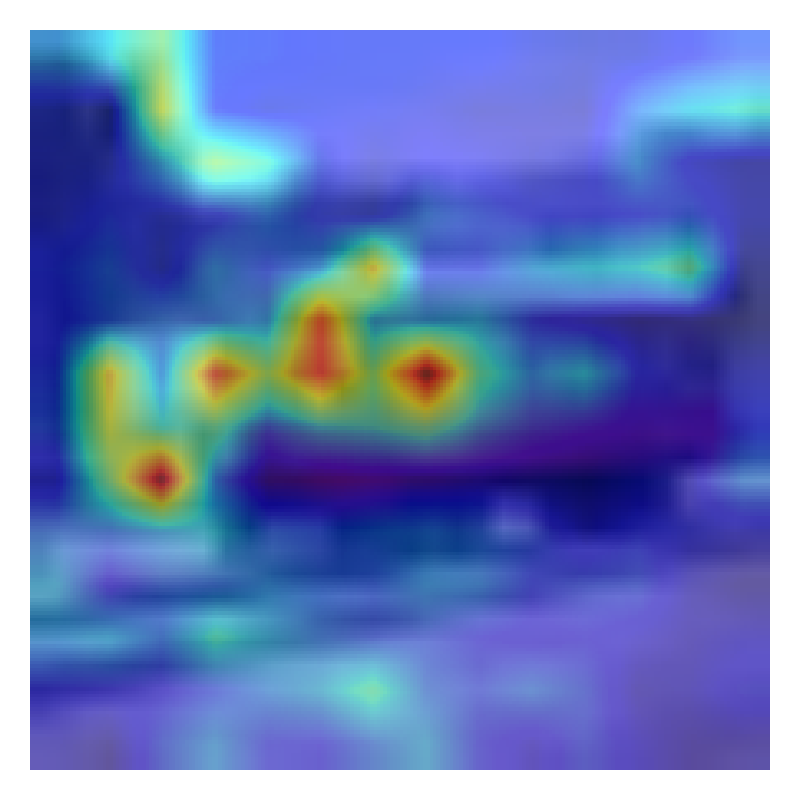} &
				\includegraphics[width=0.09\columnwidth]{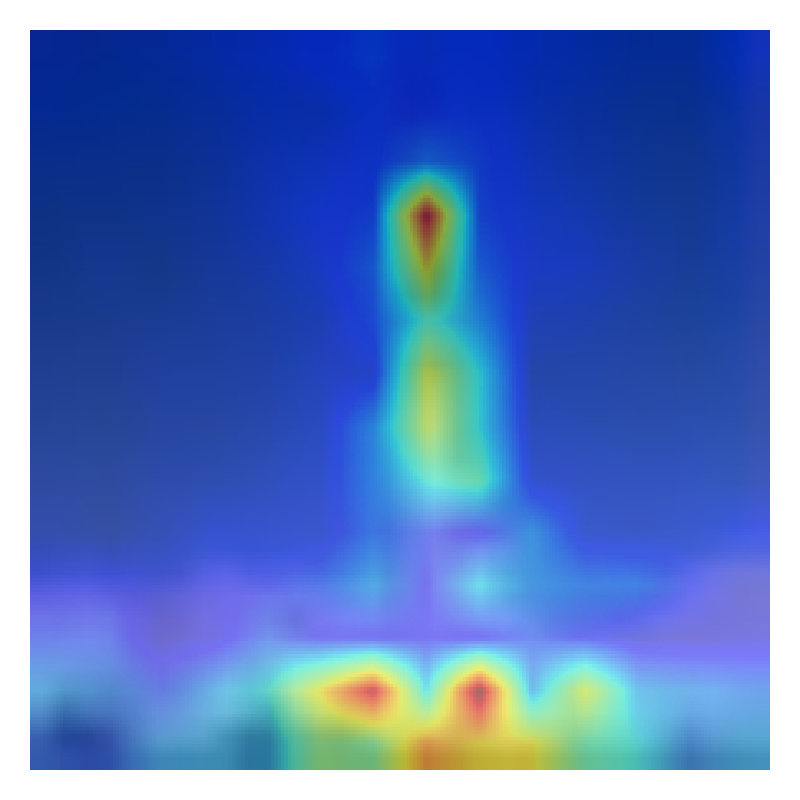} &
				\includegraphics[width=0.09\columnwidth]{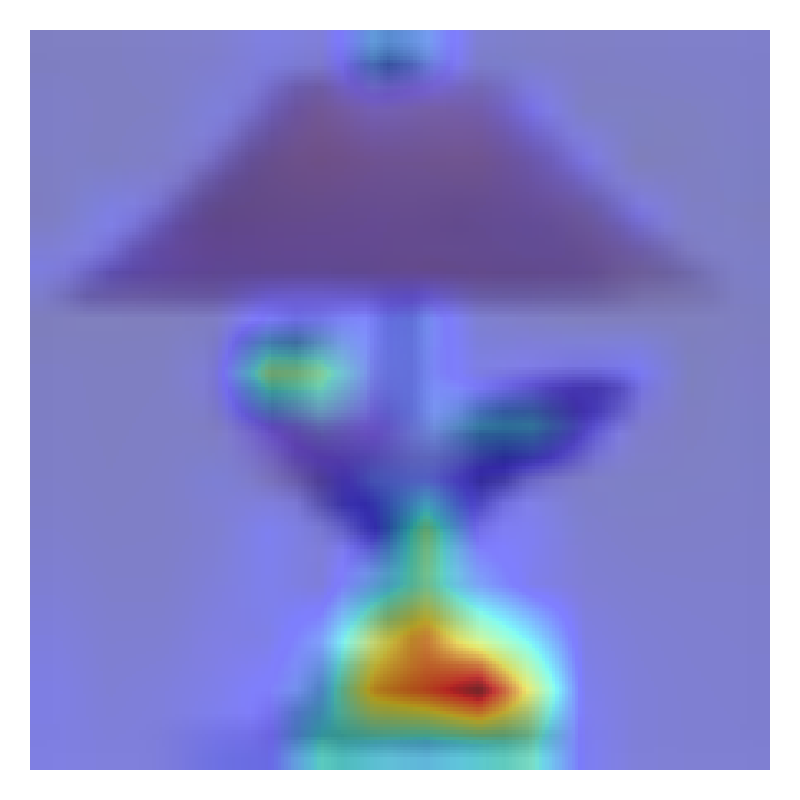} &
				\includegraphics[width=0.09\columnwidth]{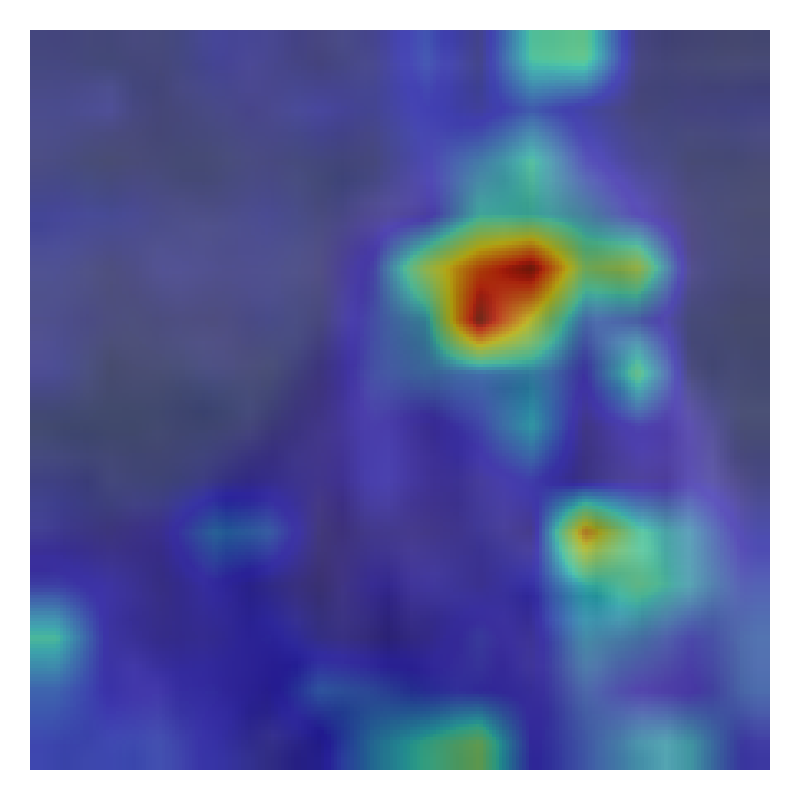} &
				\includegraphics[width=0.09\columnwidth]{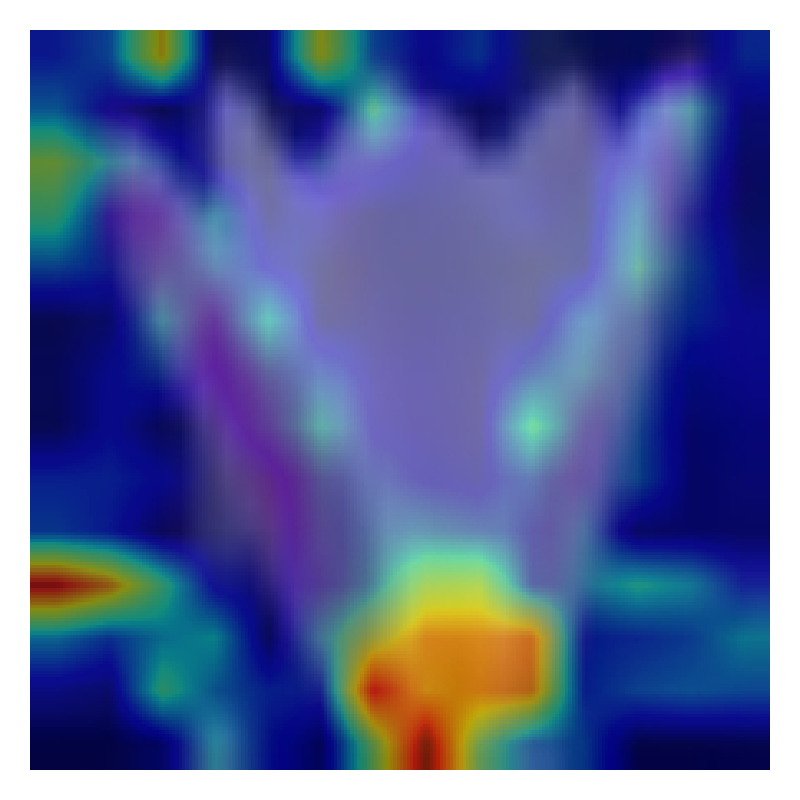}
			\end{tabular}
			
			\medskip
			
			{\small \textbf{Incorrect Predictions}}
			
			\smallskip
			
			\begin{tabular}{cccccccccc}
				\includegraphics[width=0.09\columnwidth]{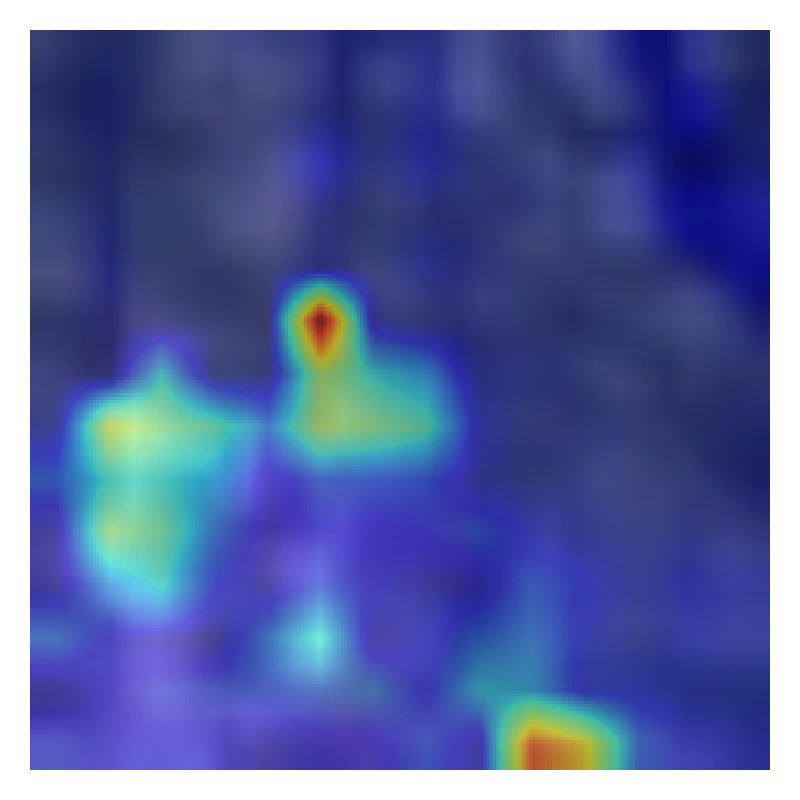} &
				\includegraphics[width=0.09\columnwidth]{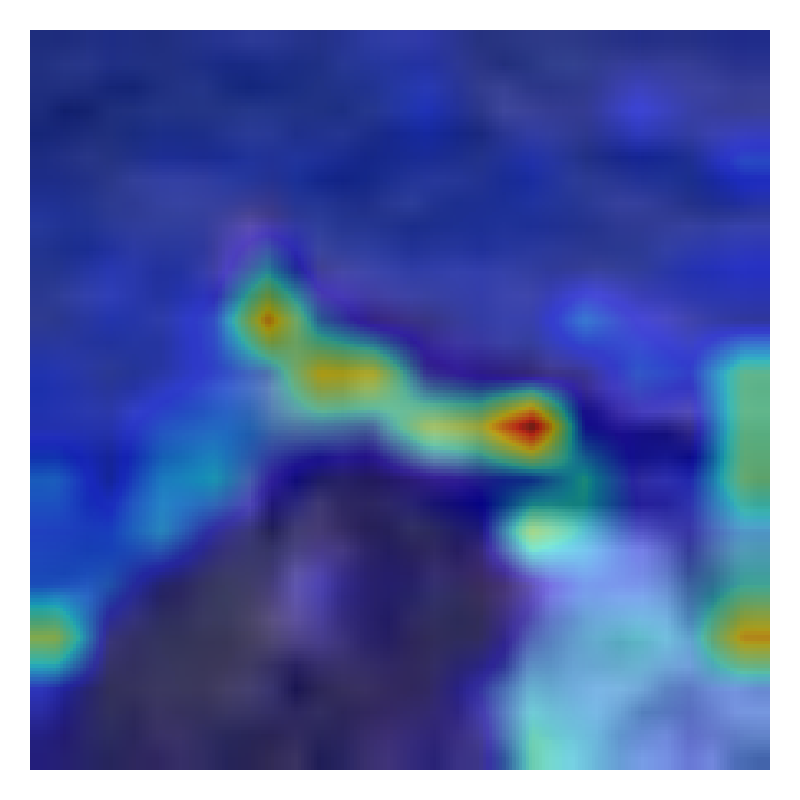} &
				\includegraphics[width=0.09\columnwidth]{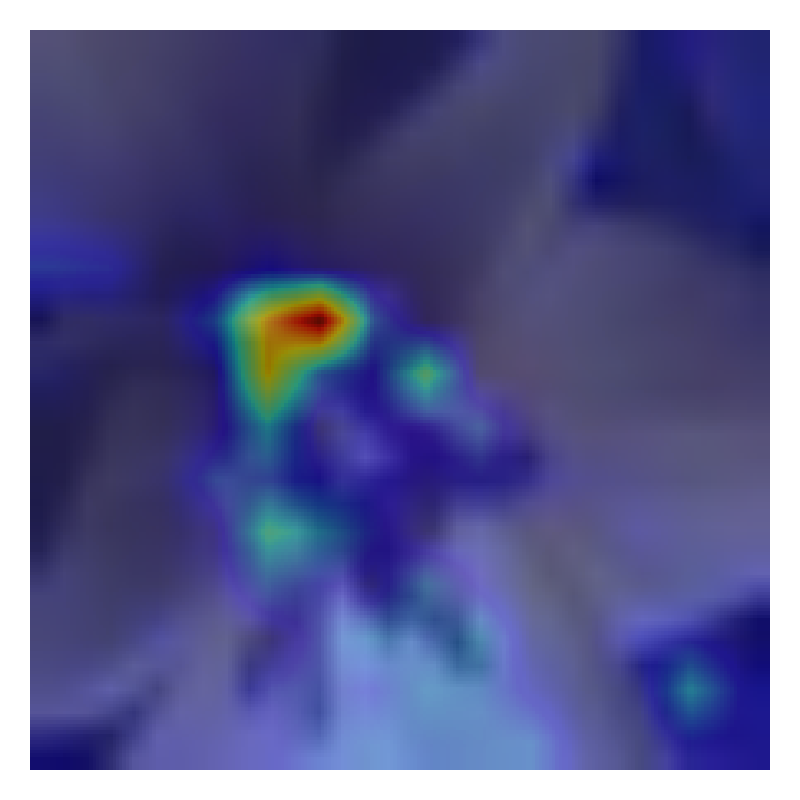} &
				\includegraphics[width=0.09\columnwidth]{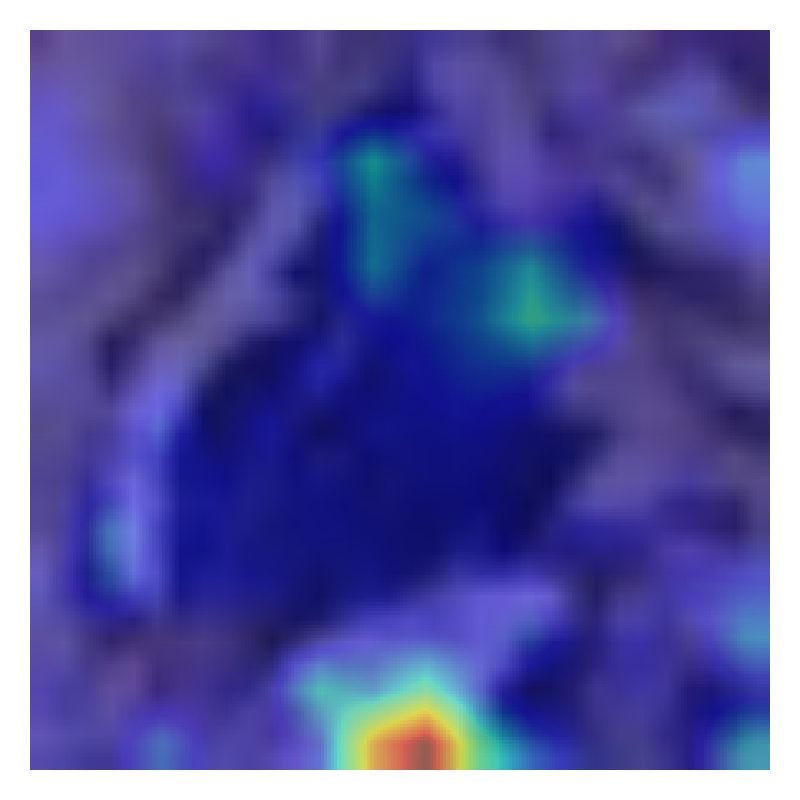} &
				\includegraphics[width=0.09\columnwidth]{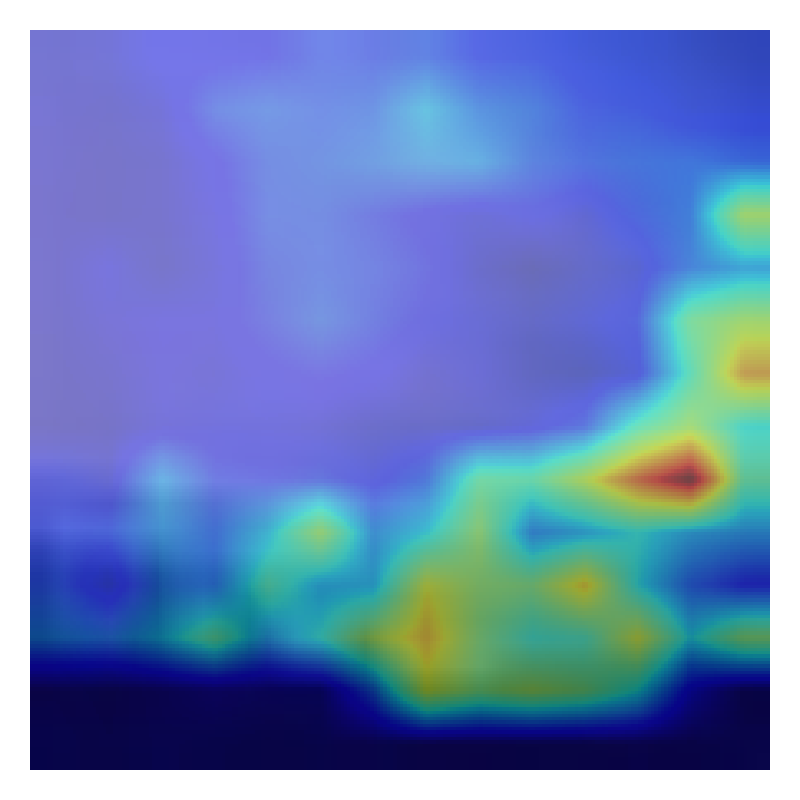} &
				\includegraphics[width=0.09\columnwidth]{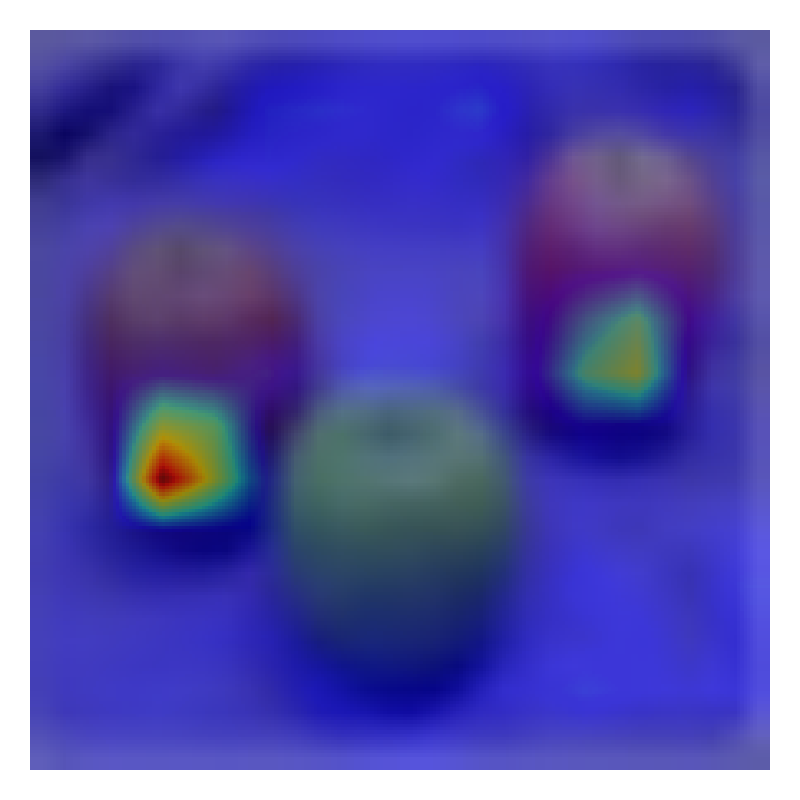} &
				\includegraphics[width=0.09\columnwidth]{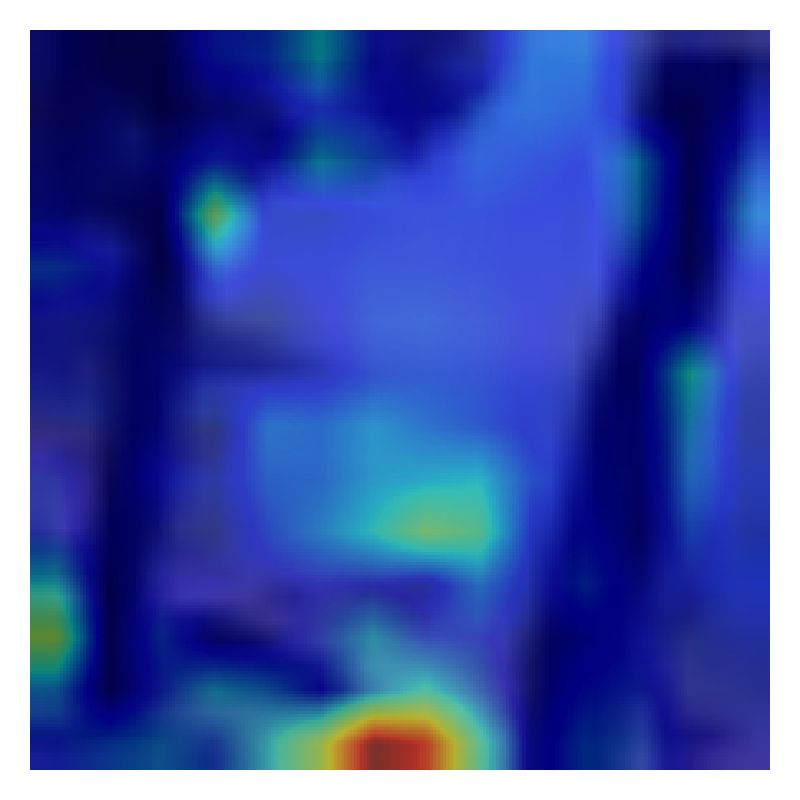} &
				\includegraphics[width=0.09\columnwidth]{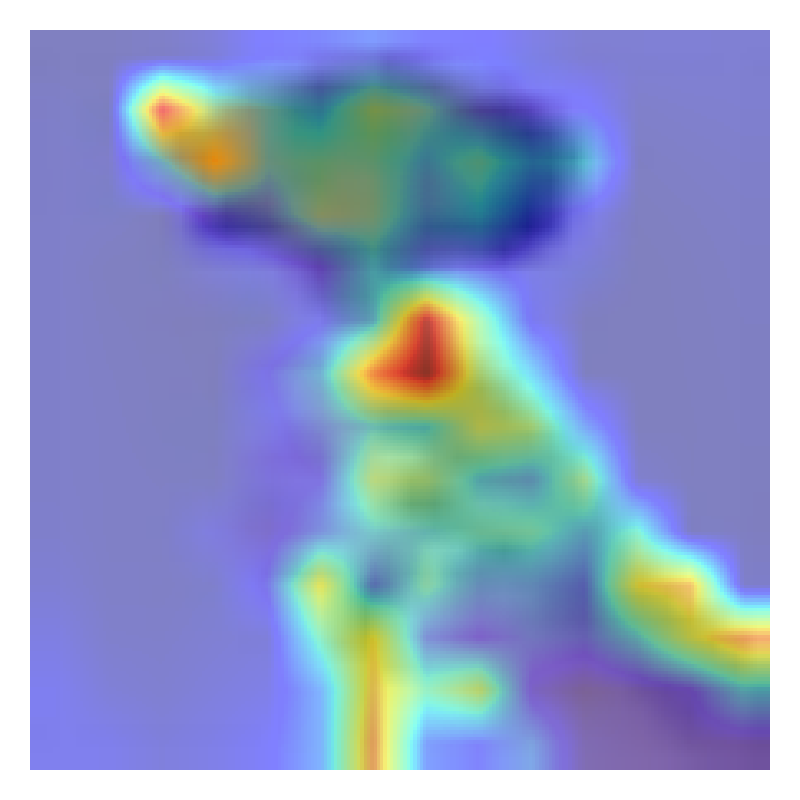} &
				\includegraphics[width=0.09\columnwidth]{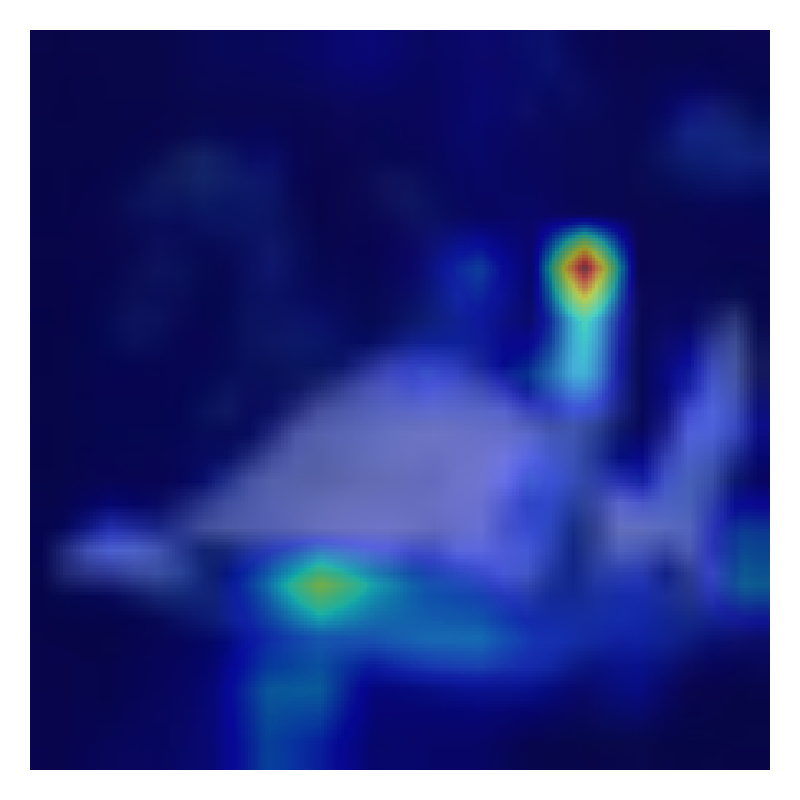} &
				\includegraphics[width=0.09\columnwidth]{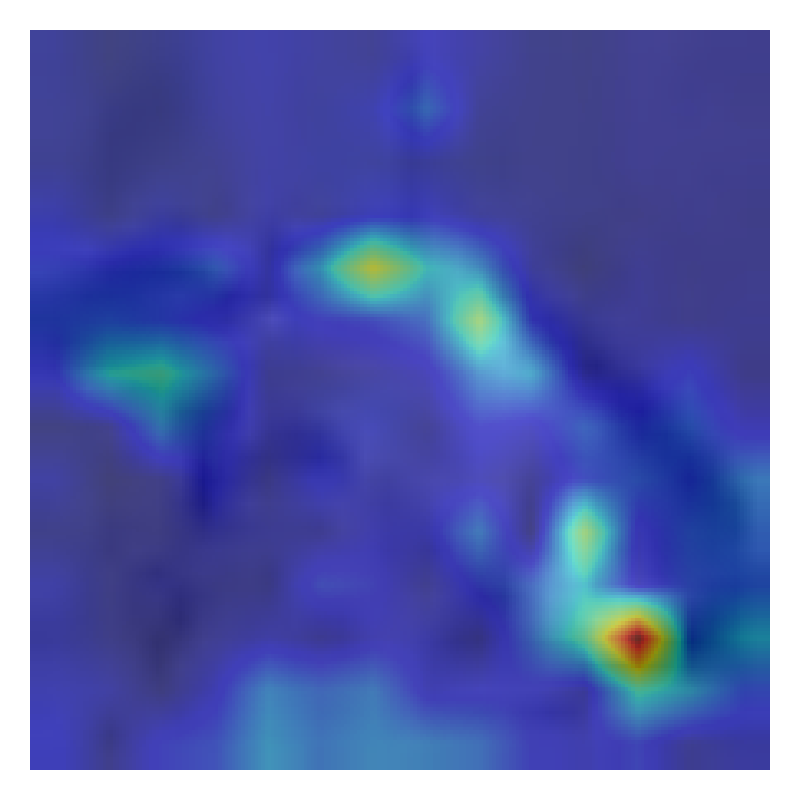}
			\end{tabular}
			
			\caption{Qualitative attention heatmaps for WePE. }
			\label{fig:wepe-heatmaps}
		\end{figure}
		
		\begin{figure}[t]
			\centering
			\includegraphics[width=0.9\linewidth]{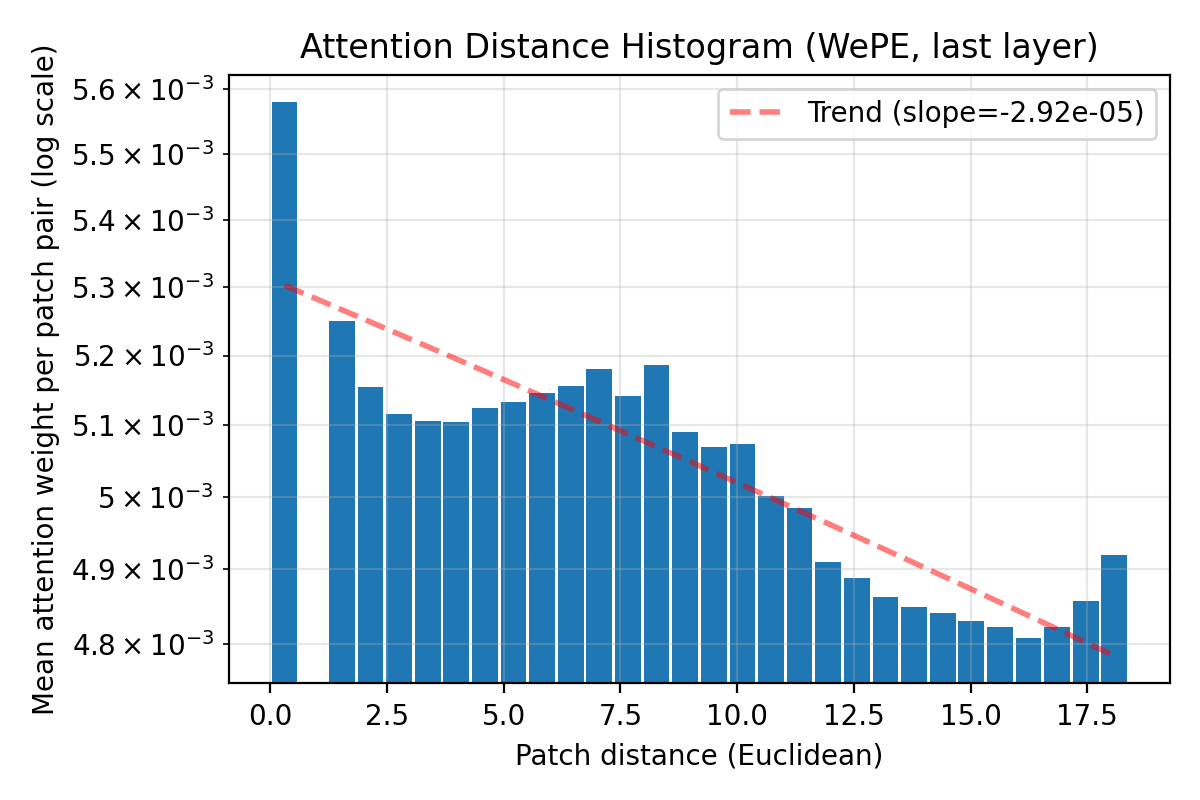}
			\caption{
				Attention--distance histogram for WePE (last layer).
				Mean attention weight per patch pair as a function of Euclidean distance
				on the $14\times14$ patch lattice.
				The dashed line shows a least-squares linear fit with a negative slope.}
			
			\label{fig:wepe-attn-distance}
		\end{figure}

		We next study how WePE distributes attention as a function of spatial distance between
		patches, providing a complementary, more global diagnostic. Using the same model and test set, we again hook the last self-attention block
		and obtain the attention tensor
		$A \in \mathbb{R}^{B\times H\times L\times L}$.
		We discard the class token and keep only patch--patch attention,
		resulting in $A_{\mathrm{pp}} \in \mathbb{R}^{B\times H\times N\times N}$ with
		$N=14\times14$.
		On the $14\times14$ patch lattice, we assign each patch an integer coordinate
		$(i,j)$ and precompute the Euclidean distance between all patch pairs,
		forming a matrix $D\in\mathbb{R}^{N\times N}$ with entries
		$d_{uv} = \|p_u - p_v\|_2$. We then discretize the range of distances into $M$ bins
		and, for each bin, aggregate the mean
		attention weight per patch pair:
		\begin{equation}
			\bar{A}(b)
			\;=\;
			\frac{1}{|\mathcal{P}_b|}
			\sum_{(u,v)\in\mathcal{P}_b}
			\frac{1}{B H} \sum_{b=1}^{B}\sum_{h=1}^{H}
			A_{\mathrm{pp}}^{(b,h)}(u,v),
		\end{equation}
		where $\mathcal{P}_b$ is the set of patch pairs whose distance falls into bin $b$.
		We finally plot $\bar{A}(b)$ against the corresponding bin centers and optionally
		fit a least-squares line to estimate the global trend. Figure~\ref{fig:wepe-attn-distance} reports the resulting attention--distance histogram.
		Although the absolute variation is modest, the fitted trend exhibits a clear
		negative slope: the mean attention weight per patch pair decreases as the
		spatial distance increases.
		This confirms that WePE indeed introduces a distance-aware inductive bias:
		interactions between nearby and mid-range patches are slightly favored over
		very long-range connections.
		At the same time, the decay is gentle rather than abrupt, indicating that
		WePE still allows non-negligible long-range attention, in line with the good
		performance on samples whose discriminative cues are moderately dispersed. We also notice a small uptick in the last one or two distance bins.
		This edge effect is largely due to the small number of patch pairs at maximal
		distances on the finite $14\times14$ grid, which makes the estimate in those bins
		noisy.
		
		Nonetheless, the overall monotonic trend remains clear when considering the
		full range of distances. This histogram analysis is inherently global: it aggregates over all images,
		layers, and heads, and therefore cannot capture head-specific strategies or
		class-dependent patterns.
		In addition, we only analyze the final layer, earlier layers may exhibit
		stronger locality, which is blurred by this averaging.
		Finally, the current WePE design imposes a fixed distance-decay profile that
		is shared across all samples, which partly explains why the model underperforms
		on rare cases requiring extremely long-range reasoning. In future work, we aim to extend WePE towards adaptive and multi-scale
		distance priors
		to better accommodate images whose discriminative cues are either unusually
		local or unusually global.

	\end{document}